\documentclass[manuscript, sigconf, dvipsnames]{acmart}

\renewcommand\footnotetextcopyrightpermission[1]{} 
\usepackage{xcolor}
\usepackage{color}
\usepackage{tikz}
\usepackage{amsmath}
\usepackage{amsfonts}

\usepackage{algorithm}
\usepackage{algpseudocode}
\usepackage[skip=0pt]{subcaption}
\usepackage{longtable}
\usepackage{indentfirst}
\usepackage{dsfont}
\usepackage{bbm}
\usepackage{hyperref}
\usepackage{url} 
\usepackage{graphicx}
\usepackage{multirow}
\usepackage{pifont}
\usepackage{amsthm}
\usepackage{balance}
\usepackage{booktabs}
\usepackage{makecell}
\usepackage{cleveref}
\usepackage{array}

\newtheorem{theorem}{Theorem}

\hypersetup
{
    pdfauthor = UnnamedOrange,
    colorlinks = true,
    citecolor = blue,
}

\usepackage{mathtools}
\usepackage{multirow}
\usepackage{arydshln}
\usepackage{tabularx}

\usepackage{caption}
\usepackage{subcaption}
\usepackage{todonotes}

\usepackage{tikz}

\newcommand*\emptycirc[1][1ex]{\tikz\draw[thick] (0,0) circle (#1);} 
\newcommand*\halfcirc[1][1ex]{%
  \begin{tikzpicture}
  \draw[fill] (0,0)-- (90:#1) arc (90:270:#1) -- cycle ;
  \draw[thick] (0,0) circle (#1);
  \end{tikzpicture}}
\newcommand*\fullcirc[1][1ex]{\tikz\fill (0,0) circle (#1);} 

\begin{document}

\title{Towards Lifecycle Unlearning Commitment Management: Measuring Sample-level Approximate Unlearning Completeness}

\author{Cheng-Long Wang$^{1}$, Qi Li$^{2}$, Zihang Xiang$^{1}$, Yinzhi Cao$^3$, Di Wang$^{1}$}
\def \authors{author one, author two, author three, author four}
\affiliation{%
\institution{$^1$King Abdullah University of Science and  Technology
}
\country{}
}
\affiliation{%
  \institution{$^2$National University of Singapore}
  \country{}
}
\affiliation{%
\institution{$^3$Johns Hopkins University}
\country{}
}

\begin{abstract}
The growing concerns surrounding data privacy and security have underscored the critical need for 'machine unlearning', aimed at fully removing data lineage from machine learning models. MLaaS (Machine Learning as a Service) providers view this as their ultimate safeguard for regulatory compliance, eliminating the need for fully retraining large models. 
It is within this context that 'approximate machine unlearning' emerges as a relevant concept. 
By adopting a more flexible definition of unlearning and adjusting the model distribution to simulate training without the targeted data, approximate machine unlearning provides a less resource-demanding alternative to the more laborious exact unlearning methods. Yet, the unlearning completeness of target samples—even when the approximate algorithms are executed faithfully without external threats—remains largely unexamined, raising questions about those approximate algorithms' ability to fulfill their commitment of unlearning during the lifecycle, specifically when it comes to sensitive, polluted, or copyrighted data.

In this paper, we introduce the task of Lifecycle Unlearning Commitment Management (LUCM) for approximate unlearning and outline its primary challenges. We propose an efficient metric designed to assess the sample-level unlearning completeness. Our empirical results demonstrate its superiority over membership inference techniques in two key areas: the strong correlation of its measurements with unlearning completeness across various unlearning tasks, and its computational efficiency, making it suitable for real-time applications. Additionally, we show that this metric is able to serve as a tool for monitoring unlearning anomalies throughout the unlearning lifecycle, including both under-unlearning and over-unlearning.

We apply this metric to evaluate the unlearning commitments of current approximate algorithms. Our analysis, conducted across multiple unlearning benchmarks, reveals that these algorithms inconsistently fulfill their unlearning commitments due to two main issues: 1) unlearning new data can significantly affect the unlearning utility of previously requested data, and 2) approximate algorithms fail to ensure equitable unlearning utility across different groups. These insights emphasize the crucial importance of LUCM throughout the unlearning lifecycle. We will soon open-source our newly developed benchmark.

\end{abstract}

\keywords{Machine Unlearning, Privacy, Machine Learning Security}


\maketitle

\section{Introduction}

Leveraging large models, Machine Learning as a Service (MLaaS) is rapidly expanding into business, workflows and personal applications, raising concerns over the risks associated with handling sensitive, polluted, or copyrighted data~\cite{CarliniIJLTZ23,BrownLMST22}. A notable example is Google's €250 million fine by France for using publisher content without authorization in its Bard model~\cite{wsj2024google}. Machine unlearning~\cite{CaoY15}, aimed at removing targeted data lineage from the model without full retraining, is now viewed by MLaaS providers as a critical solution for data compliance. This approach helps avoid the hefty expenses associated with complete model retraining. As an option, exact machine unlearning~\cite{CaoY15,BourtouleCCJTZL21} often modify the training pipeline to minimize the crossing of data lineage, thereby reducing the computational costs associated with retraining. Typically, this process involves retraining only a submodel or checkpoint that corresponds to the related subset. This strategy reduces retraining costs while maintaining model utility. Yet, the popularity of pretraining and foundational models often forbids the pre-implementation of exact unlearning. This gap has spurred the development of approximate machine unlearning algorithms~\cite{IzzoSCZ21,GolatkarARPS21,NeurIPS21Adaptive,MaLLLMR23,cvprMehtaPSR22,abs-2308-07707,abs-2108-11577}, a resource-efficient post-hoc solution that adjusts model parameters to mimic the model trained without the targeted data. 

While benefiting from the fast adaptability of approximate unlearning, a question arises: How much of the unlearning commitment is fulfilled with the approximate solution? \textit{Here, we use 'unlearning commitment' to refer to the ability of approximate unlearning algorithms to thoroughly remove all targeted data lineage from the model}. The unlearning commitment is a guarantee that once the process is executed faithfully, the unlearning system can remain consistent with the retrained system, ensuring no remnants of the unlearned data will affect future outcomes or decisions made by the system. 

\textbf{Lifecycle unlearning commitment management.} Consider further the run-time scenario, it is essential to identify and manage the risks associated with incomplete unlearning commitments, specifically providing accurate measurements on how completely the targeted data has been unlearned. We define the '\textit{unlearning lifecycle}' as the period that starts with the submission of an unlearning request, includes the faithful execution of the unlearning algorithm, and ends after the use of the unlearning outcome in the next training or inference processes.
The Lifecycle Unlearning Commitment Management (LUCM) enables the system to detect and provide early warnings for anomalies in unlearning completeness throughout the unlearning lifecycle, including both \textit{under-unlearning} and \textit{over-unlearning}. We denote under-unlearning as the scenario where approximate unlearning, despite being faithfully executed and achieving its predefined optimization goal on the requested batch, inadvertently leaves the related data lineage embedded in the model. This occurs in unlearned samples and can lead to significant privacy risks. We describe over-unlearning as the process where the unlearning action unintentionally removes more information than intended from the retained training data, undermining model robustness in unpredictable ways, even without external threats.

To reach the objectives of LUCM, we identify three practical challenges within the run-time environment: (i) providing accurate, sample-level measurements of unlearning completeness for each sample regardless of unlearning batch size variations; (ii) producing robust measurement results for continuous unlearning requests or across different unlearning groups; and (iii) ensuring that the computational costs of these measurements remain manageable and cost-effective. Beyond addressing the three challenges in LUCM, which typically involve using a fixed, pre-designed unlearning algorithm, it's also important to develop a general unlearning commitment metric applicable to most approximate unlearning algorithms, as long as they are applicable within the parameter space. We believe these metrics will meet not only operational demands but also advance analytical research in approximate unlearning.

Although many studies on approximate unlearning claim acceptable unlearning results by comparing unlearned models with exactly retrained models during algorithm design, no method convincingly assesses unlearning performance in run-time scenarios. This limitation arises because using exact retraining as a ground truth in such scenarios would contradict the primary goal of approximate unlearning—to avoid full retraining. In this paper, we first delineate the process of managing unlearning commitments during run-time and propose a robust strategy to effectively address the primary challenges encountered in achieving this goal.

\noindent {\bf Our work.} In this paper, we set out to develop a general method for measuring sample-level approximate unlearning completeness, tailored for the unlearning lifecycle. We showcase the utility of our proposed unlearning metric across various unlearning tasks, highlighting the relationship of its sample-level measurement score with unlearning completeness, and its computational efficiency during the unlearning lifecycle. We show that the proposed method can serve as a monitoring metric to detect anomalies in unlearning completeness during the unlearning lifecycle, including cases of under-unlearning and over-unlearning. Next, we apply it to benchmark the unlearning commitments of current approximate unlearning algorithms. Our investigation reveals that these algorithms fail to consistently fulfill their unlearning commitments during the unlearning lifecycle. Specifically, two risks are associated with these approximate unlearning algorithms (not present in exact unlearning): 1) the unlearning utility of previously unlearned data can continue to be disturbed by subsequent unlearning of new data; 2) almost all algorithms exhibit an unlearning equity issue, particularly when the unlearning difficulty varies across different groups. Both risks highlight the importance of monitoring unlearning completeness during the unlearning lifecycle, allowing for appropriate adjustments to the unlearning efforts so that the unlearning commitments made by these algorithms at the sample level are guaranteed. 

To summarize, this work contributes to the literature with:

1. Introducing the lifecycle unlearning commitment management task and identifying misalignments between its challenges and existing unlearning metrics, followed by the design of a general, resource-efficient solution to measure sample-level approximate unlearning completeness during the unlearning lifecycle.

2. Demonstrating the utility of our metric across different unlearning tasks by comparing its measurements with the ground truth of exact retraining results. Applying the designed metric for sample-level anomaly detection in approximate unlearning completeness during the unlearning lifecycle, addressing cases of both under-unlearning and over-unlearning.

3. Applying the designed metric to benchmark the unlearning commitments of current approximate unlearning algorithms across diverse unlearning tasks. Our evaluation uncovers a significant gap, showing that these algorithms cannot maintain their unlearning commitments during the unlearning lifecycle.

\section{Preliminaries}

Before delving into the measurement of unlearning completeness within the LUCM, in this section, we first provide an overview of unlearning algorithms and review the current metrics used to evaluate unlearning effectiveness. For related work on dataset auditing and Proof-of-(Un)Learning, please refer to Appendix~\ref{sec:appendix_related}.

\subsection{Machine Unlearning}

Machine unlearning refers to the process of removing specific training data lineage from a machine learning model without fully retraining. This is essential for models trained on sensitive or restricted data that requires periodic or on-demand deletion. Initially driven by the 'right to be forgotten,' the scope of this concept has expanded to not only meet privacy requirements but also address issues like model alignments and security vulnerabilities. Recent studies have broadened the definition of unlearning, categorizing the methods into two types: exact unlearning, which provides a definitive removal of data, and approximate unlearning, which offers greater flexibility.

\noindent {\bf \textit{i) Exact unlearning.}} The most rigorous approach to machine unlearning is exact unlearning. This method requires retraining the model from scratch using a dataset that excludes the targeted data. To reduce the retraining computation, Cao et al.~\cite{CaoY15}  proposed converting the learning algorithm into a summation form, allowing the unlearning process to simply update the model based on the updated summations. Bourtoule et al.~\cite{BourtouleCCJTZL21} partitioned the entire dataset into disjoint shards, training submodels separately on each shard, and then aggregating the predictions from these submodels. This approach ensures that unlearning only requires updating the corresponding shard and submodel. Ullah et al.~\cite{UllahM0RA21} developed a total variation stable learning algorithm for smooth convex empirical risk minimization problems. Their approach to achieving fast unlearning involved a rejection sampling strategy to reduce the probability of recomputation.
While those approach guarantees the complete removal of data lineage, they are still computationally intensive for large-scale models or scenarios where unlearning requests are frequent. Moreover, this approach is not applicable to post-hoc model unlearning, where the trained model has already been deployed.

\smallskip
\noindent {\bf \textit{ii) Approximate unlearning.}} As an alternative to the retraining-based exact unlearning methods, approximate machine unlearning algorithms aim to balance computational efficiency with unlearning utility. They provide resource-efficient post-hoc solutions for removing data from large models by adjusting model parameters. Based on how they calculate the model update, we classify those unlearning algorithms into three groups:

\textit{Log-based retrieval}: 
Amnesiac Unlearning~\cite{GravesNG21} directly subtracts the corresponding parameter updates, which have been logged during the training process, of the small batches containing targeted data from the model weights. 
Thudi et al.~\cite{ThudiDCP22} design a Standard Deviation Loss for the original training process, which is beneficial to reduce the unlearning verification error of the later unlearning by adding back the logged gradients.
Generally speaking, those log-based methods are memory-intensive for keeping gradient records during large model training. They are practical only when removing a small subset of data known in advance, such as temporary authorization data, by saving updates for a small batch. It is also challenging for these methods to remove the influence of targeted samples on later ones, whose gradients may be affected by the targeted sample.

\textit{Hessian-based update}: 
Guo et al.~\cite{GuoGHM20} develop a certified-removal mechanism for data deletion in $l_2$ regularized linear models by applying a single Newton step to the model parameters. This step aims to approximately minimize the leave-$k$-out loss, thereby removing the influence of the deleted data point. 
Izzo et al.~\cite{IzzoSCZ21} introduce the projective residual update to reduce the time complexity associated with the approximate data deletion in linear and logistic models. 
Fisher Forgetting~\cite{GolatkarCVPR20,GolatkarECCV20} approximates the scrubbing procedure of selective forgetting through a noisy Newton update, deriving it as reducing the KL divergence distance between two model distributions: one trained on the original dataset and the other trained on the retained dataset. They calculate the corresponding update by approximating the Hessian of the forgotten data using the Fisher Information Matrix. 
Despite there being a theoretical foundation for data deletion in linear models, approximating the influence of targeted samples in deep models remains challenging due to their non-convexity and the randomness of perturbations. Computation efficiency is also a concern when it comes to large models.

\textit{Dynamics Masking}: 
Forsaken~\cite{MaLLLMR23} introduces a mask gradient generator that can iteratively generate mask gradients to “stimulate” neural neurons to unlearn the memorization of given samples. 
Selective Synaptic Dampening~\cite{abs-2308-07707} uses the Fisher information matrix from training and forgetting data to identify key forget set parameters. Then, it dampens these based on their significance to the forget set relative to the overall training data. 
Jia et al.~\cite{JiaLRYLLSL23} explore the application of model sparsification via weight pruning in machine unlearning.
These methods efficiently achieve machine unlearning by masking parameter dynamics of given samples but rely on explaining their opaque performance to evaluate their unlearning utility. 

This paper will primarily focus on unlearning commitment management for approximate algorithms. Since exact unlearning provides definitive unlearning, 
approximate unlearning operations are performed on the parameter space, making its unlearning utility less transparent, an evaluation step is necessary. Additionally, approximate unlearning is more acceptable for large models due to its computational efficiency.

\subsection{Failure of Existing Unlearning Measurements}\label{sec:subsec:measurement}
Current methods for measuring unlearning vary widely, from analyzing weight distributions to testing the accuracy of data that should be forgotten. However, as Table~\ref{tab:metric_align} illustrates, these metrics do not align well with the requirements of LUCM. We compare these metrics with our designed approach to highlight their differences. In this section, we will examine these metrics in detail, identify their shortcomings, and explain why they fail to align with LUCM.

\begin{table}[h]
\centering
\caption{Misalignment of unlearning metrics with LUCM requirements. Symbols: $\fullcirc$ represents full alignment with requirements; $\halfcirc$ indicates partial or conditional alignment; $\emptycirc$ denotes no alignment.}
\begin{tabular}{lcccc}
\hline
& Sample-level & Robust & Efficient & General \\
\hline
Weight-based & \emptycirc & \halfcirc & \emptycirc & \emptycirc \\
Accuracy-based  & \emptycirc & \emptycirc & \fullcirc & \emptycirc \\
MI-based & \fullcirc & \halfcirc & \emptycirc & \fullcirc \\
\hline
Our Method & \fullcirc & \fullcirc & \fullcirc & \fullcirc \\
\hline
\end{tabular}
\label{tab:metric_align}
\end{table}

\noindent {\bf \textit{Weight distribution analysis}}~\cite{GolatkarCVPR20,GolatkarECCV20} typically assumes a proxy for the optimal (retrained) model weights. The smaller the distance between the approximate unlearned model and the proxy, the better the unlearning completeness. However, due to the unavailability of the retrained model, this criterion is often replaced by maximizing the distance between the original model and the approximate unlearned model, along the direction of the proxy. Yet, this substitute criterion fails to provide a definitive standard for complete unlearning, offering only a directional measurement. Besides, both the bounds between the unlearned model and the proxy, and between the proxy and the optimal weights, become unbounded when the unlearning batch size is large and continuous approximate unlearning updates are conducted. Another work~\cite{abs-2208-10836} evaluates unlearning algorithms in a white-box setting based on epistemic uncertainty. It calculates the residual information about the unlearned dataset using the trace of the Fisher Information Matrix (FIM). Similar to the task of dataset auditing, it fails to provide sample-level measurement results. Additionally, the computational demands of this metric are comparable to those of an approximate unlearning algorithm, potentially leading to a significant added computational burden. This would be particularly impactful for large-scale models or in situations where real-time performance is critical. Thudi et al. ~\cite{ThudiDCP22} derive a proxy named unlearning error for verification error (the $l_2$ difference between the weights of an approximately unlearned and a naively retrained model). Unfortunately, they designed it only as a regularization term and minimized it during the original training to improve the ability to later unlearn (by external algorithms) with a smaller verification error. Besides, the strong assumptions about stochasticity in this method impede its extension to general unlearning methods as the measurements for unlearning.

 \smallskip 
\noindent {\bf \textit{The accuracy evaluations}}~\cite{GravesNG21,cvprMehtaPSR22,GolatkarCVPR20,GolatkarECCV20} are misleading. This is because unlearning completeness relies not on the correctness of predictions: an incorrect but consistent prediction by both unlearned and retrained systems does not decrease completeness~\cite{CaoY15}. In other words, accuracy evaluations (regardless of using retain, forget, or test datasets) are only applicable to evaluate model utility performance. The consistency evaluation is applicable to measure unlearning completeness, but it is not practical since the retrained model is unavailable, as we have discussed. Interclass Confusion test~\cite{goel2023adversarial} introduces an invasive method by injecting a strong differentiating influence specific to the forgetting dataset into the training dataset via label manipulations. Such an inject test requires specifying the forgetting set prior to training, which is a constrained unlearning scenario, and could potentially degrade the model's performance related to the forgetting dataset.

\smallskip
\noindent {\bf \textit{Membership inference (MI)}}, a widely used technique for determining whether a sample belongs to a model's training data, has been utilized by researchers to evaluate the effectiveness of unlearning algorithms~\cite{LiuT20,HuangLL21,GravesNG21,MaLLLMR23,GolatkarECCV20}. Yet, when considering its application to LUCM, the significant computational burden makes it impractical and inefficient. The need to train a separate inference model for each target sample, coupled with the extensive computational resources required, hinders its scalability and feasibility, especially in scenarios involving the continuous unlearning of numerous data samples throughout the model's lifecycle.  

Moreover, MI techniques are primarily designed to accurately identify members while minimizing false alarms. This aim is evident in their optimization for a high True Positive Rate at a low False Positive Rate, which we denote as \textit{MI\_TPR@LowFPR}. As Carlini et al.~\cite{CarliniCN0TT22} pointed out: 'If a membership inference attack can reliably violate the privacy of even just a few users in a sensitive dataset, it has succeeded.' However, when the target sample is 'unlearned' from the model, the focus shifts to measuring the completeness of the unlearning for that sample. In this context, we are particularly concerned with the unlearned members, who, in an ideal unlearning update, should be non-members of the model. Consequently, this shifts the problem to Non-membership Inference (NMI).

\smallskip
\noindent \textbf{Why are effective MI techniques inadequate for NMI tasks?}
To understand this, let's first examine the confusion matrix shifts from MI to NMI shown in Figure~\ref{fig:shifts}. As we can derive from the figure, a high \textit{NMI\_TPR@LowFPR} corresponds to the high \textit{MI\_TNR@LowFNR}, following a consistent notation with previous. However, the design and optimization of MI techniques results in less attention being paid to this corner.

\begin{figure}[h]
    \centering
    \includegraphics[width=0.9\linewidth]{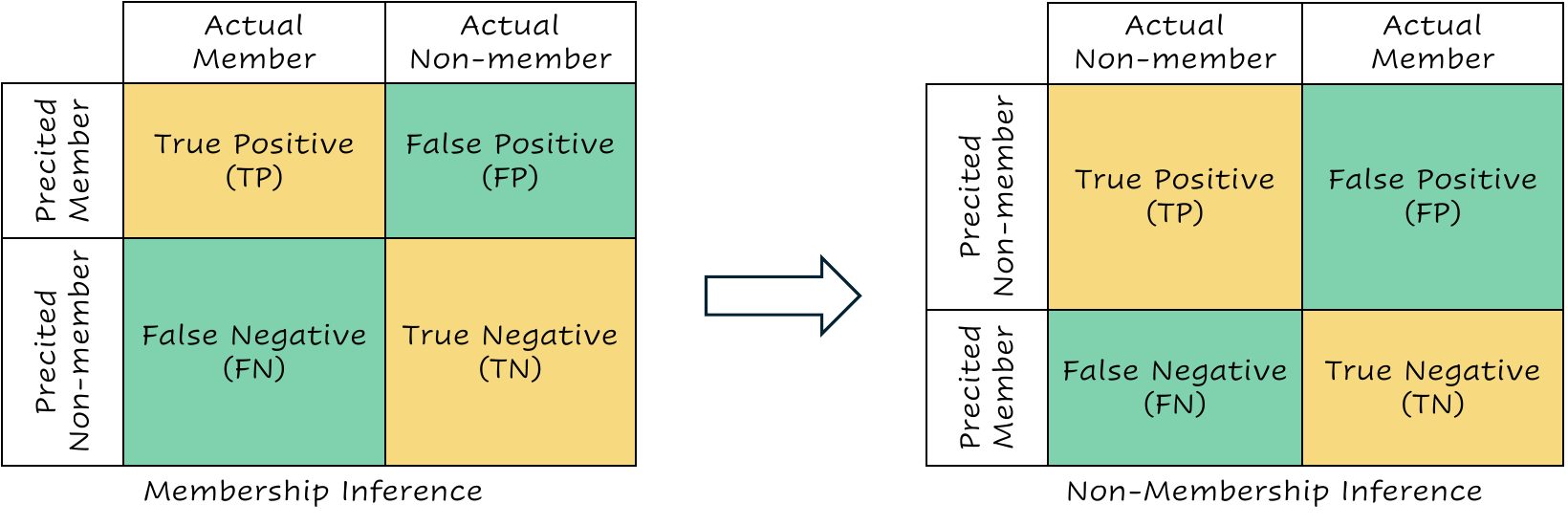}
    \caption{Confusion Matrix Shifts from Membership to Non-Membership}
    \label{fig:shifts}
\end{figure}

We further illustrate the impact of such design biases on inference performance through a concrete example in Figure~\ref{fig:prefer}. Both Model 0 and Model 1 achieve the same Area Under the Curve (AUC) score within the membership inference task. Model 0 is characterized by a high \textit{MI\_TPR@LowFPR} but demonstrates a low \textit{MI\_TNR@LowFNR}. It excels at membership inference but fails to perform well in non-membership inference. We  will also see similar phenomenons in our experiments, see Figure \ref{fig:retrain_results_tprs} and \ref{fig:retrain_results_auc} for details. A more detailed description of this shift is provided in Section~\ref{sec:method}.

\begin{figure}[h]
    \centering
    \includegraphics[width=0.6\linewidth]{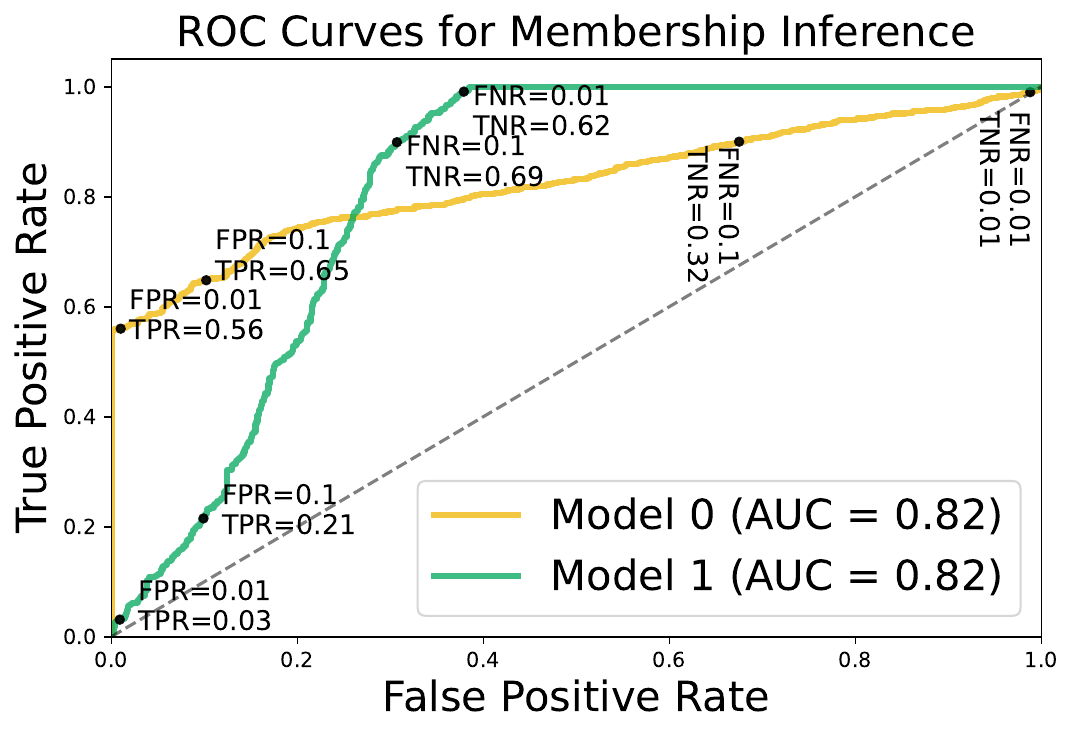}
    \caption{Membership Prefer vs. Non-Membership Prefer}
    \label{fig:prefer}
\end{figure}

Note that achieving a perfect score (100\% AUC) in membership inference undoubtedly signifies mastery in both membership preference and non-membership preference at the same time. Yet, consistently reaching such a level remains a practical challenge. Thus, it leaves space to enhanced strategies specifically aimed at improving non-membership inference. Implementing focused approaches for non-membership inference could effectively reduce the existing gap in measuring the completeness of unlearning processes.

\section{Methodology}\label{sec:method}

As we have discussed, the inherent limitations of approximate unlearning algorithms underscore a significant drawback: their unavoidable inability to guarantee the complete removal of data lineage.
Thus, the question of managing lifecycle unlearning commitment arises for approximate unlearning systems, with the aim of sensing and providing early warning for both inadvertent and malicious loss of control during unlearning.

\subsection{Approximate Unlearning System}

Approximate unlearning systems typically find applications in scenarios involving large-scale models or frequent unlearning requests. In such cases, the computational cost of complete retraining or exact unlearning becomes prohibitive, making approximate techniques a more viable and scalable solution. This system is triggered upon receipt of an unlearning request. Considering the specific requirements of the task and available resource constraints, the system selects an appropriate approximate unlearning technique from a range of options to efficiently address the request. After completion of the unlearning process, the system delivers the output and incorporates the updated results into the application, while providing transparent and auditable provenance of the unlearning operation, enabling stakeholders to track and monitor potential risks and uncertainties.

Despite their computational benefits, approximate unlearning algorithms introduce a trade-off between efficiency and the completeness of data removal. As a result, managing the lifecycle unlearning commitment and ensuring accountability becomes crucial. These systems require robust mechanisms to monitor and quantify the extent of data removal, as well as safeguards to prevent inadvertent or malicious loss of control during the unlearning process.

\smallskip 
\noindent \textbf{Problem Statement:} In real-world unlearning systems, an Unlearning Commitment Manager has black-box access to the original and unlearned models, meaning they can query the output for given samples from both models. It only accepts the unlearning requests for samples certified as training members of the original model. Upon receiving an unlearning request, a post-hoc algorithm is faithfully applied to the original model to efficiently approximate the unlearning of the requested samples, referred to as unlearned members. The remaining samples in the training set are termed retained members. Furthermore, there exists a dataset of non-members associated with the original model; these datasets are not involved in the update process between the original and unlearned models, but can be accessed by the manager. The primary role of the Unlearning Commitment Manager is to timely monitor the sample-level unlearning completeness for target samples throughout the entire unlearning lifecycle, providing early warnings for any anomalies related to unlearning. This includes issuing alerts for under-unlearning of unlearned samples and over-unlearning of retained samples.

\subsection{Formalizing Unlearning Completeness}

The varying criteria for defining approximate unlearning significantly complicate the process of assessing unlearning completeness at the sample level. This issue is further exacerbated by the complexity of deep learning models and diversity of data distributions.
Rather than focusing on specific unlearning algorithms, we aim to understand how a model behaves after unlearning instead, which is key to this assessment. To design a general assessment method, we first discuss whether the issue of unlearning completeness measurement can be translated into the problem of non-membership inference attacks. 

As we mentioned, membership inference approaches are primarily designed to differentiate training members from non-members. However, their application could become problematic when there are retained members, unlearned members, and non-members simultaneously present. In the unlearning context, researchers typically use these methods to assess how well they can differentiate between unlearned members and non-members to measure the unlearning utility~\cite{GravesNG21,MaLLLMR23,GolatkarECCV20}. If these methods fail to distinguish between the two—essentially leading to results similar to random guessing—such outcomes suggest successful unlearning. 

\begin{figure}[h]
    \centering
    \includegraphics[width=0.8\linewidth]{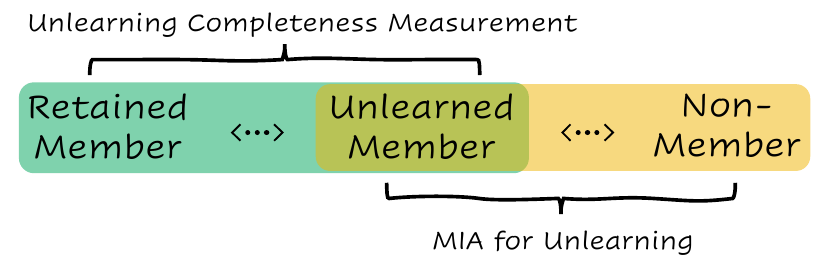}
    \caption{MIA for Unlearning vs. Unlearning Completeness Measurement}
    \label{fig:diff_setting}
\end{figure}

But when they are applied to differentiate between retained members and unlearned members (a critical problem of LUCM), results that approximate random guessing in this context signal unsuccessful unlearning. This subtlety shows that the way we measure unlearning completeness differs from traditional membership inference settings, as illustrated in Figure~\ref{fig:diff_setting}. Some may argue that the inability to distinguish between unlearned members and non-members could indicate unlearning completeness. However, as discussed in Section~\ref{sec:subsec:measurement}, the membership-preference design of existing methods may cause them to be inadequate for achieving a high \textit{MI\_TNR@LowFNR} in the ideal unlearning scenario, where both unlearned members and non-members should be classified as non-members.

Let event $A$ represent the scenario where a target sample $x$ is included in the training set of the original model $\theta_{\text{ori}}$. Conversely, let event $B$ indicate that the same target sample $x$ is not part of the training set of the model after 'unlearning' $\theta_{\text{unl}}$. Utilizing Bayes' Theorem, we can calculate the conditional probability of event $B$ given that event $A$ is true—the successful unlearning of sample $x$ from the trained model. This allows us to assess the sample $x$'s unlearning likelihood. The formula provided by Bayes' Theorem is as follows:

\begin{equation}
    P(B|A) = \frac{P(A|B) \cdot P(B)}{P(A)}.
\end{equation}
Here, $P(B|A)$ quantifies the probability that sample $x$ has been unlearned (event $B$), given that it was originally part of the model's training set (event $A$). $P(A|B)$ is the likelihood of sample $x$ being in the training set of $\theta_{\text{ori}}$ given it is absent from the unlearned model $\theta_{\text{unl}}$'s training set.

Unlike the privacy attack settings, in LUCM, the server could reject the unlearning request if the target sample $x$ was not in the training set, indirectly confirming its prior inclusion ($A$). This makes $P(B)$—a key metric for assessors, especially when $A$'s occurrence is confirmed, shifting their focus to this probability as a measure of unlearning success. Bayes' Theorem simplifies to:

\begin{equation}
    P(B|A) = P(B).
\label{eq2}
\end{equation} 

Eq (\ref{eq2}) directly illustrates that we should pay attention to non-membership inference, not the membership inference. A more rigorous analysis on Eq (\ref{eq2}) is provided in Appendix~\ref{theoremcpua_appendix}.

\subsection{Non-membership Inference in LUCM}

Within the LUCM framework, enhancing the design of non-membership inference starts by understanding its unique setup. As part of this framework, we define samples within three distinct subsets based on their membership status relative to two given models $\theta_{\text{ori}}$ and $\theta_{\text{unl}}$. Let $M_{\theta}(z)$ be the membership indicator function for a model with parameters $\theta$, where $M_{\theta}(z) = 1$ if $z$ is a member of the model (that is, in the training set on which the model $\theta$ was trained) and $M_{\theta}(z) = 0$ otherwise. These subsets are defined as follows:
\begin{equation}
\begin{split}
      D_{rm} & = \{ z \in D \mid M_{\theta_{ori}}(z) = 1 \text{ and } M_{\theta_{unl}}(z) = 1 \}, \\ 
      D_{um} & = \{ z \in D \mid M_{\theta_{ori}}(z) = 1 \text{ and } M_{\theta_{unl}}(z) = 0 \}, \\
      D_{nm} & = \{ z \sim \pi \ \mid M_{\theta_{ori}}(z) = 0 \text{ and } M_{\theta_{unl}}(z) = 0 \}, 
\end{split}
\label{eq7}
\end{equation}
where $D_{rm}$  is the set of retained members, $D_{um}$ includes all unlearned members, and $D_{nm}$ comprises non-members.

As a manager, one has known the member information for the original model $\theta_{\text{ori}}$. The primary objective then is predicting their non-membership status in the model $\theta_{\text{unl}}$ after unlearning has been applied. This differs from traditional (non)membership inference approaches, which are typically adversarial to the threat model and lack direct knowledge of (non)member data; consequently, they must mimic the (non)member behavior using shadow models. In contrast, within the LUCM setting, the manager can directly observe both the (non)member behavior associated with $\theta_{\text{ori}}$ and the behavior of $D_{nm}$ derived from both $\theta_{\text{ori}}$ and $\theta_{\text{unl}}$ without the need for shadow models and without compromising the setting. This setting thereby creates opportunities for implementing non-membership inference more efficiently. Next, we design two unlearning completeness scores and integrate them into a single, enhanced score.

For a target sample $z=(x,y)$, let us denote the loss of a model $\theta$ on this sample as $\mathcal{L}_{\theta}(z)$. The model $\theta$ assigns confidence scores $f_{\theta}(x, y)$ to its ground truth label. The logit scaled confidence, denoted as $g_{\theta}(z) = \log(f_{\theta}(x, y) / (1 - f_{\theta}(x, y))$, is a transformation of the confidence score $f_{\theta}(x, y)$ using the logit function. This transformation maps the confidence scores, which typically lie between 0 and 1, to a more interpretable scale ranging from negative infinity to positive infinity.

\smallskip
 \textbf{Likelihood's Difference Score}: This method quantifies non-membership likelihoods for each model individually, and calculates the differences between the measurements on original and the unlearned models. When provided with the original model $\theta_{ori}$ and the non-member set $D_{nm}$, we estimate a Gaussian distribution $G_{ori} \sim \mathcal{N}(\mu_{ori}, \sigma^{2}_{ori})$, characterized by its mean $\mu_{ori}$ and variance $\sigma^{2}_{ori}$, based on the logit scaled model confidences of $\theta_{ori}$ within $D_{nm}$, following the implementation in ~\cite{CarliniCN0TT22}. Similarly, for the unlearned model, we obtain another Gaussian distribution $G_{unl} \sim \mathcal{N}(\mu_{unl}, \sigma^{2}_{unl})$. The lower the target model's confidence relative to the estimated mean value for non-members, the more likely it is that the query sample is a non-member. For a target sample $z$, we compute two separate non-membership likelihoods for each model:
\begin{equation}
\begin{split}
      h_{\theta_{ori}}(z) & = Pr[g_{\theta_{ori}}(z)<G_{ori}],\\ 
      h_{\theta_{unl}}(z) & = Pr[g_{\theta_{unl}}(z)<G_{unl}].   
\end{split}
\end{equation}

We then normalize the difference between the two single-model-based likelihoods to obtain the likelihood difference score:
\begin{equation}
    \text{L-Diff}(z) = \frac{1+h_{\theta_{unl}}(z) - h_{\theta_{ori}}(z)}{2}.
\end{equation}

\smallskip
\textbf{Difference's Likelihood Score}: The second method tracks changes in logit-scaled model confidences between the original model $\theta_{ori}$ and the unlearned model $\theta_{unl}$ for the provided samples. It then directly estimates a single likelihood score that reflects the degree of unlearning based on the observed changes.

In the context of the learning and unlearning processes, let us define a fixed-membership-status group $D_{fix} = D_{rm} \cup D_{nm}$, which comprises samples whose membership status remains unchanged throughout these processes. Based on the assumption that an optimal approximate unlearning algorithm could minimize side effects on the model behavior on retained samples and the model generalizability, it is reasonable to infer that samples belonging to $D_{fix}$ exhibit only minor changes in confidence. For this group, we could estimate a distribution to model these changes. It involves calculating the confidence changes of its samples from $\theta_{ori}$ to $\theta_{unl}$ and fitting these to a Gaussian distribution $G_{fix}$, characterized by the mean and variance. 

In contrast, samples in the update-status group, for example $D_{upt} = D_{um}$, should display more distinct confidence changes. The higher the confidence changes relative to the estimated mean value for $D_{fix}$, the more likely it is that the query sample update its status.

However, given that model confidence is bounded to the range $[0,1]$, the confidence change is constrained to $[-1,1]$. This range implies a non-normal distribution. To address this, we employ a three-step boosted strategy: 1) applying two distinct adjustment methods to parameterize the non-normal confidence change distribution; 2) calculating the likelihoods of a target sample's confidence change relative to both distributions, separately; 3) taking the average of the two likelihoods to obtain our final unlearning score. The two adjustment methods are logit scaling and Median Absolute Deviation (MAD,~\cite{LEYS2013764}). We provide details in the following. 

Again with the logit scaling, for a sample $z$, we compute its logit scaled confidence change as $\phi_{A}(z) = g_{\theta_{unl}}(z) - g_{\theta_{ori}}(z)$. Because $g_{\theta}(z)$ has mapped the confidence score from a finite interval to an infinite range of values, the confidence change $\phi_{A}(z)$ can now take values over the entire real line, making it easier to estimate and model using statistical techniques.

For $D_{fix}$, we can approximate the distribution of the corresponding logit scaled confidence changes using a Gaussian distribution $G_{fix, A}$ with $\mu_{A}$ and variance $\sigma^{2}_{A}$. The likelihood that the target sample $z$ belongs to $D_{upt}$ based on logit-scaled model confidence is calculated as follows:
\begin{equation*}
    D_{A}\text{Lik}(z) = 1 - \text{Pr}[\phi_{A}(z) > G_{fix, A}], \text{where } G_{fix, A} \sim \mathcal{N}(\mu_{A}, \sigma^{2}_{A}).
\end{equation*}

For MAD, we fist compute the samples confidence changes without logit scaling as $\phi_{B}(z) = f_{\theta_{unl}}(x, y) - f_{\theta_{\text{ori}}}(x, y)$, where $z=(x,y)$. Then we directly calculate the mean $\mu_{B}$ of the model confidence change for $D_{fix}$ and replacing the variance calculation with MAD, for the values are bounded in an interval. Specifically, $var_{B}=c\cdot \textit{MAD}$, where $c$ is a constant to make MAD consistent with the standard deviation for normal distribution. $\textit{MAD}$, defined as the median of the absolute deviations from the data's median, offers a more robust estimate, being able to handle skewed distributions and less influenced by outliers compared to the standard deviation or variance. Therefore, we obtain $G_{fix, B}$. The likelihood that $z$ belongs to $D_{upt}$ directly directly based on model confidence is calculated as follows:
\begin{equation*}
    D_{B}\text{Lik}(z) = 1 - \text{Pr}[\phi_{B}(z) > G_{fix, B}], \text{where } G_{fix, B} \sim \mathcal{N}(\mu_{B}, var_{B}).
\end{equation*}

Further, the difference's likelihood score is:
\begin{equation}
    \text{D-Liks}(z) = \text{Mean} (D_{A}\text{Lik}(z),D_{B}\text{Lik}(z)).
\end{equation}

Finally, the measured score of unlearning completeness for sample $z$ is a boosted version based on $\text{L-Diff}(z)$ and $\text{D-Liks}(z)$. For the sake of simplicity, we choose to use the arithmetic average as our boosting function in this work.
\begin{equation}
    \textbf{UnleScore}(z) = \text{Boost} (\text{L-Diff}(z),\text{D-Liks}(z)).
\label{eq:unlearningscore}
\end{equation}
 
Note that $D_{rm}$ cannot be identified from $D_{um}$ in advance, as it is the objective of non-membership inference in LUCM. Consequently, $D_{fix}$ is not available. Instead, we use $D_{nm}$ to replace $D_{fix}$, and calculate the approximate mean and variance of $G_{fix}$. Experimental results will demonstrate the effectiveness of this replacement.

\smallskip
\noindent \textbf{Discussions.} For sample-level unlearning completeness measurement, our designed non-membership inference outputs unlearning scores scaled within the range of $[0,1]$ that represent unlearning completeness, rather than straightforward binary decisions. This approach is grounded in the rationale that a continuous score offers a more nuanced and detailed understanding of the model's behavior, specifically in the setting of LUCM.
This approach not only streamlines the quantitative measurements of unlearning completeness but also reduces the computational effort by eliminating the need for shadow model training, offering a more practical and efficient way for LUCM. The black-box design makes it generally applicable to a variety of unlearning methods.
When applying the designed metric to evaluate the performance of approximate unlearning algorithms, we prioritize the \textit{NMI\_TPR@LowFPR} as our primary metric. Additionally, we report on the balanced area under the curve (AUC) in our experimental analyses, following the evaluation methods implemented in \cite{CarliniCN0TT22}.

\section{Experimental Evaluation Setup}

We take 5 datasets in our evaluation benchmark, consisting of two image classification datasets (Cifar10~\cite{krizhevsky2009learning}, Cifar100~\cite{krizhevsky2009learning}), a shopping record dataset (Purchase100~\cite{ShokriSSS17}), a hospital record dataset (Texas100~\cite{ShokriSSS17}), and a location dataset (Location30~\cite{ShokriSSS17}). These datasets provide a varied testing ground for machine learning models, balancing the need for diverse data types with privacy considerations. Further information about these datasets and data processing can be found in Appendix~\ref{dataset_appendix}.

To validate the measurement utility of the designed metrics, we use an exactly retrained model as the benchmark for ground truth, obtaining outputs for both exact unlearned and retained samples. When applying this metric to benchmark approximate unlearning algorithms, we incorporate 7 approximate machine unlearning methods into our evaluation framework. These include \textbf{Fine Tuning~\cite{GolatkarCVPR20}}, \textbf{Gradient Ascent~\cite{abs-2111-08947}}, \textbf{Fisher Forgetting~\cite{GolatkarCVPR20}}, \textbf{Forsaken~\cite{MaLLLMR23}}, \textbf{L-Codec~\cite{cvprMehtaPSR22}}, \textbf{Boundary Unlearning~\cite{ChenGL0W23}}, and \textbf{SSD~\cite{abs-2308-07707}}. Appendix~\ref{approximate_alg_appendix} provides the detailed descriptions of these algorithms.

\subsection{Metrics Baselines}

We compare our designed metrics with the state-of-the-art privacy leakage-based metrics related to machine unlearning, focusing on exact retraining auditing tasks to demonstrate the utility of our metrics.
\begin{figure*}[h]
    \centering
    \includegraphics[width=0.95\linewidth]{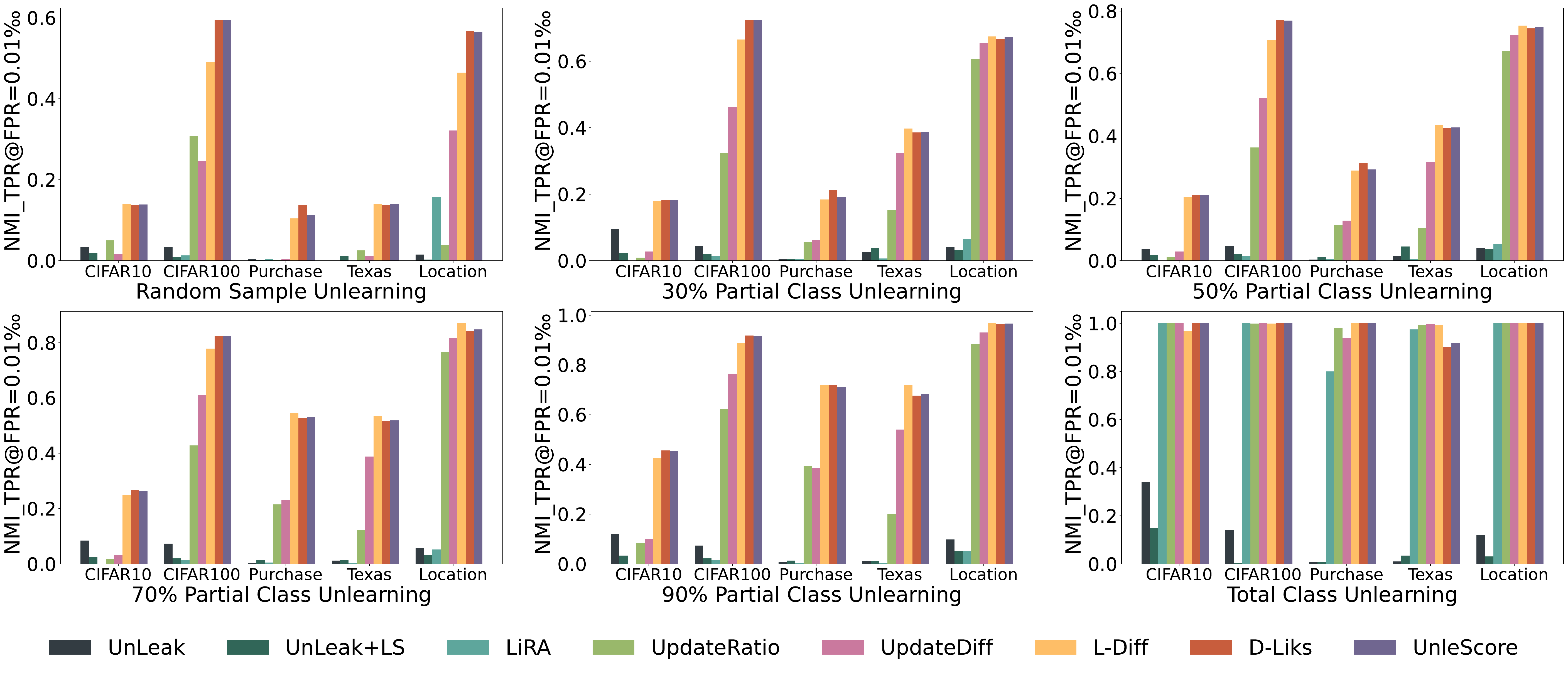}
    \caption{Statistical results (\textit{NMI\_TPR@FPR=0.01‰}) of metric utility  on 3 types of exact unlearning. }
    \label{fig:retrain_results_tprs}
\end{figure*}

\noindent \textbf{UnLeak:} Following the instruction of ~\cite{Chen000HZ21}, we train a baseline model, referred to as the `shadow original model', on the entire shadow dataset. Subsequently, we train 16 separate `shadow exact retraining models' on 16 distinct subsets of the shadow dataset, each termed a `shadow retaining set'. These shadow retaining sets, randomly sampled from 80\% of the shadow dataset, constituted 50\% of the shadow dataset's size each time. The adversary then processes the remaining 20\% of the shadow dataset samples, feeding them into their corresponding shadow retraining models and the shadow original model to obtain their posterior outputs. We train the attack model using the features constructed from these posteriors. Finally, the attack model can make predictions for target samples based on their constructed features.

\noindent \textbf{UnLeak+LS:} To enhance the performance of UnLeak, we further apply the logit scaling to the obtained posteriors, as implemented in LiRA~\cite{CarliniCN0TT22}. Then, we train the attack model using the newly constructed features from the scaled logits. 

\noindent \textbf{LiRA:} We implement the offline LiRA~\cite{CarliniCN0TT22} on $\theta_{unl}$ to predict if the target point is a non-member of $\theta_{unl}$. These shadow models replicate the architectural design of the original models. Additionally, the hyper parameters used in training the shadow models are aligned with those used in training the original models. Similarly, we train 128 shadow models on randomly sampled shadow datasets. For each target sample, we can query its non-member outputs on those shadow datasets, estimate the corresponding Gaussian distribution, and calculate the LiRA score. 

\noindent \textbf{UpdateRatio, UpdateDiff:} Following the procedure outlined by Jagielski et al.~\cite{Jagielski23}, we calculate the LiRA scores for a given target sample using both the original trained model and the unlearned model.  For estimating the LiRA scores, 128 shadow models were employed. We then combined two scores into a single score using the `ScoreRatio' and `ScoreDiff' methods used in ~\cite{Jagielski23}, separately. Based on the scoring function used, we named the two adversarial methods `UpdateRatio' and `UpdateDiff' for convenience.

\noindent \textbf{L-Diff, D-Liks, UnleScore:} As the proposed metrics, we will explore \textbf{L-Diff}, \textbf{D-Liks}, and \textbf{UnleScore} separately to understand their individual superiority in various unlearning measurement tasks in Section~\ref{sec:subsec:overallvalidation}.  Subsequently, we will employ \textbf{UnleScore} to detect unlearning anomalies and benchmark the performance of approximate unlearning algorithms.

It's important to note that LiRA, UpdateRatio, and UpdateDiff were initially designed for membership inference purposes. In our experiments, we flip the outcomes to adapt their outputs from indicating membership probability to reflecting non-membership probability.

\section{Unlearning Metric Validation}\label{sec:validation}

In this section, we begin by validating the effectiveness of our newly designed metrics for measuring the sample-level unlearning completeness. We accomplish this through a comprehensive analysis that includes measurements in exact unlearning tasks, analysis of score distributions across groups with varying non-membership statuses, and detection of unlearning anomalies.

\subsection{Assessing Metrics' Measurement Utility in Exact Unlearning Tasks}\label{sec:subsec:overallvalidation}

We first apply our designed metrics and baselines to three different exact unlearning tasks, each varying in measurement difficulty due to the different distribution overlaps between retained and exact unlearned samples. These include:

\begin{itemize}
    \item \textbf{Random Sample Unlearning:} This process targets samples within a \textit{randomly selected batch} without focusing on any specific class or data characteristic.

    \item \textbf{Partial Class Unlearning:}  This approach involves unlearning a portion of the samples from a specific class in the original model.
    
    \item \textbf{Total Class Unlearning:} The model is required to unlearn all instances belonging to a specified class.     
\end{itemize}

For random sample unlearning, we remove 500 randomly chosen samples from the training set and retrain the model on the retained set from scratch for each dataset. For partial class unlearning, we set the portions to 30\%, 50\%, 70\%, and 90\%, respectively. For both partial and total class unlearning, we select the first 10 classes from each dataset and perform unlearning for each class individually, averaging the results across these 10 classes. 

We compare the results of L-Diff, D-Liks, and five other related MIAs to measure the unlearning completeness for all exactly unlearned samples, as well as for retained samples from all training members. Given that this is a scoring task, we first use the output scores from baseline methods directly for all retained and exact unlearned samples. Then we collect the scores of all samples and their non-membership status to calculate the \textit{NMI\_TPR@FPR=0.01‰} for each dataset. This involves setting a binary threshold for the scores to achieve an FPR of 0.01‰, and then computing the corresponding TPR. We prioritized the \textit{NMI\_TPR@FPR=0.01‰} as the primary metric for interpreting the statistical measurement performance. By setting such a stringent threshold (0.01‰ FPR), we ensure that only the most clear-cut cases are classified as unlearned samples if a binary output is required based on the measuring score. 

As shown in Figure~\ref{fig:retrain_results_tprs}, L-Diff, D-Liks, and UnleScore outperform other baselines by a considerable margin. For LiRA, UpdateRatio, and UpdateDiff, their \textit{NMI\_TPR@FPR=0.01‰} results match the performance of L-Diff and D-Liks \textbf{only} in the total class unlearning task. Compared to L-Diff, D-Liks performs better in random sample unlearning and partial class unlearning tasks, while L-Diff excels in total class unlearning tasks. UnleScore, which averages the scores of L-Diff and D-Liks at the sample level, results in a trade-off in performance. We will demonstrate in Section~\ref{sec:subsec:correlation} how the general superiority of our designed metric is useful for detecting unlearning anomalies.
We also present the AUC scores in Figure~\ref{fig:retrain_results_auc} in the Appendix for reference; however, we do not consider them as a meaningful metric for this analysis. AUC scores, which are summarized results across all possible thresholds, fail to capture fine-grained differences in scenarios with low FPR. There can be different values of \textit{NMI\_TPR@FPR=0.01‰} even when the AUC scores are identical. The results of Figure~\ref{fig:retrain_results_tprs} and Figure~\ref{fig:retrain_results_auc} highlight our earlier point: membership inference techniques do not necessarily perform well in non-membership inference tasks.

\noindent {\bf \textit{Measurability challenges vary by task.}} 
The difficulty of obtaining accurate unlearning measurements varies significantly across tasks, primarily due to the degree of distribution overlap between unlearned and retained samples. Let's discuss each task individually. For the total class unlearning measurement task, nearly all metrics achieved perfect \textit{NMI\_TPR@FPR=0.01‰} across five datasets, except for UnLeak and UnLeak+LS. For partial class unlearning tasks, as the unlearned portion of the targeted class increases from 30\% to 90\%, both our designed metrics and the baselines improve their results. However, our metrics consistently maintain a distinctly superior performance compared to the baselines.

The most challenging task, random sample unlearning, often results in similar distributions for unlearned and retained sets, significantly complicating measurements. This overlap prevents any metric from consistently achieving a perfect result. Nevertheless, in datasets like CIFAR10, Purchase, and Texas, our designed metrics significantly outperform other metrics, achieving a tenfold increase in \textit{NMI\_TPR@FPR=0.01‰}. In the CIFAR100 and Location datasets, the performance advantages of the designed metrics are approximately twice that of other metrics. More analyses are provided in Appendix~\ref{sec:subsec:overallvalidation_add}.

After individually analyzing the contributions of L-Diff and D-Liks, we will use \textbf{UnleScore}—the average of the two metrics—for subsequent analyses for convenience.

\begin{figure*}[h]
    \centering
    \begin{subfigure}[b]{0.16\linewidth}
         \centering
         \includegraphics[height=0.08\textheight]{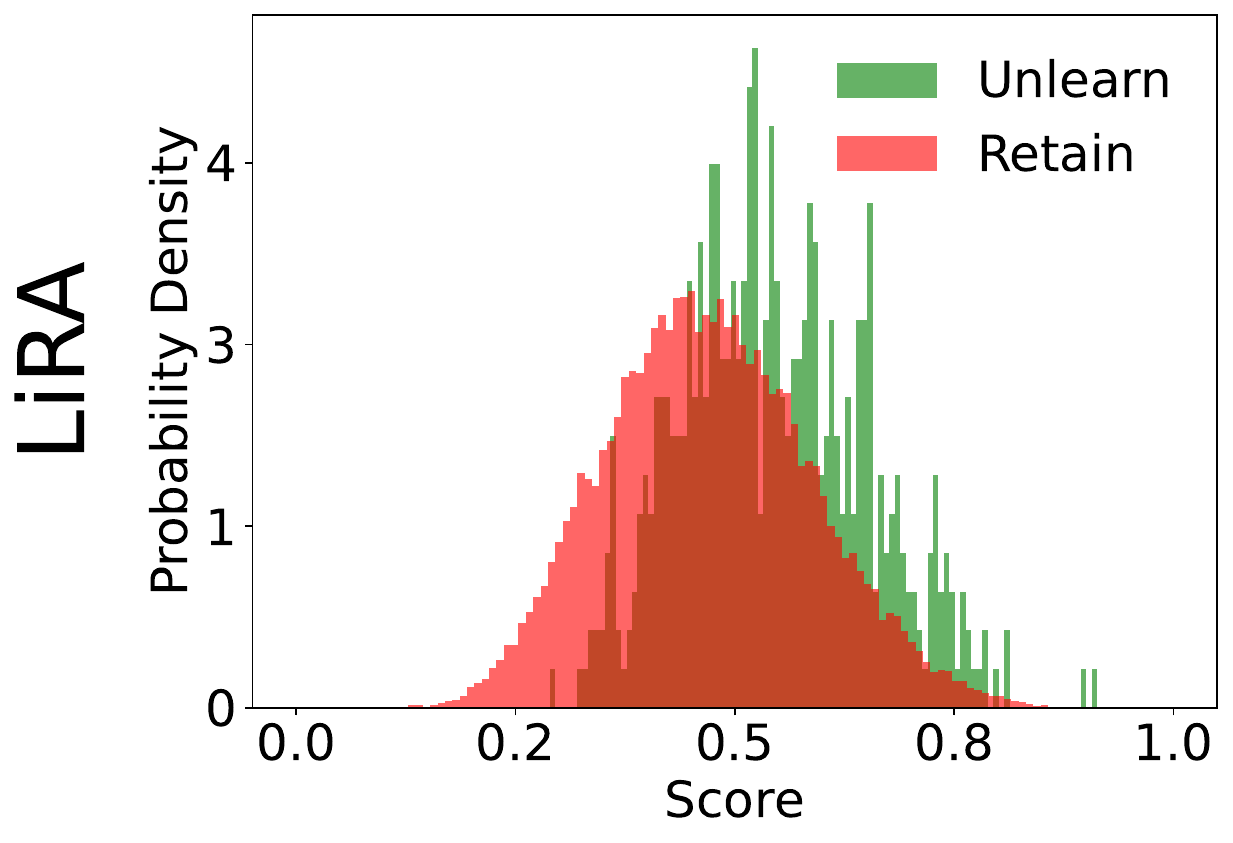}
         \caption{Random Sample}
     \end{subfigure}
     \begin{subfigure}[b]{0.15\linewidth}
         \centering
         \includegraphics[height=0.08\textheight]{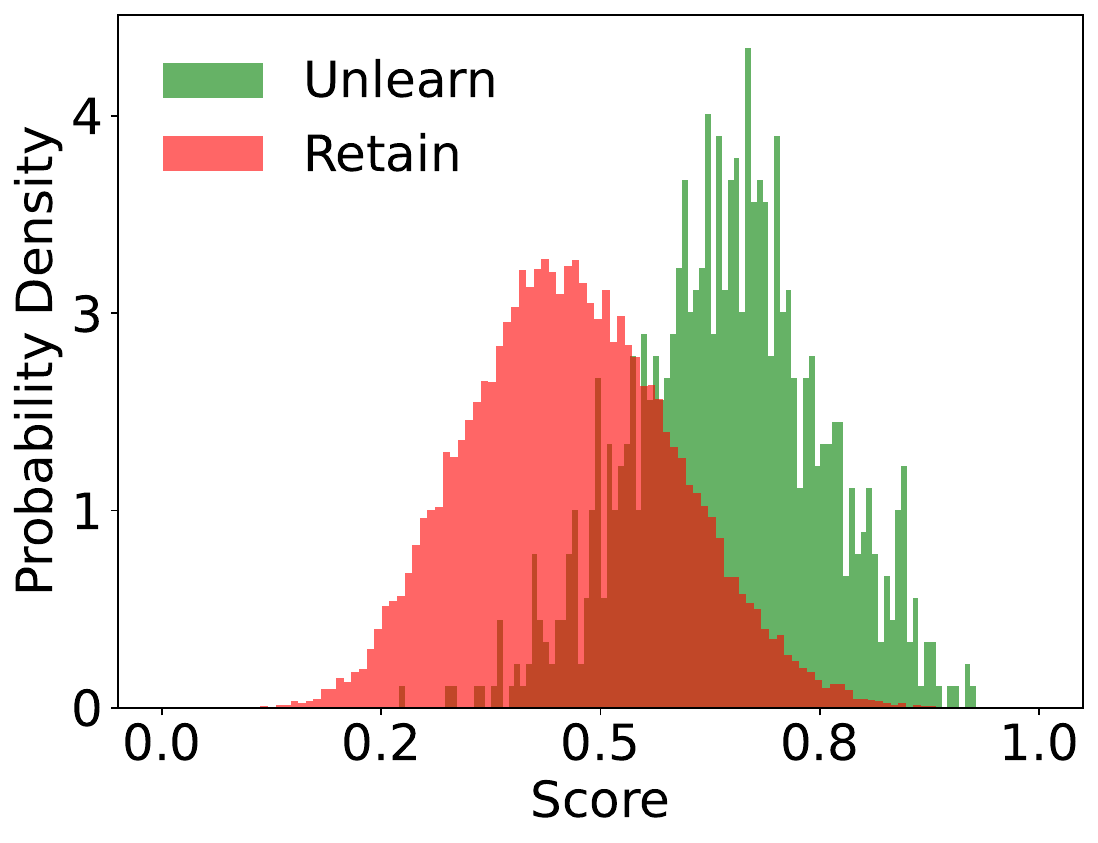}
         \caption{30\% Part-Class}
     \end{subfigure}
     \begin{subfigure}[b]{0.15\linewidth}
         \centering
         \includegraphics[height=0.08\textheight]{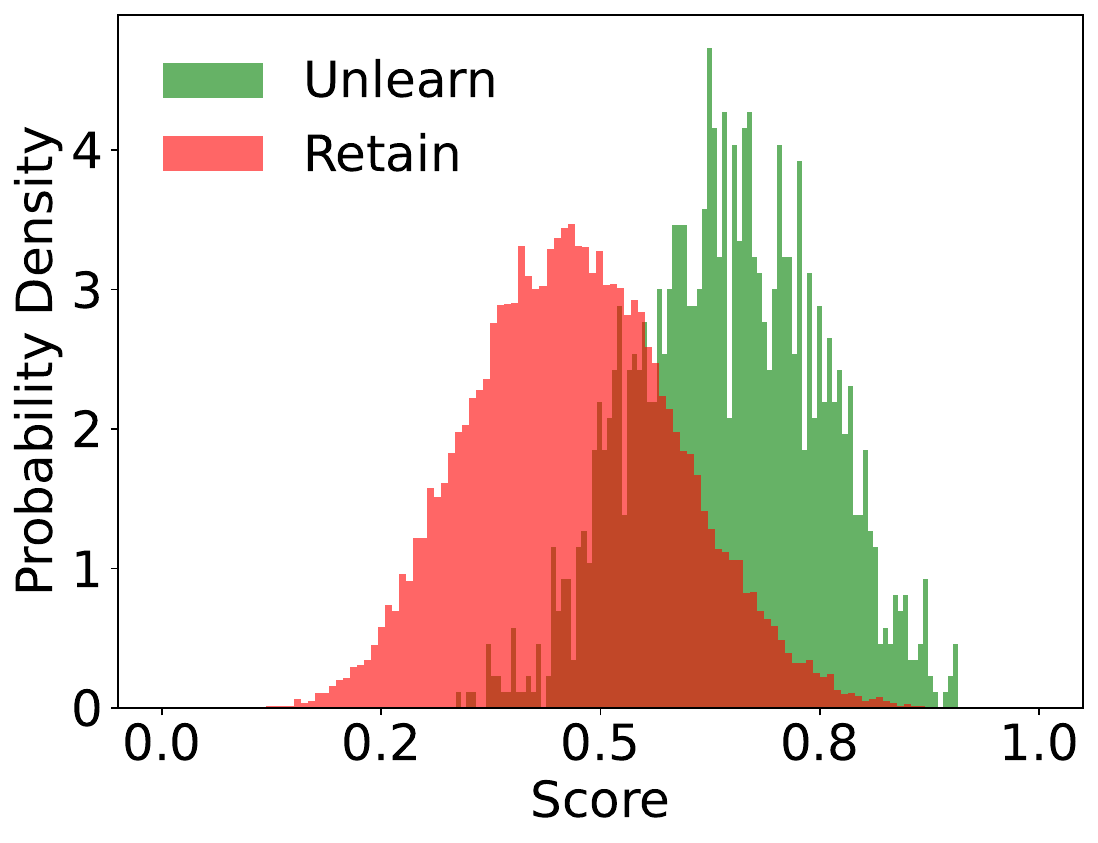}
         \caption{50\% Part-Class}
     \end{subfigure}
     \begin{subfigure}[b]{0.15\linewidth}
         \centering
         \includegraphics[height=0.08\textheight]{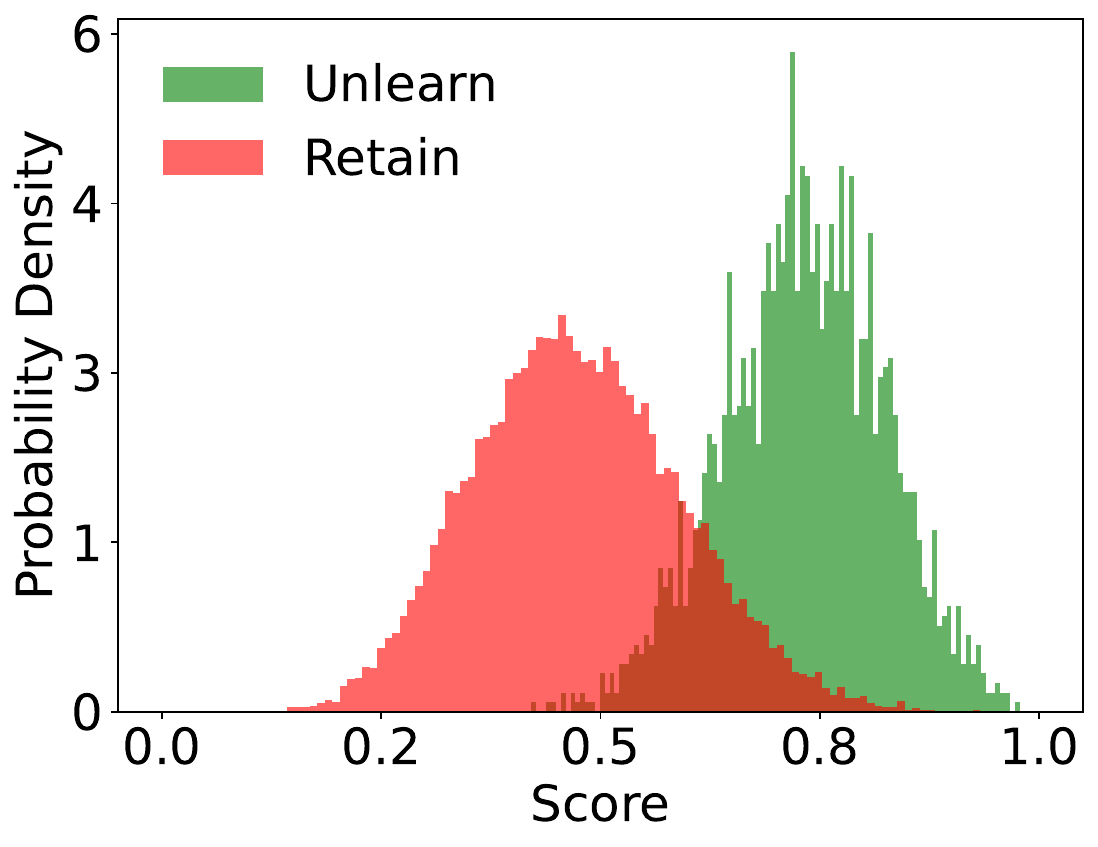}
         \caption{70\% Part-Class}
     \end{subfigure}
     \begin{subfigure}[b]{0.15\linewidth}
         \centering
         \includegraphics[height=0.08\textheight]{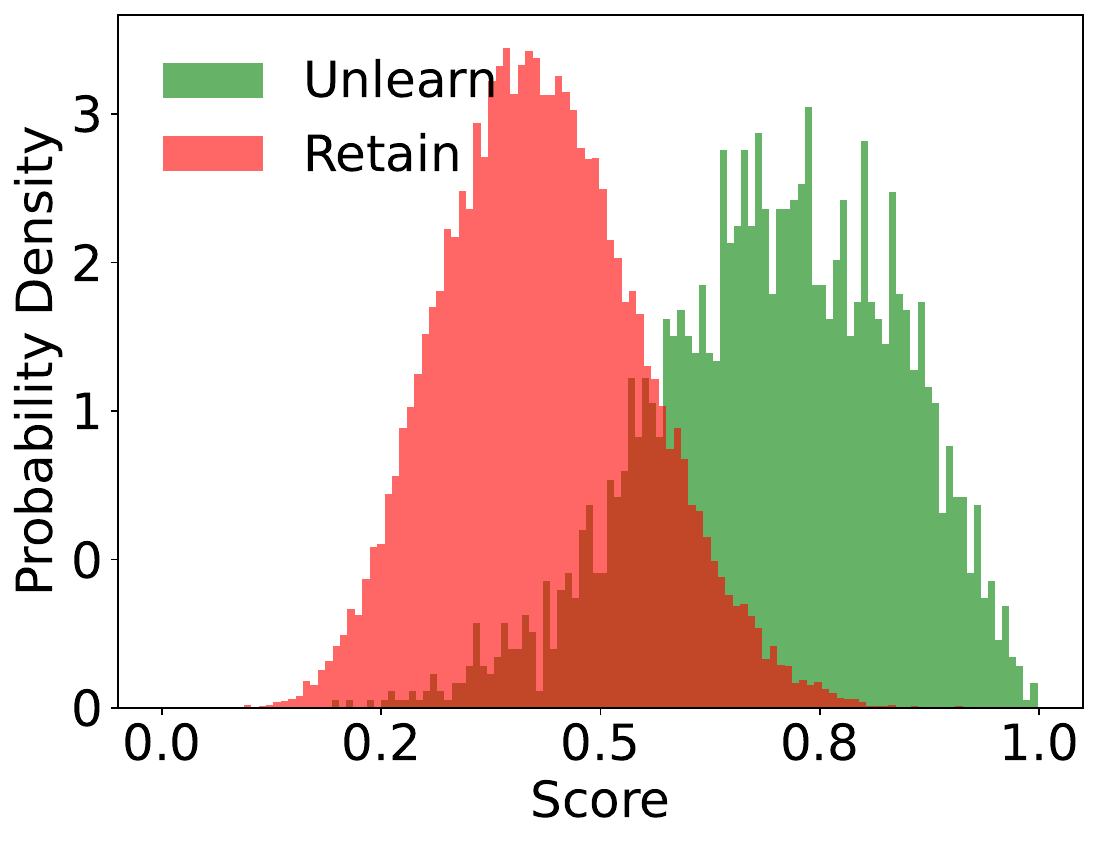}
         \caption{90\% Part-Class}
     \end{subfigure}
     \begin{subfigure}[b]{0.15\linewidth}
         \centering
         \includegraphics[height=0.08\textheight]{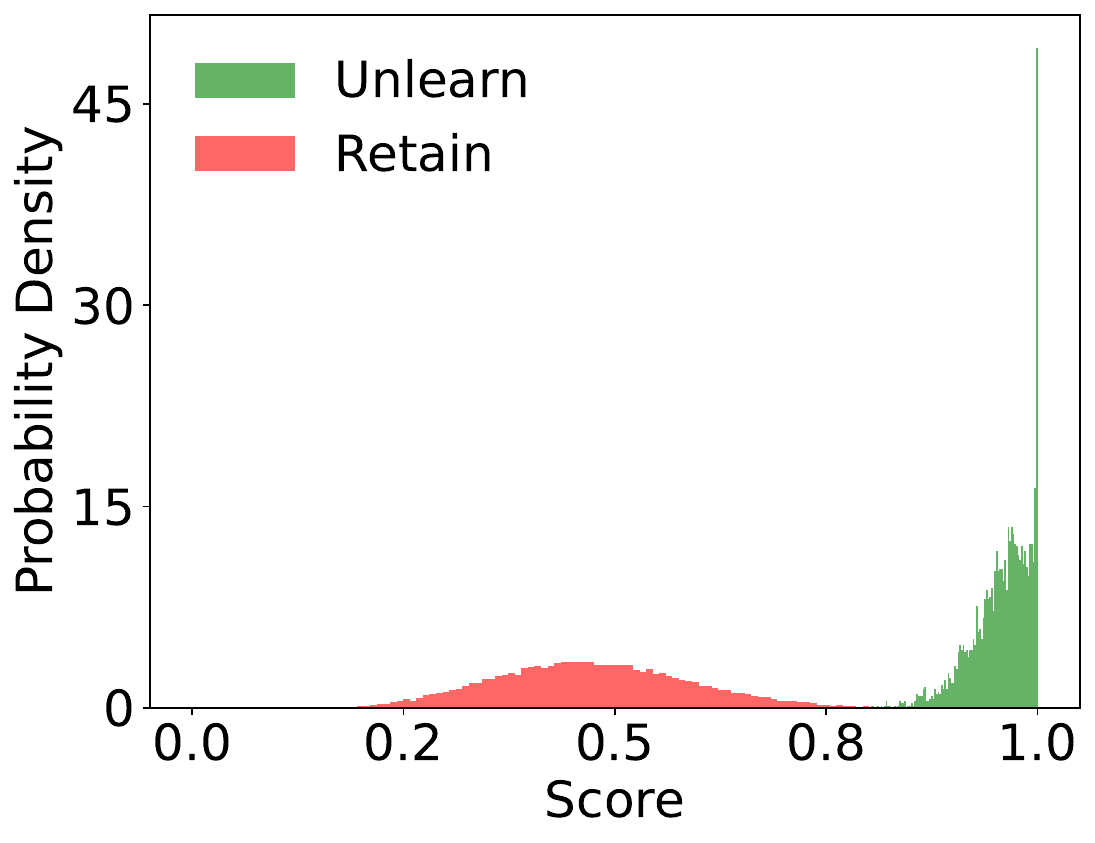}
         \caption{Total Class}
     \end{subfigure}
     \hfill
     \begin{subfigure}[b]{0.16\linewidth}
         \centering
         \includegraphics[height=0.08\textheight]{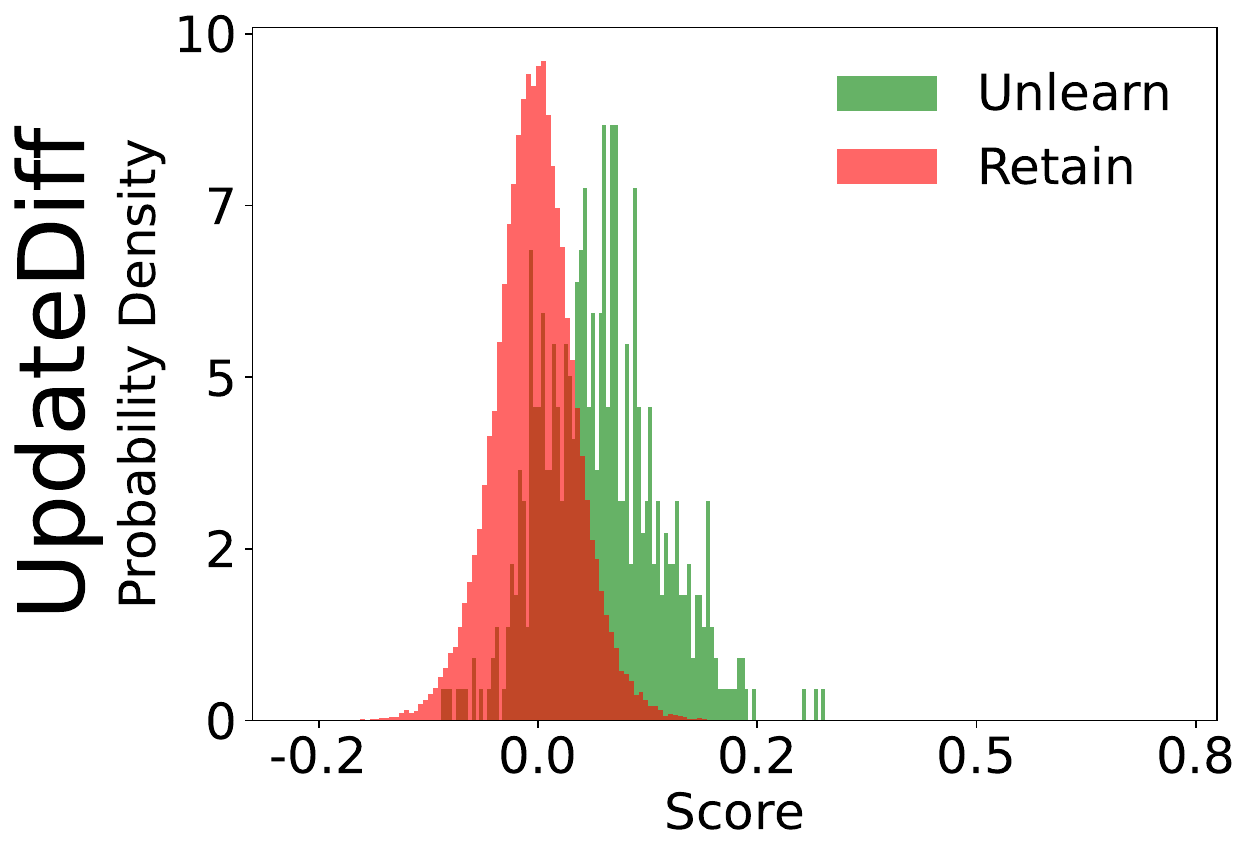}
         \caption{Random Sample}
     \end{subfigure}     
\begin{subfigure}[b]{0.15\linewidth}
         \centering
         \includegraphics[height=0.08\textheight]{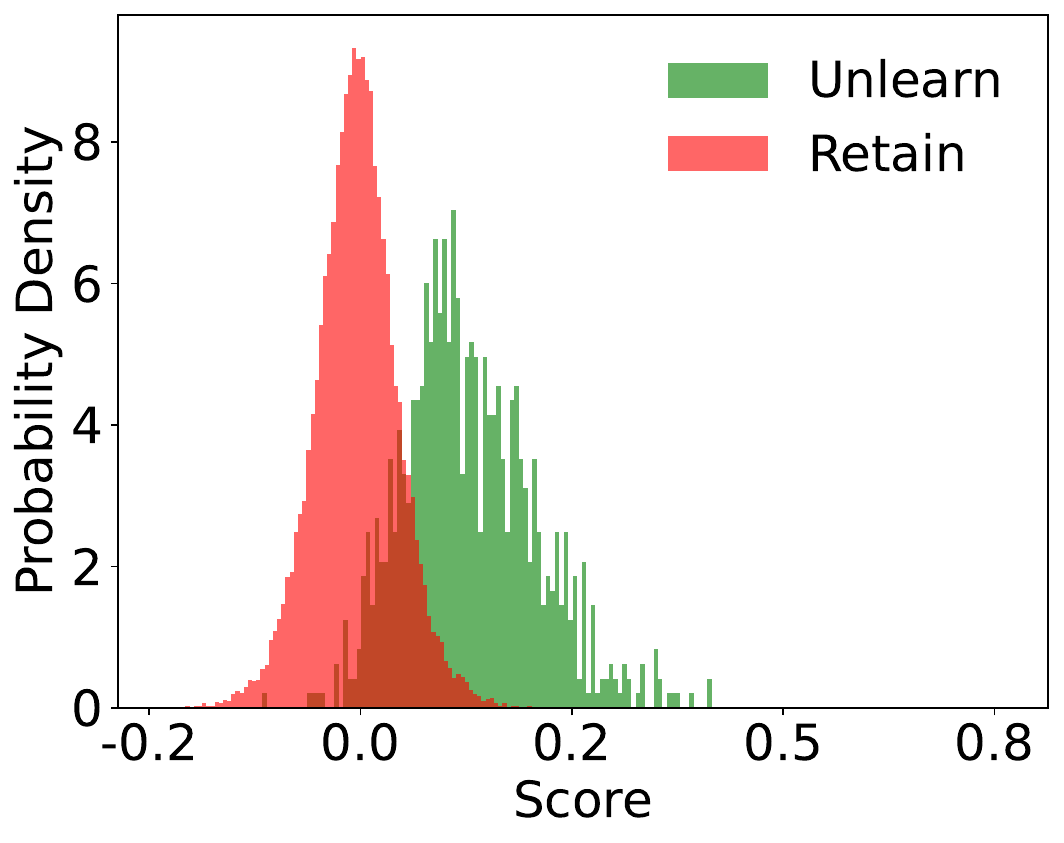}
         \caption{30\% Part-Class}
     \end{subfigure}     
\begin{subfigure}[b]{0.15\linewidth}
         \centering
         \includegraphics[height=0.08\textheight]{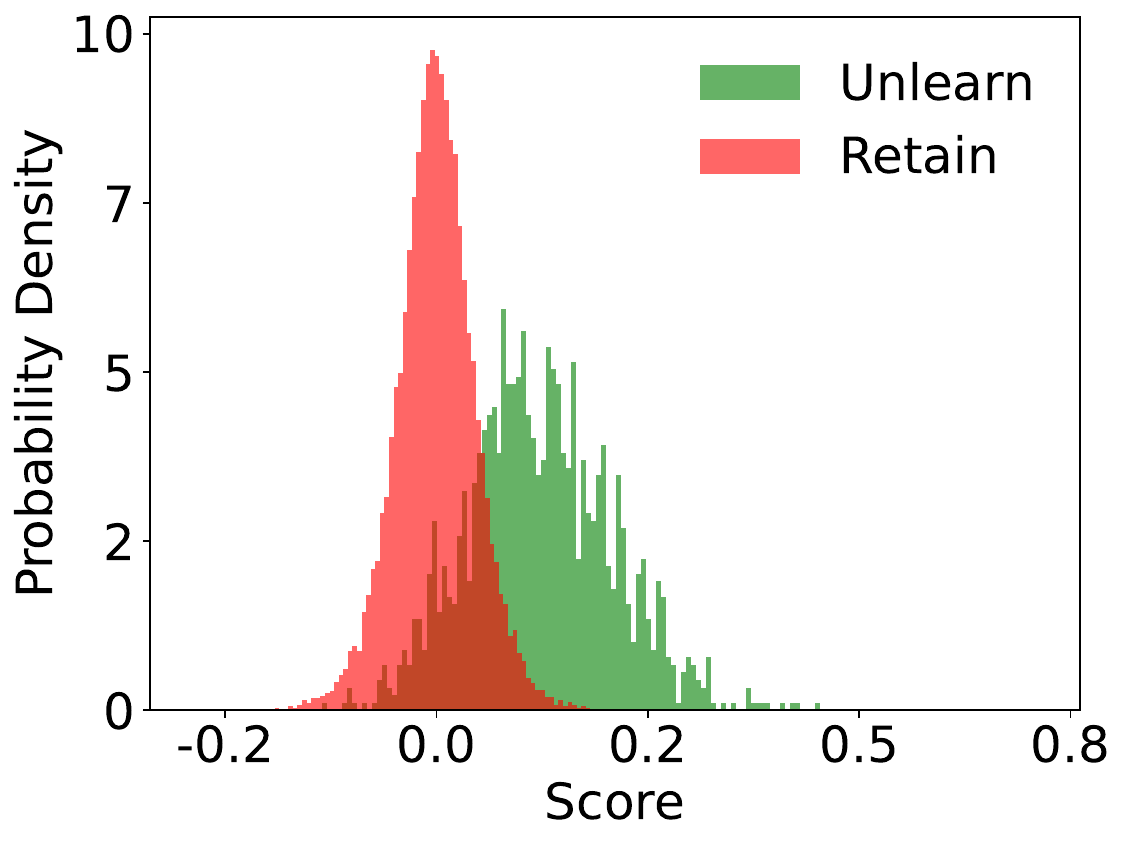}
         \caption{50\% Part-Class}
     \end{subfigure}     
\begin{subfigure}[b]{0.15\linewidth}
         \centering
         \includegraphics[height=0.08\textheight]{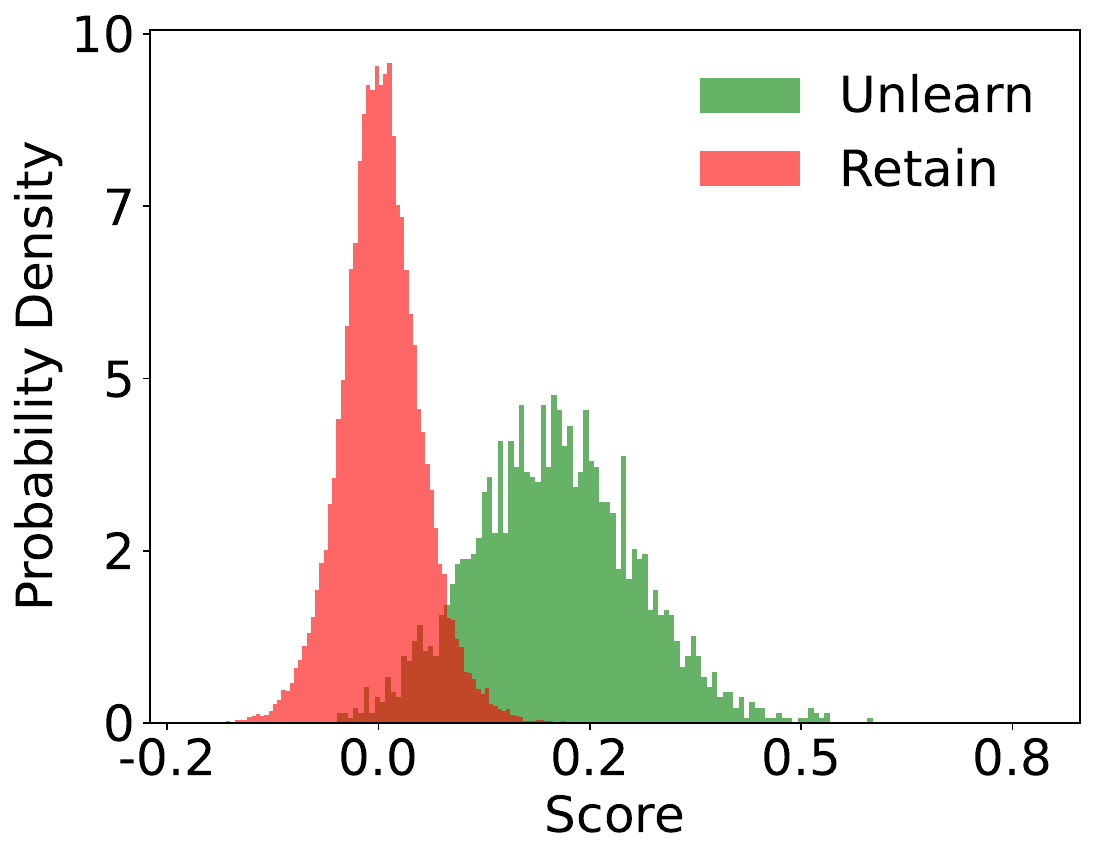}
         \caption{70\% Part-Class}
     \end{subfigure}     
\begin{subfigure}[b]{0.15\linewidth}
         \centering
         \includegraphics[height=0.08\textheight]{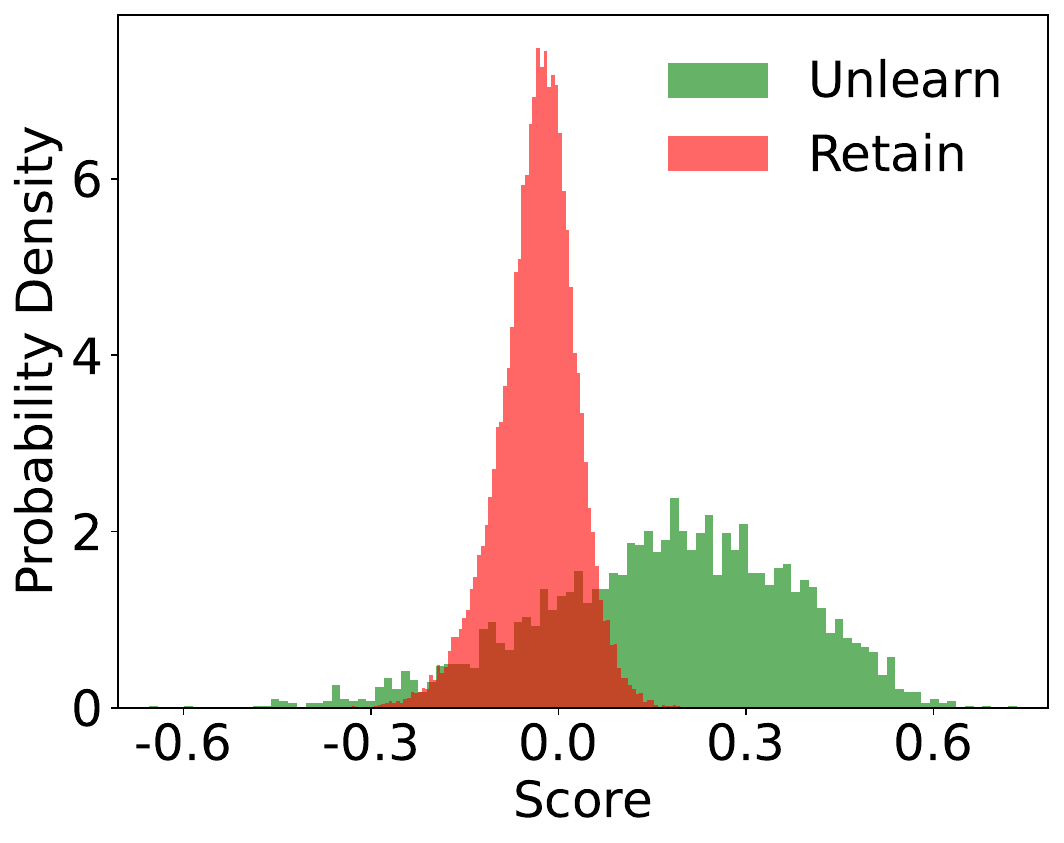}
         \caption{90\% Part-Class}
     \end{subfigure}     
\begin{subfigure}[b]{0.15\linewidth}
         \centering
         \includegraphics[height=0.08\textheight]{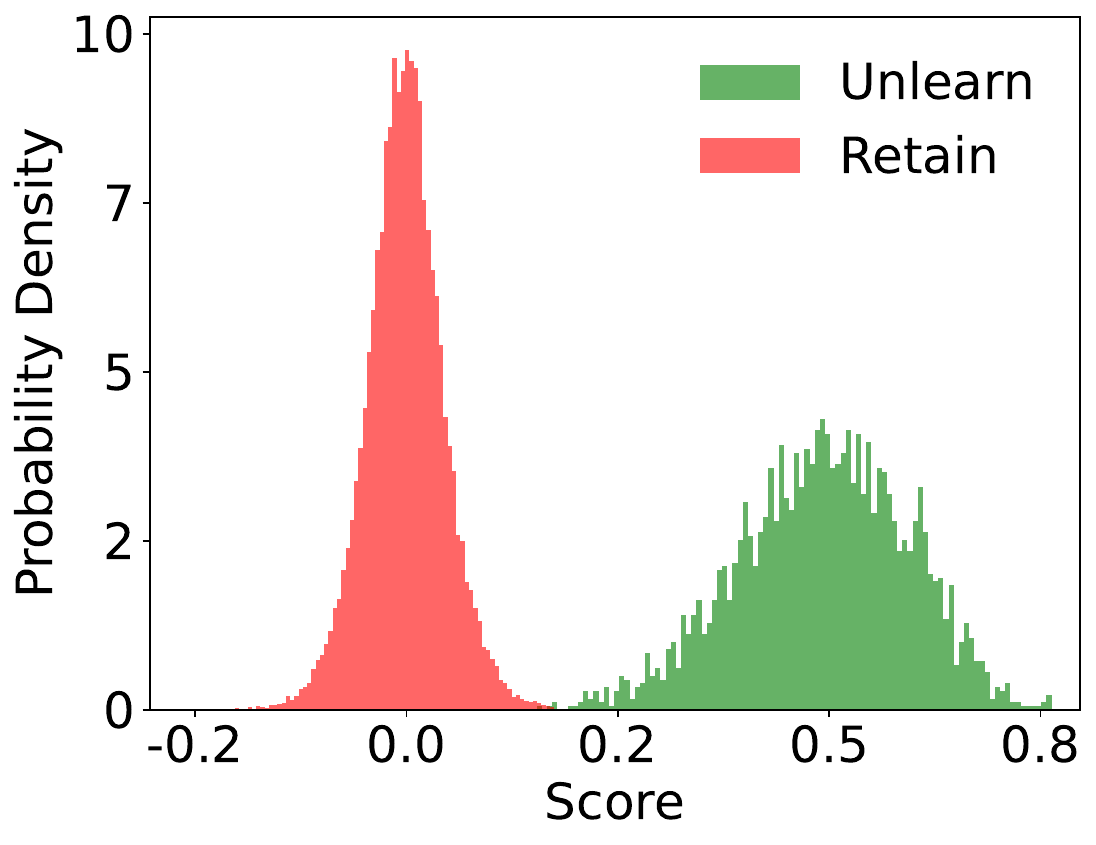}
         \caption{Total Class}
     \end{subfigure}     
     \hfill
\begin{subfigure}[b]{0.16\linewidth}
         \centering
         \includegraphics[height=0.08\textheight]{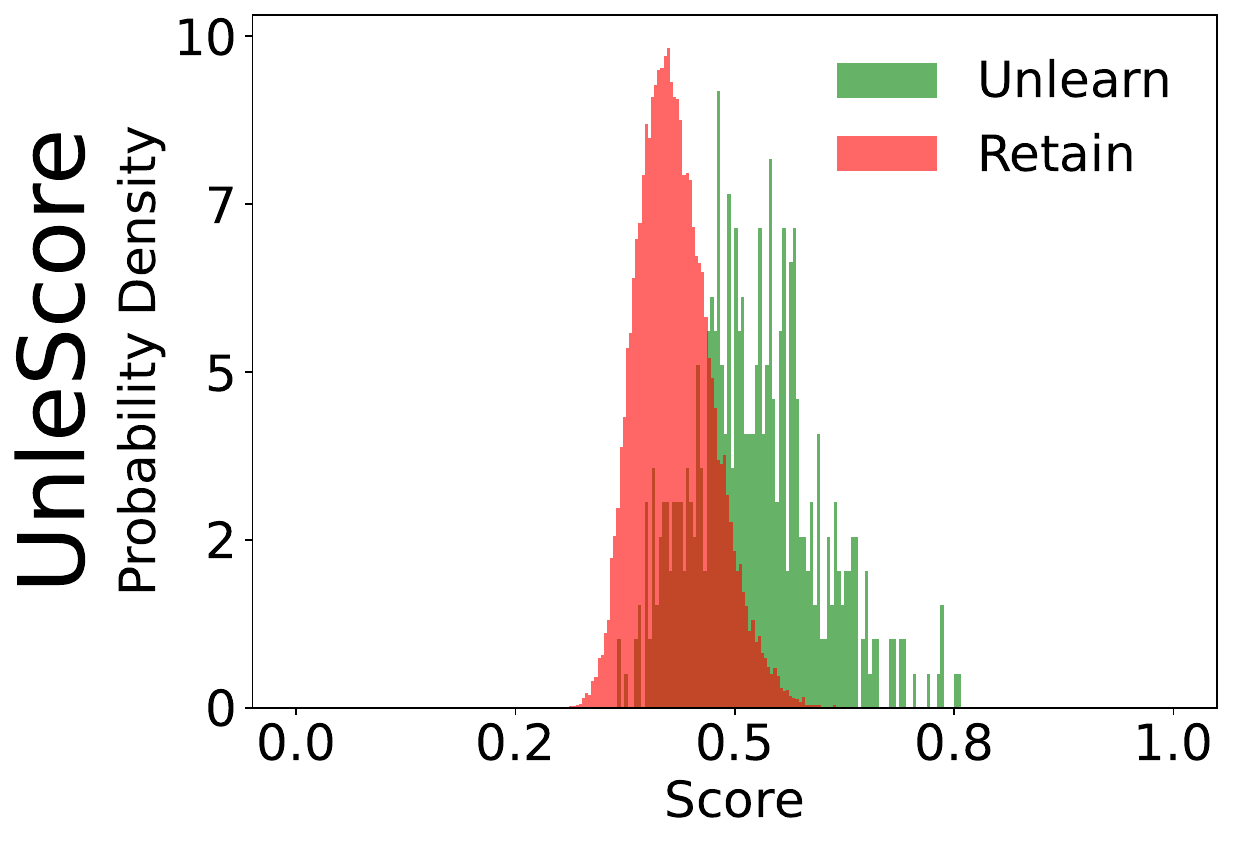}
         \caption{Random Sample}
     \end{subfigure}    
\begin{subfigure}[b]{0.15\linewidth}
         \centering
         \includegraphics[height=0.08\textheight]{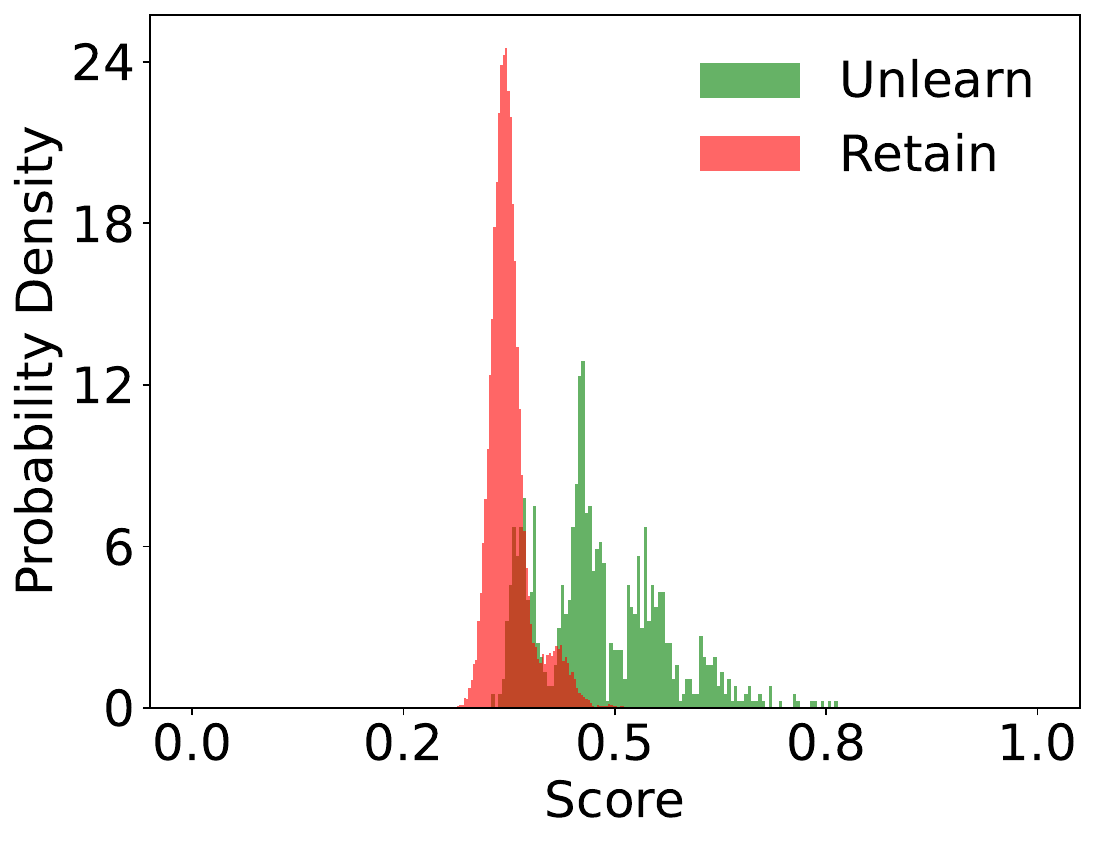}
         \caption{30\% Part-Class}
     \end{subfigure}    
\begin{subfigure}[b]{0.15\linewidth}
         \centering
         \includegraphics[height=0.08\textheight]{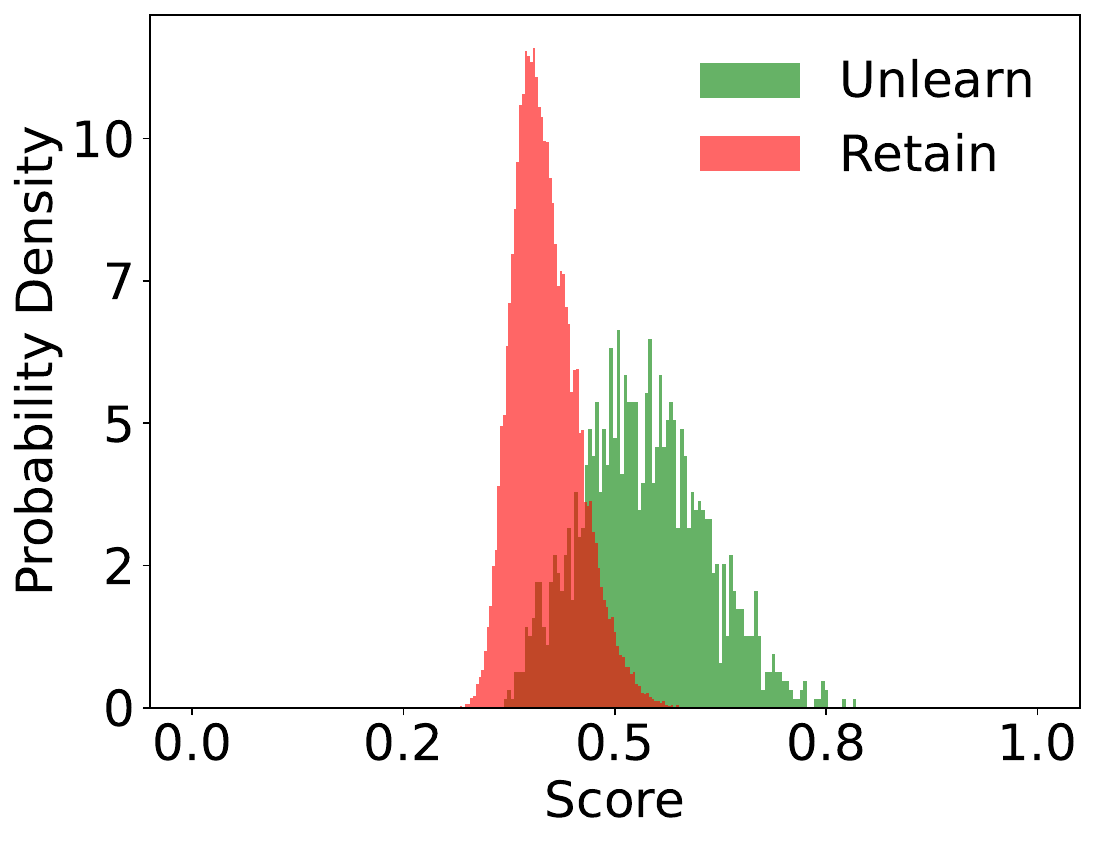}
         \caption{50\% Part-Class}
     \end{subfigure}    
\begin{subfigure}[b]{0.15\linewidth}
         \centering
         \includegraphics[height=0.08\textheight]{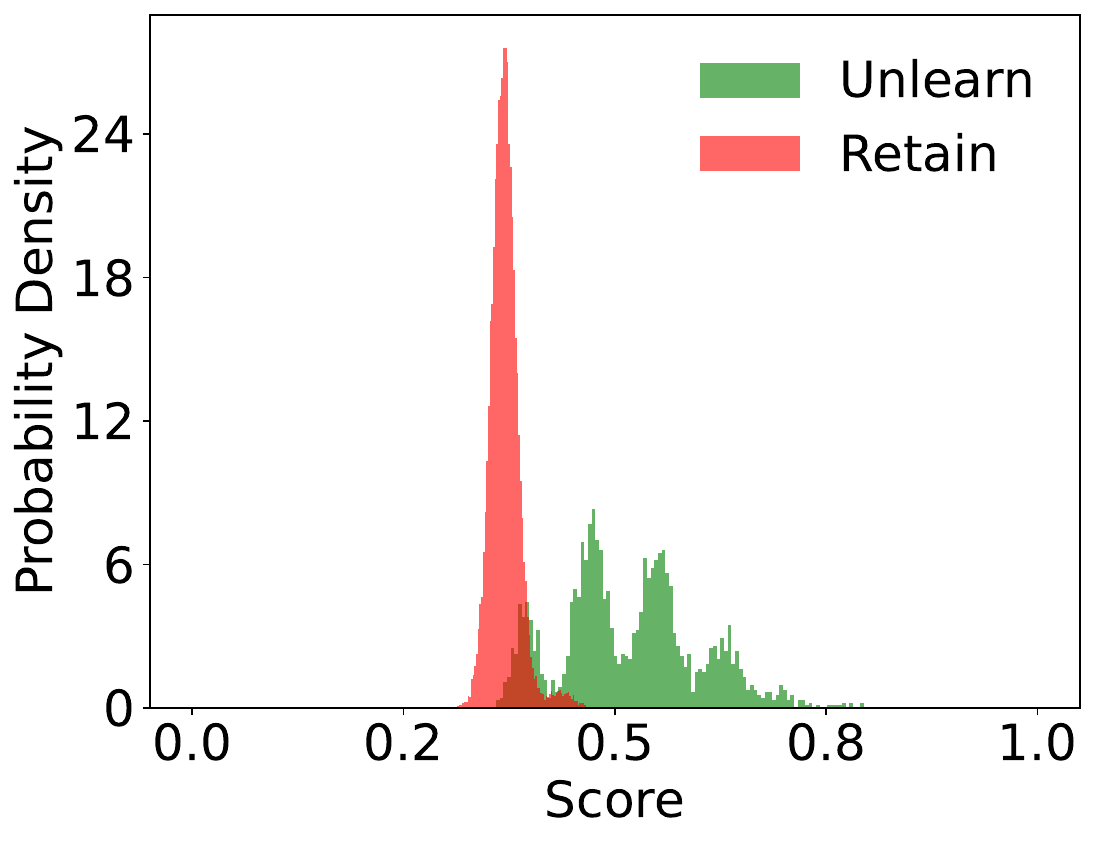}
         \caption{70\% Part-Class}
     \end{subfigure}    
\begin{subfigure}[b]{0.15\linewidth}
         \centering
         \includegraphics[height=0.08\textheight]{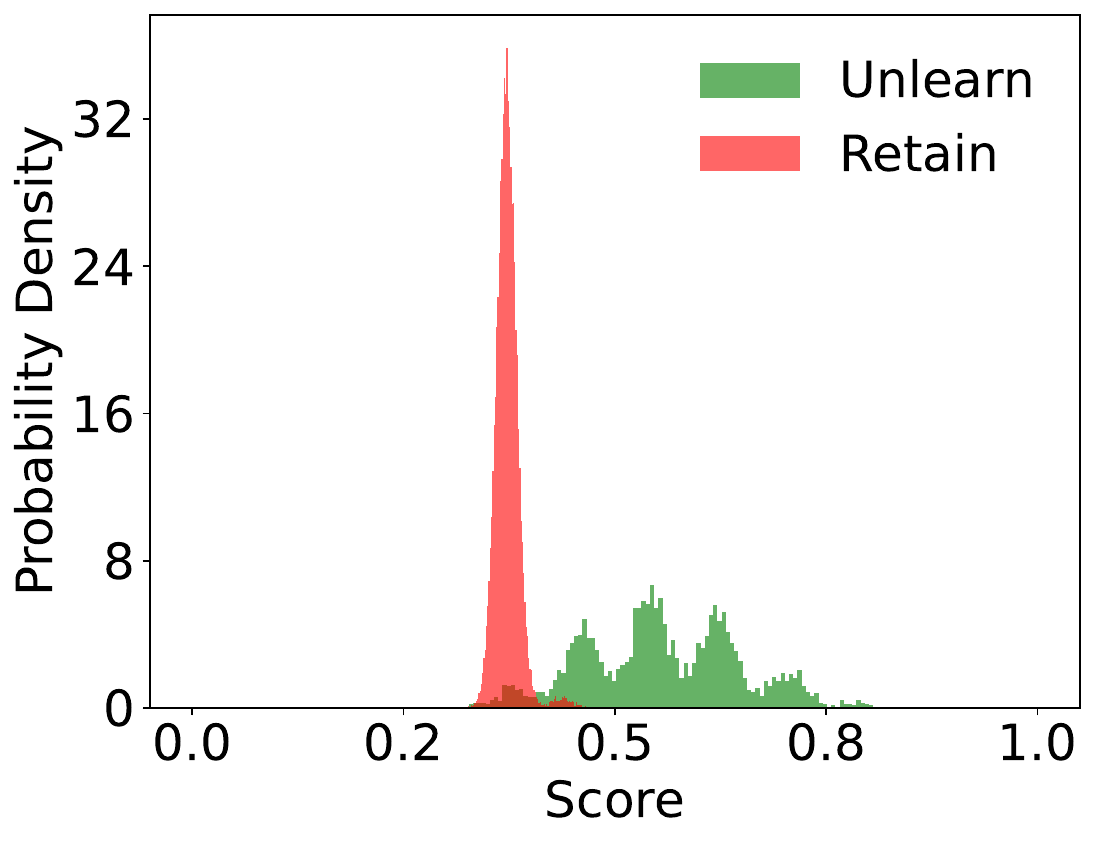}
         \caption{90\% Part-Class}
     \end{subfigure}    
\begin{subfigure}[b]{0.15\linewidth}
         \centering
         \includegraphics[height=0.08\textheight]{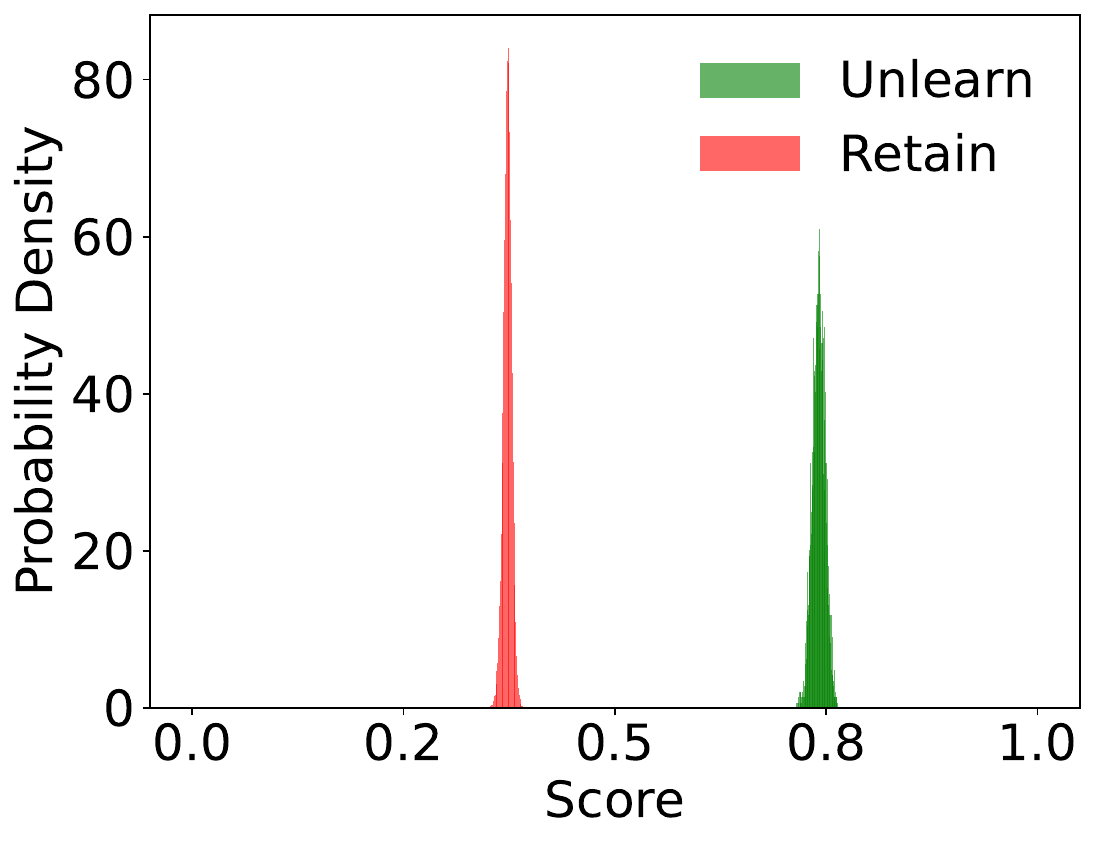}
         \caption{Total Class}
     \end{subfigure}       
     \caption{Score Distributions of Unlearning Metrics with Different Unlearning Tasks on CIFAR10. See Appendix~\ref{sec:subsec:metricthreshold_add} for results on other datasets.}
    \label{fig:unlearningscores_cifar10}
\end{figure*}

\subsection{Metric Threshold for Sufficient Unlearning}\label{sec:subsec:metricthreshold}

Is achieving a metric close to 100\% truly a sufficient condition for exact unlearning? Our metric, which ranges from 0 to 1, is designed to differentiate between unlearned and retained samples through an unlearning score. While we do not aim to establish a binary threshold for this measurement task, analyzing the score distributions for both unlearned and retained samples during exact retraining may offer valuable insights into the adequacy of the high score as a sufficient condition for exact unlearning. It's important to note that attaining a 100\% metric score for exact retraining is also challenging. We just expect the unlearned samples will have higher unlearning scores than those retained samples. Consequently, we first explore what metric score level would constitute a sufficient condition for exact unlearning.

Figure~\ref{fig:unlearningscores_cifar10} shows the detailed measurement scoring results of UnleScore, LiRA and UpdateDiff for unlearned and retained samples. LiRA, although a leading MIA technique, fails to binary differentiate between the scores of unlearned and retained samples in random sample unlearning and partial class unlearning tasks. UpdateDiff, even though it employs a difference calculation similar to ours, reports broad score ranges for both retained and unlearned samples, with many retained samples scoring near the average for unlearned samples. In contrast, UnleScore delivers more robust results by concentrating the scores of retained samples within a very narrow range. In most cases, setting a threshold below 0.4 reliably differentiates between retained and unlearned samples in the exact unlearning setting. The reason that LiRA and UpdateDiff perform worse than our metrics in measuring random sample unlearning and partial class unlearning may be that they are designed for binary membership inference tasks and fail to provide discriminative scores for retained groups and unlearned groups when there is a large distribution overlap between two groups.

The specific output distribution of UnleScore makes it particularly well-suited for assessing the utility of approximate unlearning for two main reasons: i) The concentration score distributions for retained samples and exact unlearned samples leave a wide blank score area between them, specifically for the total class unlearning task. This space could be utilized to position the scores of approximate unlearning samples, thereby differentiating them from the distributions of both retained and exact samples and identifying anomalies. ii) The robustness of its scoring enables meaningful comparisons between different unlearning procedures. Particularly, when one algorithm achieves a higher UnleScore than another on the same task, it could indicate the superior unlearning efficacy of the former. Additionally, while LiRA and UpdateDiff effectively distinguish between unlearned and retained samples in total class unlearning, their performance on approximate unlearned samples is notably poor (it's reasonable, as they were originally designed for binary decisions). We will illustrate this point in the next section.

\subsection{Correlation Between UnleScore and the Degree of Unlearning Implementation}\label{sec:subsec:correlation}

\begin{figure}[h]
    \centering
    \begin{subfigure}[b]{0.32\linewidth}
         \centering
         \includegraphics[width=\linewidth]{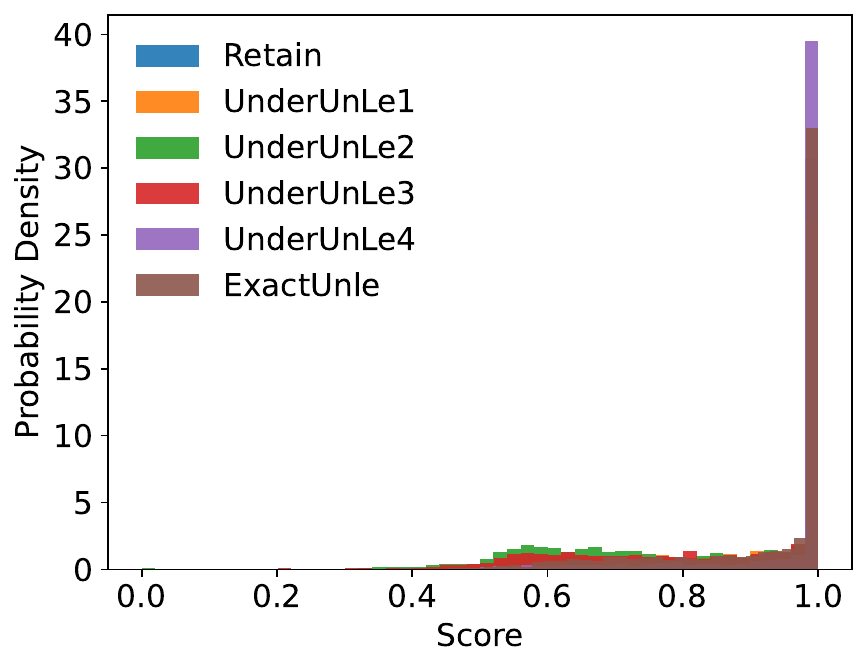}
         \caption{LiRA}
     \end{subfigure}
     \begin{subfigure}[b]{0.32\linewidth}
         \centering
         \includegraphics[width=\linewidth]{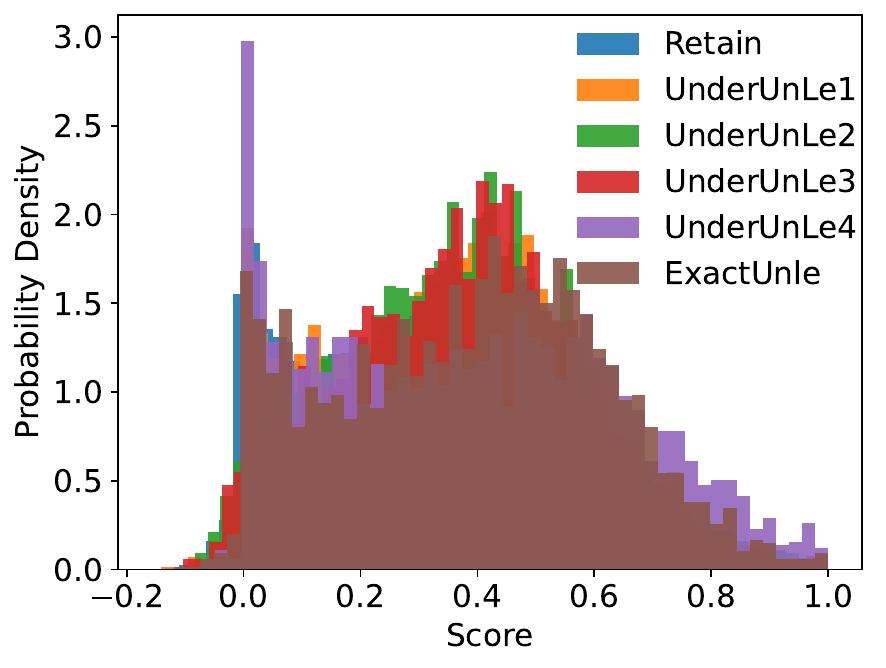}
         \caption{UpdateDiff}
     \end{subfigure}
     \begin{subfigure}[b]{0.32\linewidth}
         \centering
         \includegraphics[width=\linewidth]{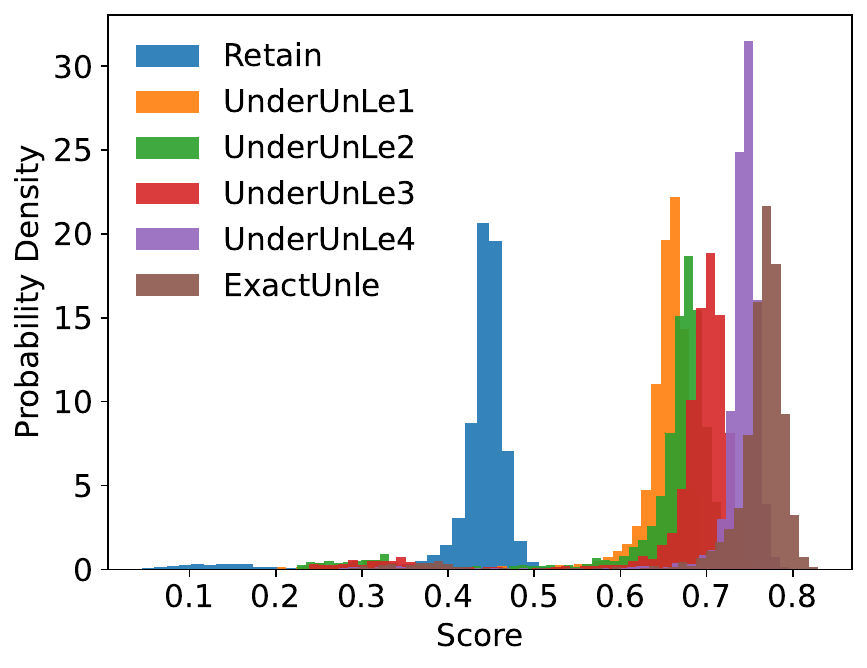}
         \caption{UnleScore}
     \end{subfigure}
     \caption{Scores of unlearning metrics on retained, under-unlearned, and exact unlearned groups within CIFAR10. Additional dataset results are in Appendix~\ref{sec:subsec:correlation_add}.}
        \label{fig:correlation_distributions}
\end{figure}

\noindent {\bf \textit{Under-unlearned group identification.}}  To further validate the correlation between the proposed metric and the completeness of data unlearning, we conduct a controlled experiment. We set up a task for CIFAR10 where classes 0-4 are requested for unlearning, and the remaining classes are retained. We start by dividing the training dataset into several disjoint groups: the \textbf{retained} set, four \textbf{under-unlearned} classes, and the \textbf{exact unlearned} class, each specified with a varying degree of unlearning. The original model is trained on the entire training set. The unlearned model is initially retrained solely on the retained set. To control its unlearning degree for the under-unlearned groups, we continue to finetune this unlearned model via these groups. We incrementally finetune the unlearned model on the first under-unlearned group for 10 epochs, the second under-unlearned group for 20 epochs, and so forth, each using a small learning rate. Note that we never train/finetune this model on the exact unlearned set. Therefore, the final unlearned model has varying memorization degrees corresponding to the unlearning levels of 6 groups.

With knowledge of the ground truth under-unlearned unlearning degrees for different groups, we can now analyze the correlation between the unlearning scores and the completeness of unlearning. The scoring results are provided in Figure~\ref{fig:correlation_distributions}. If we roughly define the unlearning levels as $[0, 1, 2, 3, 4, 5]$, the Pearson correlation coefficients between the scores and the unlearning levels are as follows: {\bf LiRA: 0.032, UpdateDiff: 0.071, and UnleScore: 0.830}. It is evident that our designed metric, UnleScore, achieves a superior correlation with the unlearning levels, whereas the baseline metrics fail to demonstrate a significant correlation. Consequently, it enables the identification of the corresponding unlearning risk areas for the target samples based on these results.

\smallskip

\noindent \textbf{Over-unlearned group identification in a camouflage case.} We now examine a more adversarial scenario in camouflage unlearning, where an approximate unlearning algorithm opts to camouflage unlearned samples by relabeling them with retained sample labels to achieve high UnleScores. Take 'airplane' images as an example: if all airplane images are relabeled as 'car' and the original model is fine-tuned on these camouflaged samples, the model’s confidence in recognizing airplanes as 'airplane' will decrease, potentially resulting in high UnleScores. However, this method merely changes the labels of the data without truly unlearning the information. Can it effectively circumvent our UnleScore measurements? Are we able to detect such superficial unlearning?

We implemented this camouflage-unlearning approach in the total class unlearning task of CIFAR10. Setting class-0 as the unlearned class, we adopted two strategies to 'camouflage' it. In the first strategy, all class-0 labels were replaced with class-1 (Template) labels. In the second strategy, the labels of class-0 samples were replaced with random labels. We report the corresponding UnleScores for both cases in Figure~\ref{fig:adv_cifar10}. In Case 1, following the camouflage operation, the UnleScores for the Camouflage Class are indeed higher than those of the retained classes. However, this operation significantly degrades the performance of Template Class predictions and also extends to degrading the model's ability to discriminate this class from other classes. Such over-unlearning anomalies can be detected by analyzing the distribution of the UnleScores of retained samples, especially if there is no distinct peak within a very wide score range. In Case 2, where the Camouflage Class is randomly relabeled, no specific target class suffers degradation in terms of prediction performance. Nevertheless, the model output for some retained samples is still affected. By identifying the existence of such a small subset of retained samples whose UnleScores deviate from the majority, we can conclude that over-unlearning anomalies are present.

\begin{figure}[t]
    \centering
     \begin{subfigure}[b]{0.48\linewidth}
         \centering
         \includegraphics[width=\linewidth]{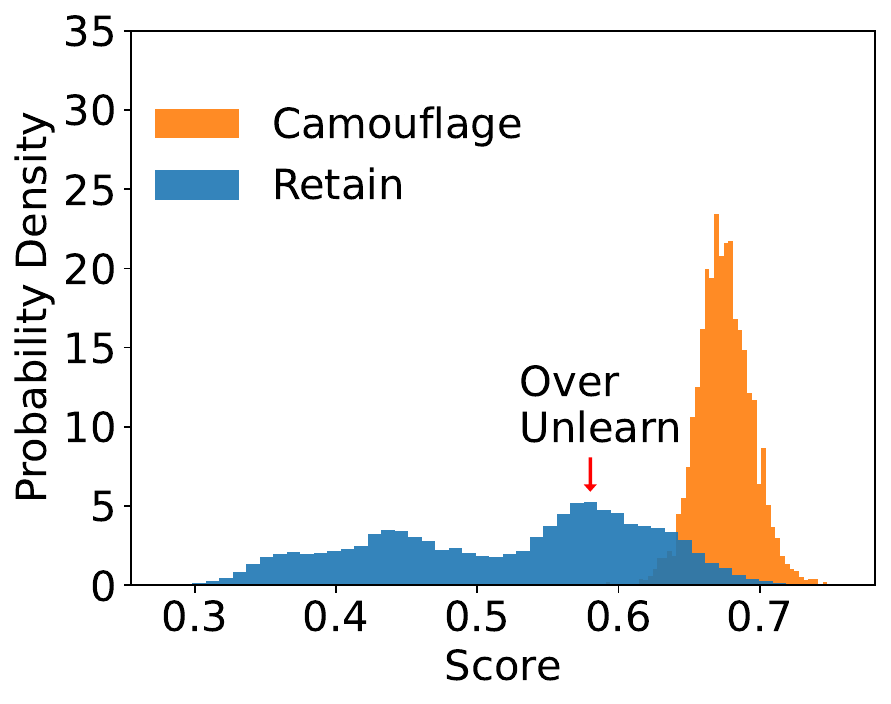}
         \caption{Case1}
     \end{subfigure}
    \begin{subfigure}[b]{0.48\linewidth}
         \centering
         \includegraphics[width=\linewidth]{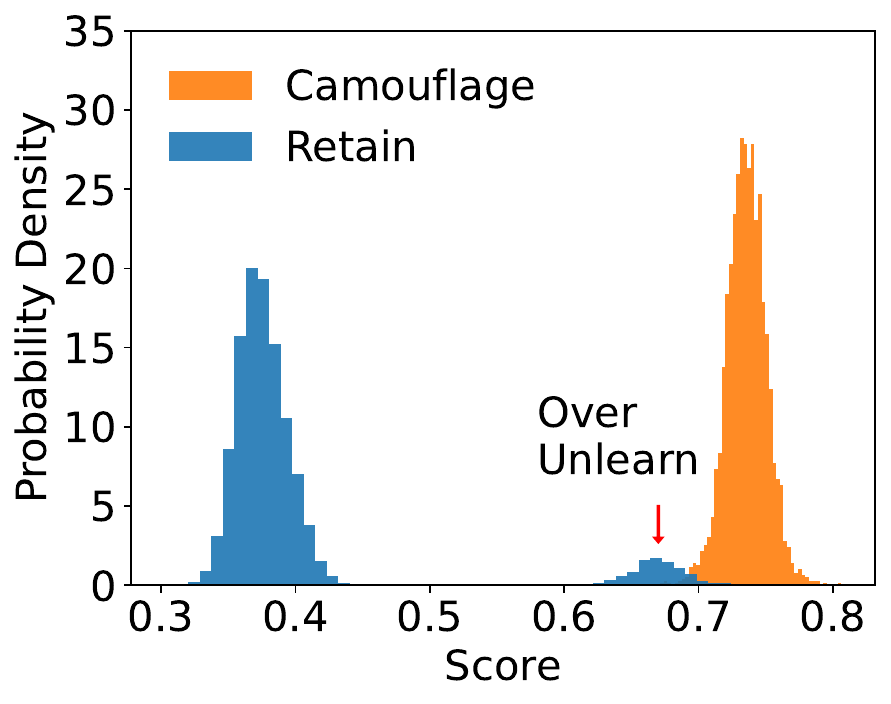}
         \caption{Case2}
     \end{subfigure}
     \caption{UnleScores for samples within Camouflage 'Unlearning' Cases within CIFAR10. Refer to Appendix~\ref{sec:subsec:over_add} for results on other datasets.}
        \label{fig:adv_cifar10}
\end{figure}
\subsection{Computation Efficiency}

We measured the total computation time of UnleScore for various unlearning measurement tasks across five datasets, contrasting it with the computation times of baseline metrics. The results are summarized in Table~\ref{query_time}. Notably, UnleScore achieves $~10\times $ greater efficiency compared to other metrics. This efficiency arises from UnleScore's capability to operate without the need for training any shadow models, making it scalable for larger datasets and models.

\begin{table}[h]
\centering
\caption{Computation Time of Unlearning Metrics (s)}
\label{query_time}
\begin{tabularx}{\linewidth}{>
{\hsize=0.95\hsize}X  >
{\hsize=0.7\hsize\centering\arraybackslash}X  >{\hsize=0.7\hsize\centering\arraybackslash}X  >{\hsize=0.65\hsize\centering\arraybackslash}X  >{\hsize=0.65\hsize\centering\arraybackslash}X  >{\hsize=0.6\hsize\centering\arraybackslash}X
}
\toprule\toprule
Dataset          & CIFAR10 & CIFAR100 & Location & Purchase & Texas   \\ 
\#Query Set  & 30000   & 30000    & 3008     & 117859   & 33864   \\ 
\midrule\midrule
UnLeak           & 67.07   & 306.42   & 14.19    & 219.81   & 162.68  \\
UnLeak+LS        & 67.35   & 291.95   & 14.14    & 203.96   & 159.50  \\
LiRA             & 262.18  & 262.63   & 2.41     & 42.72    & 18.73   \\
UpdateRatio      & 264.66  & 266.17   & 2.37     & 43.15    & 18.39   \\
UpdateDiff       & 264.66  & 266.17   & 2.37     & 43.15    & 18.39   \\
\midrule
UnleScore        & \textbf{11.09}   & \textbf{11.26}    & \textbf{0.87}     & \textbf{3.38}     & \textbf{2.08}    \\
\bottomrule\bottomrule
\end{tabularx}
\end{table}

\begin{table*}[ht]
\caption{Unlearning Results (NMI\_TPR@FPR=0.01‰) of 7 Approximate Unlearning Baselines}
\begin{tabularx}{\textwidth}{ >{\hsize=0.1\hsize}X  >{\hsize=0.8\hsize}X  >{\hsize=0.85\hsize\centering\arraybackslash}X  >{\hsize=0.85\hsize\centering\arraybackslash}X  >{\hsize=0.85\hsize\centering\arraybackslash}X  >{\hsize=0.85\hsize\centering\arraybackslash}X  >{\hsize=0.85\hsize\centering\arraybackslash}X  >{\hsize=0.85\hsize\centering\arraybackslash}X  >{\hsize=0.85\hsize\centering\arraybackslash}X  >{\hsize=0.85\hsize\centering\arraybackslash}X }
\toprule
\toprule
& Dataset & \textcolor{Blue}{Retrain} & \textcolor{BurntOrange}{FT} & Ascent & Forsaken & Fisher & L-Codec & Boundary & SSD \\
\midrule
\midrule
{\multirow{5}{*}{\rotatebox[origin=c]{90}{Random}}} &CIFAR10 & \textcolor{Blue}{13.73$\pm$0.54} & \textcolor{BurntOrange}{0.30$\pm$0.12} & \textcolor{Green}{0.32$\pm$0.19} & 0.19$\pm$0.15 & 0.09$\pm$0.05 & 0.05$\pm$0.04 & - & \textcolor{Green}{0.37$\pm$0.21} \\
&CIFAR100 & \textcolor{Blue}{59.18$\pm$0.79}   & \textcolor{BurntOrange}{0.52$\pm$0.25} & 0.18$\pm$0.11 & 0.20$\pm$0.11 & 0.20$\pm$0.13 & 0.10$\pm$0.07 & - & 0.38$\pm$0.09 \\
&Purchase & \textcolor{Blue}{11.34$\pm$0.27}   & \textcolor{BurntOrange}{4.04$\pm$0.21} & 0.20$\pm$0.02 & 0.18$\pm$0.03 & 0.11$\pm$0.04 & 0.11$\pm$0.04 & - & 0.27$\pm$0.03 \\
&Texas & \textcolor{Blue}{13.92$\pm$0.38}   & \textcolor{BurntOrange}{8.22$\pm$0.20} & 0.23$\pm$0.02 & 0.10$\pm$0.02 & 0.15$\pm$0.02 & 0.19$\pm$0.03 & - & 4.85$\pm$0.14 \\
&Location & \textcolor{Blue}{56.49$\pm$0.36}   & \textcolor{BurntOrange}{15.41$\pm$0.83} & 0.02$\pm$0.01 & 0.01$\pm$0.01 & 0.08$\pm$0.02 & 0.20$\pm$0.05 & - & 0.42$\pm$0.06 \\
\midrule
\midrule
{\multirow{5}{*}{\rotatebox[origin=c]{90}{Partial}}} &CIFAR10 & \textcolor{Blue}{21.06$\pm$0.67}   & \textcolor{BurntOrange}{0.24$\pm$0.06} & \textcolor{Green}{1.32$\pm$0.22} & \textcolor{Green}{1.33$\pm$0.24} & 0.08$\pm$0.03 & \textcolor{Green}{1.23$\pm$0.27} & \textcolor{Green}{1.35$\pm$0.21} & 0.17$\pm$0.05 \\
&CIFAR100 & \textcolor{Blue}{76.34$\pm$1.14}   & \textcolor{BurntOrange}{6.44$\pm$1.11} & \textcolor{Green}{55.36$\pm$2.69} & \textcolor{Green}{55.23$\pm$3.29} & 0.58$\pm$0.30 & \textcolor{Green}{23.54$\pm$1.57} & \textcolor{Green}{49.51$\pm$2.63} & 3.17$\pm$0.56 \\
&Purchase & \textcolor{Blue}{29.04$\pm$0.81}   & \textcolor{BurntOrange}{19.36$\pm$1.06} & 0.21$\pm$0.02 & 0.11$\pm$0.01 & 0.03$\pm$0.00 & 0.00$\pm$0.00 & {40.35$\pm$2.49} & 12.29$\pm$0.93 \\
&Texas & \textcolor{Blue}{41.63$\pm$1.08}   & \textcolor{BurntOrange}{47.41$\pm$1.29} & {57.18$\pm$1.88} & 1.74$\pm$0.21 & 0.07$\pm$0.02 & - & {73.44$\pm$2.23} & {60.17$\pm$2.28} \\
&Location & \textcolor{Blue}{75.49$\pm$0.52}   & \textcolor{BurntOrange}{60.95$\pm$0.95} & 35.44$\pm$1.03 & 2.47$\pm$0.28 & 3.12$\pm$0.31 & 1.75$\pm$0.21 & \textcolor{Green}{70.61$\pm$3.00} & 42.13$\pm$2.00 \\
\midrule
\midrule
{\multirow{5}{*}{\rotatebox[origin=c]{90}{Total}}} &CIFAR10 & \textcolor{Blue}{99.98$\pm$0.01}   & \textcolor{BurntOrange}{21.12$\pm$1.86} & \textcolor{Green}{75.45$\pm$0.88} & 9.79$\pm$0.80 & 0.04$\pm$0.02 & \textcolor{Green}{46.02$\pm$1.62} & \textcolor{Green}{44.37$\pm$2.43} & \textcolor{Green}{25.35$\pm$1.52} \\
&CIFAR100 & \textcolor{Blue}{100.00$\pm$0.00}   &\textcolor{BurntOrange}{71.39$\pm$1.81} & 69.19$\pm$1.49 & 58.62$\pm$1.89 & 0.26$\pm$0.12 & 15.28$\pm$0.27 & 65.10$\pm$1.75 & 2.06$\pm$0.41 \\
&Purchase & \textcolor{Blue}{100.00$\pm$0.00}   &\textcolor{BurntOrange}{100.00$\pm$0.00} & 0.10$\pm$0.01 & 0.02$\pm$0.00 & 0.02$\pm$0.00 & 0.01$\pm$0.00 & 99.96$\pm$0.00 & 17.04$\pm$0.52 \\
&Texas & \textcolor{Blue}{91.64$\pm$0.31}   & \textcolor{BurntOrange}{93.09$\pm$0.15} & 56.70$\pm$0.97 & 0.61$\pm$0.09 & 29.16$\pm$0.03 & - & 87.21$\pm$0.14 & 59.50$\pm$1.27 \\
&Location & \textcolor{Blue}{100.00$\pm$0.00}   &\textcolor{BurntOrange}{83.84$\pm$0.36} & 41.73$\pm$0.46 & 0.93$\pm$0.11 & 11.37$\pm$0.21 & 2.67$\pm$0.14 & \textcolor{Green}{100.00$\pm$0.00} & 59.56$\pm$1.36 \\
\bottomrule
\bottomrule
\end{tabularx}
\label{unlearning_utility_tprs}
     {\raggedright *We use \textcolor{Blue}{Retrain} results as the ground truth and \textcolor{BurntOrange}{FT (Fine Tuning)}  as the lower threshold. An \textcolor{Green}{acceptable approximate unlearning algorithm} should yield results that, at least, fall within the middle area between these two benchmarks.\par}
\end{table*}

\section{Benchmarking Approximate Unlearning Algorithms}\label{audit_unle}

In this section, we utilize the proposed unlearning metrics to benchmark the performance of existing approximate unlearning baselines. Our evaluation includes the implementation and detailed examination of seven approximate unlearning methods, focusing on their unlearning utility, resilience, and equity. We begin by quantifying their unlearning utility, measuring the UnleScore for unlearned samples and retained samples, and calculating the statistical \textit{NMI\_TPR@LowFPR} results. We then analyze the resilience of these methods in a continuous unlearning setting, observing how their unlearning utility varies in response to variations in the times of unlearning requests. Furthermore, we explore unlearning equity by investigating the variance in the difficulty of unlearning across different samples or classes, revealing that some classes might be unlearned more easily than others. Like the exact retraining method, we apply three types of unlearning requests to these baselines.

\subsection{Unlearning Utility}

Following the implementation of metric validation in the last section, we measure the approximate unlearning results of various baseline algorithms. Table~\ref{unlearning_utility_tprs}  summarizes their \textit{NMI\_TPR@FPR=0.01‰} results across different unlearning tasks with exact retraining serving as the ground truth. The corresponding AUC scores are available in Table~\ref{unlearning_utility_auc}, located in Appendix~\ref{addition_utility}. We categorize acceptable approximate unlearning algorithms (those falling between Retrain and Fine Tuning, considering their approximate properties) as green.

The effectiveness of these algorithms varies significantly across datasets and scenarios, with no baseline method consistently achieving acceptable results. Specifically, in the random sample unlearning scenario, nearly all approximate algorithms underperform compared to Finetune, let alone Retrain, emphasizing the challenges of unlearning within groups that have overlap distributions with the retained set. In the partial class unlearning scenario, algorithms such as Ascent, Forsaken, L-Codec, and Boundary occasionally achieve acceptable levels. Notably, the use of the Boundary algorithm in this context indicates a case of over-unlearning because it was initially designed for total class scenarios. In total class unlearning, almost all algorithms show improved performance; except Forsaken and Fisher, all achieve acceptable results at least once. The Fisher algorithm often fails to achieve high unlearning performance, potentially due to the implementation aimed at reducing its significant time complexity but at the cost of performance.

\subsection{Unlearning Resilience}

\begin{figure}[ht]
    \centering
    \centering
     \begin{subfigure}[b]{0.47\textwidth}
         \centering
         \includegraphics[width=\linewidth]{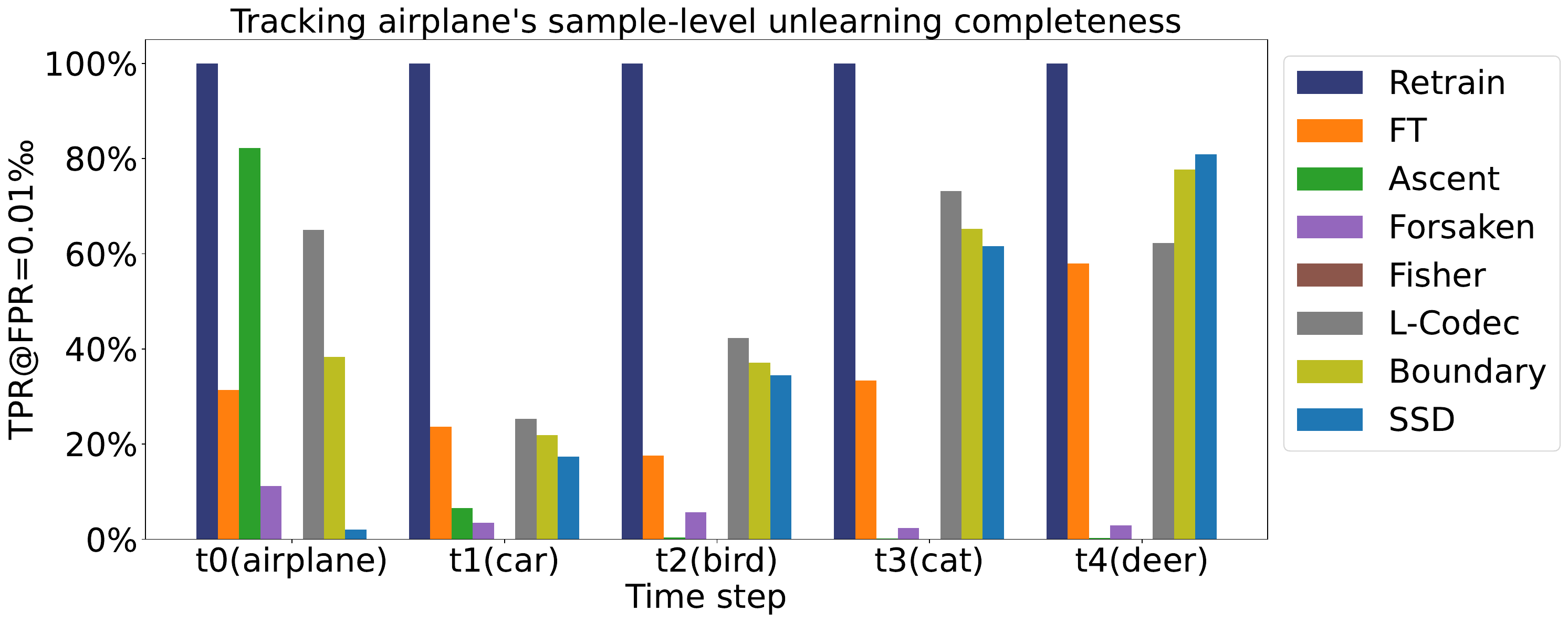}
         \caption{CIFAR10}
         \label{fig:y equals x}
     \end{subfigure}
     \hfill
     \begin{subfigure}[b]{0.47\textwidth}
         \centering
         \includegraphics[width=\linewidth]{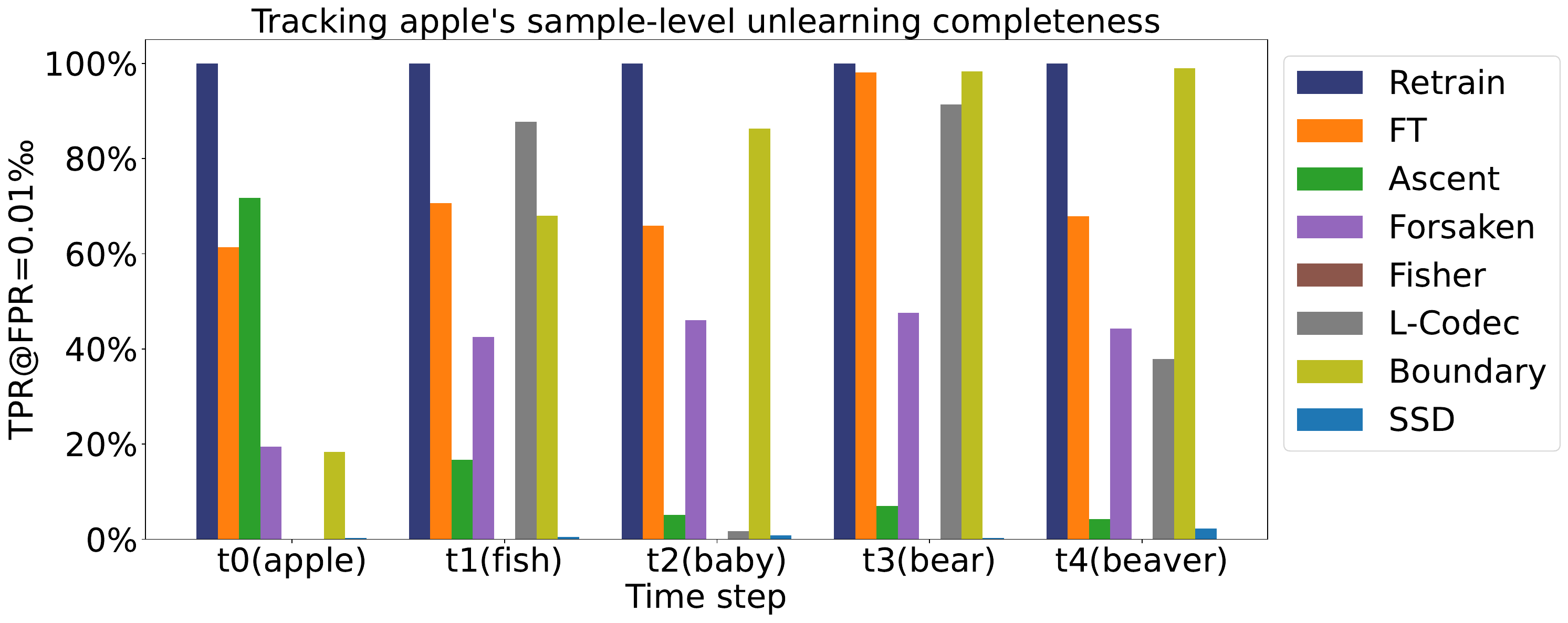}
         \caption{CIFAR100}
         \label{fig:three sin x}
     \end{subfigure}
    \caption{Statistical Measurement Outcomes (\textit{NMI\_TPR@FPR=0.01‰}) of Approximate Baselines for All Samples in the 0-th Class throughout the Unlearning Lifecycle. The unlearned class is indicated at the corresponding time step.}
    \label{continue_tpr_cifar}
\end{figure}

In this section, we examine a scenario of continual unlearning,  where unlearning requests arrive sequentially, with each request targeting a different batch. To simulate this, we randomly select five distinct subsets from each dataset’s training set. For random sample unlearning, we randomly draw five non-overlapping groups from the training set, each containing 100 samples. For partial and total class unlearning, we randomly select five classes to form the unlearning queue. The chosen algorithms are tasked with sequentially unlearning these five subsets on an original model, ensuring that by the end of the process, each subset has been addressed. To assess the impact of successive unlearning actions on previously unlearned groups, we monitor and report the unlearning performance of all samples in the first group after each iteration. This methodology allows us to observe the cumulative effects of unlearning on the model’s behavior over time and evaluate the resilience of different unlearning methods. Figure~\ref{continue_tpr_cifar} presents the total class unlearning results for the CIFAR10 and CIFAR100 image datasets, Additional results are available in Appendix~\ref{addition_resil}.

The CIFAR10 dataset's unlearning sequence starts with the 'airplane' class, followed by the 'automobile' class, denoted as 'car' in Figure~\ref{continue_tpr_cifar}. Initially, when only the 'airplane' class is removed, algorithms such as FT, Ascent, L-Codec, and Boundary maintain a TPR@FPR=0.01‰ above 30\%. Their performance on 'airplane' significantly shifts after the unlearning of the 'car' class, which shares similar textures with 'airplane'; a marked decrease in the unlearning scores for 'airplane' is observed. This may be due to the removal of the 'car' class impairing the model's residual capability to distinguish 'car' from 'airplane'. Then, unlearning classes less similar to 'airplane', such as 'bird', 'cat', and 'deer', leads to a progressive improvement in the unlearning score for 'airplane'. This may be attributed to the reduction in the number of retained classes, which degrades the model's overall capability and further exacerbates the forgetting of the 'airplane' class. For CIFAR100, algorithms such as FT, Forsaken, and Boundary exhibit similar trends of improved unlearning results with each update. However, it is notable that Ascent’s initially perfect unlearning performance (for both CIFAR10 and CIFAR100) sharply declines after only the second unlearning request. This shows that its unlearning resilience is indeed problematic.

\smallskip
\noindent {\bf \textit{Unlearning resilience risk vs. privacy onion effect.}} This phenomenon, where the impact of unlearning specific classes on the model is related to their correlation with previously unlearned classes, is similar to the 'privacy onion effect'~\cite{CarliniJZPTT22}—removing the "layer" of outlier points that are most vulnerable to a privacy attack exposes a new layer of previously vulnerable points. However, the 'privacy onion effect' in unlearning only occurs in approximate methods. For exact unlearning via retraining, unlearning-related classes do not affect their unlearning utility compared to previous classes. It is able to maintain near-perfect TPR scores consistently throughout the unlearning lifecycle. The low unlearning resilience exposes a significant drawback of approximate unlearning algorithms: \textit{\textbf{previously unlearned samples may be inadvertently reactivated by subsequent model updates!}} This emphasizes the urgency of managing unlearning commitments made by approximate unlearning algorithms throughout the unlearning lifecycle.

\subsection{Unlearning Equity}

\begin{figure}[h]
    \centering
    \includegraphics[width=0.95\linewidth]{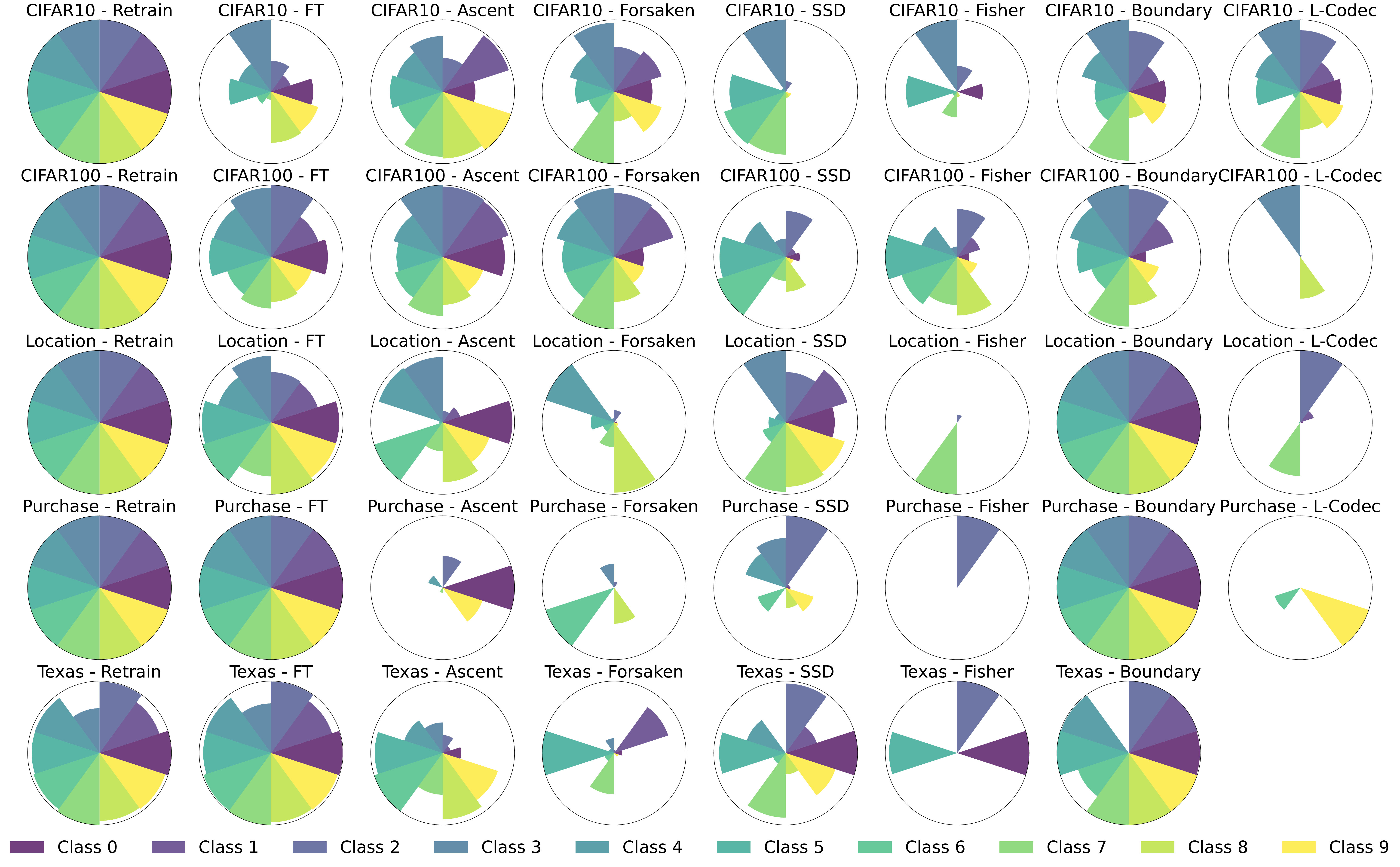}
    \caption{Total Class Unlearning Performance (NMI\_TPR@FPR=0.01‰) for Top 10 Classes Across 5 Datasets. The results are presented as relative values, normalized against the best-performing class among the 10 classes.}
    \label{fairtpr_totalclass}
\end{figure}

We continue our analysis to assess the variations in unlearning equity across different unlearning groups for each unlearning baseline. We perform total class unlearning individually for the top 10 classes of each dataset, then compare the unlearning performance of the approximate unlearning algorithms across these different classes. The results are summarized in Figure~\ref{fairtpr_totalclass}. Several algorithms, including Ascent, Forsaken, SSD, Fisher, and L-Codec, exhibit problematic issues with lower unlearning equity, failing to provide consistent unlearning utility across different request groups. In contrast, FT and Boundary perform better, although they still exhibit significant biases across classes in certain datasets, such as with CIFAR10. This could be due to the limited number of classes in CIFAR10 (only 10), which tends to highlight class biases more clearly. 

The difficulty of unlearning can vary significantly between classes within the same dataset. For example, in CIFAR10, class 3 achieves better unlearning utility with almost all algorithms compared to other classes. This suggests that additional efforts are necessary for groups that are harder to unlearn. Furthermore, LUCM is essential for identifying such groups throughout the unlearning lifecycle when using approximate unlearning algorithms.

The results of 50\% partial class unlearning are provided in Appendix~\ref{addition_equity}, supporting the same conclusion. Additional AUC scores are also provided, but AUC again fails to identify such unlearning inequity.

\section{Discussions}

\noindent \textbf{Practical impacts of UnleScore.} Unlike MIAs for binary decisions, UnleScore is more suitable for providing quantitative measurements for 'approximate unlearned samples'. Not only does it serve as a monitoring metric in LUCM, but it can also be used to provide more fine-grained analysis for research related to approximate unlearning, such as model editing. The efficiency of UnleScore makes it lightweight enough to be used frequently to analyze all samples simultaneously. This analysis helps to understand how data residual memorization changes over time. By detailing the nuances of unlearning performance, UnleScore helps enhance the security and effectiveness of machine unlearning systems.

\noindent \textbf{Approximate unlearning risks.} In this paper, we design tests for continual unlearning and multi-class unlearning, and identify two risks (unlearning resilience and equity) associated with approximate unlearning algorithms in maintaining lifecycle unlearning commitments. Consequently, we believe that a one-time UnleScore does not conclusively verify exact unlearned samples in an approximate unlearning context; it is only appropriate for revealing the residual memorization of approximate unlearned samples. We expect that the continual unlearning and multi-class tests could become the benchmark for evaluating approximate unlearning. We do not claim that this work covers all potential risks of approximate unlearning. However, these are emergencies because they exist even without external threats.

\noindent \textbf{Limitations} 1) \textit{Utility}. As we have also shown in experiments, UnleScore still cannot provide nearly perfect measurement performance for random sample unlearning and small-ratio partial class unlearning, although it remains superior to baselines. The overlapping distributions between retained data and unlearned data can significantly affect the metric’s ability to accurately measure unlearning completeness. More powerful metrics are still needed. We hope our research serves as the first step in this area and attracts community attention to this problem. 2) \textit{Privacy Risks.} UnleScore requires knowledge of the membership of the original model, as it is designed to serve as an integral part of the machine learning system. However, this may pose increased privacy leakage risks if the measurement results are obtained by an external party. Therefore, it is essential to study how to combine it with state-of-the-art privacy protection techniques, such as differential privacy, without impacting the measurement utility of unlearning.

\section{Conclusion}

This work first introduces the Lifecycle Unlearning Commitment Management (LUCM) task for approximate unlearning, identifying its special challenges beyond traditional MIAs. We then design UnleScore to efficiently measure the sample-level unlearning completeness and show how it can be utilized to detect unlearning anomalies during approximate unlearning, including under-unlearning and over-unlearning. We apply it to benchmark existing approximate unlearning algorithms and reveal two risks of approximate unlearning (not present in exact unlearning): the resilience risk and the equity risk. Both risks highlight the importance of LUCM when using approximate unlearning algorithms. As approximate unlearning becomes the de facto choice for post-hoc unlearning solutions of large models, LUCM will grow increasingly important.


\bibliographystyle{plain}
\bibliography{reference}

\newpage
\appendix

\section{Related Work}\label{sec:appendix_related}

\subsection{Dataset Auditing}

Dataset auditing is a process to verify whether a query dataset has been removed from a trained model. Operating in a black-box setting, the auditor has access only to the training algorithm and the model's outputs, not the training dataset or model parameters. Unlike membership inference, which assesses individual sample status, dataset auditing provides a binary decision at the dataset level, indicating whether the entire query dataset was used in training the model.

\noindent {\bf \textit{Calibrating}}: A recent work ~\cite{LiuT20} considers the dataset auditing problem by pointing out that Membership Inference Attacks (MIAs) always return false positives when the query dataset and the training dataset overlap, which frequently occurs in the real world. To overcome the drawback of MIAs in data auditing, a calibrated model is first created on a calibration dataset sampled with a distribution similar to the training dataset but with no overlap with the query dataset. Based on the output distribution of the shadow models trained on the query dataset and calibrated set, it introduces the Kolmogorov-Smirnov (K-S) distance to detect if the target model has used or forgotten the query dataset.

\noindent {\bf \textit{EMA}}: The Ensembled Membership Auditing (EMA) method, proposed by~\cite{HuangLL21}, is a two-stage method. It first conducts MIA for each query sample using various metrics. Subsequently, it selects a filter threshold for sample-wise predictions that maximizes the balanced accuracy and aggregates the results into a binary decision.

These methods typically address dataset membership as a collective issue, thus overlooking the nuanced requirements of auditing at the individual sample level. Such broad approaches contrast with the fine-grained challenges faced by unlearning commitment managers, who must assess each sample individually. This emphasis on sample-specific measurements introduces a more complex task, highlighting the need for tailored strategies that go beyond dataset-wide assessments to effectively mitigate risks associated with sample-level unlearning anomalies.

\subsection{Proof-of-(Un)Learning}

Proof-of-learning (PoL) is a proof mechanism that allows the author party to generate proof supporting its claim on the computational efforts necessary for training. Jia et al.~\cite{JiaYCDTCP21} first proposed a PoL solution by logging training checkpoints to ensure that spoofing is as costly as honestly obtaining the proof through actual model training. However, recent work ~\cite{thudi2022necessity} points out that the PoL framework cannot be used for unlearning auditing, as there are forging attacks that could spoof audits of unlearning in the parameter space by synthesizing the unlearned model to bypass necessary unlearning computations. In response, Weng et al.~\cite{WengYDHWW24} present a trusted hardware-empowered instantiation using an Intel SGX enclave to achieve Proof-of-Unlearning (PoUL) from the perspective of trusted execution environments, verifying the necessary computations during the unlearning process.

\noindent \textbf{Distinguishing LUCM from Po(U)L.}
Note that Po(U)L, which centers on providing proof of computation execution in the (Un)- Learning process, is distinct from LUCM. LUCM aims to provide sample-level unlearning completeness measurements for the output of an honestly unlearning operation within unlearning systems. This is crucial because, for many approximate unlearning algorithms, faithful execution does not naturally result in the complete removal of data lineage, due to the algorithms' prior relaxation of the criteria defining unlearning. Recent works have pointed out the fragility of the relied-upon foundational techniques of these algorithms, specifically in deep learning, from both empirical~\cite{BasuPF21} and theoretical sides~\cite{ChourasiaS23}. Some studies also reveal the external threats originating from the opacity of approximate unlearning~\cite{DiDAKS22, abs-2309-08230, QianZLMH23}. Therefore, it is essential to focus efforts on developing and refining unlearning completeness measurements for effective risk management. Discussions related to attacks on Proof-of-(Un)Learning, such as forging attacks, also fall outside our scope. We have to claim that the fact that retrained parameter states can be constructed non-uniquely, as induced by the forging map, does not undermine their validity as exact unlearning outputs.

\section{Proof of Theorem~\ref{theoremcpua} }\label{theoremcpua_appendix}

Underpinned by Theorem~\ref{theoremcpua}, we can equate $P(B|A)$ with $P(B)$, effectively framing the measuring of approximate unlearning completeness as a non-membership inference challenge, i.e.,  it clarifies that upon confirming $A$, analyzing unlearning completeness shifts seamlessly to evaluating non-membership, streamlining the assessing process.

\begin{theorem}[Conditional Probability of Unlearning Auditing]
\label{theoremcpua}
The conditional probability $P(B|A)$ equals the marginal probability $P(B)$ when event $A$—the inclusion of sample $x$ in the original model $\theta_{\text{ori}}$'s training set—is confirmed ($P(A)=1$), simplifying the assessment of unlearning's efficacy.
\end{theorem}

\begin{proof}

Given events $A$ and $B$ where $A$ represents the inclusion of a target sample $x$ in the training set of an original model $\theta_{\text{ori}}$, and $B$ represents the absence of sample $x$ from the training set after the model has undergone a process of so-called `unlearning' to become $\theta_{\text{unl}}$, the conditional probability $P(B|A)$ is equal to the marginal probability $P(B)$, when we have validated the occurrence of event $A$ ($P(A)=1$). 

We analyze the relationship between $A$ and $B$ under two distinct scenarios.

\noindent
\textbf{Case 1:} The events $A$ and $B$ are statistically independent, which can occur if the server performing the unlearning process has access to sample $x$. Under this condition, the presence or absence of $x$ in the original model $\theta_{\text{ori}}$ has no bearing on its presence in the unlearned model $\theta_{\text{unl}}$. For example, a dishonest server could incrementally include $x$ in the training data for the unlearned model $\theta_{\text{unl}}$ without it having been in the original model. Formally, this independence implies that:

\begin{equation}
    P(B|A) = P(B).
\end{equation}

\noindent
\textbf{Case 2:} The events $A$ and $B$ are not statistically independent, which can occur if the server lacks access to sample $x$ during unlearning. Here, the probability $P(B)$ is conditioned by $P(A)$. For example, $P(B)=1$ when $P(A)=0$.

Applying Bayes' Theorem under the assumption that $P(A)=1$, we have:

\begin{equation}
    \begin{split}
      P(B|A) & = \frac{P(A|B) \cdot P(B)}{P(A)} \\
      & = \frac{P(A|B) \cdot P(B)}{P(A|B) \cdot P(B) + P(A|B^{c}) \cdot P(B^{c})} \\
      & = \frac{1 - P(A|B^{c}) \cdot P(B^{c})}{1} \\
      & = 1 - P(A|B^{c}) \cdot P(B^{c}).
    \end{split}
\end{equation}

Therefore, $P(B|A)=P(B)$ if and only if $P(A|B^{c}) = 1$. It's obvious that when sample $x$ is part of the training set of the unlearned model $\theta_{\text{unl}}$, it must also be part of the training set of original model $\theta_{\text{ori}}$, i.e., $P(A|B^{c}) = 1$. Thus, the equality $P(B|A) = P(B)$ holds when $A$ and $B$ are independent. In other words, leveraging non-membership inference allows us to perform unlearning auditing seamlessly.

\end{proof}

\section{Data and Setup} \label{dataset_appendix}

\subsection{Datasets}

\noindent \textbf{Cifar10}: This dataset contains a diverse set of 60,000 small, 32x32 pixel color images categorized into 10 distinct classes, each represented by 6,000 images. It is organized into 50,000 training images and 10,000 test images. The classes in Cifar10 are exclusive, featuring a range of objects like birds, cats, and trucks, making it ideal for basic tasks.

\noindent \textbf{Cifar100}: Cifar100 is similar to Cifar10 in its structure, consisting of 60,000 32x32 color images. However, it expands the complexity with 100 unique classes, which can be further organized into 20 superclasses. Each image in Cifar100 is associated with two types of labels: a `fine' label identifying its specific class, and a `coarse' label indicating the broader superclass it belongs to. This dataset is suited for more nuanced evaluation.

\noindent \textbf{Purchase100}: It contains 197,324 anonymized data about customer purchases across 100 different product categories. Each record in the dataset represents an individual purchase transaction and includes details such as product category, quantity, and transaction time, which is useful for analyzing consumer behavior patterns.

\noindent \textbf{Texas100}: This dataset comprises 67,330 hospital discharge records from the state of Texas. It includes anonymized patient data such as diagnosis, procedure, length of stay, and other relevant clinical information. The data is grouped into 100 classes, and used for healthcare data analysis and predictive modeling.

\noindent \textbf{Location30}: This dataset includes 5,010 location "check-in" records of different individuals. It is organized into 30 distinct categories, representing different types of geosocial behavior. The 446 binary attributes correspond to various regions or location types, denoting whether or not the individual has visited each area. The primary classification task involves using these 446 binary features to accurately predict an individual's geosocial type.

\subsection{Data Processing} In our experiments, we divide the datasets into three distinct sets: training, test, and shadow. The training set is employed to train the original model, and the test set is used to assess the performance of the trained model. The shadow set is used to develop attack models employed by baseline methods and serves as the non-member set used by our designed method. Depending on the types of unlearning required, the retained set is formed by excluding the requested unlearning samples from the original training set.

For the Cifar10 and Cifar100 datasets, we have randomly chosen 20,000 images from their training datasets to form the shadow set for each. The rest of the images are utilized for training classifiers, while the predefined test images make up the test set. In the case of the Purchase100, Texas100, and Location30 datasets, we randomly select 20\% of the records to the test set for each. Subsequently, we select 40,000, 20,000, and 1,000 records as the shadow set for the Purchase100, Texas100, and Location30 datasets, respectively, leaving the remaining records as the training set for each dataset.

\subsection{Original Models} We employ the Resnet18 model as the original model for learning tasks on the Cifar10 and Cifar100 datasets. For classification tasks involving the Purchase100, Texas100, and Location30 datasets, we have implemented a four-layer fully connected neural network as the original model. This architecture comprises hidden layers with 1024, 512, 256, and 128 neurons, respectively.

\section{Unlearning Baselines}\label{approximate_alg_appendix}

\noindent \textbf{Exact Retraining:} The model is initialized and retrained only on the retained dataset, using the same random seeds and hyperparameters as the original model. We consider the results of this method as the ground truth for exact unlearning.

\noindent \textbf{Fine Tuning:} We fine-tune the originally trained model on the retained set for 5 epochs with a large learning rate. This method is intended to leverage the catastrophic forgetting characteristic of deep learning models~\cite{PNAS17James,GolatkarCVPR20}, wherein directly fine-tuning the model without the requested subset may make the model to forget it. Google has also adopted this approach as the starting point for their unlearning challenge.\footnote{https://unlearning-challenge.github.io} Given its simplicity, we use this method as the lower baseline for unlearning benchmarks.

\noindent \textbf{Gradient Ascent:}  Initially, we train the initial model on the unlearning set to record the accumulated gradients. Subsequently, we update the original trained model by adding the recorded gradients as the inverse of the gradient descent learning process.

\noindent \textbf{Fisher Forgetting:} As per ~\cite{GolatkarCVPR20}, we utilize the Fisher Information Matrix (FIM) of samples related to the retaining set to calculate optimal noise for erasing information of the unlearning samples. Given the huge memory requirement of the original Fisher Forgetting implementation, we employ an elastic weight consolidation technique (EWC) (as suggested by ~\cite{EWCpnas.1611835114}) for a more efficient FIM estimation.

\noindent \textbf{Forsaken:}  We implement the Forsaken~\cite{MaLLLMR23} method by masking the neurons of the original trained model with gradients (called mask gradients) that are trained to eliminate the memorization of the unlearning samples.

\noindent \textbf{L-Codec:} Similar to Fisher Forgetting, L-Codec uses optimization-based updates to achieve approximate unlearning. To make the Hessian computation process scalable with the model's dimensions, ~\cite{cvprMehtaPSR22} leverages a variant of a new conditional independence coefficient to identify a subset of the model parameters that have the most semantic overlap at the individual sample level.

\noindent \textbf{Boundary Unlearning:} Targeting class-level unlearning tasks, this method~\cite{ChenGL0W23} shifts the original trained model's decision boundary to imitate the decision-making behavior of a model retrained from scratch.

\noindent \textbf{SSD:} Selective Synaptic Dampening (SSD)~\cite{abs-2308-07707} is a fast, approximate unlearning method. SSD employs the first-order FIM to assess the importance of parameters associated with the unlearning samples. It then induces forgetting by proportionally dampening these parameters according to their relative importance to the unlearning set in comparison to the broader training dataset.

\section{Additional measurement results of unlearning metrics}\label{addition_metric}

\subsection{Additional Results of Metric Utility}\label{sec:subsec:overallvalidation_add}

\begin{figure}[h]
    \centering
    \includegraphics[width=0.95\linewidth]{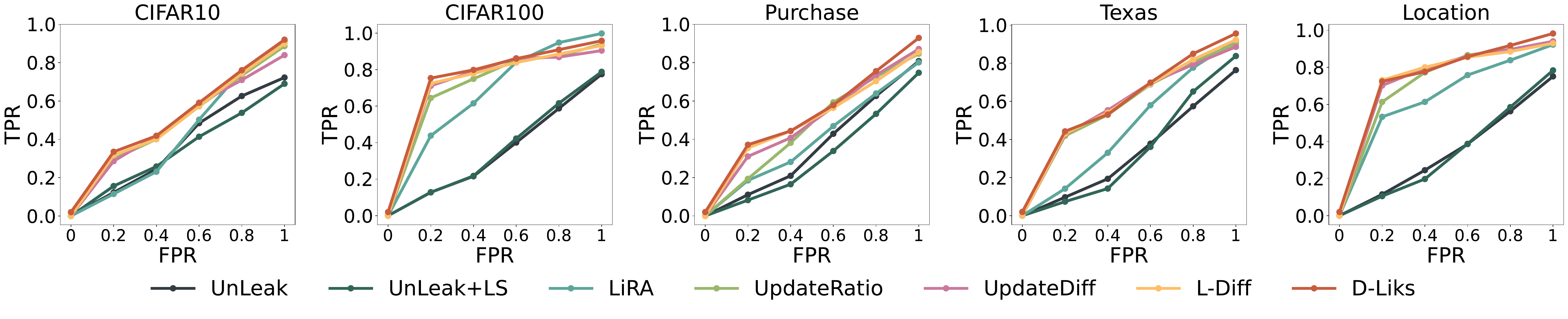}
    \caption{ROC curves on exact random sample unlearning}
    \label{retrain_curves_full}
\end{figure}
\begin{figure}[h]
    \centering
    \includegraphics[width=0.95\linewidth]{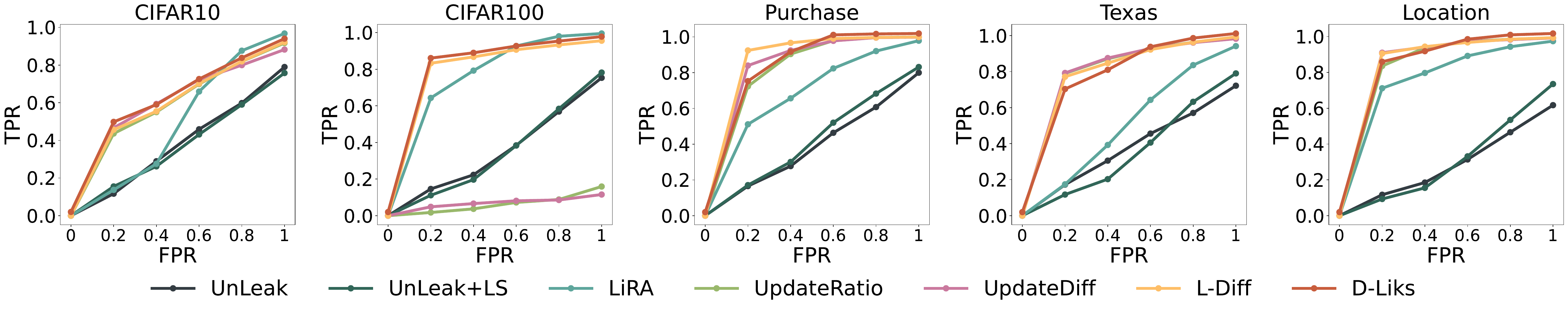}
    \caption{ROC curves on exact partial class unlearning}
    \label{partclass_curves_full}
\end{figure}

\begin{figure}[h]
    \centering
    \includegraphics[width=0.95\linewidth]{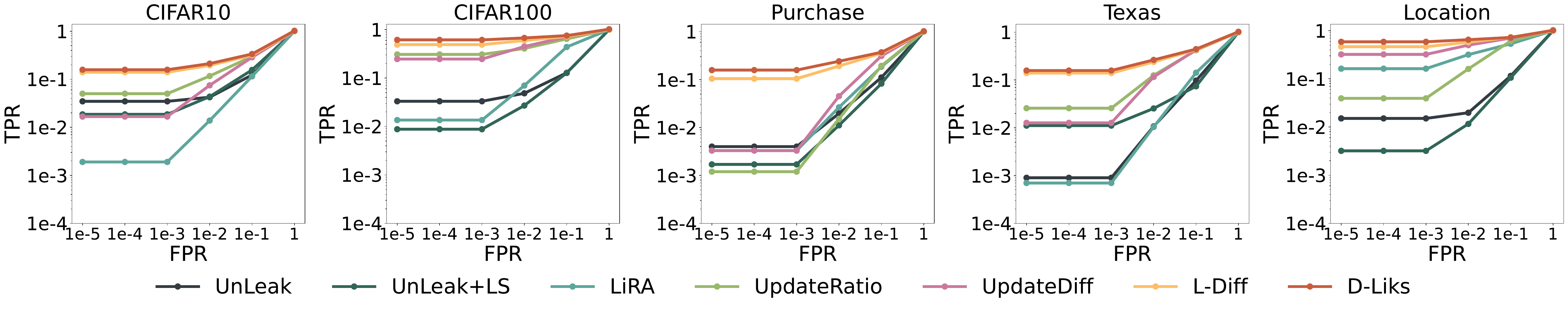}
    \caption{Measuring results on exact random sample unlearning}
    \label{retrain_curves}
\end{figure}

\begin{figure*}[t]
    \centering
    \includegraphics[width=0.95\linewidth]{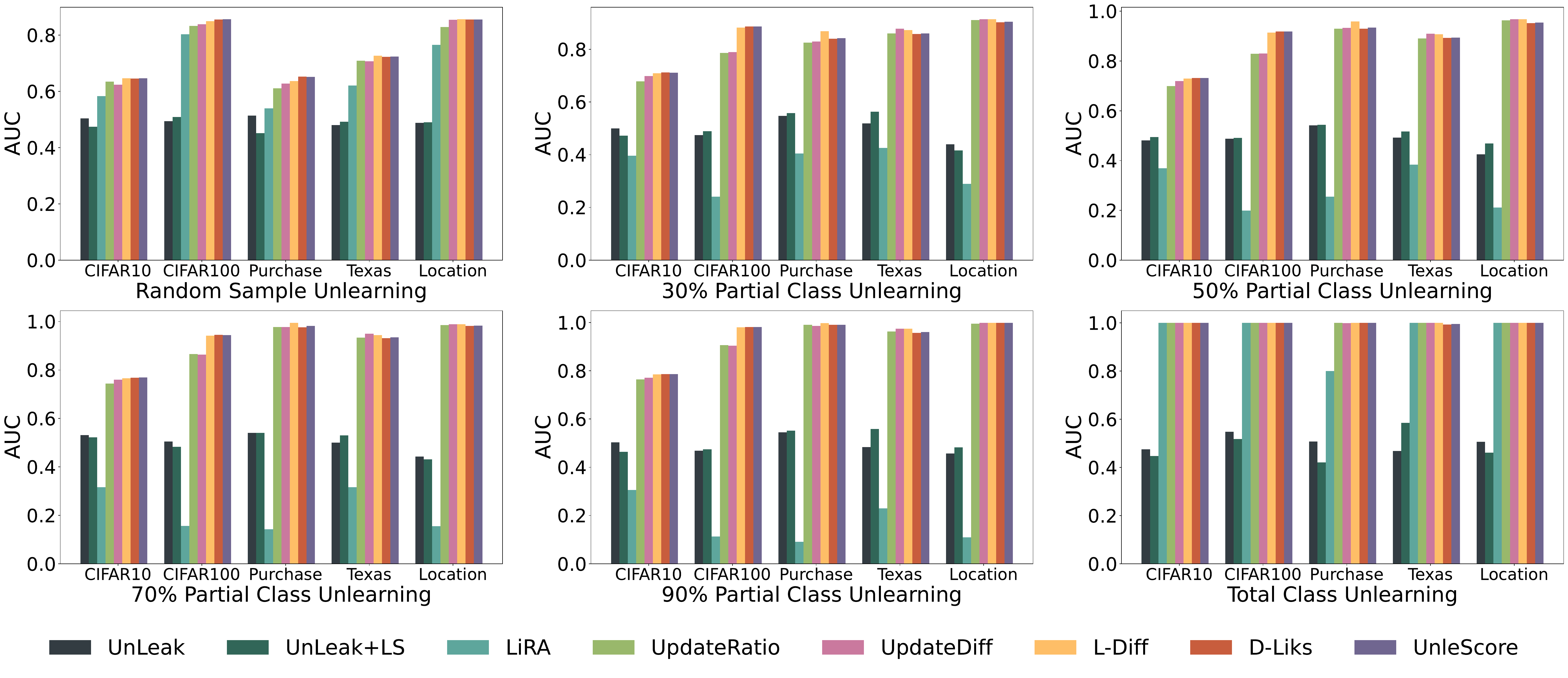}
    \caption{Statistical results (AUC) of metric utility on 3 types of exact unlearning.}
    \label{fig:retrain_results_auc}
\end{figure*}

Figure~\ref{fig:retrain_results_auc} presents the summarized AUC scores for all metrics across three exact unlearning tasks applied to five datasets. Although LiRA, UpdateRatio, and UpdateDiff seem comparable to our designed metrics, they exhibit distinctly different score distributions, as illustrated in Figure~\ref{fig:unlearningscores_cifar10}, and different results at lower FPR, as shown in Figure~\ref{retrain_curves_full} and \ref{partclass_curves_full}. Consequently, we do not regard the AUC as a meaningful statistic for evaluating these measurement results.

For random sample unlearning, we further analyze TPR variations as the FPR shifts from 1e-5 to 1e-1, as shown in Figure~\ref{retrain_curves}. On a logarithmic scale, L-Diff and D-Liks significantly outperform other baselines, particularly at lower FPRs, with minor differences between each other. Besides, the ROC curves for LiRA, UpdateRatio, and UpdateDiff align with those of L-Diff and D-Liks only after the FPR exceeds $0.5$, which is a high threshold, as detailed in Figure~\ref{retrain_curves_full}. Despite their relative superiority, L-Diff and D-Liks achieve a maximum \textit{NMI\_TPR@FPR=0.01‰} of only 10\%, underscoring the substantial challenges in measuring unlearning completeness at the sample level.

\subsection{Under-unlearned Group}\label{sec:subsec:correlation_add}

Figure~\ref{fig:correlation_dist_cifar100}, \ref{fig:correlation_dist_purchase}, \ref{fig:correlation_dist_texas}, and \ref{fig:correlation_dist_location} present the measurement results for LiRA, UpdateDiff, and UnleScore across 6 sample groups (retained, under-unlearned 1-4, exact unlearned), using the CIFAR100, Purchase, Texas, and Location datasets, respectively. For LiRA and UpdateDiff, their measurements on under-unlearned groups with different unlearning levels always overlap. In contrast, UnleScore offers clearer separation compared to LiRA and UpdateDiff.

\begin{figure}[h]
    \centering
    \begin{subfigure}[b]{0.32\linewidth}
         \centering
         \includegraphics[height=0.08\textheight]{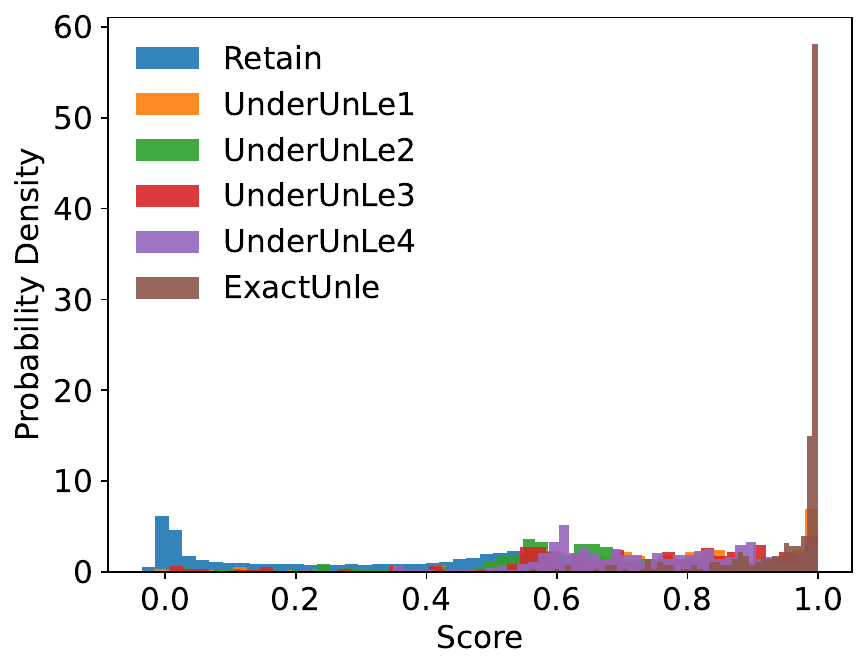}
         \caption{LiRA}
     \end{subfigure}
     \begin{subfigure}[b]{0.32\linewidth}
         \centering
         \includegraphics[height=0.08\textheight]{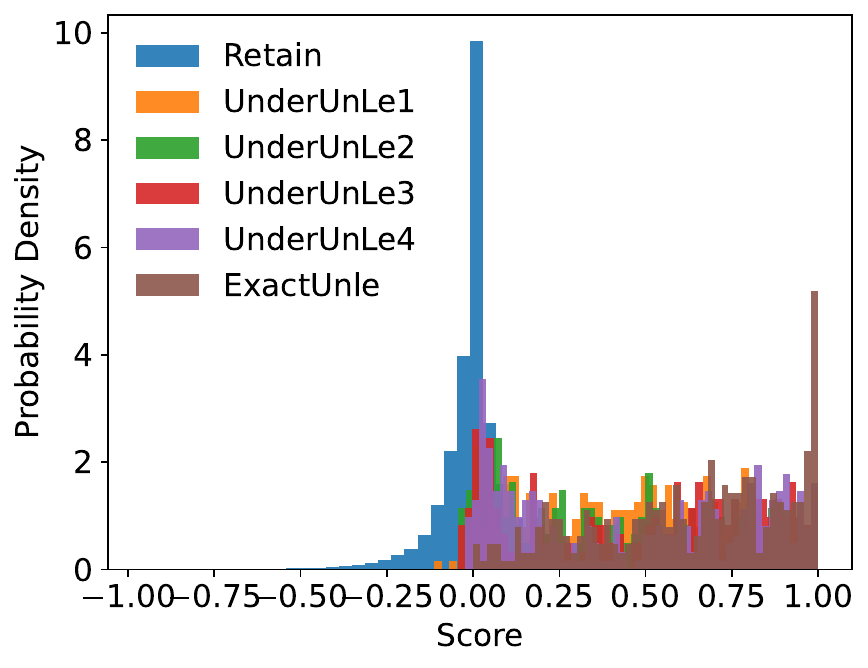}
         \caption{UpdateDiff}
     \end{subfigure}
     \begin{subfigure}[b]{0.32\linewidth}
         \centering
        \includegraphics[height=0.08\textheight]{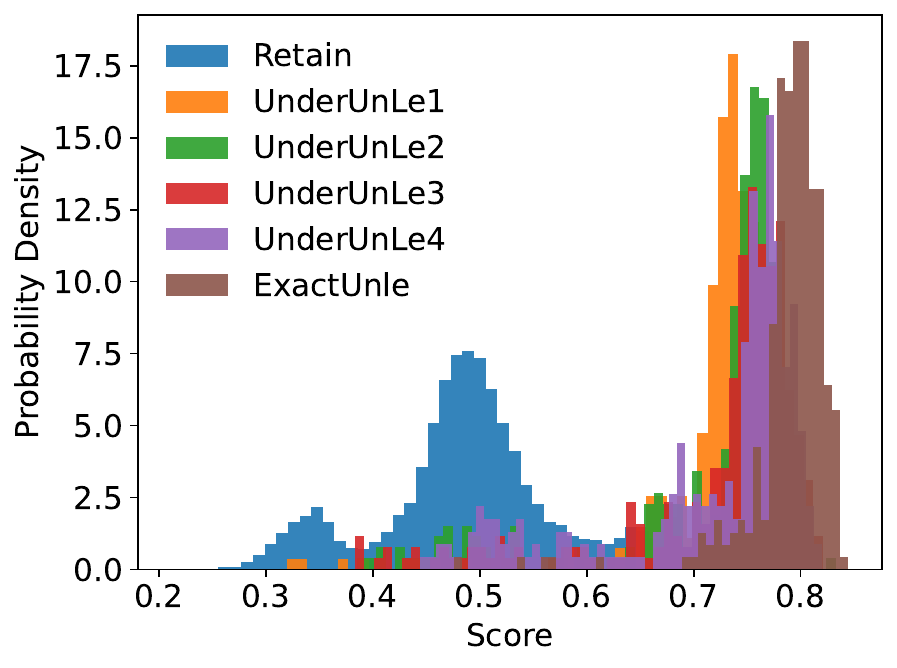}
         \caption{UnleScore}
     \end{subfigure}
     \caption{Scores of unlearning metrics on retained, under-unlearned, and exact unlearned groups within CIFAR100.}
\label{fig:correlation_dist_cifar100}
\end{figure}
\begin{figure}[h]
    \centering
    \begin{subfigure}[b]{0.32\linewidth}
         \centering
         \includegraphics[width=\linewidth]{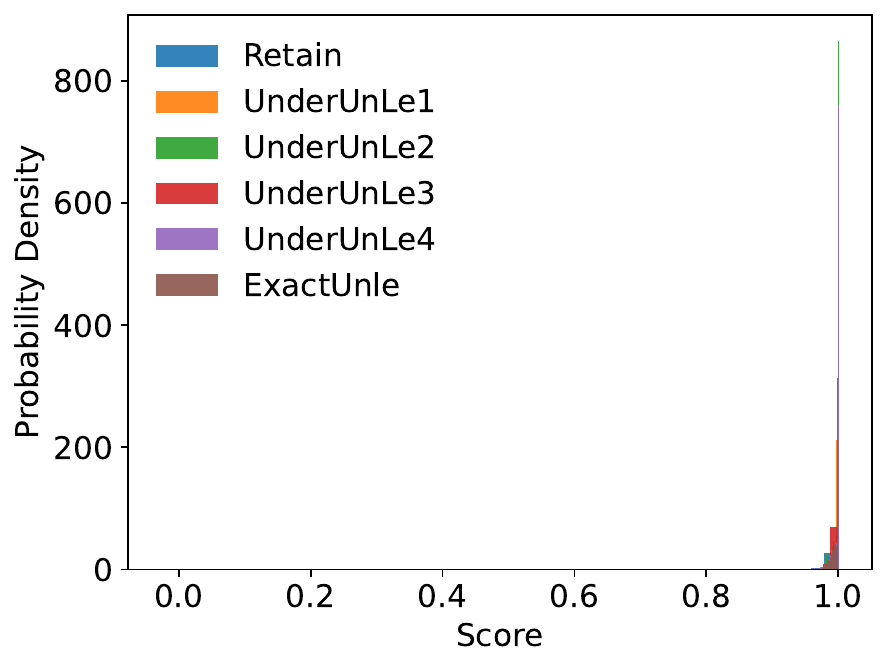}
         \caption{LiRA}
     \end{subfigure}
     \begin{subfigure}[b]{0.32\linewidth}
         \centering
         \includegraphics[width=\linewidth]{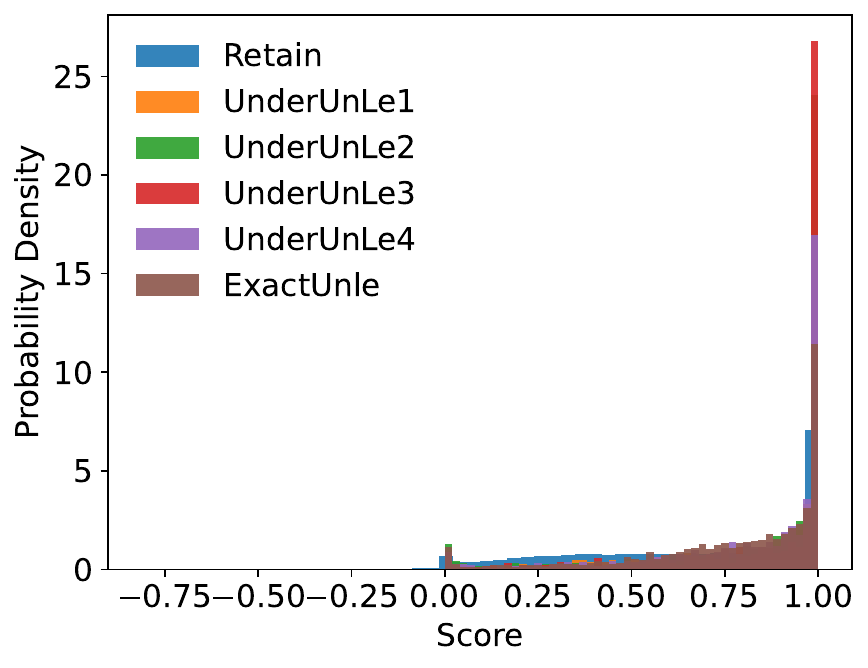}
         \caption{UpdateDiff}
     \end{subfigure}
     \begin{subfigure}[b]{0.32\linewidth}
         \centering
        \includegraphics[width=\linewidth]{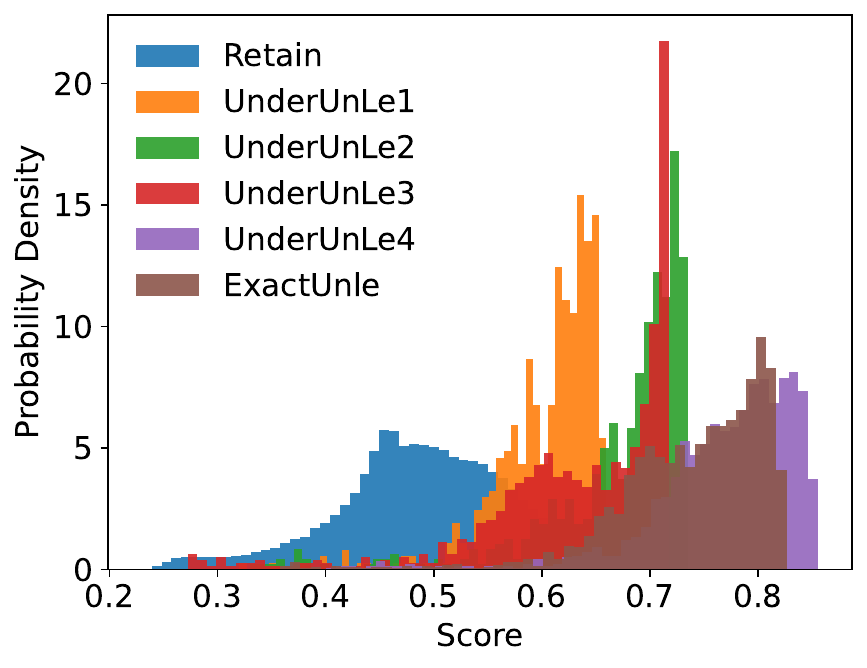}
         \caption{UnleScore}
     \end{subfigure}
     \caption{Scores of unlearning metrics on retained, under-unlearned, and exact unlearned groups within Purchase.}
\label{fig:correlation_dist_purchase}
\end{figure}
\begin{figure}[h]
    \centering
    \begin{subfigure}[b]{0.32\linewidth}
         \centering
         \includegraphics[width=\linewidth]{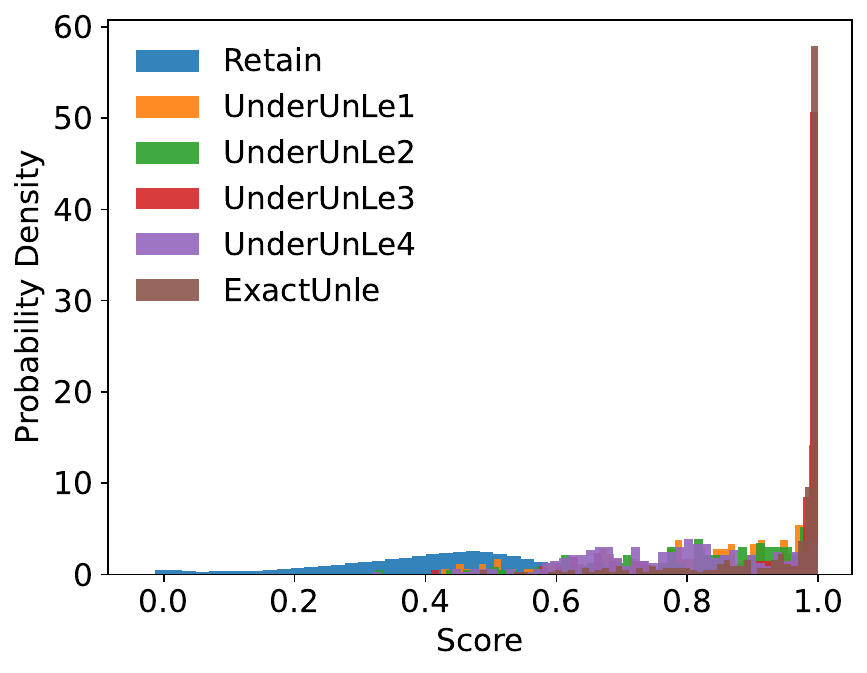}
         \caption{LiRA}
     \end{subfigure}
     \begin{subfigure}[b]{0.32\linewidth}
         \centering
         \includegraphics[width=\linewidth]{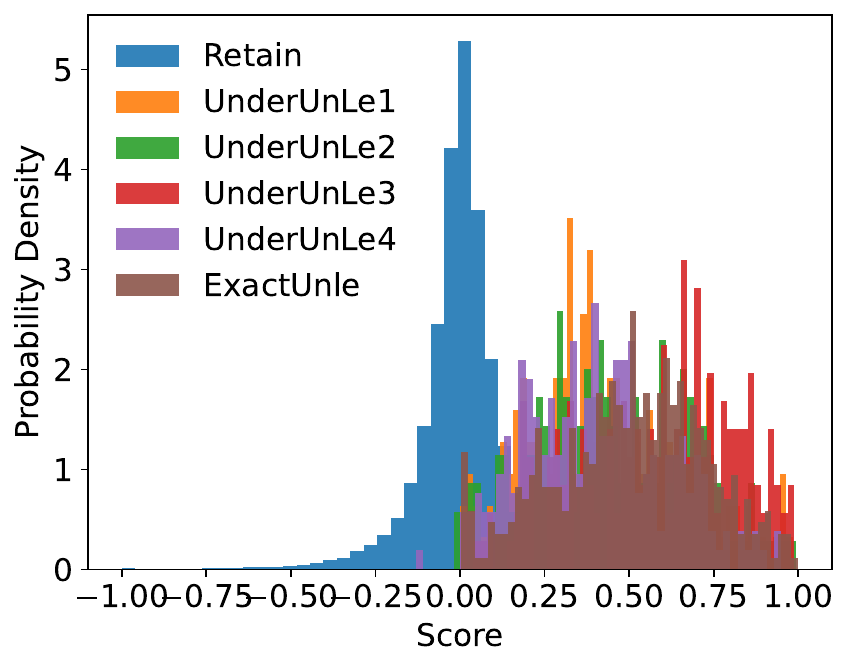}
         \caption{UpdateDiff}
     \end{subfigure}
     \begin{subfigure}[b]{0.32\linewidth}
         \centering
        \includegraphics[width=\linewidth]{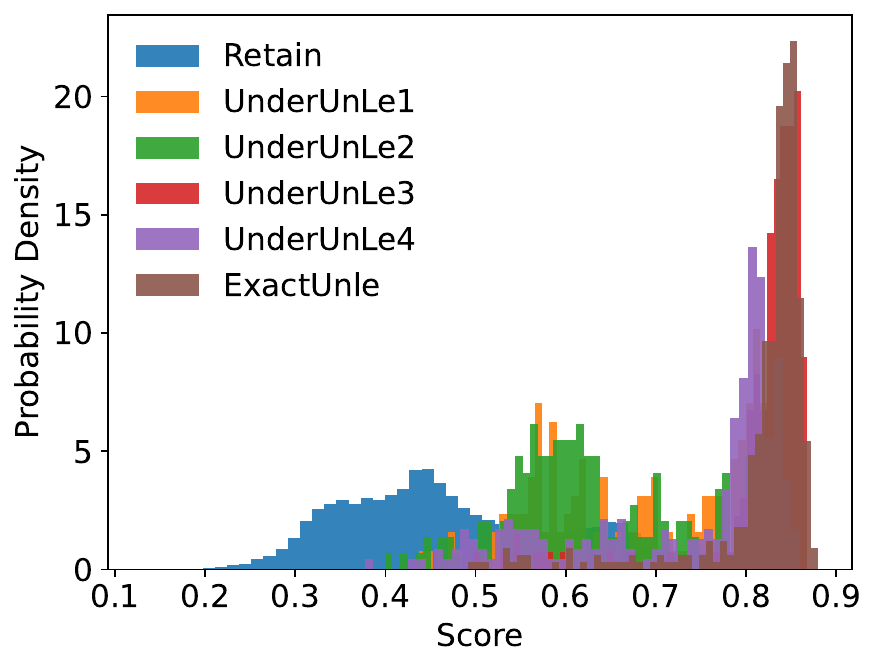}
         \caption{UnleScore}
     \end{subfigure}
     \caption{Scores of unlearning metrics on retained, under-unlearned, and exact unlearned groups within Texas.}
\label{fig:correlation_dist_texas}
\end{figure}
\begin{figure}[h]
    \centering
    \begin{subfigure}[b]{0.32\linewidth}
         \centering
         \includegraphics[width=\linewidth]{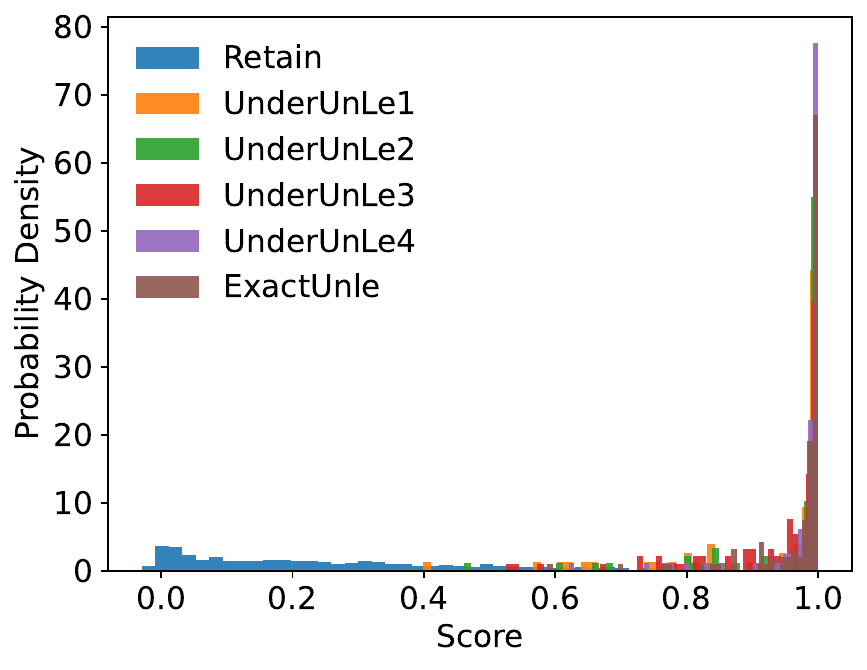}
         \caption{LiRA}
     \end{subfigure}
     \begin{subfigure}[b]{0.32\linewidth}
         \centering
         \includegraphics[width=\linewidth]{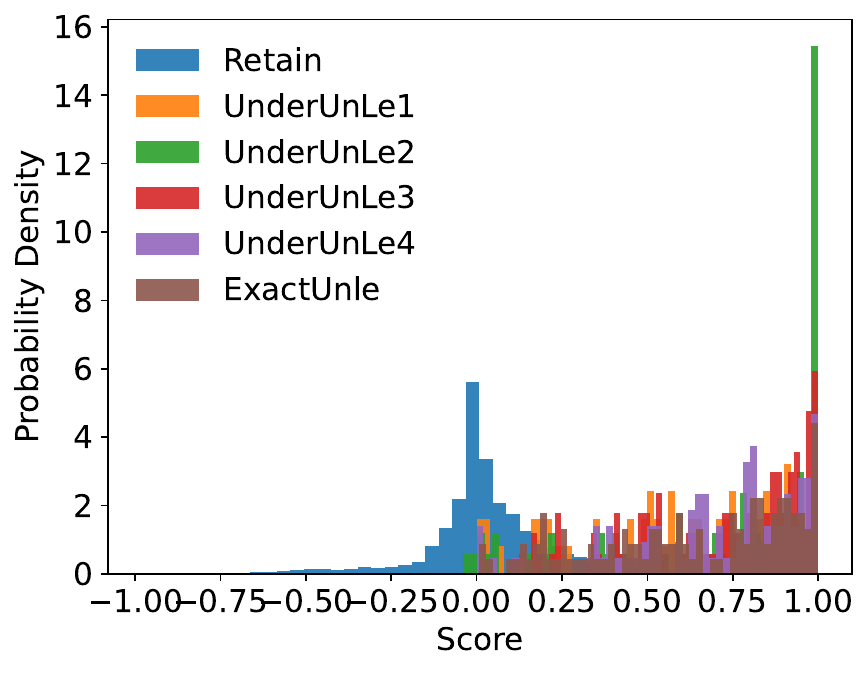}
         \caption{UpdateDiff}
     \end{subfigure}
     \begin{subfigure}[b]{0.32\linewidth}
         \centering
        \includegraphics[width=\linewidth]{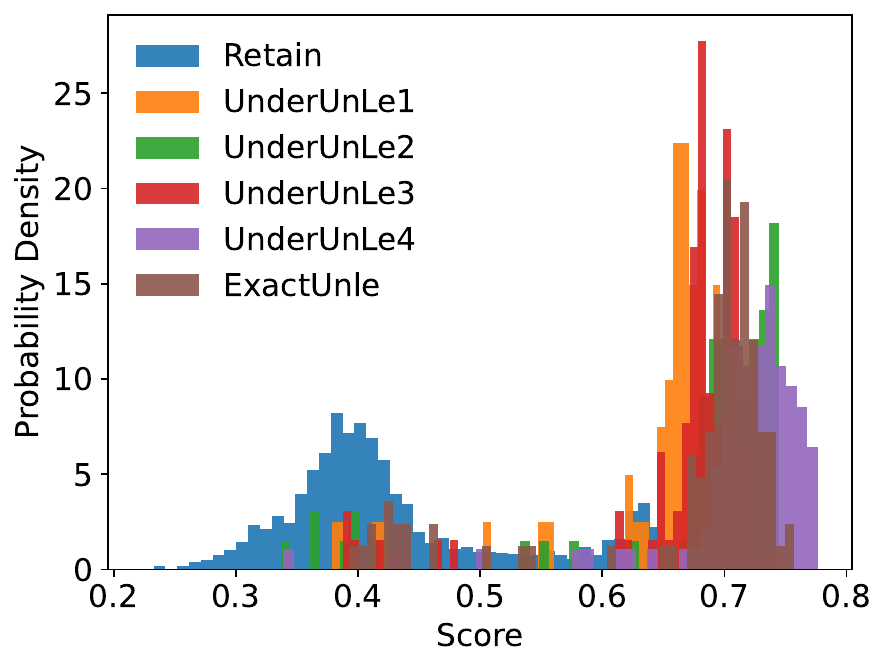}
         \caption{UnleScore}
     \end{subfigure}
     \caption{Scores of unlearning metrics on retained, under-unlearned, and exact unlearned groups within Location.}
\label{fig:correlation_dist_location}
\end{figure}

\subsection{Over-unlearned Group}\label{sec:subsec:over_add}

Figure~\ref{fig:adv_cifar100}, \ref{fig:adv_purchase}, \ref{fig:adv_texas}, and \ref{fig:adv_location} report the measurement results of UnleScore for the camouflage unlearning case across four additional datasets. These results demonstrate smaller differences between case1 and case2 compared to those in CIFAR10, likely due to the significantly smaller size of the class-0 subset in these datasets. CIFAR100, Purchase, and Texas each have 100 classes, while Location has 30 classes, compared to only 10 in CIFAR10. As a result, relabeling class-0 to other classes or to a single class produces a similar distribution, with no distinct peak observed in the scores of retained samples across a broad score range. This allows us to infer the presence of over-unlearning.

\begin{table*}[h]
\caption{Unlearning results (AUC) of 7 approximate unlearning baselines}
\begin{tabularx}{\textwidth}{ >{\hsize=0.1\hsize}X  >{\hsize=0.8\hsize}X  >{\hsize=0.85\hsize\centering\arraybackslash}X  >{\hsize=0.85\hsize\centering\arraybackslash}X  >{\hsize=0.85\hsize\centering\arraybackslash}X  >{\hsize=0.85\hsize\centering\arraybackslash}X  >{\hsize=0.85\hsize\centering\arraybackslash}X  >{\hsize=0.85\hsize\centering\arraybackslash}X  >{\hsize=0.85\hsize\centering\arraybackslash}X  >{\hsize=0.85\hsize\centering\arraybackslash}X }
\toprule
\toprule
& Dataset & Retrain & FT & Ascent & Forsaken & Fisher & L-Codec & Boundary & SSD \\
\midrule
\midrule
{\multirow{5}{*}{\rotatebox[origin=c]{90}{Random}}} &CIFAR10 & 65.48$\pm$0.89 & 53.17$\pm$0.70 & 51.82$\pm$0.84 & 52.98$\pm$0.83 & 50.63$\pm$0.57 & 48.31$\pm$0.49 & - & 53.48$\pm$0.41 \\
&CIFAR100 & 86.31$\pm$0.61 & 57.83$\pm$0.77 & 50.84$\pm$0.59 & 53.24$\pm$0.30 & 50.97$\pm$0.77 & 56.48$\pm$0.20 & - & 55.82$\pm$0.32 \\
&Purchase & 65.19$\pm$0.16 & 64.31$\pm$0.07 & 49.24$\pm$0.13 & 50.33$\pm$0.07 & 49.96$\pm$0.15 & 48.72$\pm$0.08 & - & 47.98$\pm$0.14 \\
&Texas & 72.46$\pm$0.16 & 68.81$\pm$0.23 & 50.08$\pm$0.19 & 48.21$\pm$0.19 & 49.61$\pm$0.17 & 49.74$\pm$0.09 & - & 56.02$\pm$0.21 \\
&Location & 85.47$\pm$0.18 & 79.08$\pm$0.22 & 49.19$\pm$0.13 & 46.41$\pm$0.22 & 51.88$\pm$0.09 & 51.74$\pm$0.16 & - & 50.21$\pm$0.08 \\
\bottomrule
\midrule
{\multirow{5}{*}{\rotatebox[origin=c]{90}{Partial Class}}} &CIFAR10 & 74.02$\pm$0.33 & 36.91$\pm$0.34 & 97.31$\pm$0.05 & 96.54$\pm$0.05 & 49.94$\pm$0.52 & 94.22$\pm$0.06 & 96.45$\pm$0.04 & 77.80$\pm$0.23 \\
&CIFAR100 & 91.33$\pm$0.57 & 77.22$\pm$0.94 & 98.90$\pm$0.11 & 99.29$\pm$0.06 & 49.94$\pm$1.25 & 70.45$\pm$0.62 & 98.91$\pm$0.08 & 89.61$\pm$0.35 \\
&Purchase & 93.81$\pm$0.05 & 95.29$\pm$0.04 & 59.70$\pm$0.12 & 50.94$\pm$0.14 & 49.43$\pm$0.19 & 27.55$\pm$0.11 & 99.75$\pm$0.01 & 98.91$\pm$0.03 \\
&Texas & 89.20$\pm$0.17 & 91.43$\pm$0.18 & 93.53$\pm$0.13 & 50.92$\pm$0.32 & 39.23$\pm$0.27 & - & 96.23$\pm$0.10 & 97.56$\pm$0.09 \\
&Location & 94.96$\pm$0.16 & 95.66$\pm$0.14 & 87.30$\pm$0.22 & 51.70$\pm$0.41 & 53.40$\pm$0.45 & 32.98$\pm$0.45 & 99.23$\pm$0.09 & 97.04$\pm$0.16 \\
\bottomrule
\midrule
{\multirow{5}{*}{\rotatebox[origin=c]{90}{Total Class}}} &CIFAR10 & 100.00$\pm$0.00 & 97.01$\pm$0.07 & 99.94$\pm$0.00 & 98.56$\pm$0.04 & 49.99$\pm$0.25 & 99.31$\pm$0.02 & 99.18$\pm$0.02 & 91.57$\pm$0.07 \\
&CIFAR100 & 100.00$\pm$0.00 & 99.52$\pm$0.04 & 99.20$\pm$0.06 & 99.42$\pm$0.04 & 50.30$\pm$1.11 & 62.34$\pm$0.62 & 99.22$\pm$0.05 & 91.26$\pm$0.26 \\
&Purchase & 100.00$\pm$0.00 & 100.00$\pm$0.00 & 57.48$\pm$0.10 & 52.60$\pm$0.09 & 49.11$\pm$0.10 & 56.23$\pm$0.11 & 100.00$\pm$0.00 & 99.14$\pm$0.02 \\
&Texas & 99.52$\pm$0.01 & 98.58$\pm$0.04 & 93.72$\pm$0.09 & 54.22$\pm$0.27 & 52.74$\pm$0.18 & - & 96.27$\pm$0.03 & 97.97$\pm$0.06 \\
&Location & 100.00$\pm$0.00 & 99.20$\pm$0.04 & 88.23$\pm$0.15 & 48.66$\pm$0.37 & 56.54$\pm$0.37 & 50.95$\pm$0.34 & 100.00$\pm$0.00 & 98.50$\pm$0.06 \\
\bottomrule
\bottomrule
\end{tabularx}
\label{unlearning_utility_auc}
\end{table*}

Figure~\ref{fig:unlearningscores_cifar100}, \ref{fig:unlearningscores_Purchase}, \ref{fig:unlearningscores_Texas}, and \ref{fig:unlearningscores_Location} show the score distributions of three unlearning metrics for the unlearning tasks on CIFAR100, Purchase, Texas, and Location, respectively. The conclusions on the three metrics remain consistent with our previous analysis on CIFAR10.

\begin{figure}[h]
    \centering
     \begin{subfigure}[b]{0.45\linewidth}
         \centering
         \includegraphics[width=\linewidth]{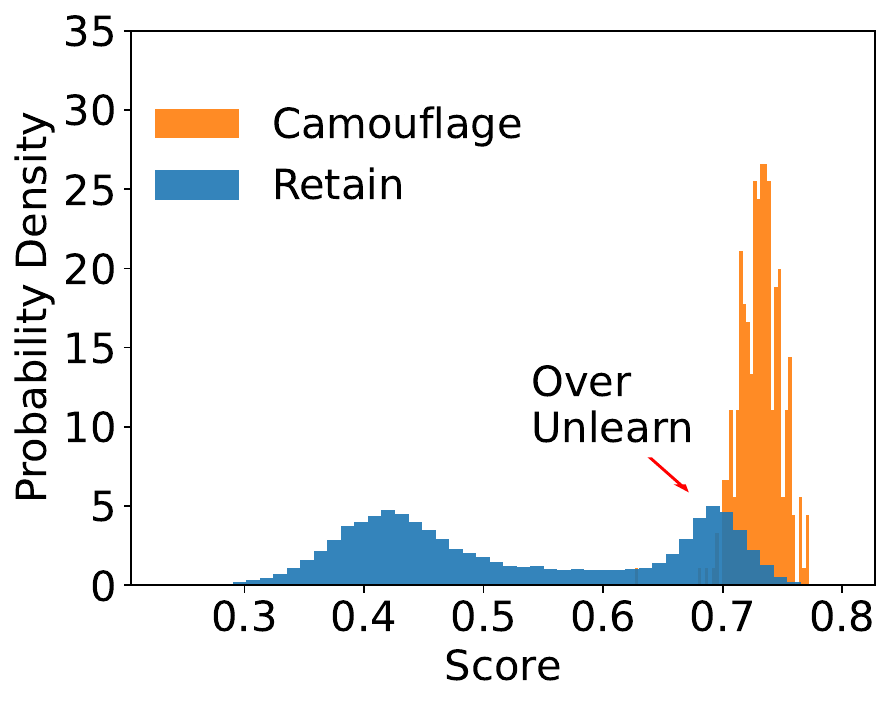}
         \caption{Case1}
     \end{subfigure}
    \begin{subfigure}[b]{0.45\linewidth}
         \centering
         \includegraphics[width=\linewidth]{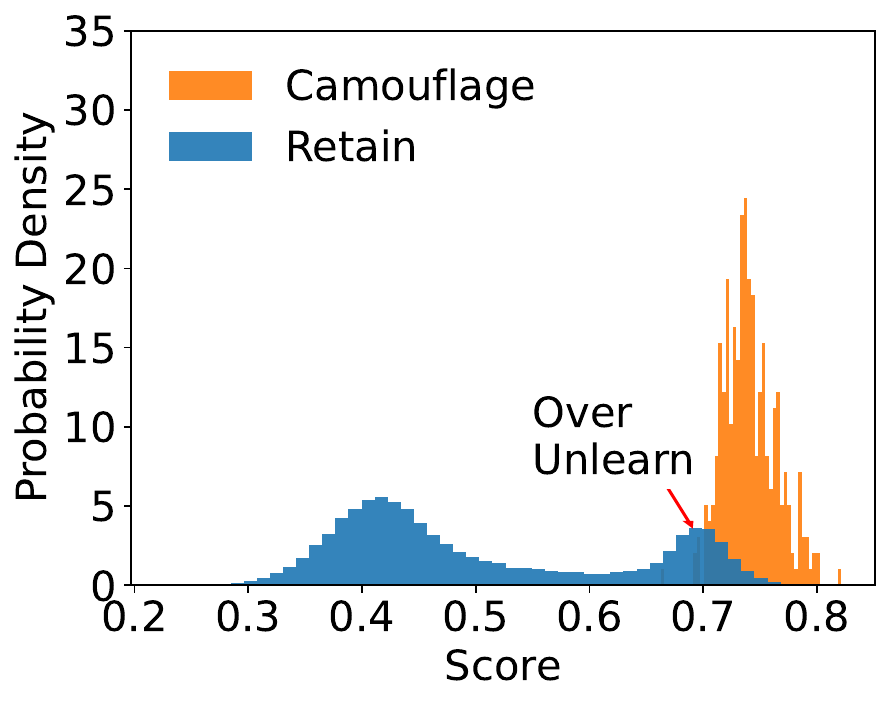}
         \caption{Case2}
     \end{subfigure}
     \caption{UnleScores for samples within Camouflage 'Unlearning' Cases of CIFAR100.}
     \label{fig:adv_cifar100}
\end{figure}
\begin{figure}[h]
    \centering
     \begin{subfigure}[b]{0.45\linewidth}
         \centering
         \includegraphics[width=\linewidth]{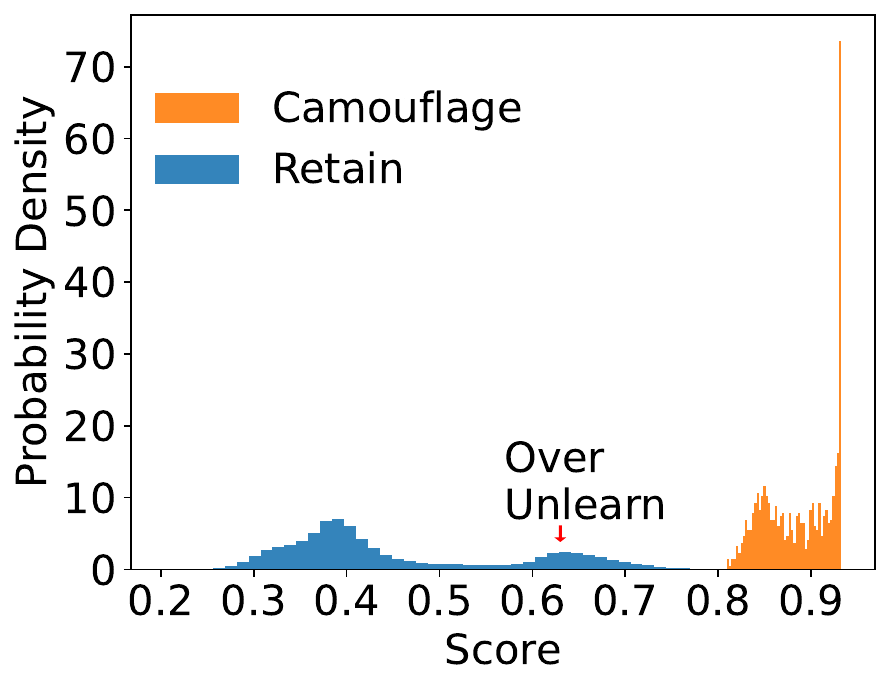}
         \caption{Case1}
     \end{subfigure}
    \begin{subfigure}[b]{0.45\linewidth}
         \centering
         \includegraphics[width=\linewidth]{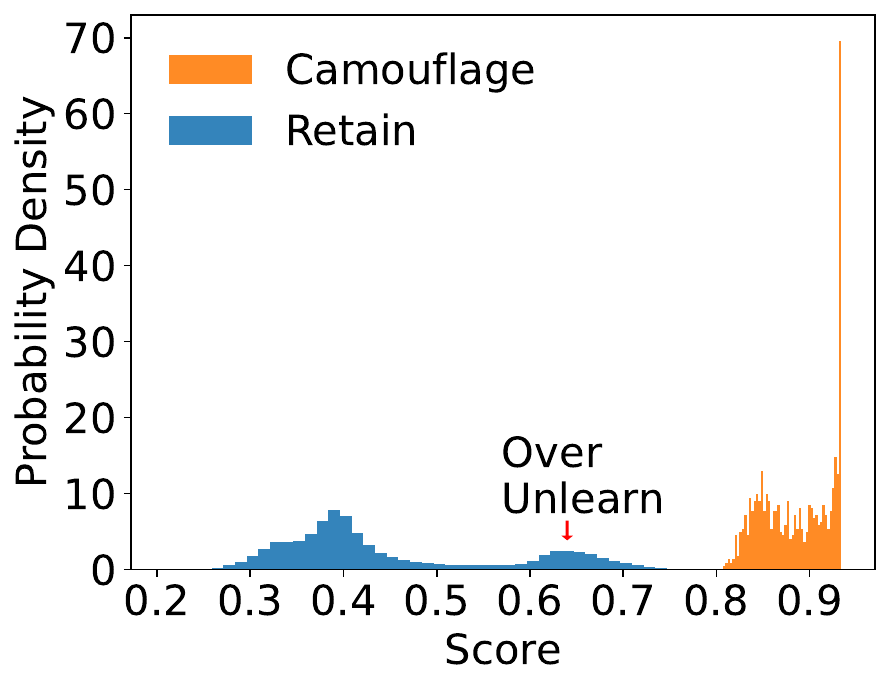}
         \caption{Case2}
     \end{subfigure}
     \caption{UnleScores for samples within Camouflage 'Unlearning' Cases of Purchase.}
     \label{fig:adv_purchase}
\end{figure}
\begin{figure}[h]
    \centering
     \begin{subfigure}[b]{0.45\linewidth}
         \centering
         \includegraphics[width=\linewidth]{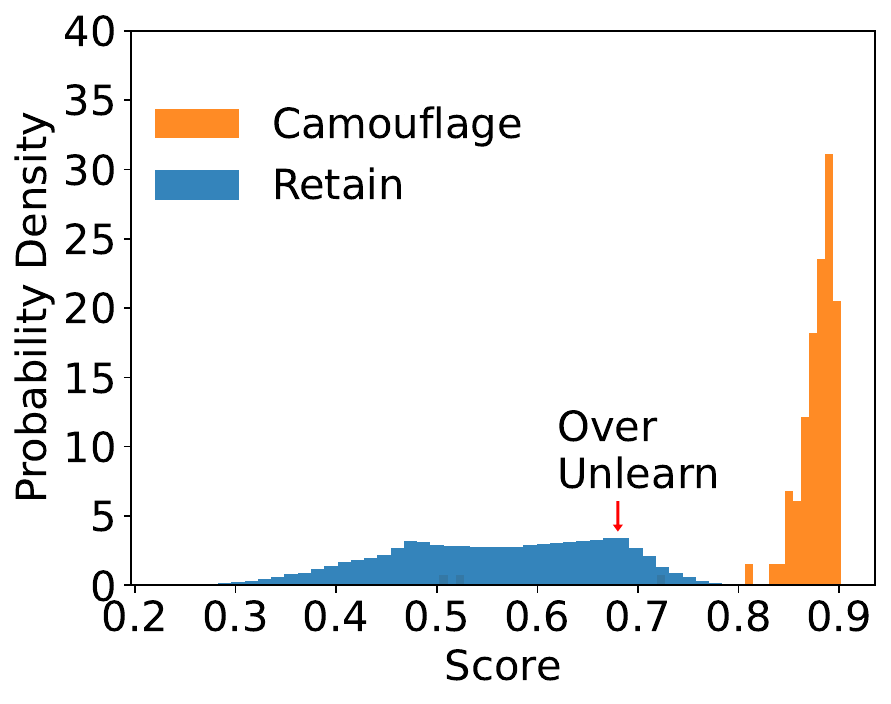}
         \caption{Case1}
     \end{subfigure}
    \begin{subfigure}[b]{0.45\linewidth}
         \centering
         \includegraphics[width=\linewidth]{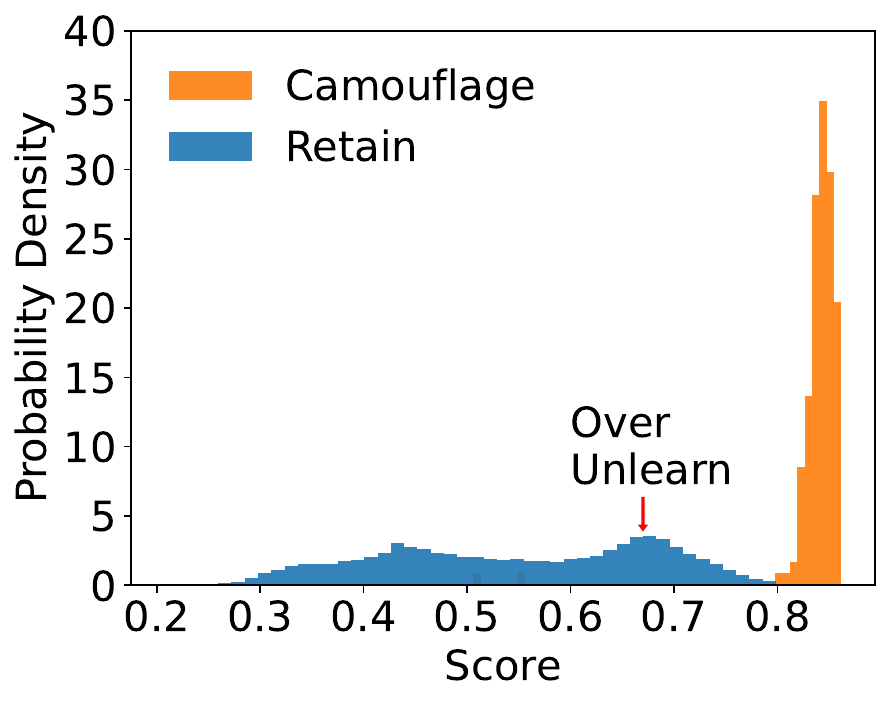}
         \caption{Case2}
     \end{subfigure}
     \caption{UnleScores for samples within Camouflage 'Unlearning' Cases of Texas.}
     \label{fig:adv_texas}
\end{figure}
\begin{figure}[h]
    \centering
     \begin{subfigure}[b]{0.45\linewidth}
         \centering
         \includegraphics[width=\linewidth]{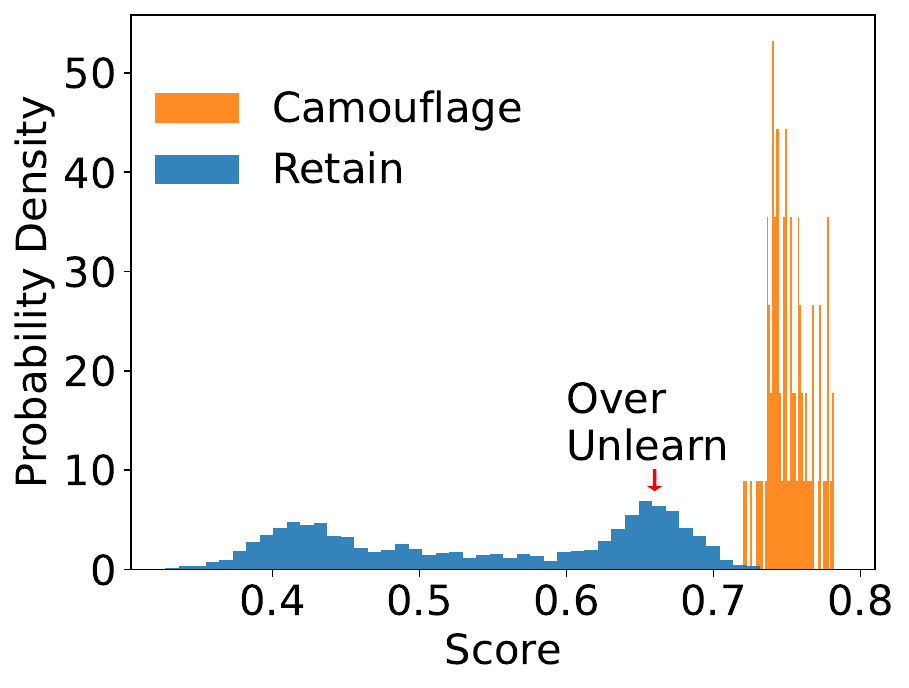}
         \caption{Case1}
     \end{subfigure}
    \begin{subfigure}[b]{0.45\linewidth}
         \centering
         \includegraphics[width=\linewidth]{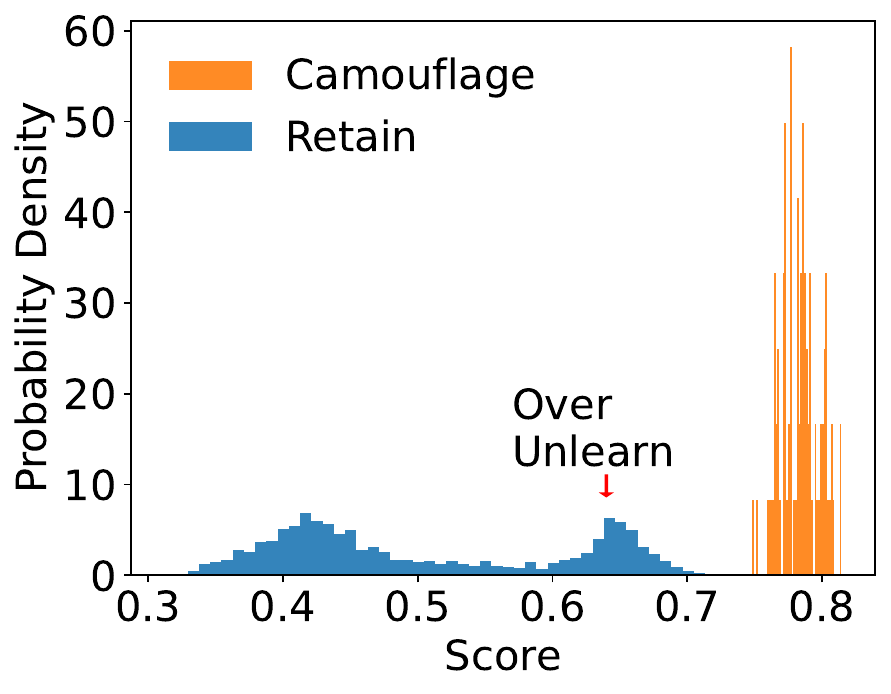}
         \caption{Case2}
     \end{subfigure}
     \caption{UnleScores for samples within Camouflage 'Unlearning' Cases of Location.}
     \label{fig:adv_location}
\end{figure}

\begin{figure}[ht]
    \centering
    \centering
     \begin{subfigure}[b]{0.45\textwidth}
         \centering
         \includegraphics[width=\linewidth]{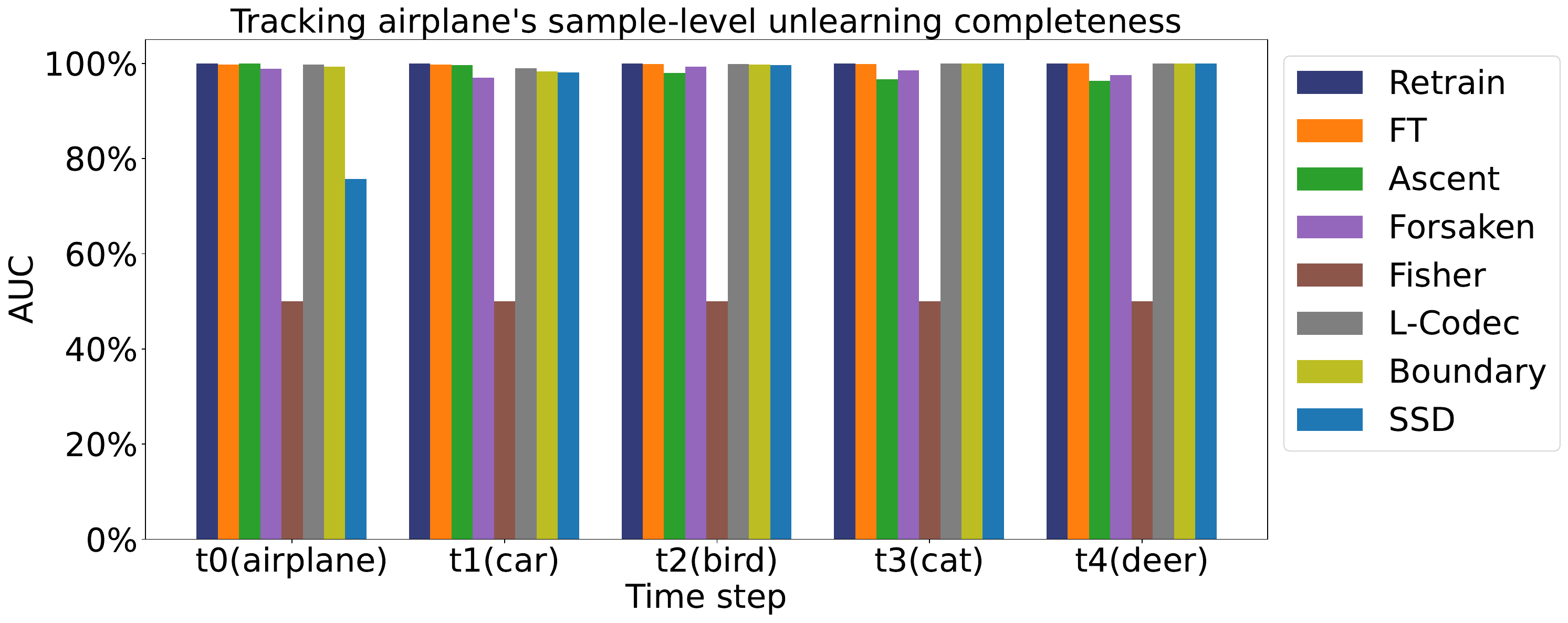}
         \caption{CIFAR10}
         \label{fig:y equals x}
     \end{subfigure}
     \hfill
     \begin{subfigure}[b]{0.45\textwidth}
         \centering
         \includegraphics[width=\linewidth]{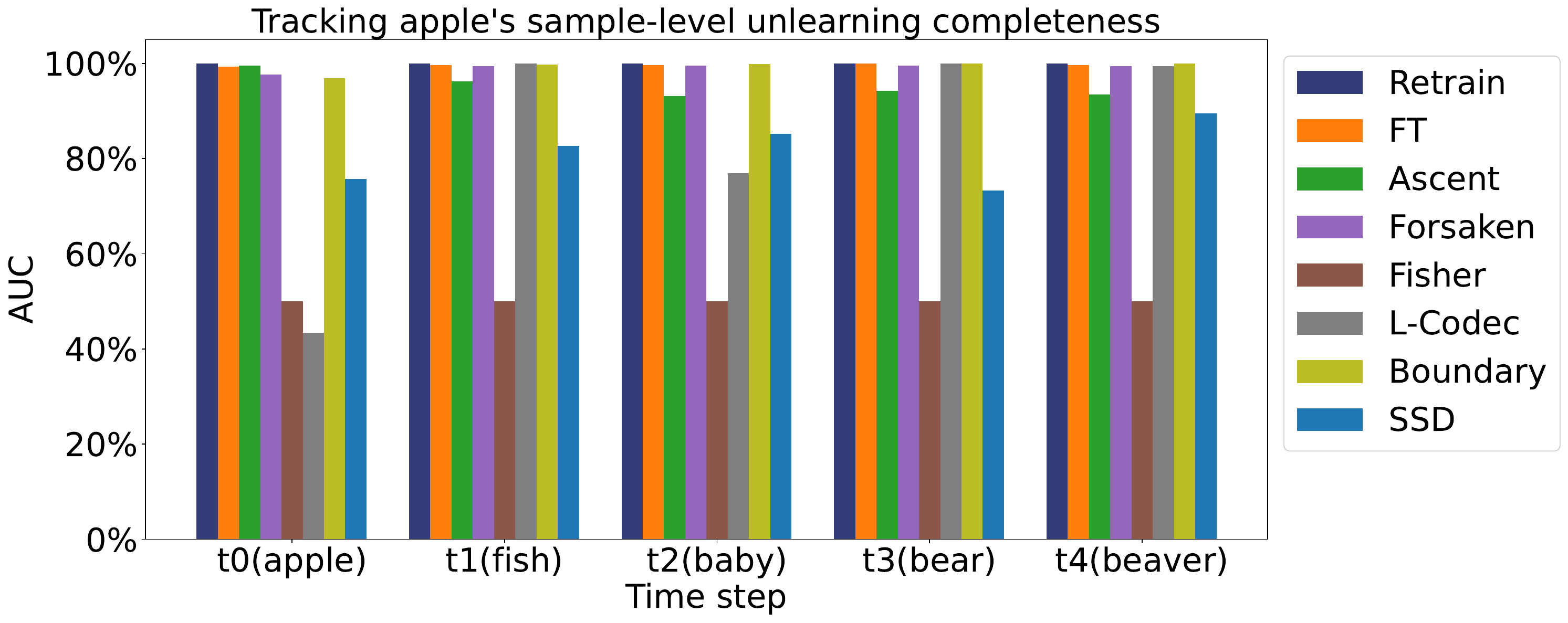}
         \caption{CIFAR100}
         \label{fig:three sin x}
     \end{subfigure}
    \caption{Total class unlearning resilience results (AUC) of baselines on CIFAR10 and CIFAR100}
    \label{continue_auc_cifar10_cifar100}
\end{figure}

\subsection{Additional Results of Score Distribution}\label{sec:subsec:metricthreshold_add}
\begin{figure*}[h]
    \centering
    \begin{subfigure}[b]{0.16\linewidth}
         \centering
         \includegraphics[height=0.08\textheight]{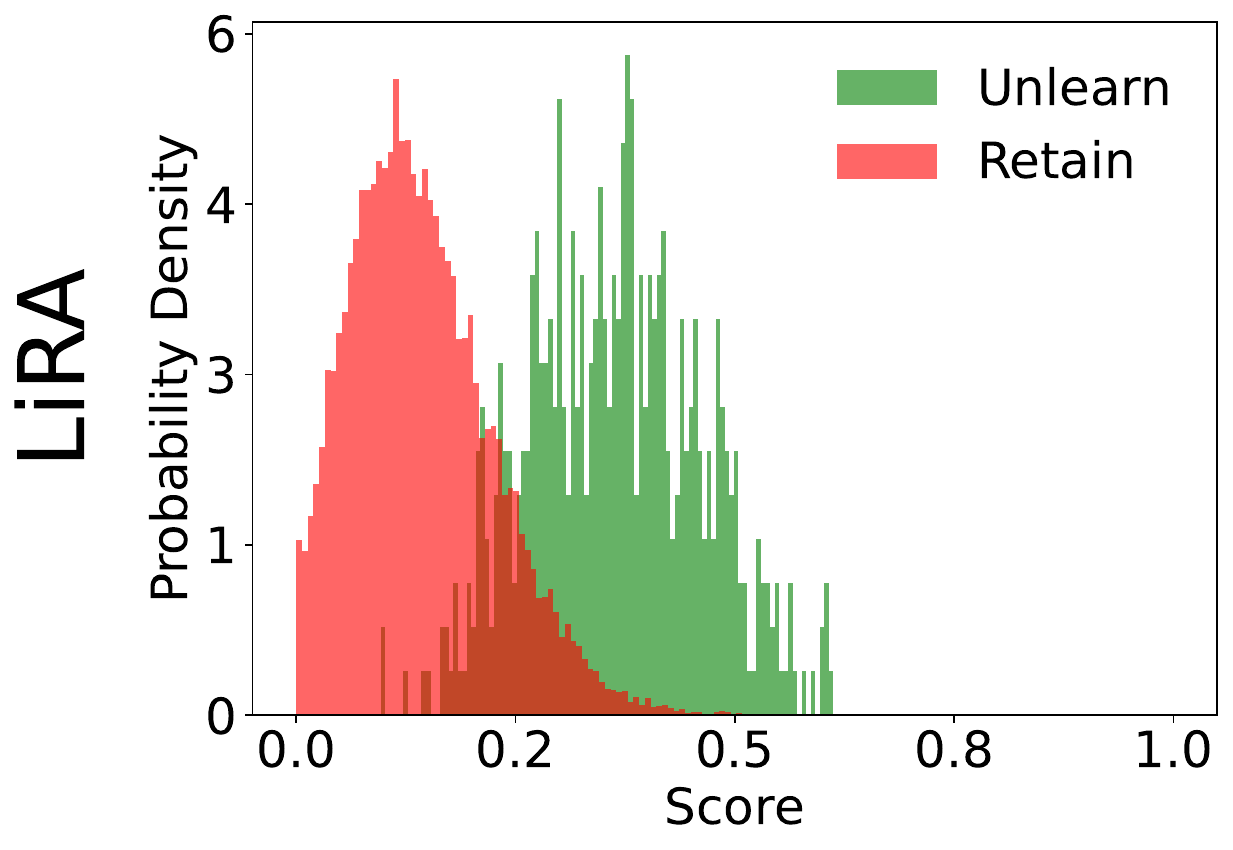}
         \caption{Random Sample}
     \end{subfigure}
     \begin{subfigure}[b]{0.15\linewidth}
         \centering
         \includegraphics[height=0.08\textheight]{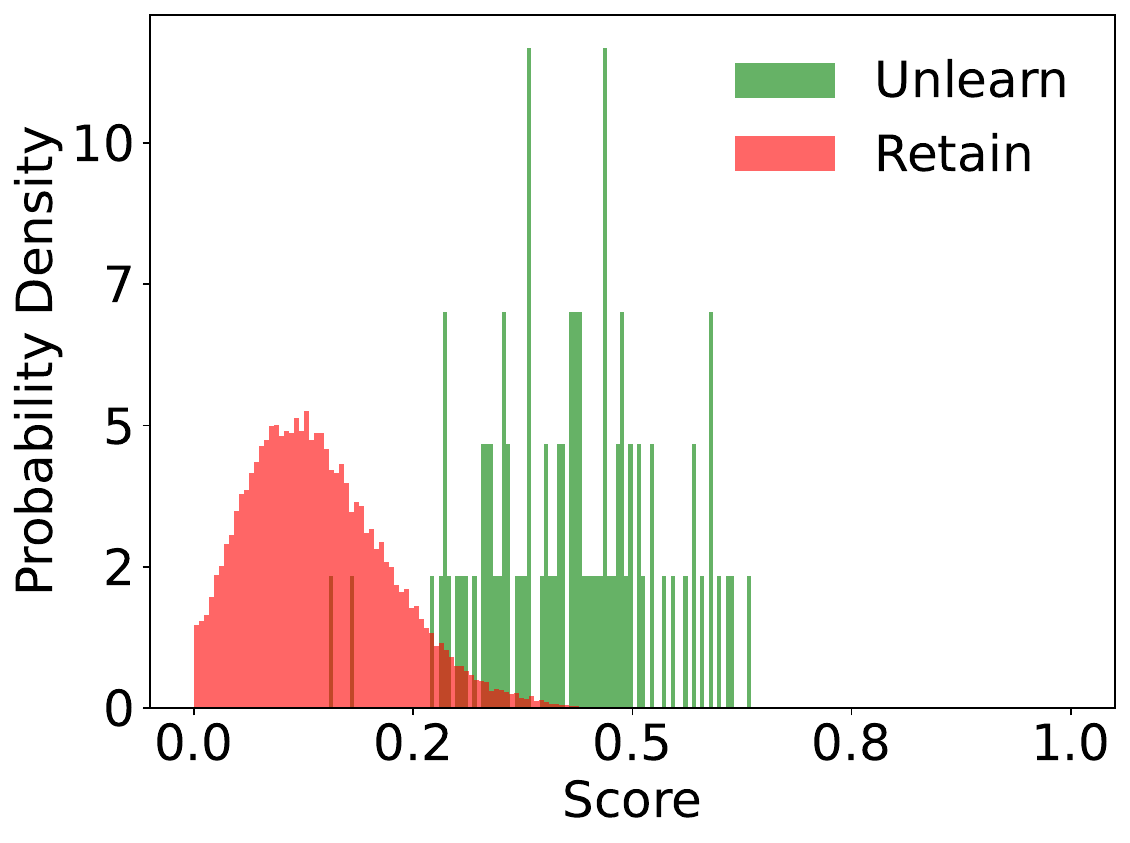}
         \caption{30\% Part-Class}
     \end{subfigure}
     \begin{subfigure}[b]{0.15\linewidth}
         \centering
         \includegraphics[height=0.08\textheight]{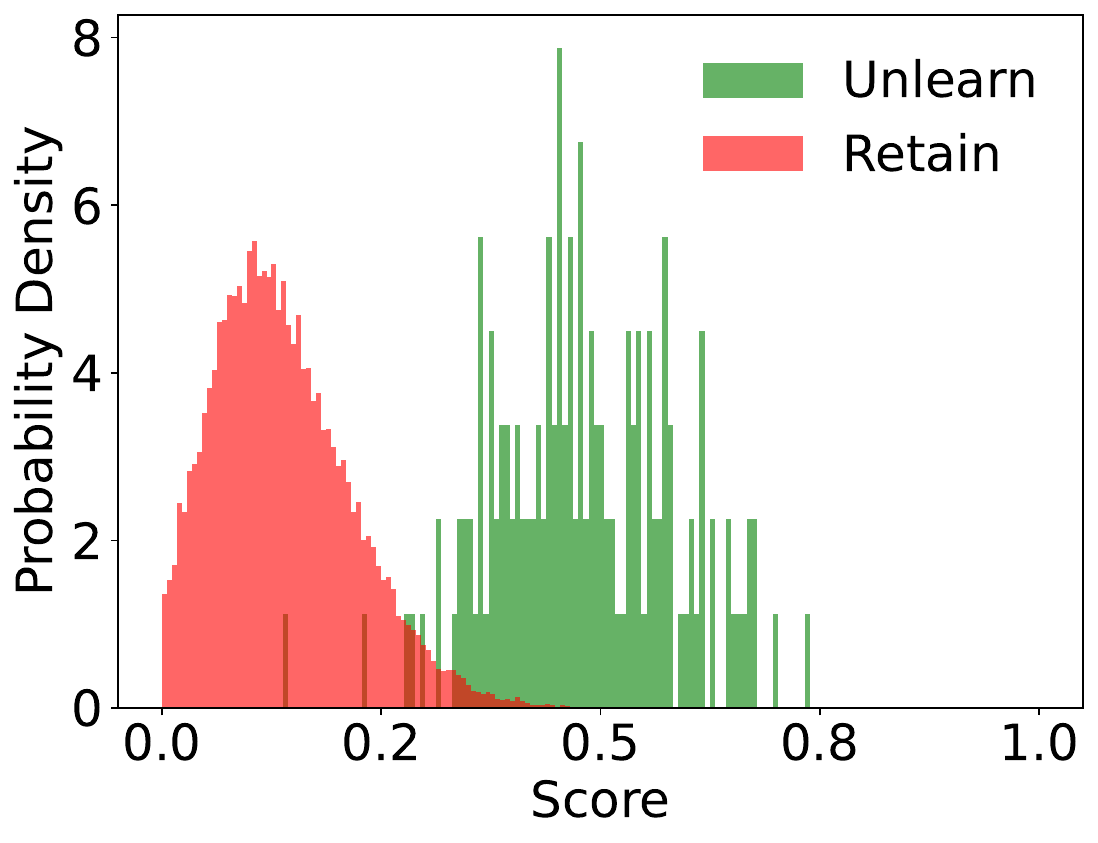}
         \caption{50\% Part-Class}
     \end{subfigure}
     \begin{subfigure}[b]{0.15\linewidth}
         \centering
         \includegraphics[height=0.08\textheight]{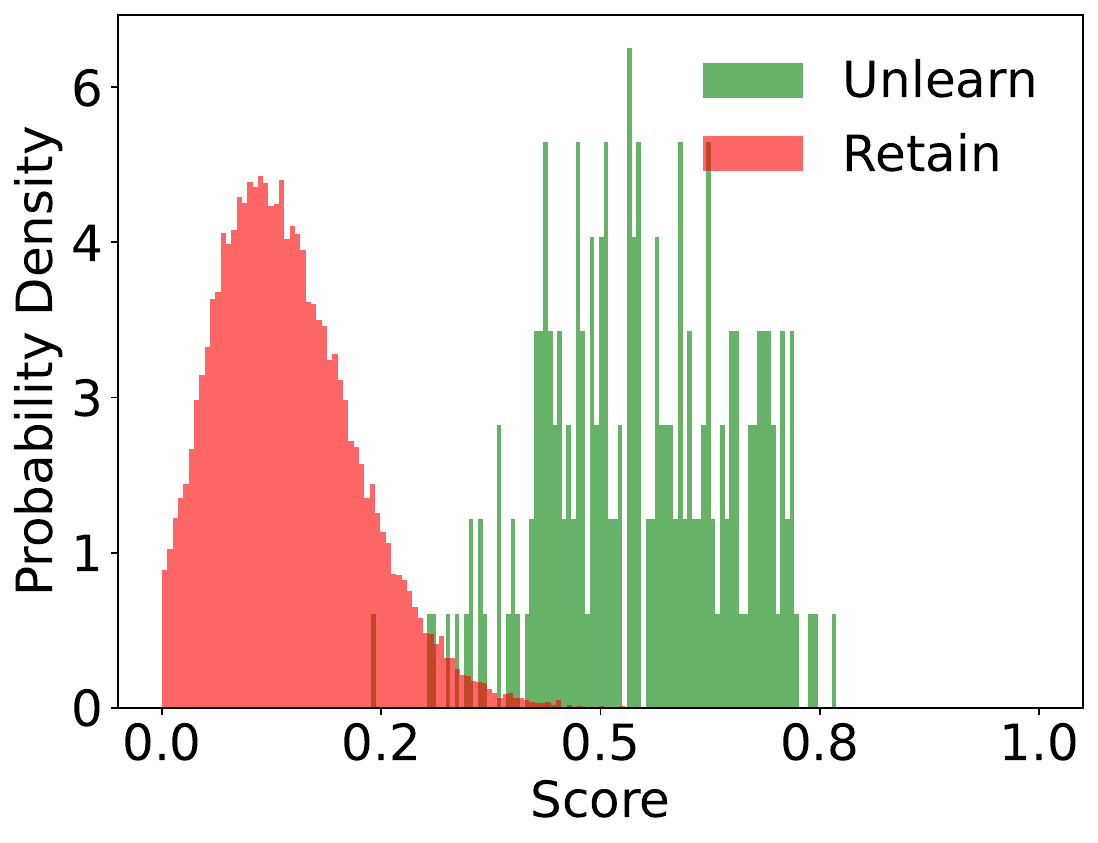}
         \caption{70\% Part-Class}
     \end{subfigure}
     \begin{subfigure}[b]{0.15\linewidth}
         \centering
         \includegraphics[height=0.08\textheight]{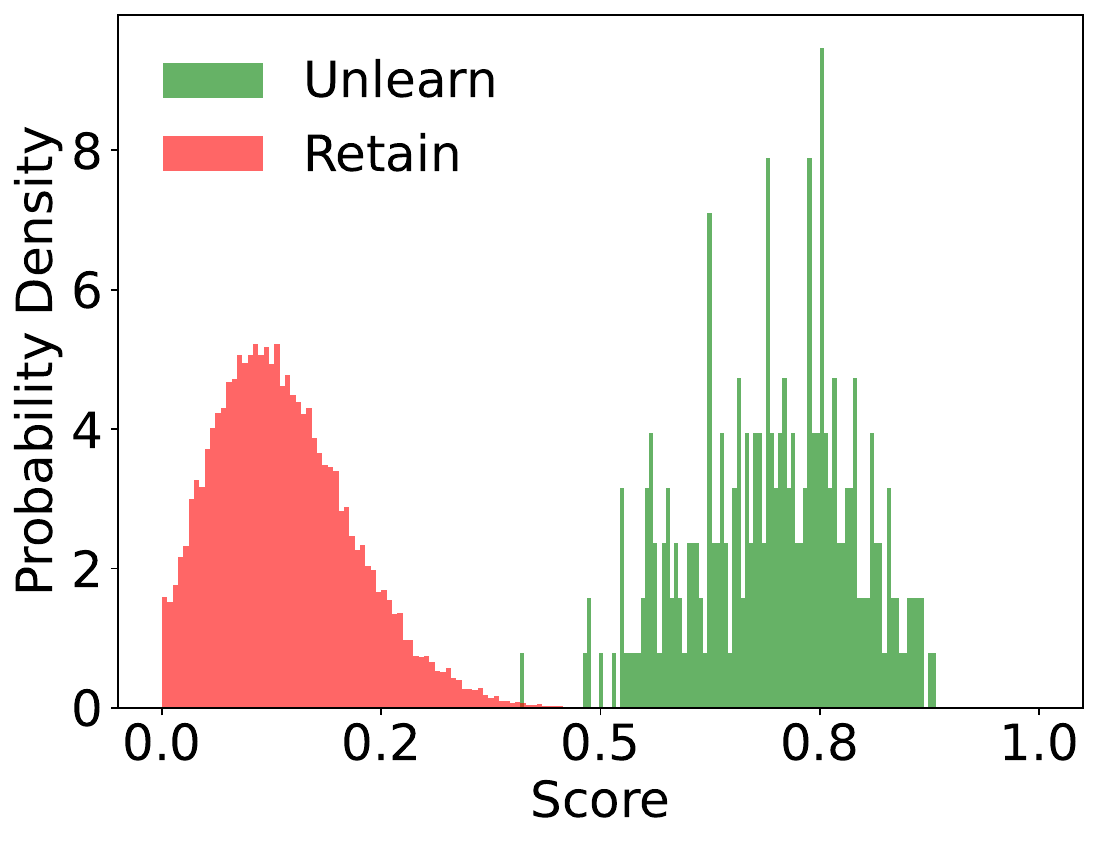}
         \caption{90\% Part-Class}
     \end{subfigure}
     \begin{subfigure}[b]{0.15\linewidth}
         \centering
         \includegraphics[height=0.08\textheight]{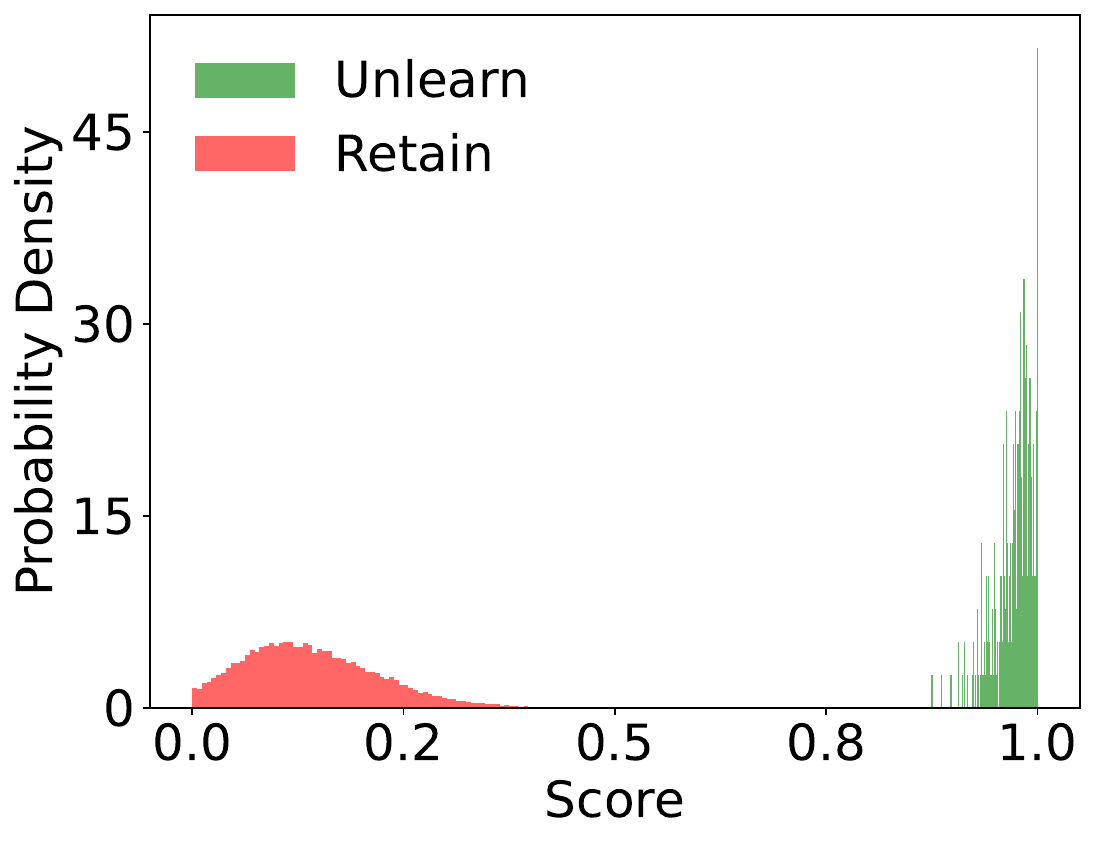}
         \caption{Total Class}
     \end{subfigure}
     \hfill
     \begin{subfigure}[b]{0.16\linewidth}
         \centering
         \includegraphics[height=0.08\textheight]{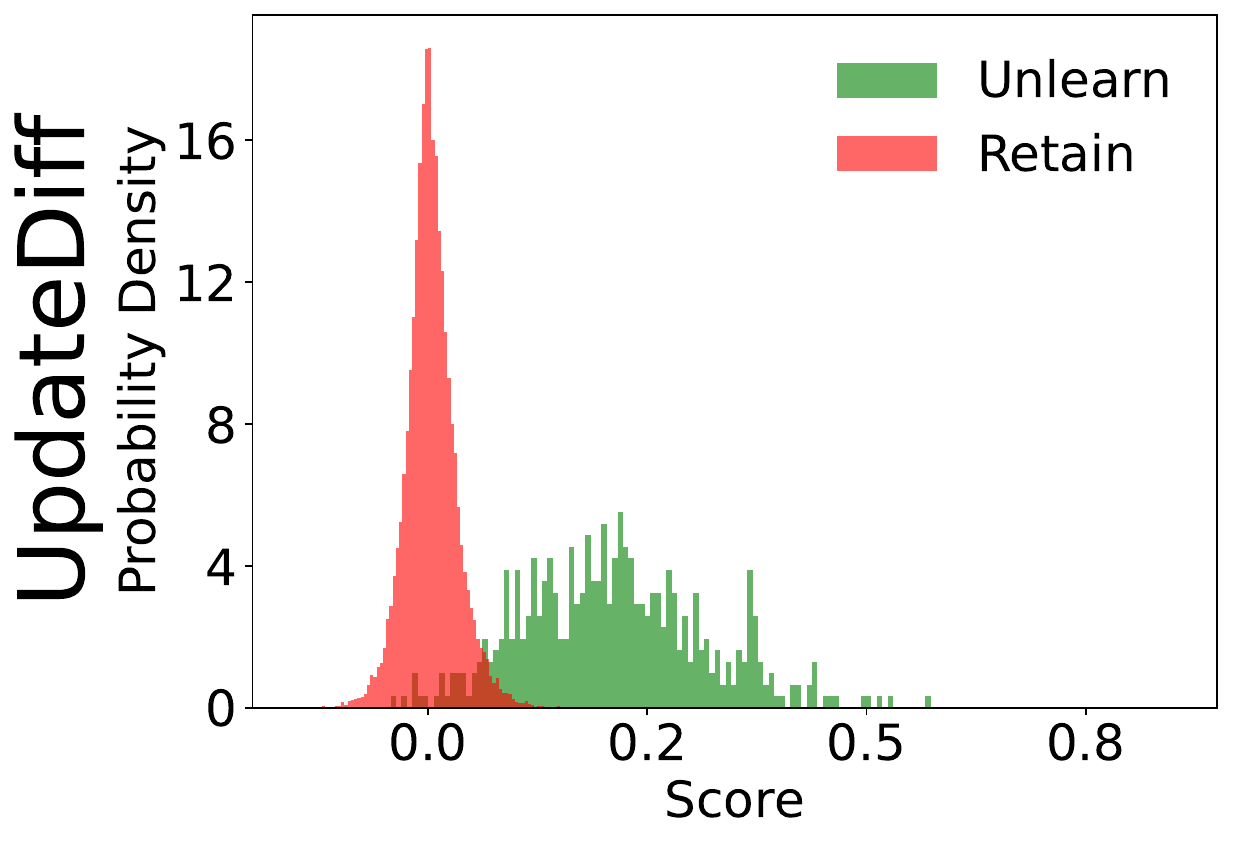}
         \caption{Random Sample}
     \end{subfigure}     
\begin{subfigure}[b]{0.15\linewidth}
         \centering
         \includegraphics[height=0.08\textheight]{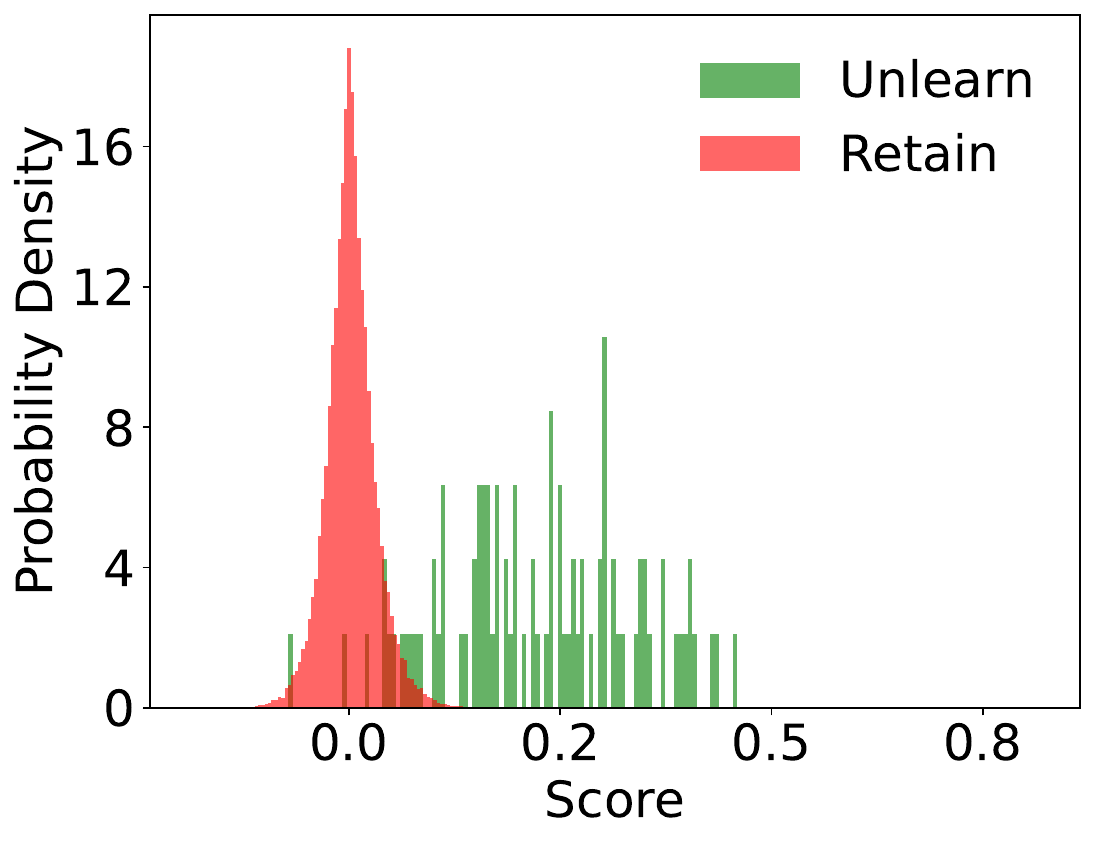}
         \caption{30\% Part-Class}
     \end{subfigure}     
\begin{subfigure}[b]{0.15\linewidth}
         \centering
         \includegraphics[height=0.08\textheight]{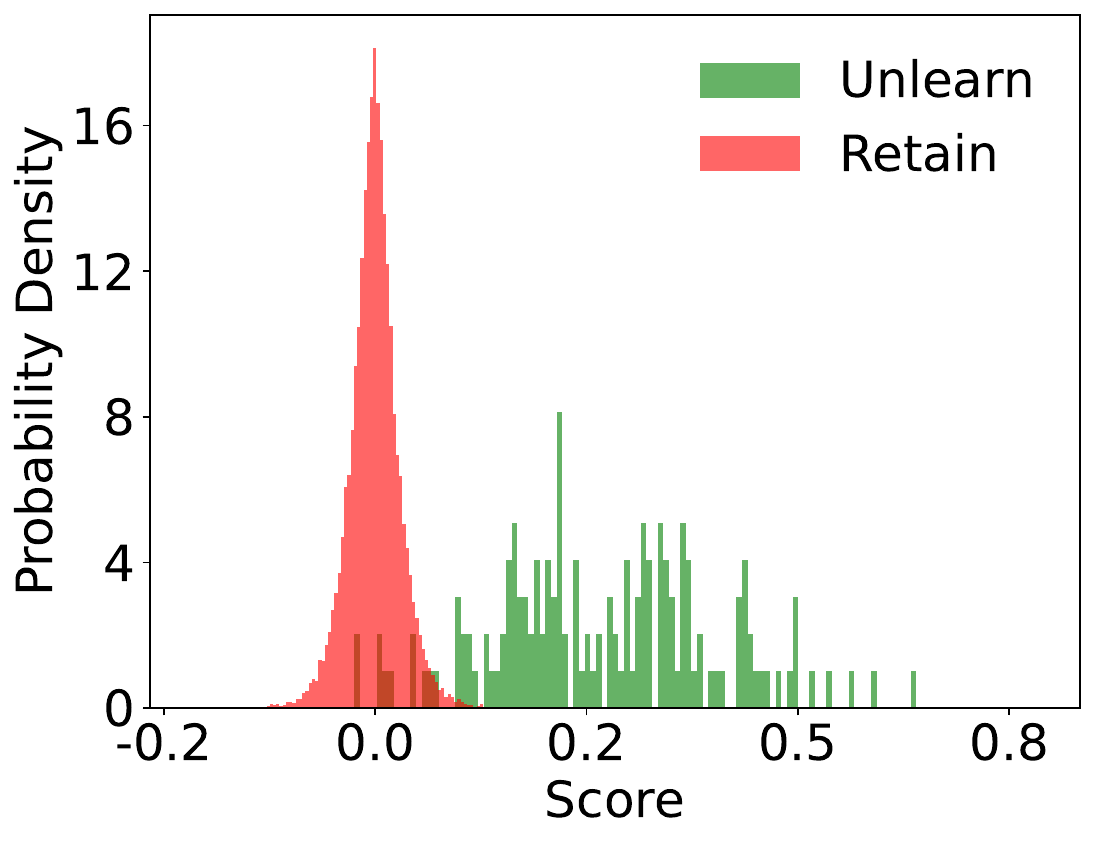}
         \caption{50\% Part-Class}
     \end{subfigure}     
\begin{subfigure}[b]{0.15\linewidth}
         \centering
         \includegraphics[height=0.08\textheight]{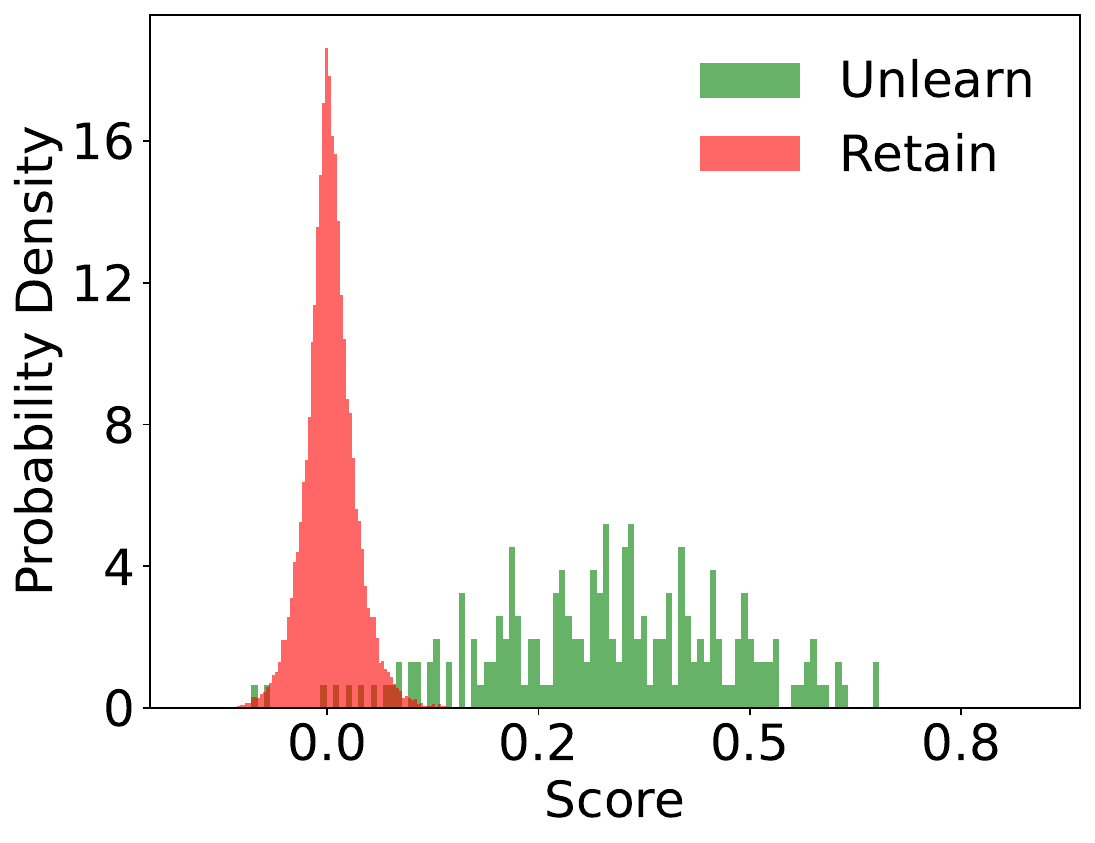}
         \caption{70\% Part-Class}
     \end{subfigure}     
\begin{subfigure}[b]{0.15\linewidth}
         \centering
         \includegraphics[height=0.08\textheight]{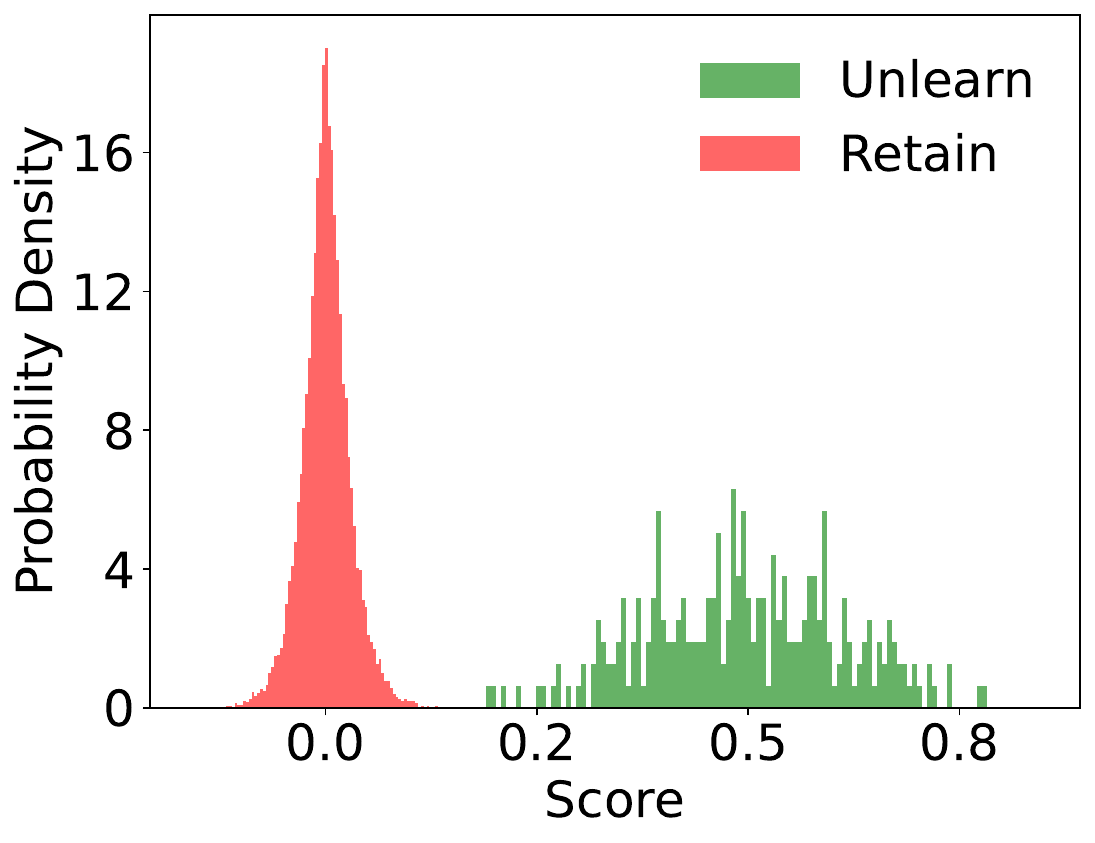}
         \caption{90\% Part-Class}
     \end{subfigure}     
\begin{subfigure}[b]{0.15\linewidth}
         \centering
         \includegraphics[height=0.08\textheight]{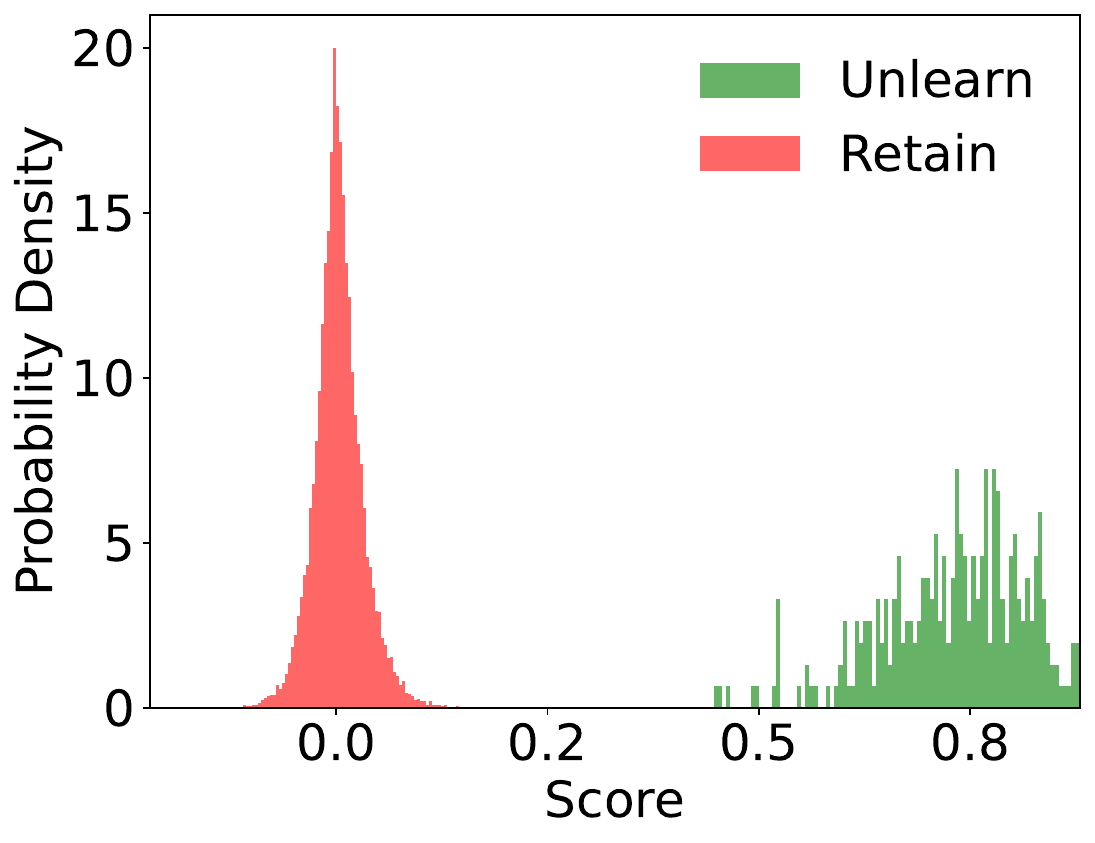}
         \caption{Total Class}
     \end{subfigure}     
     \hfill
\begin{subfigure}[b]{0.16\linewidth}
         \centering
         \includegraphics[height=0.08\textheight]{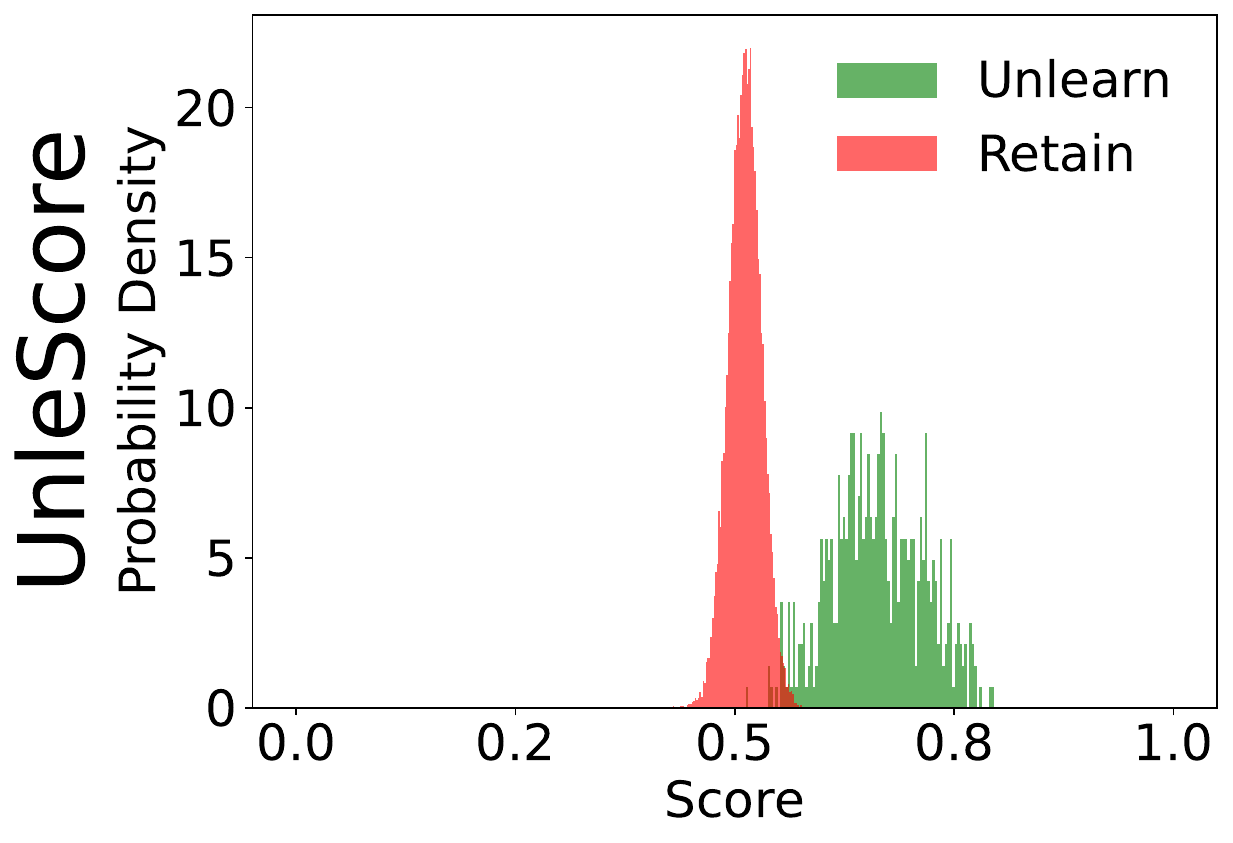}
         \caption{Random Sample}
     \end{subfigure}    
\begin{subfigure}[b]{0.15\linewidth}
         \centering
         \includegraphics[height=0.08\textheight]{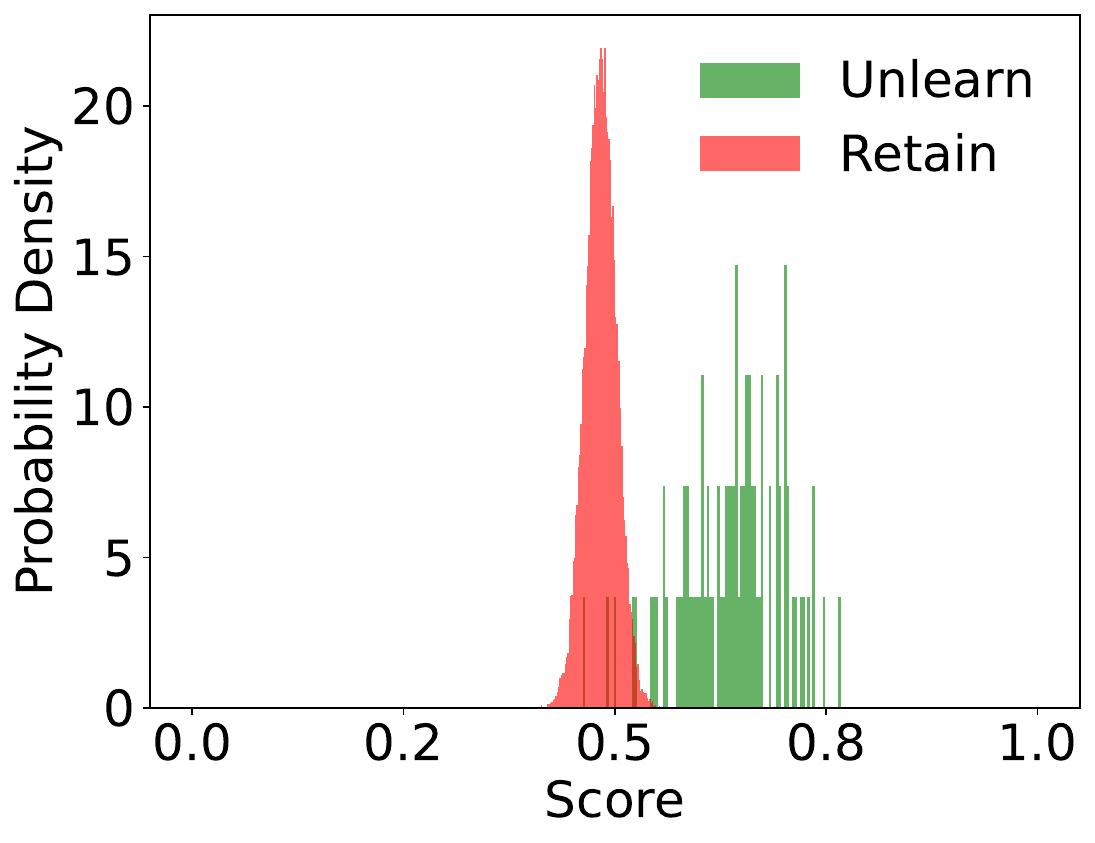}
         \caption{30\% Part-Class}
     \end{subfigure}    
\begin{subfigure}[b]{0.15\linewidth}
         \centering
         \includegraphics[height=0.08\textheight]{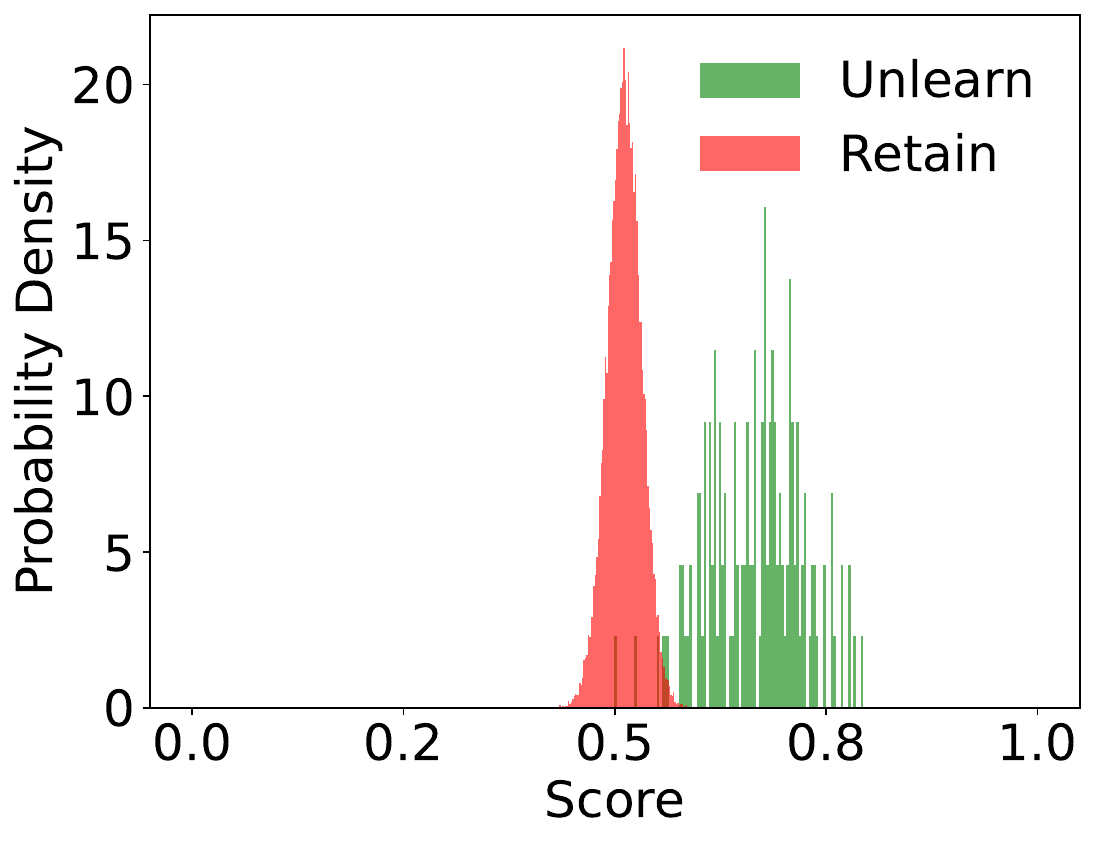}
         \caption{50\% Part-Class}
     \end{subfigure}    
\begin{subfigure}[b]{0.15\linewidth}
         \centering
         \includegraphics[height=0.08\textheight]{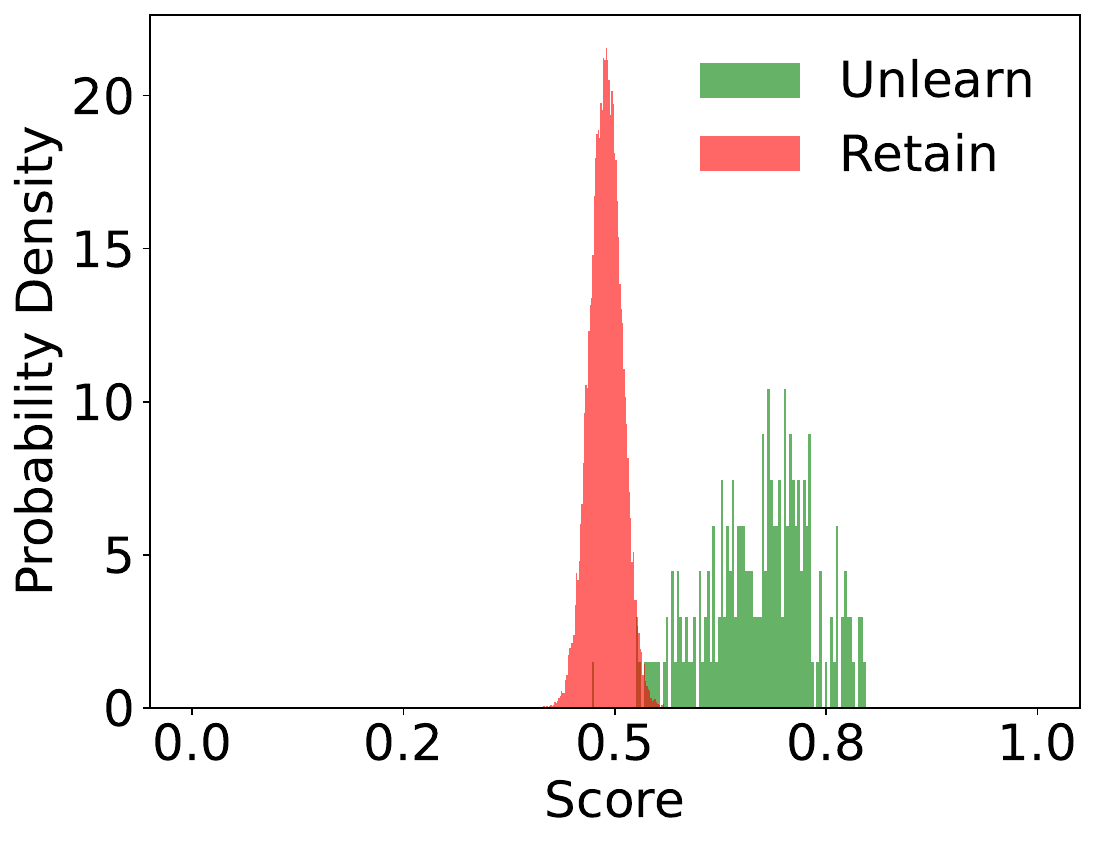}
         \caption{70\% Part-Class}
     \end{subfigure}    
\begin{subfigure}[b]{0.15\linewidth}
         \centering
         \includegraphics[height=0.08\textheight]{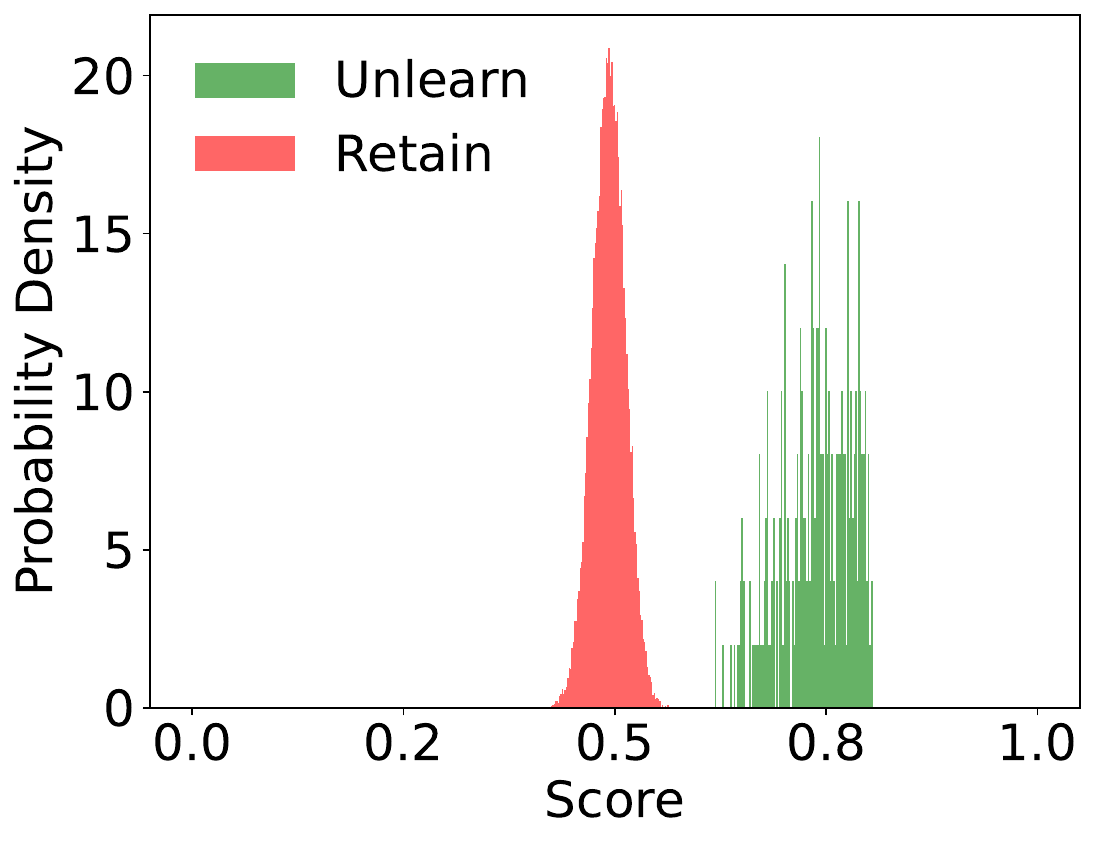}
         \caption{90\% Part-Class}
     \end{subfigure}    
\begin{subfigure}[b]{0.15\linewidth}
         \centering
         \includegraphics[height=0.08\textheight]{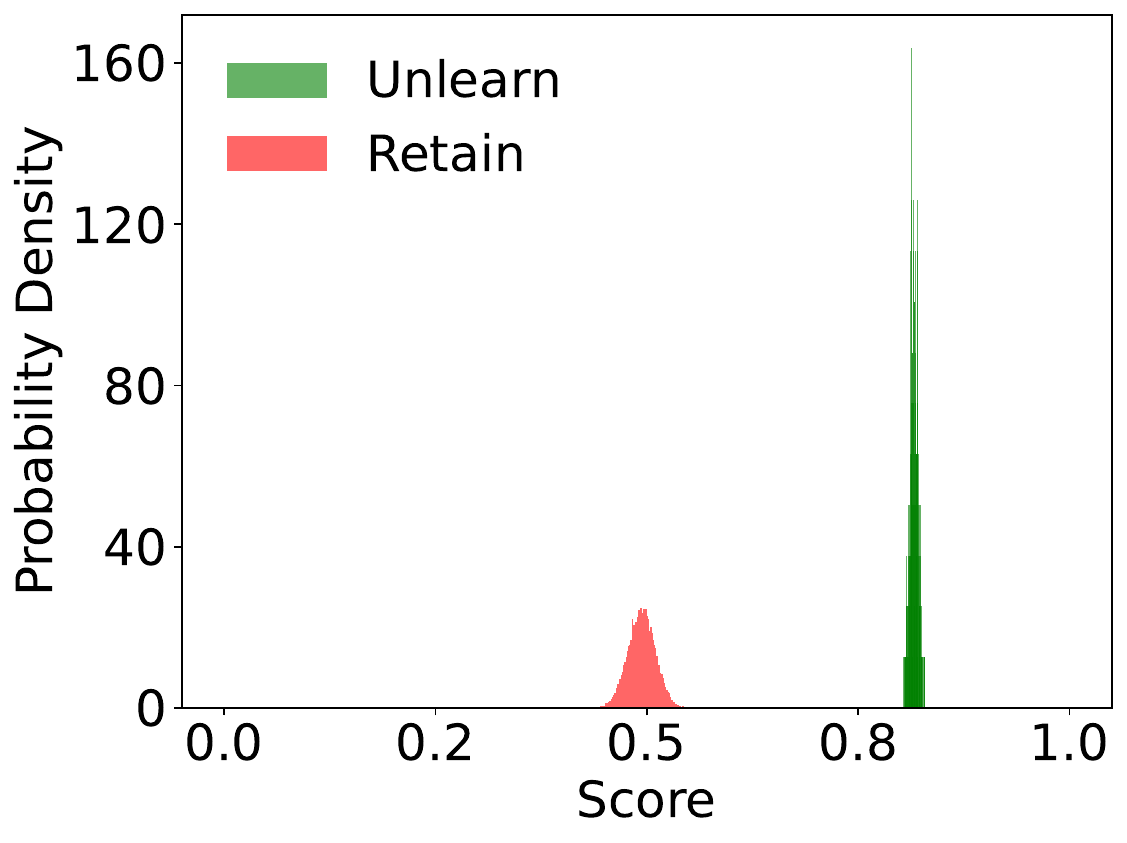}
         \caption{Total Class}
     \end{subfigure}       
     \caption{Score Distributions of Unlearning Metrics with Different Unlearning Tasks on CIFAR100.}
    \label{fig:unlearningscores_cifar100}
\end{figure*}

\begin{figure*}[h]
    \centering
    \begin{subfigure}[b]{0.16\linewidth}
         \centering
         \includegraphics[height=0.08\textheight]{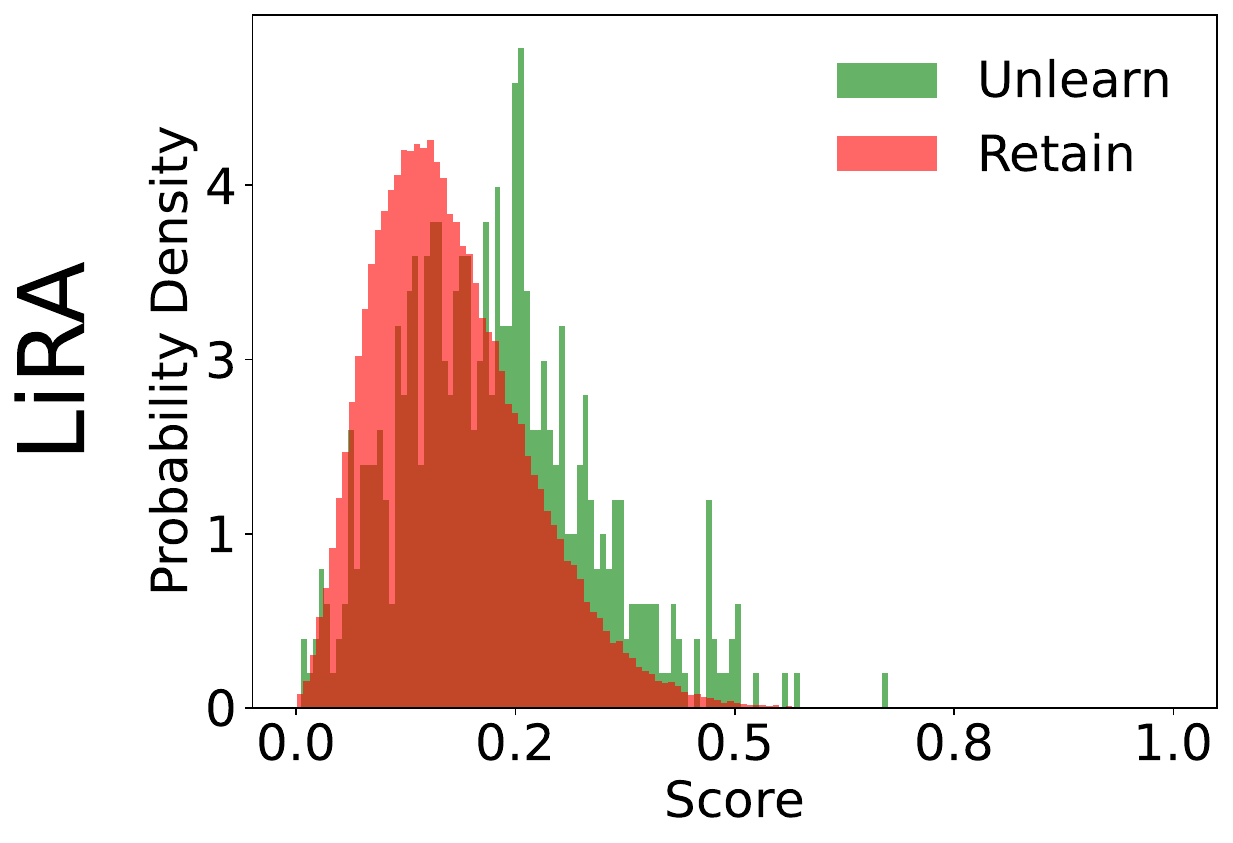}
         \caption{Random Sample}
     \end{subfigure}
     \begin{subfigure}[b]{0.15\linewidth}
         \centering
         \includegraphics[height=0.08\textheight]{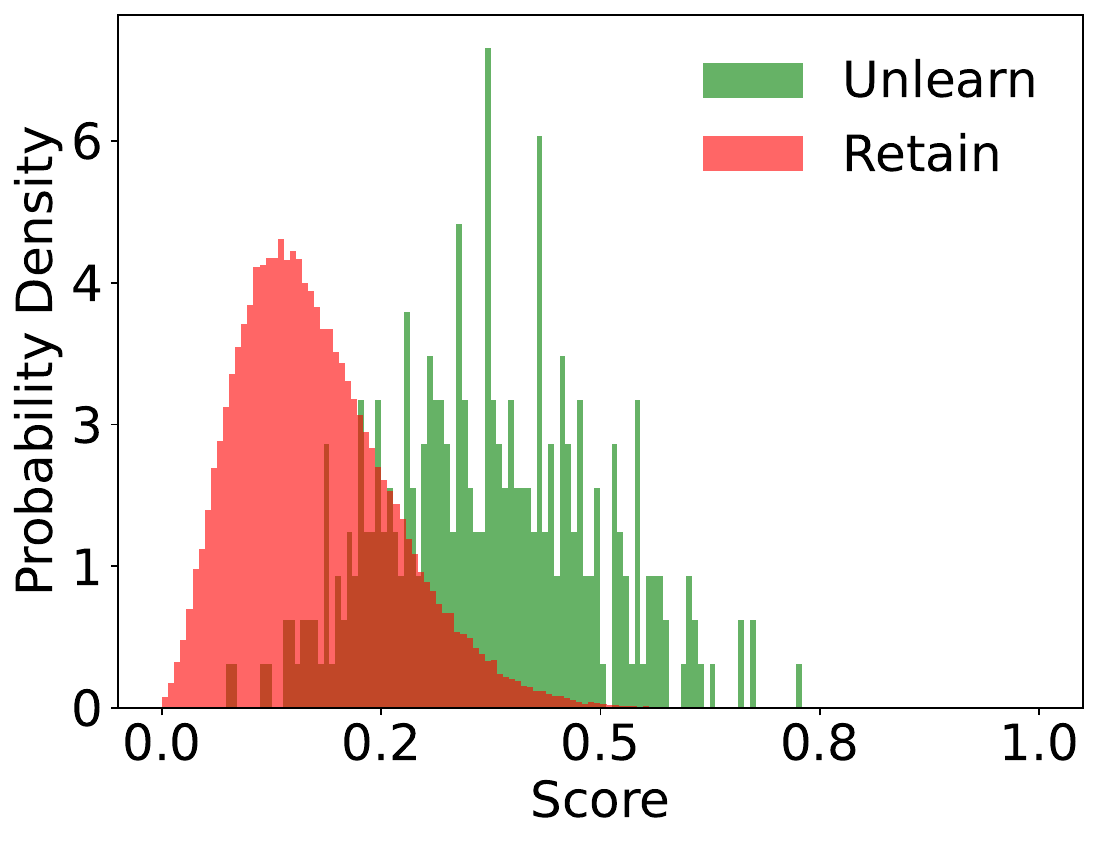}
         \caption{30\% Part-Class}
     \end{subfigure}
     \begin{subfigure}[b]{0.15\linewidth}
         \centering
         \includegraphics[height=0.08\textheight]{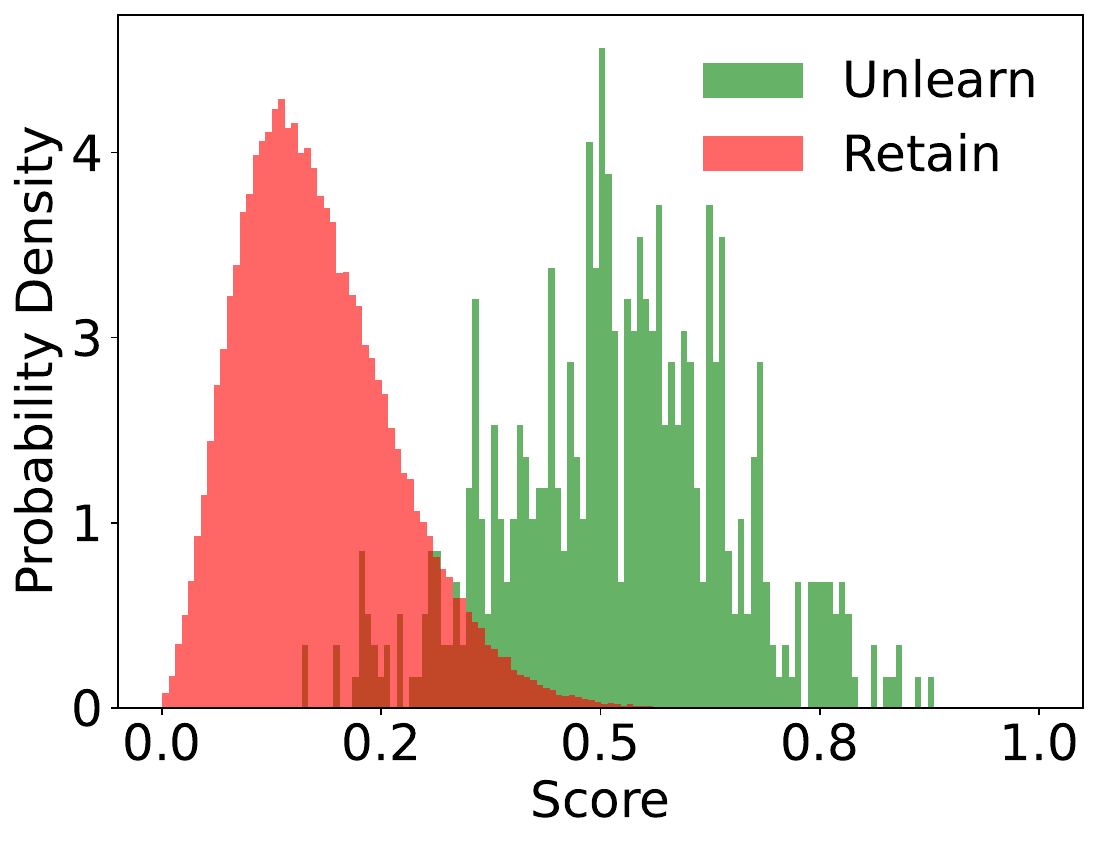}
         \caption{50\% Part-Class}
     \end{subfigure}
     \begin{subfigure}[b]{0.15\linewidth}
         \centering
         \includegraphics[height=0.08\textheight]{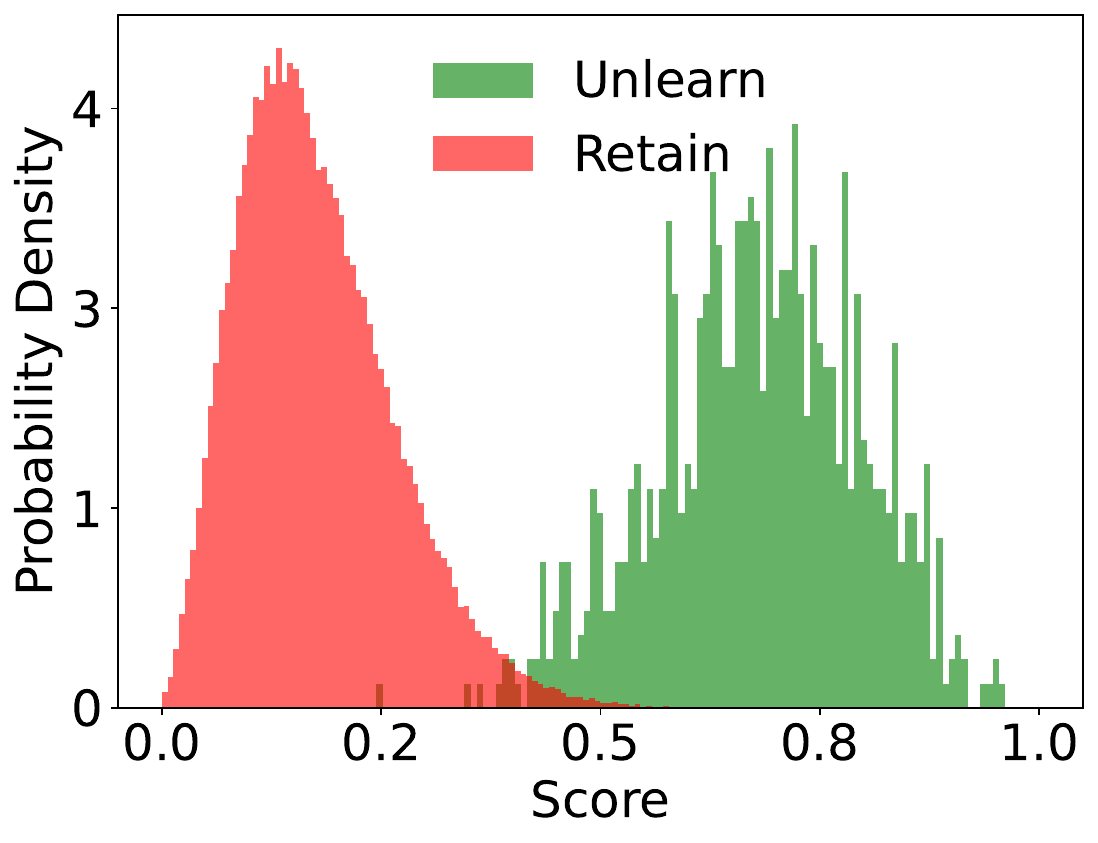}
         \caption{70\% Part-Class}
     \end{subfigure}
     \begin{subfigure}[b]{0.15\linewidth}
         \centering
         \includegraphics[height=0.08\textheight]{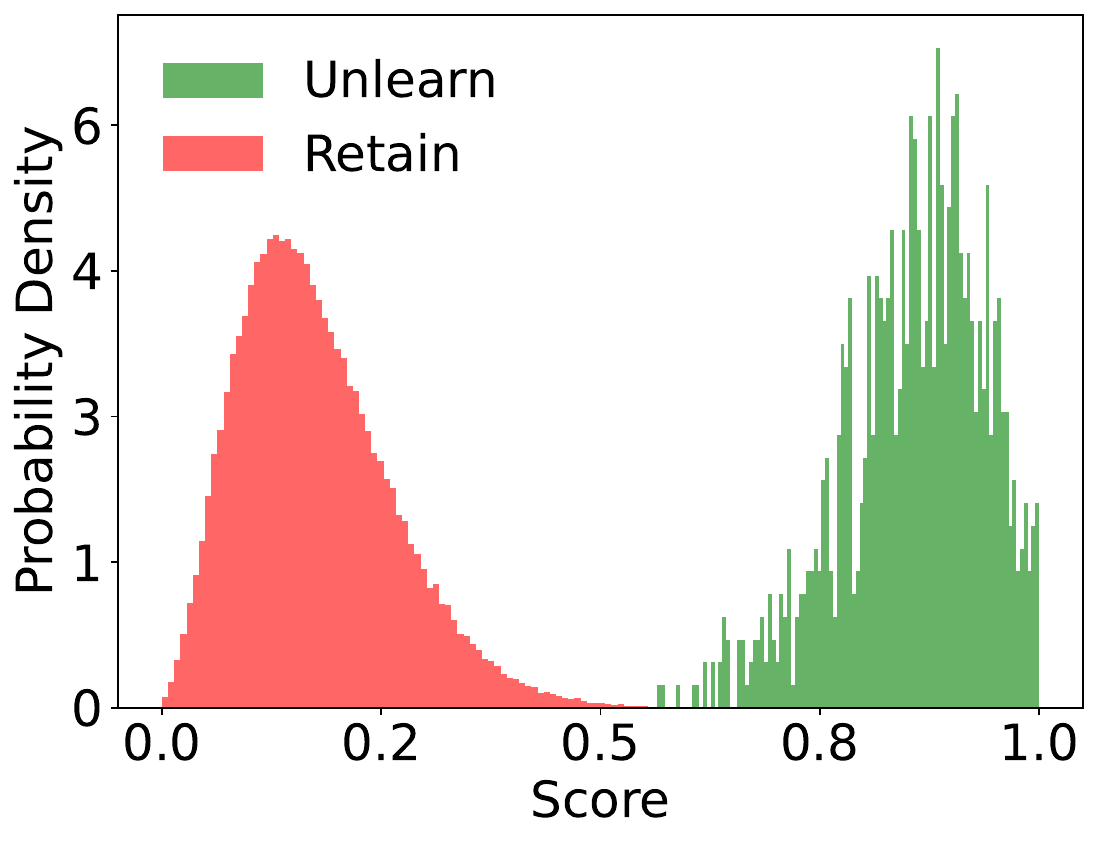}
         \caption{90\% Part-Class}
     \end{subfigure}
     \begin{subfigure}[b]{0.15\linewidth}
         \centering
         \includegraphics[height=0.08\textheight]{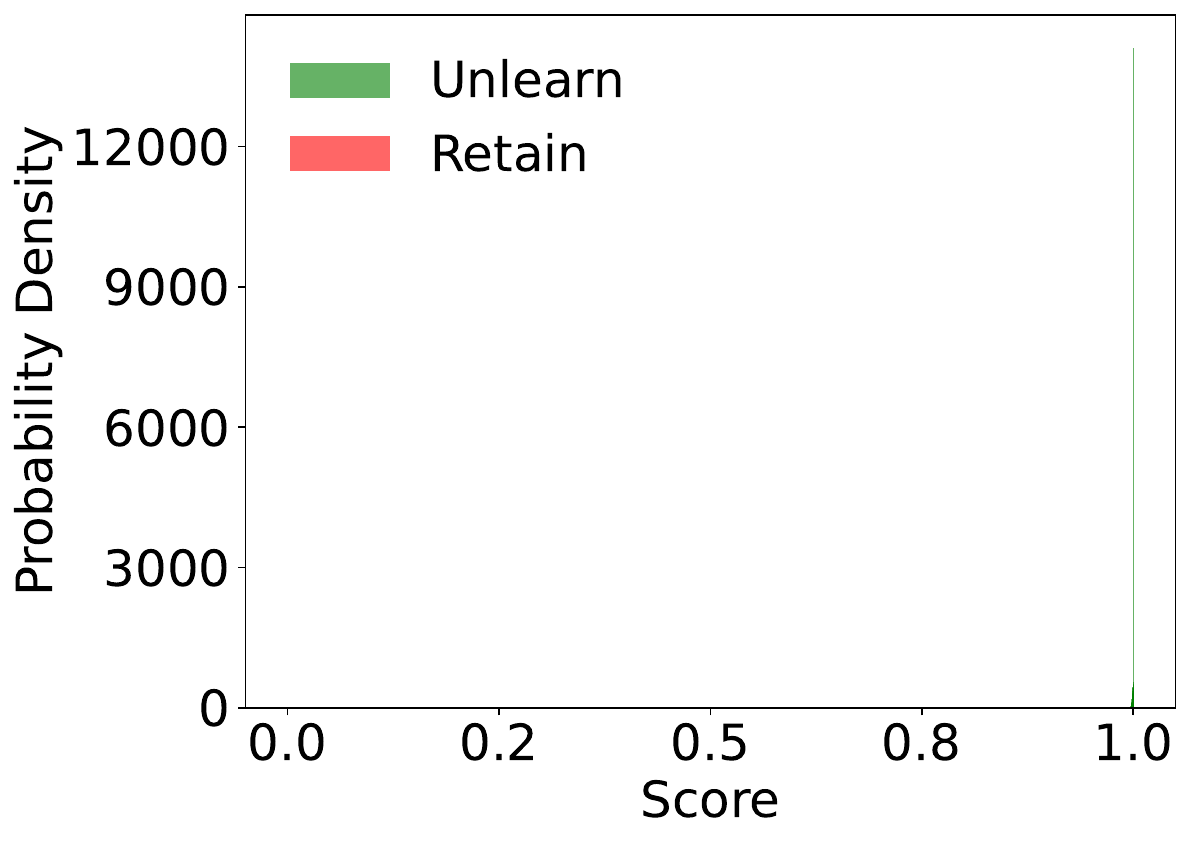}
         \caption{Total Class}
     \end{subfigure}
     \hfill
     \begin{subfigure}[b]{0.16\linewidth}
         \centering
         \includegraphics[height=0.08\textheight]{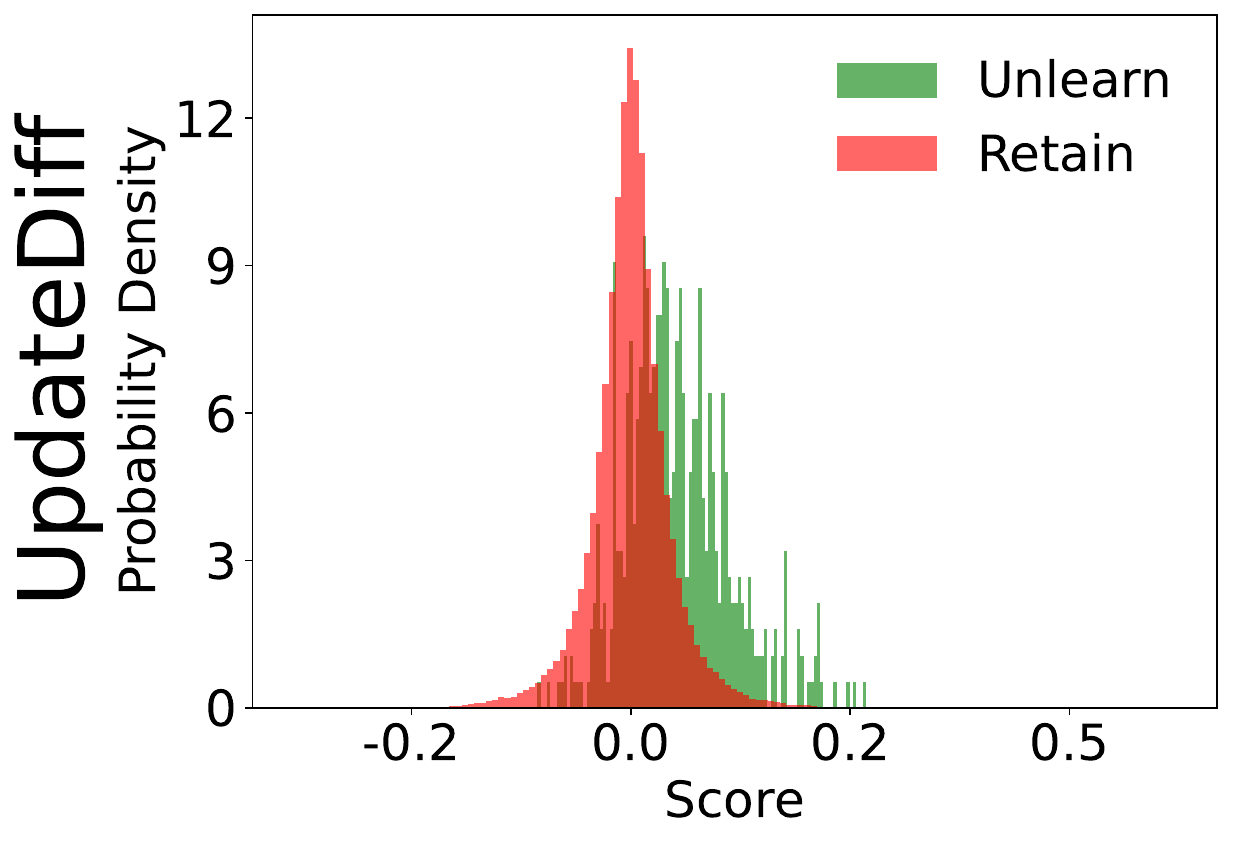}
         \caption{Random Sample}
     \end{subfigure}     
\begin{subfigure}[b]{0.15\linewidth}
         \centering
         \includegraphics[height=0.08\textheight]{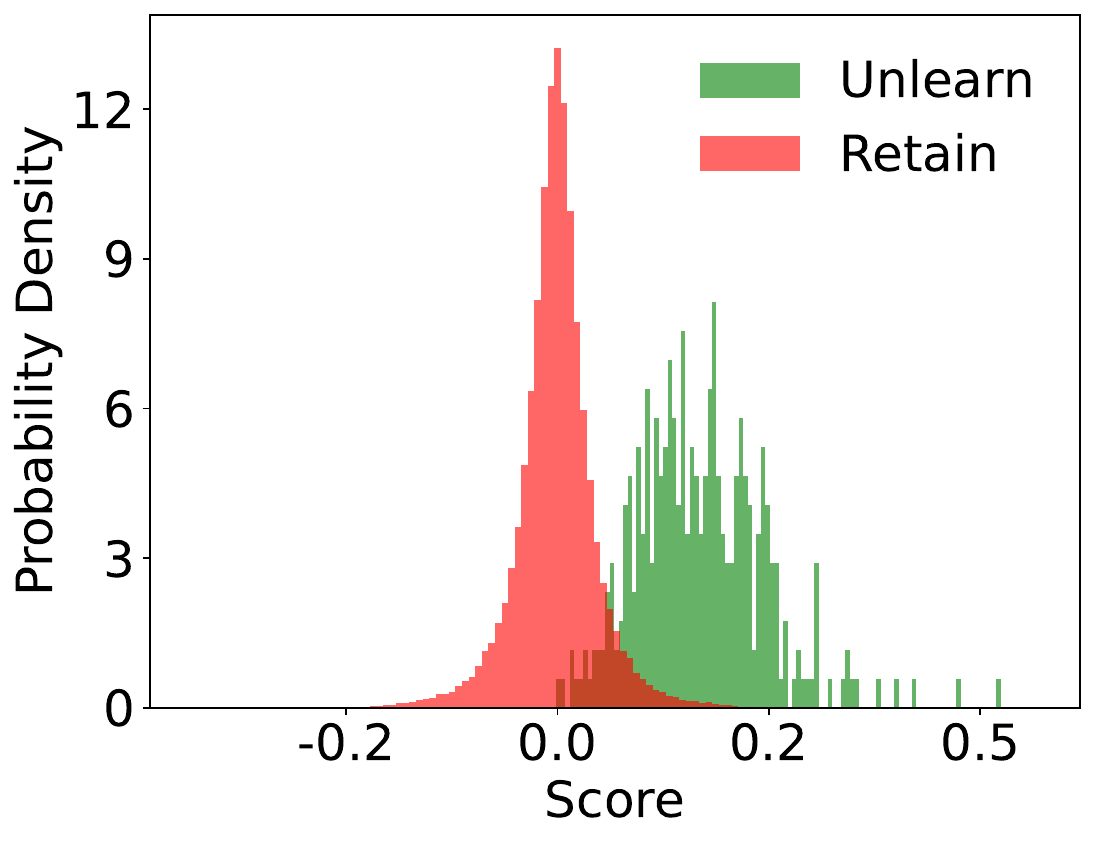}
         \caption{30\% Part-Class}
     \end{subfigure}     
\begin{subfigure}[b]{0.15\linewidth}
         \centering
         \includegraphics[height=0.08\textheight]{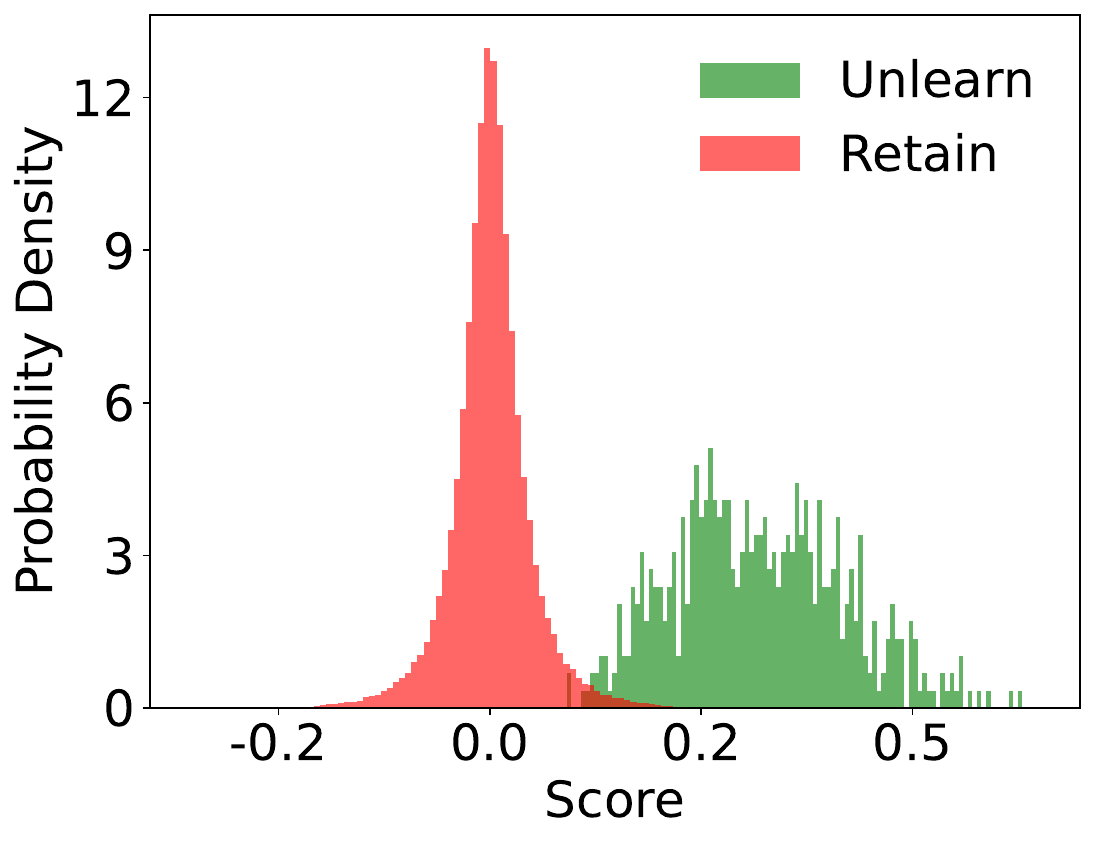}
         \caption{50\% Part-Class}
     \end{subfigure}     
\begin{subfigure}[b]{0.15\linewidth}
         \centering
         \includegraphics[height=0.08\textheight]{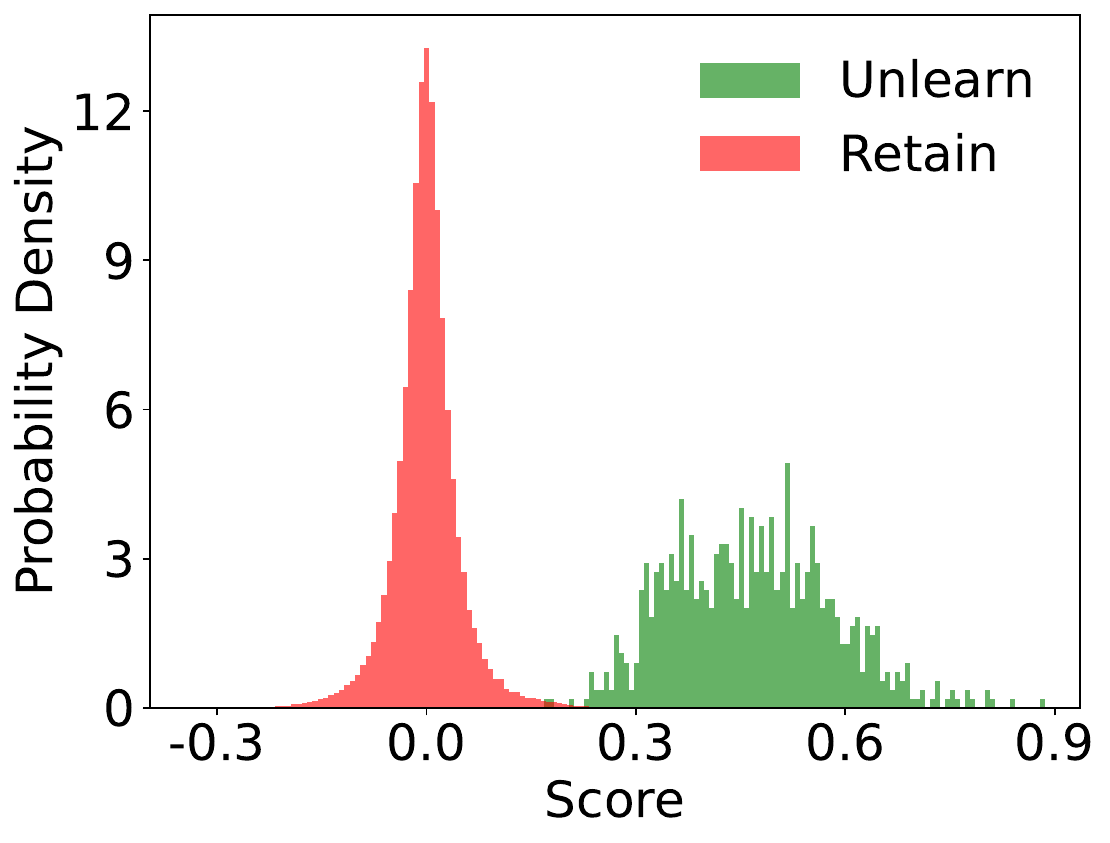}
         \caption{70\% Part-Class}
     \end{subfigure}     
\begin{subfigure}[b]{0.15\linewidth}
         \centering
         \includegraphics[height=0.08\textheight]{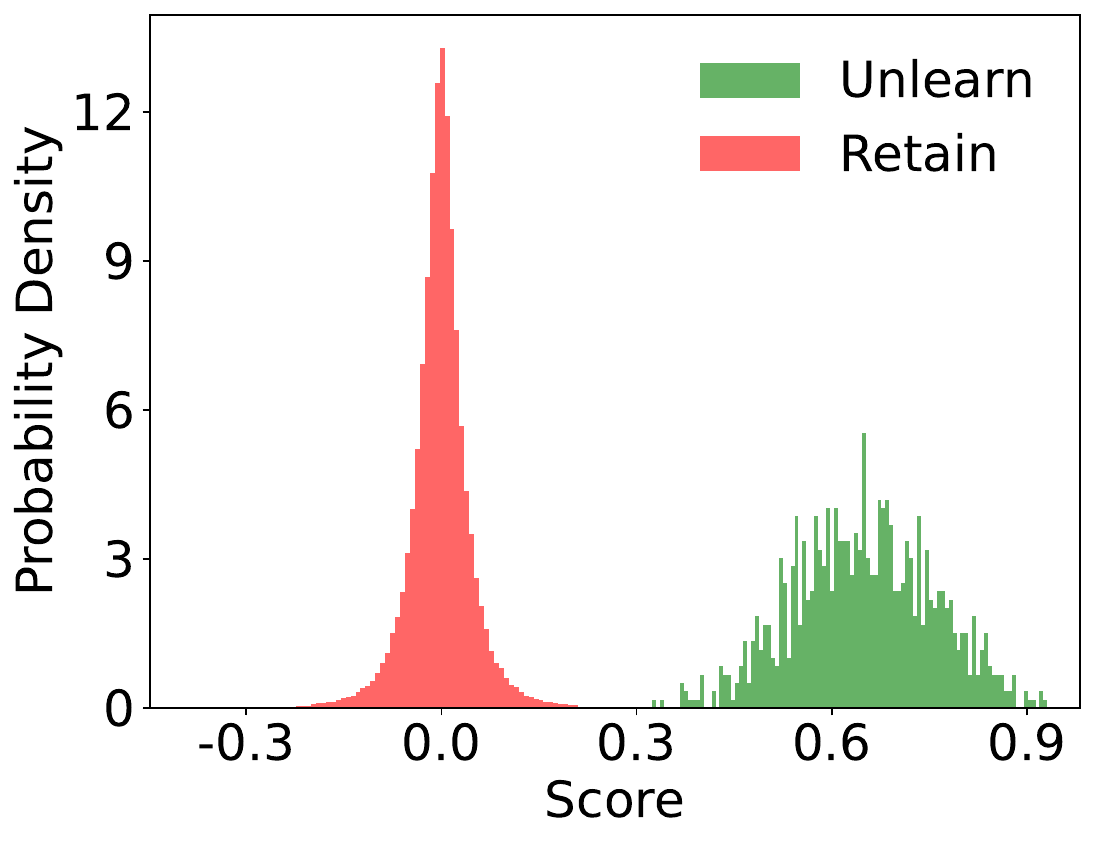}
         \caption{90\% Part-Class}
     \end{subfigure}     
\begin{subfigure}[b]{0.15\linewidth}
         \centering
         \includegraphics[height=0.08\textheight]{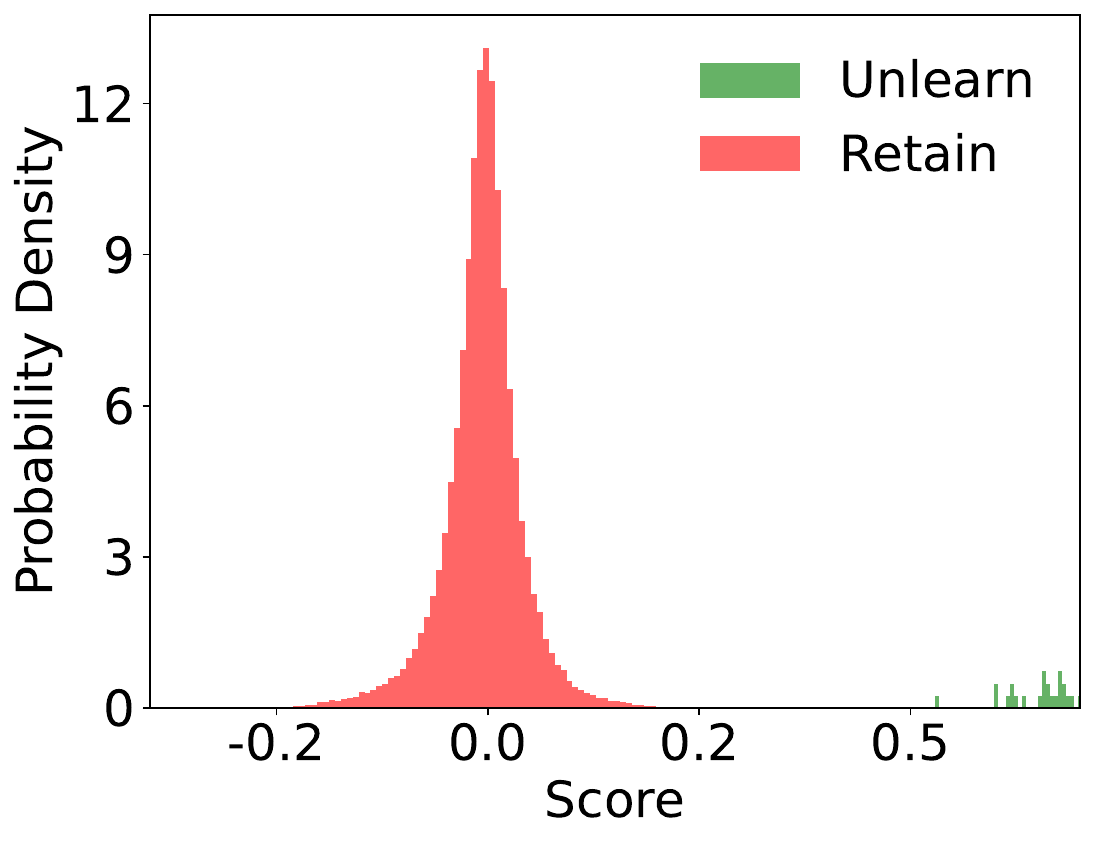}
         \caption{Total Class}
     \end{subfigure}     
     \hfill
\begin{subfigure}[b]{0.16\linewidth}
         \centering
         \includegraphics[height=0.08\textheight]{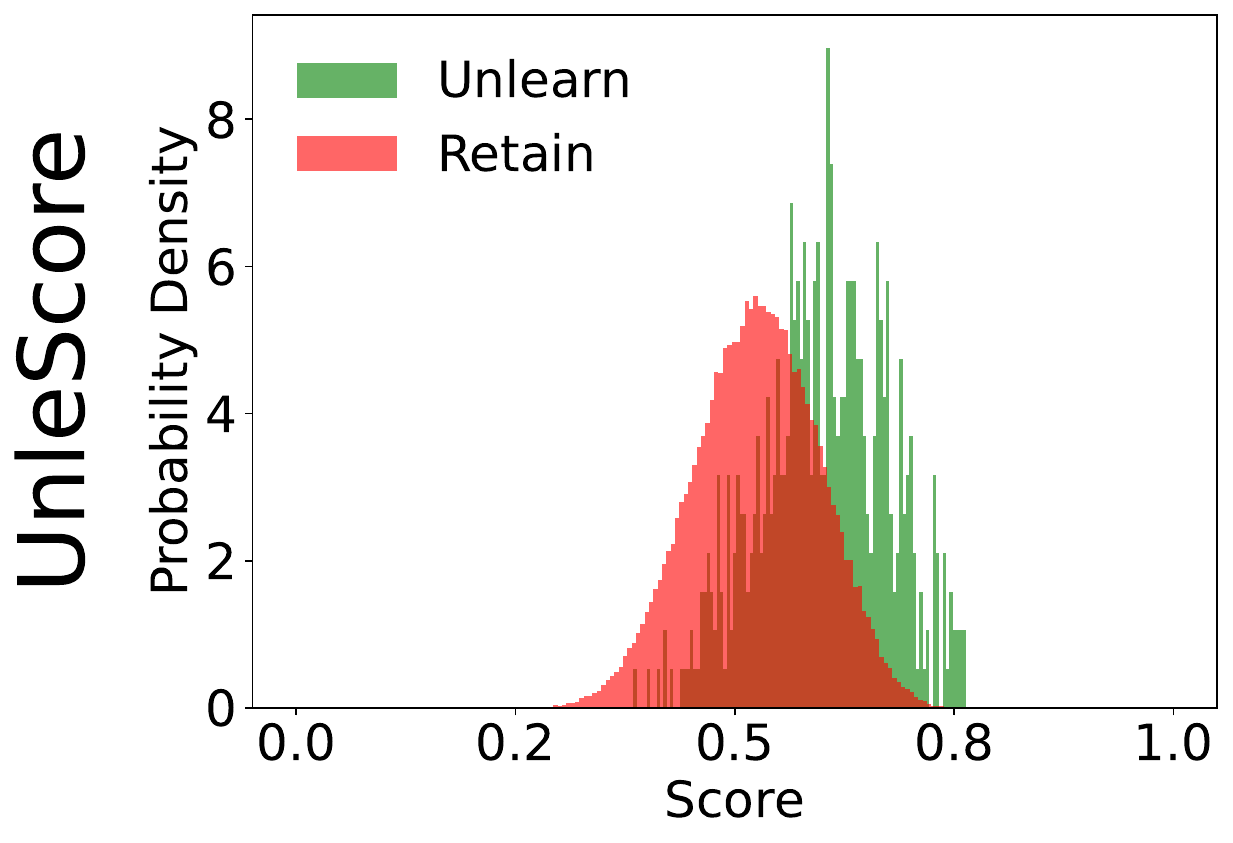}
         \caption{Random Sample}
     \end{subfigure}    
\begin{subfigure}[b]{0.15\linewidth}
         \centering
         \includegraphics[height=0.08\textheight]{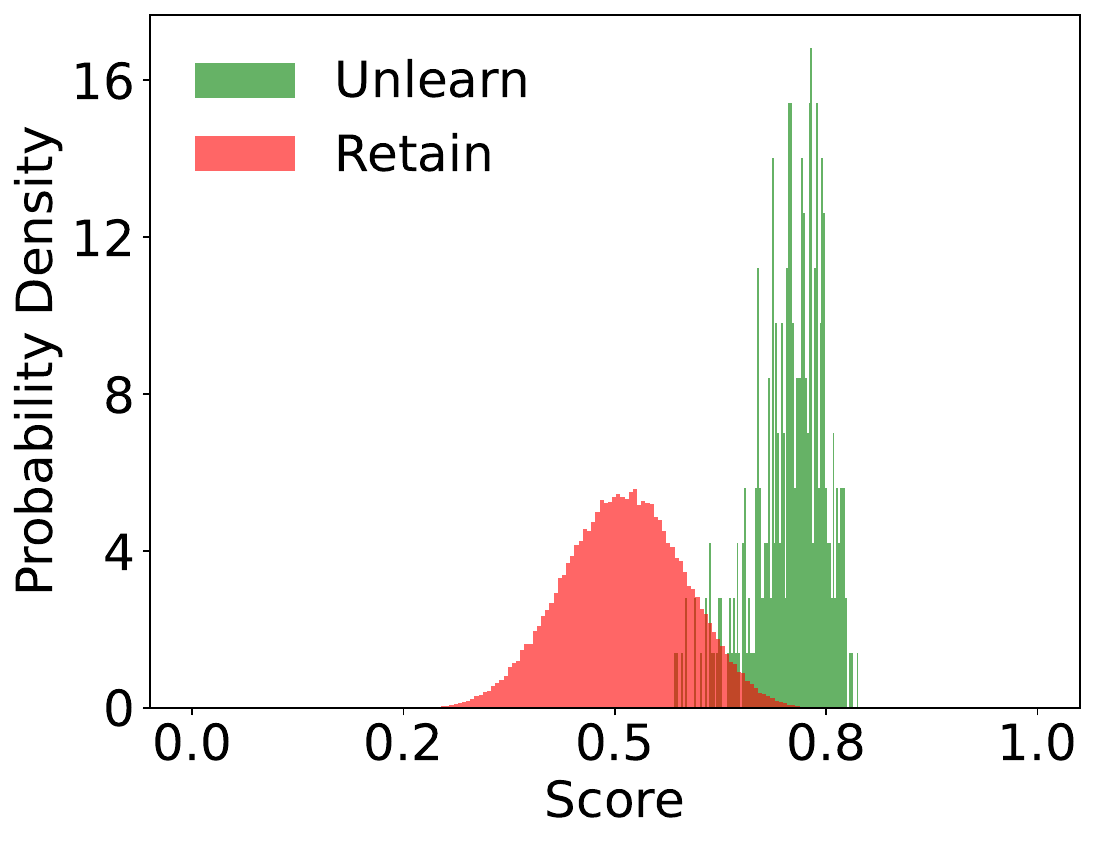}
         \caption{30\% Part-Class}
     \end{subfigure}    
\begin{subfigure}[b]{0.15\linewidth}
         \centering
         \includegraphics[height=0.08\textheight]{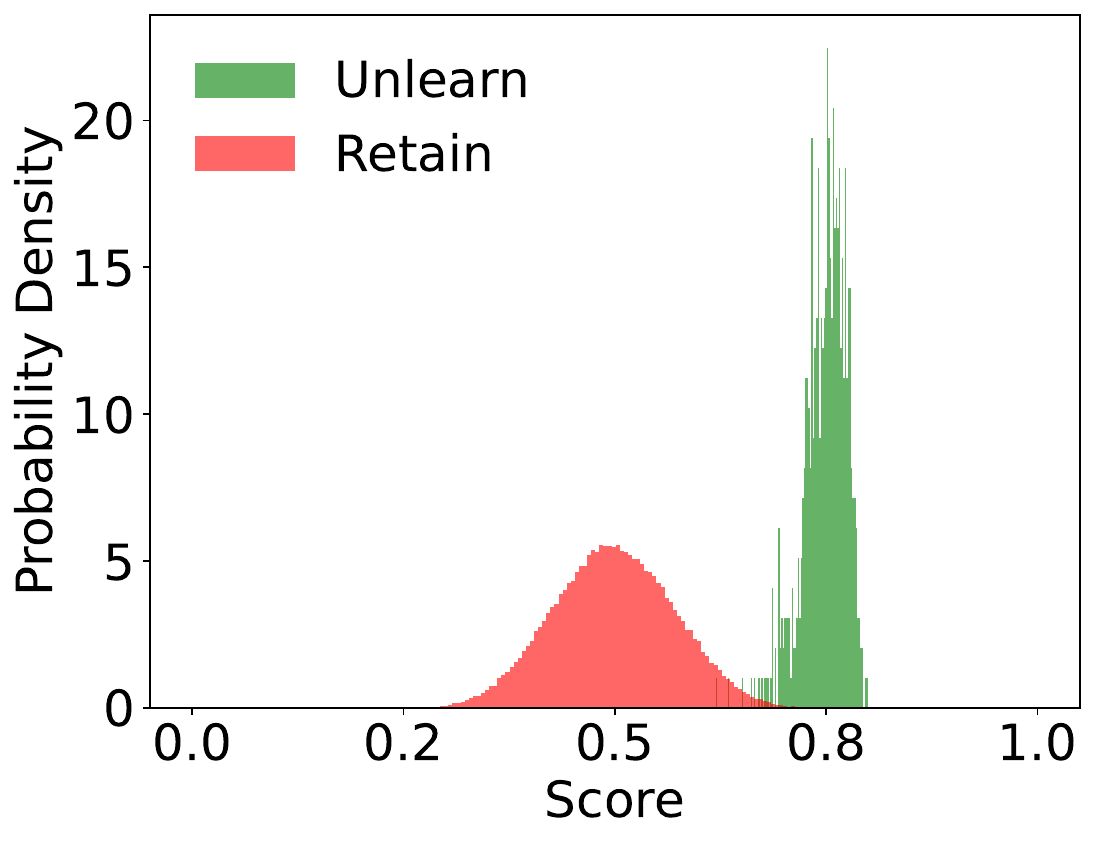}
         \caption{50\% Part-Class}
     \end{subfigure}    
\begin{subfigure}[b]{0.15\linewidth}
         \centering
         \includegraphics[height=0.08\textheight]{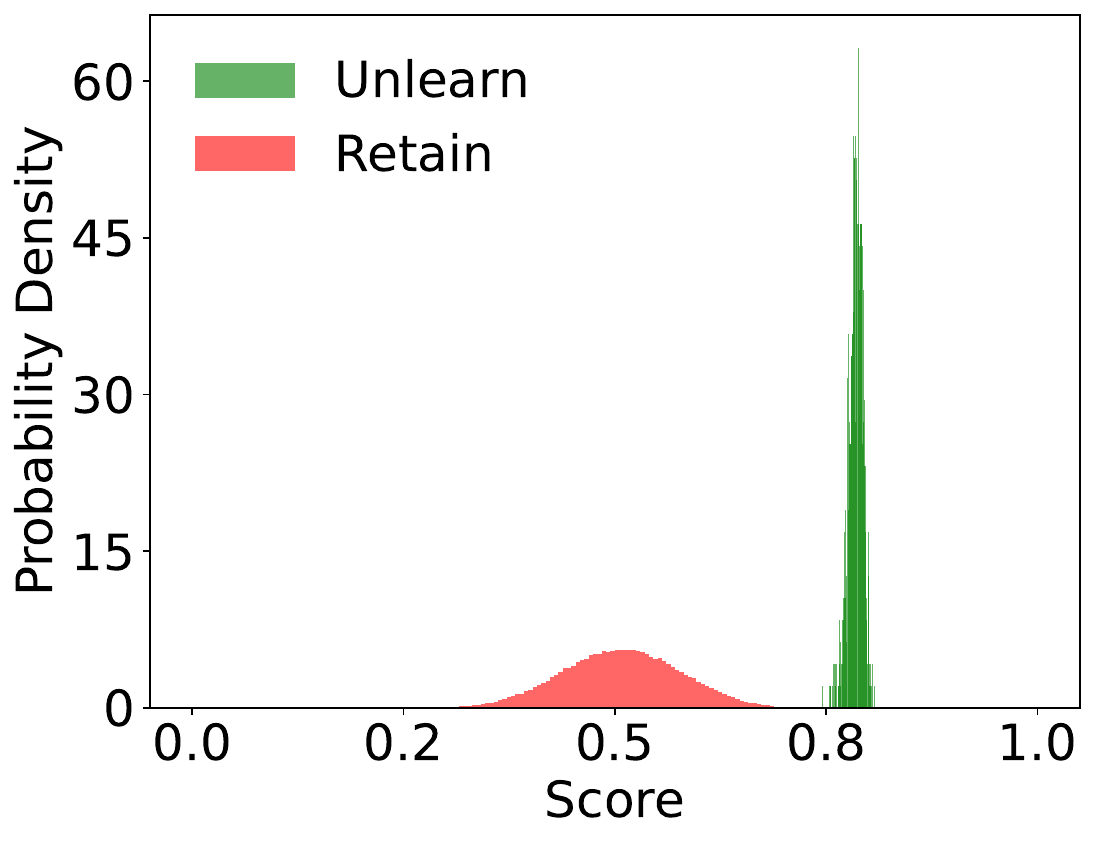}
         \caption{70\% Part-Class}
     \end{subfigure}    
\begin{subfigure}[b]{0.15\linewidth}
         \centering
         \includegraphics[height=0.08\textheight]{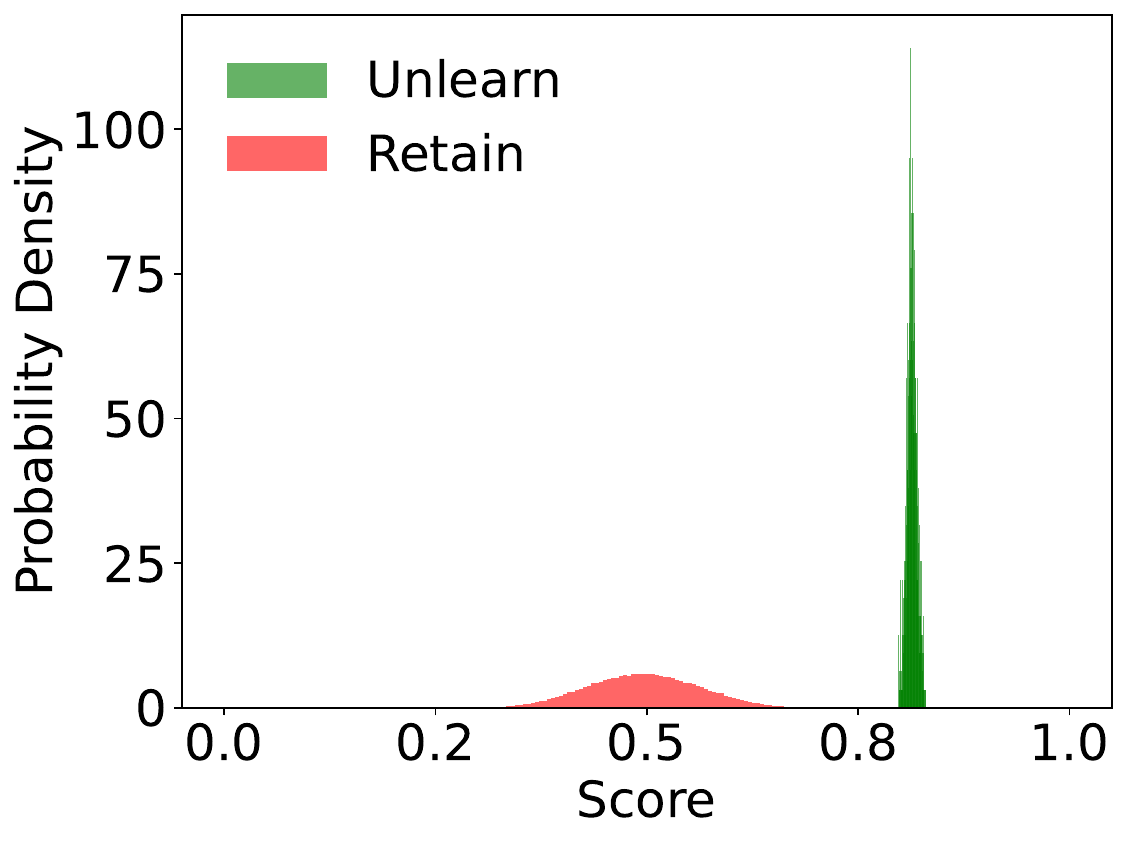}
         \caption{90\% Part-Class}
     \end{subfigure}    
\begin{subfigure}[b]{0.15\linewidth}
         \centering
         \includegraphics[height=0.08\textheight]{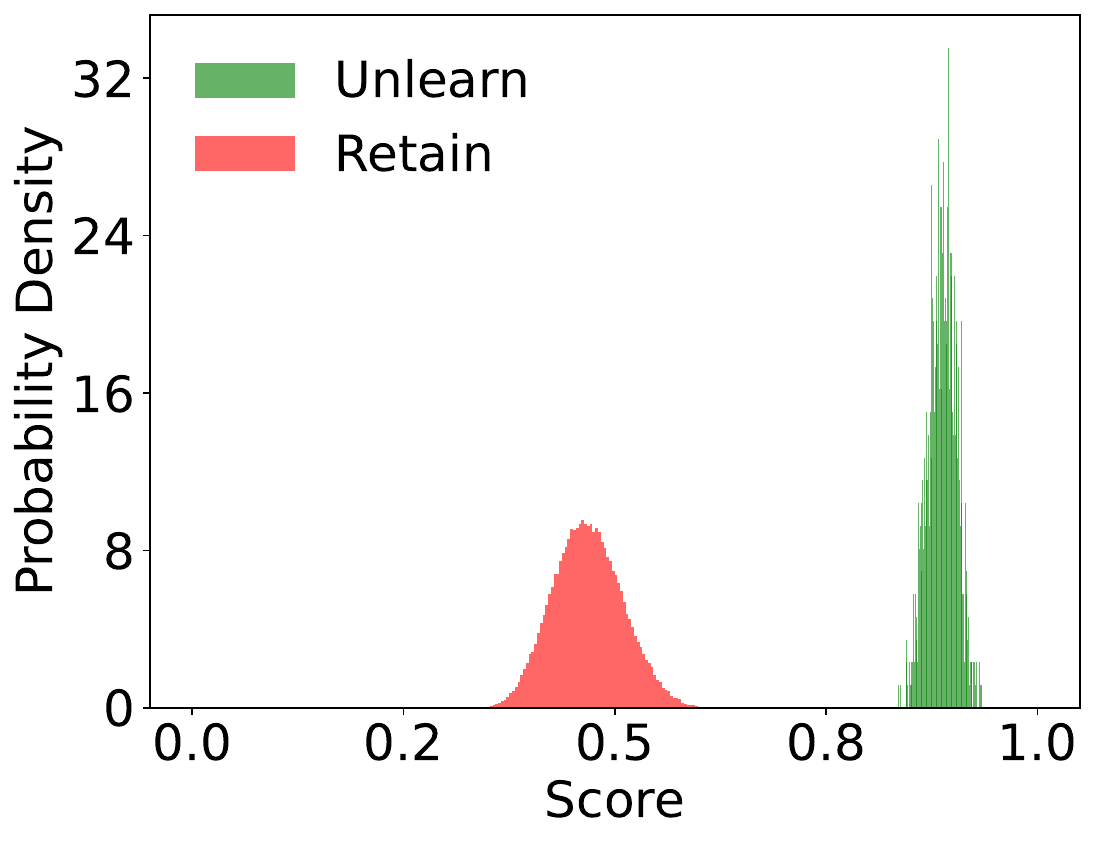}
         \caption{Total Class}
     \end{subfigure}       
     \caption{Score Distributions of Unlearning Metrics with Different Unlearning Tasks on Purchase.}
    \label{fig:unlearningscores_Purchase}
\end{figure*}

\begin{figure*}[h]
    \centering
    \begin{subfigure}[b]{0.16\linewidth}
         \centering
         \includegraphics[height=0.08\textheight]{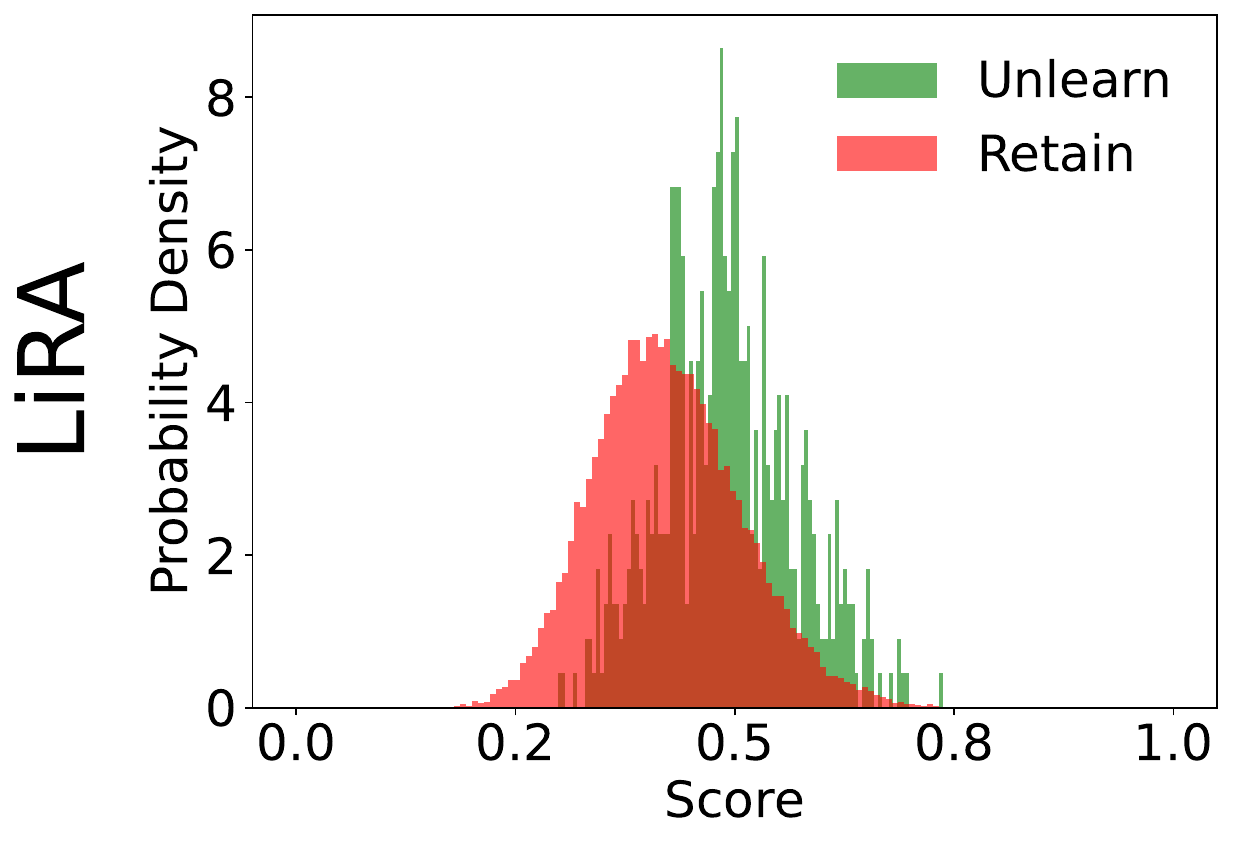}
         \caption{Random Sample}
     \end{subfigure}
     \begin{subfigure}[b]{0.15\linewidth}
         \centering
         \includegraphics[height=0.08\textheight]{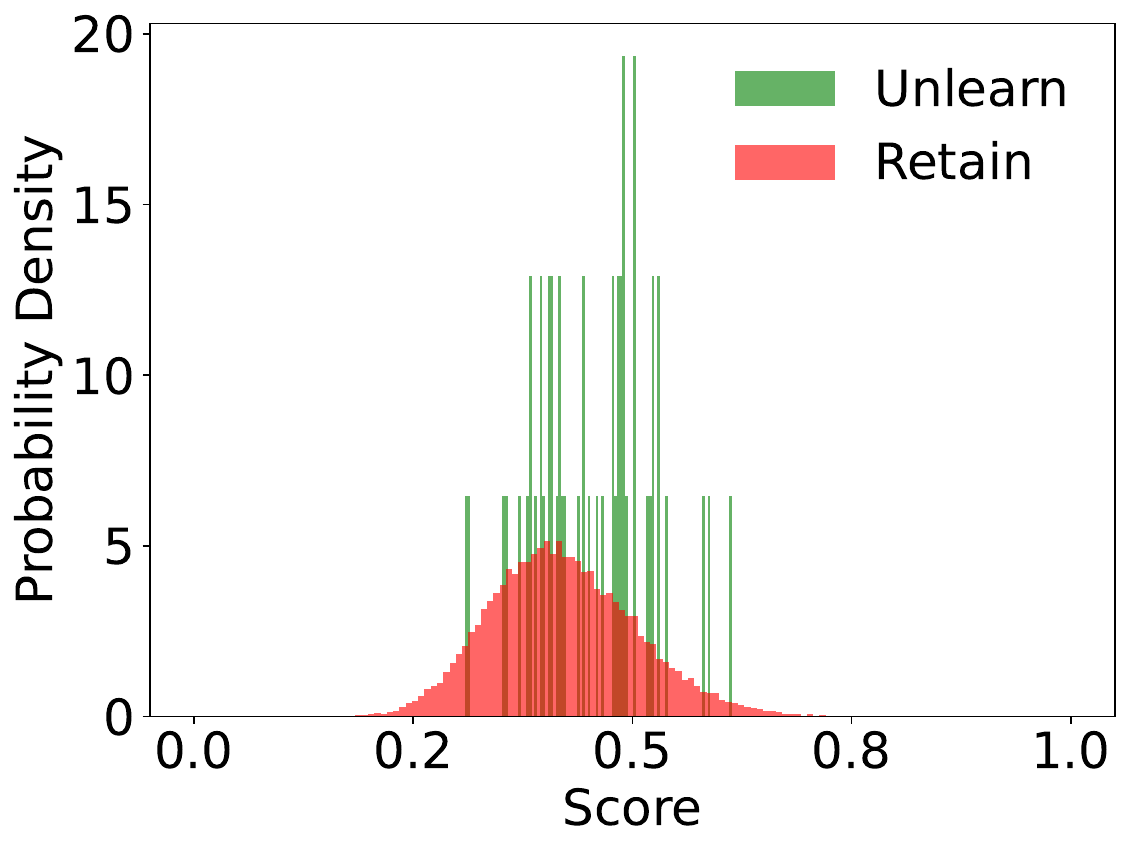}
         \caption{30\% Part-Class}
     \end{subfigure}
     \begin{subfigure}[b]{0.15\linewidth}
         \centering
         \includegraphics[height=0.08\textheight]{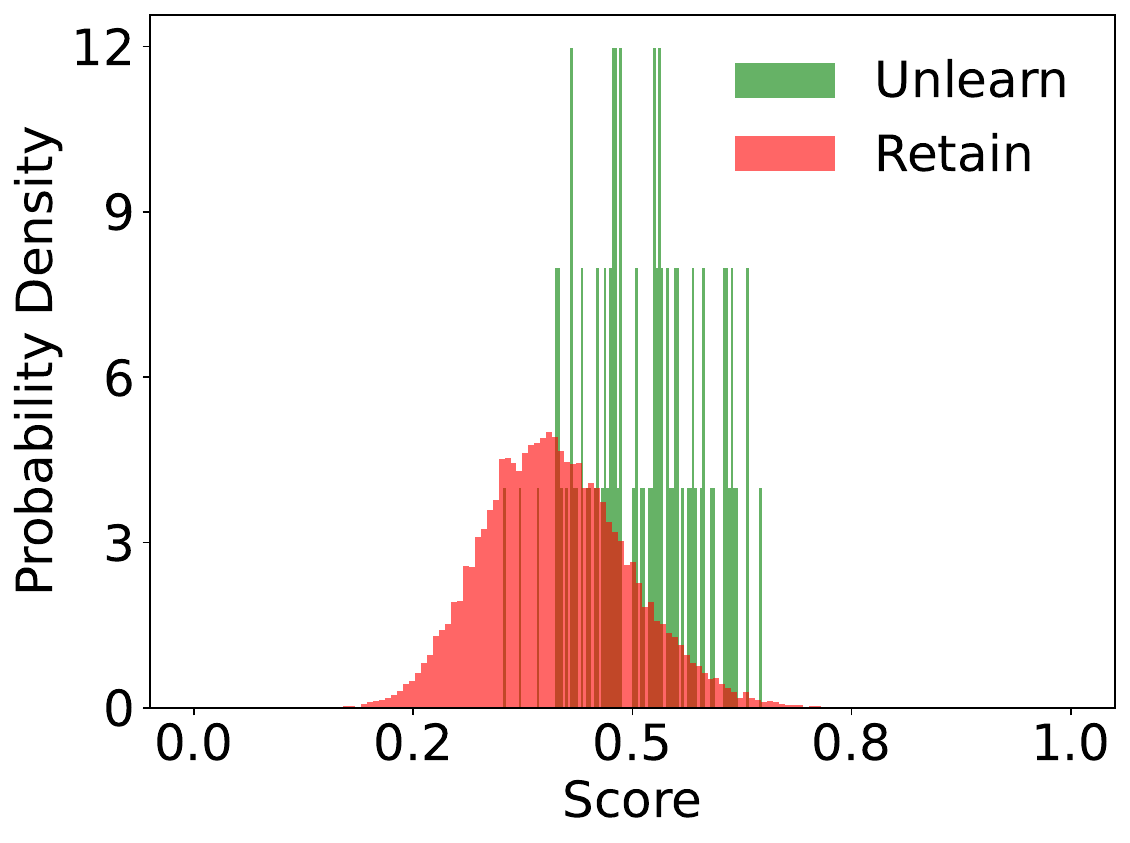}
         \caption{50\% Part-Class}
     \end{subfigure}
     \begin{subfigure}[b]{0.15\linewidth}
         \centering
         \includegraphics[height=0.08\textheight]{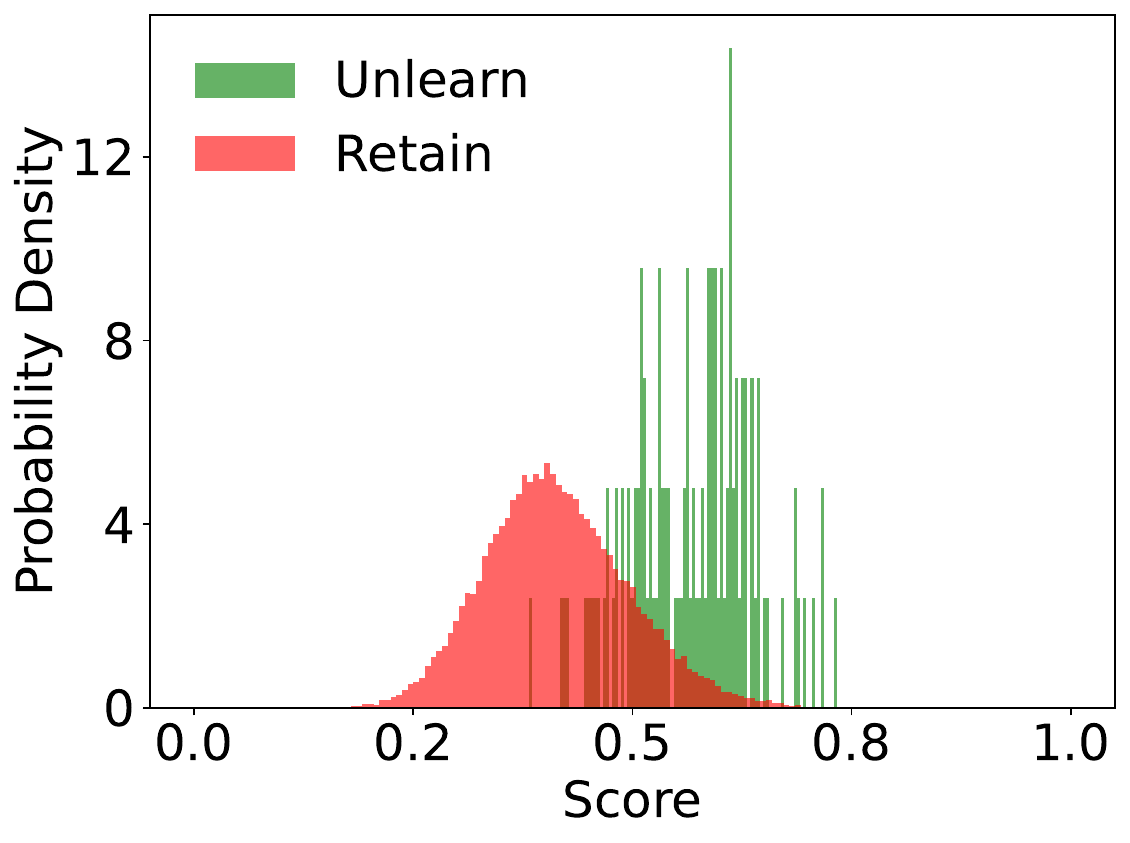}
         \caption{70\% Part-Class}
     \end{subfigure}
     \begin{subfigure}[b]{0.15\linewidth}
         \centering
         \includegraphics[height=0.08\textheight]{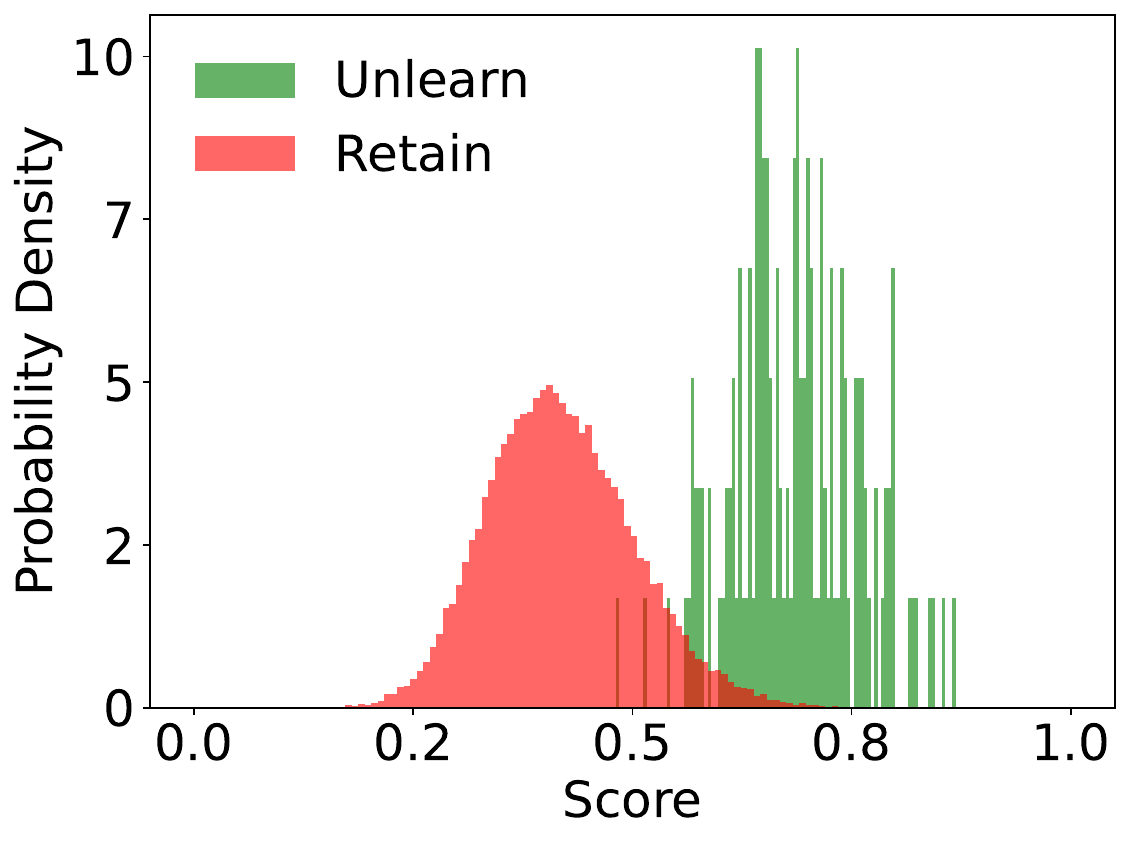}
         \caption{90\% Part-Class}
     \end{subfigure}
     \begin{subfigure}[b]{0.15\linewidth}
         \centering
         \includegraphics[height=0.08\textheight]{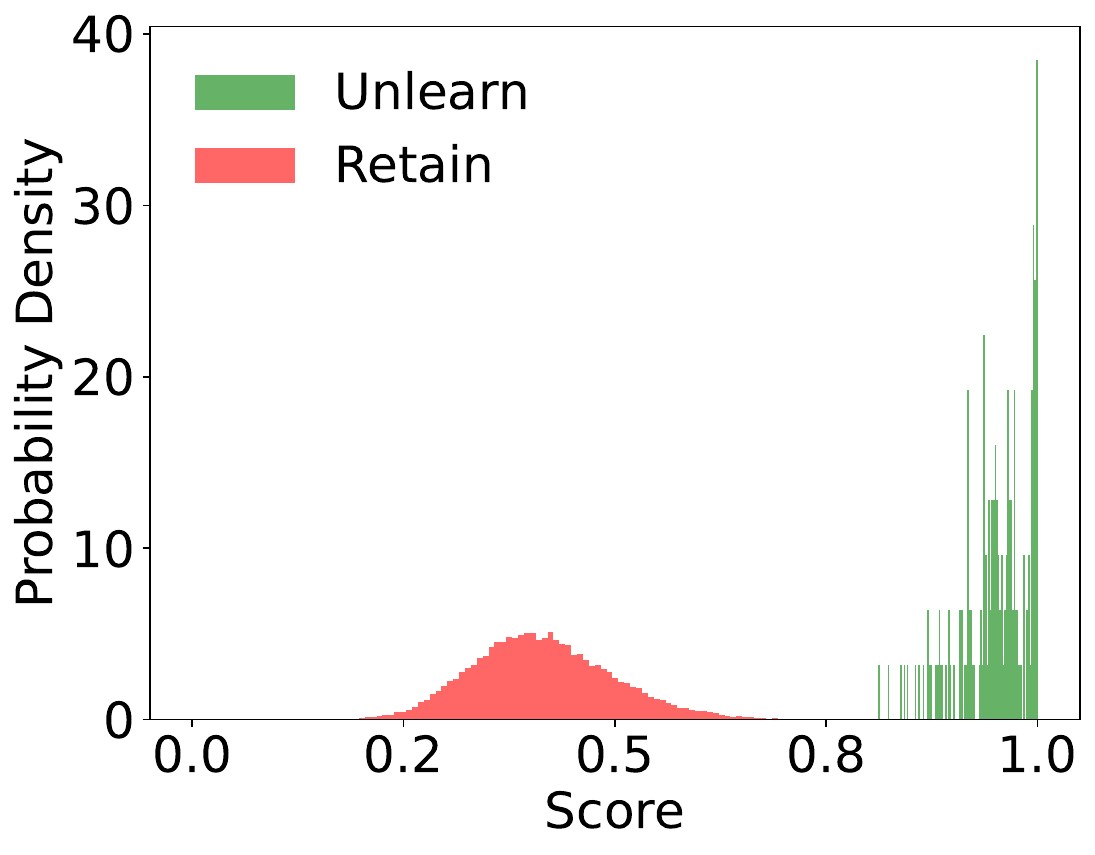}
         \caption{Total Class}
     \end{subfigure}
     \hfill
     \begin{subfigure}[b]{0.16\linewidth}
         \centering
         \includegraphics[height=0.08\textheight]{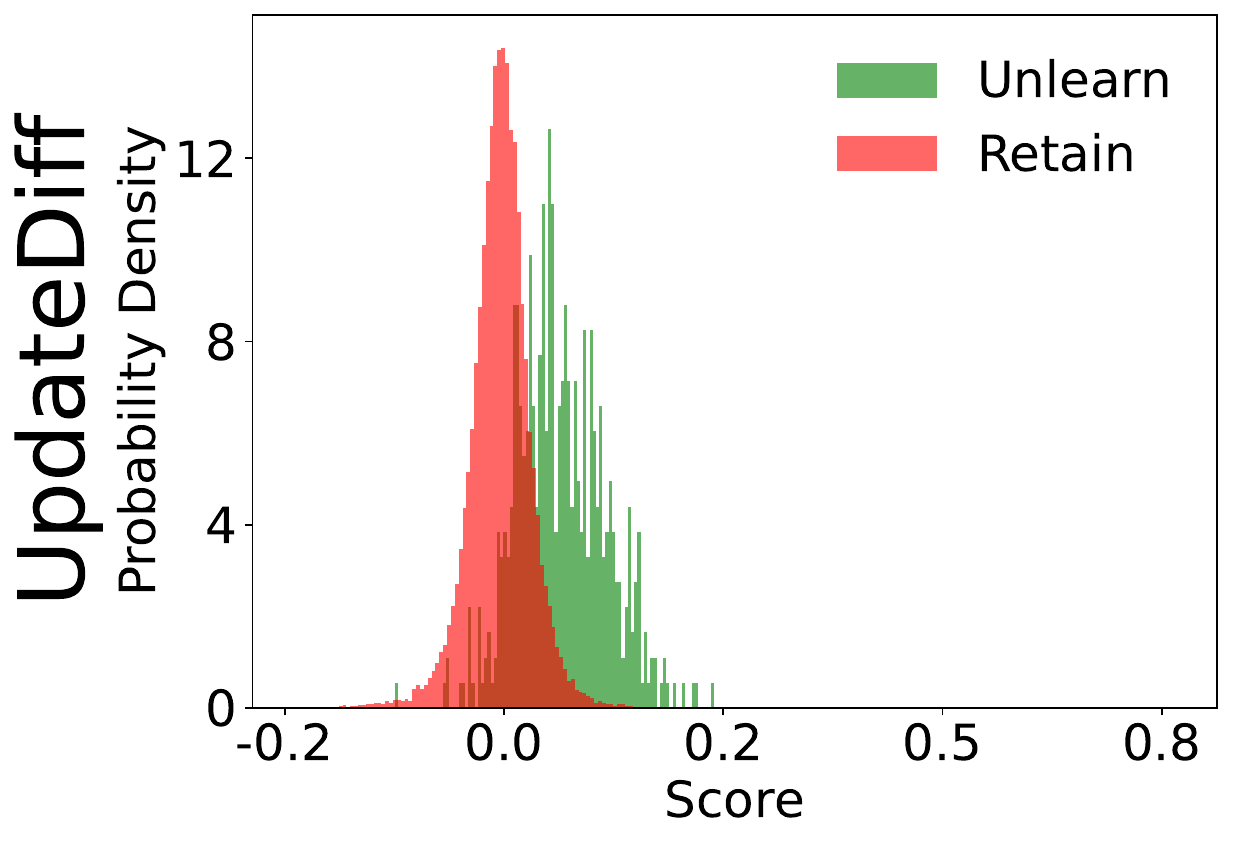}
         \caption{Random Sample}
     \end{subfigure}     
\begin{subfigure}[b]{0.15\linewidth}
         \centering
         \includegraphics[height=0.08\textheight]{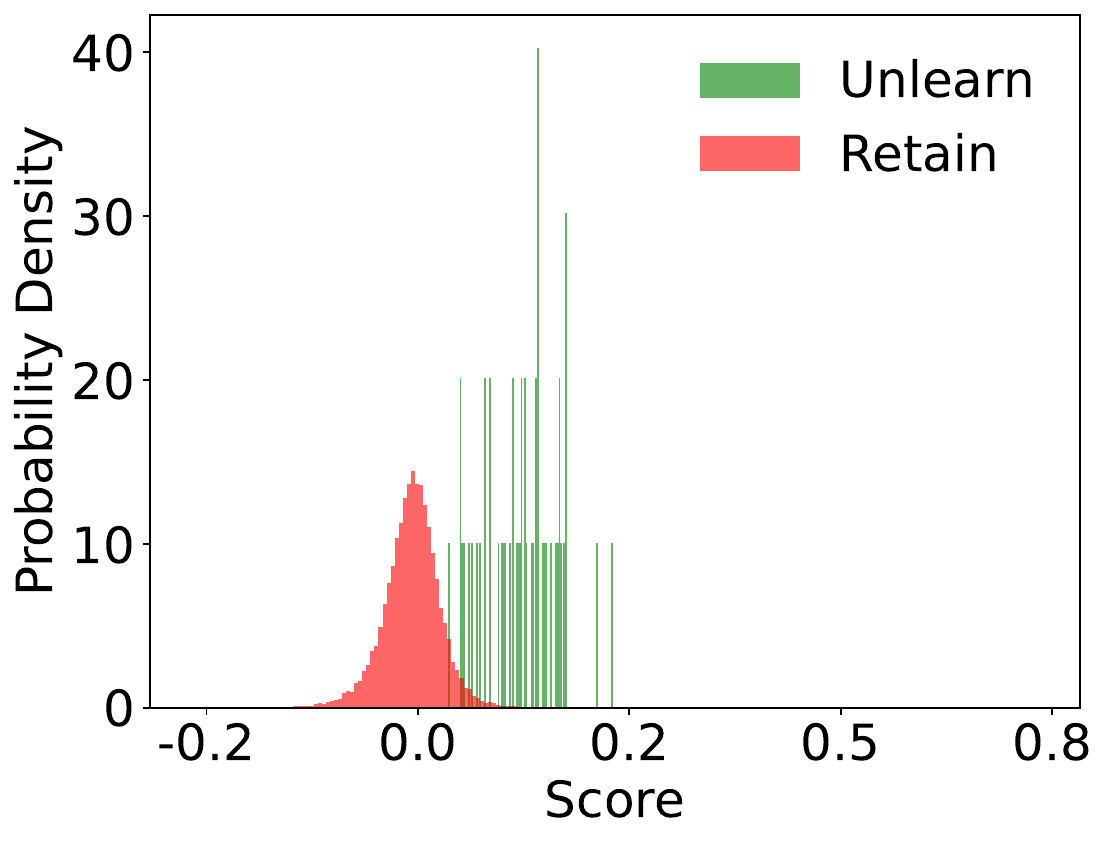}
         \caption{30\% Part-Class}
     \end{subfigure}     
\begin{subfigure}[b]{0.15\linewidth}
         \centering
         \includegraphics[height=0.08\textheight]{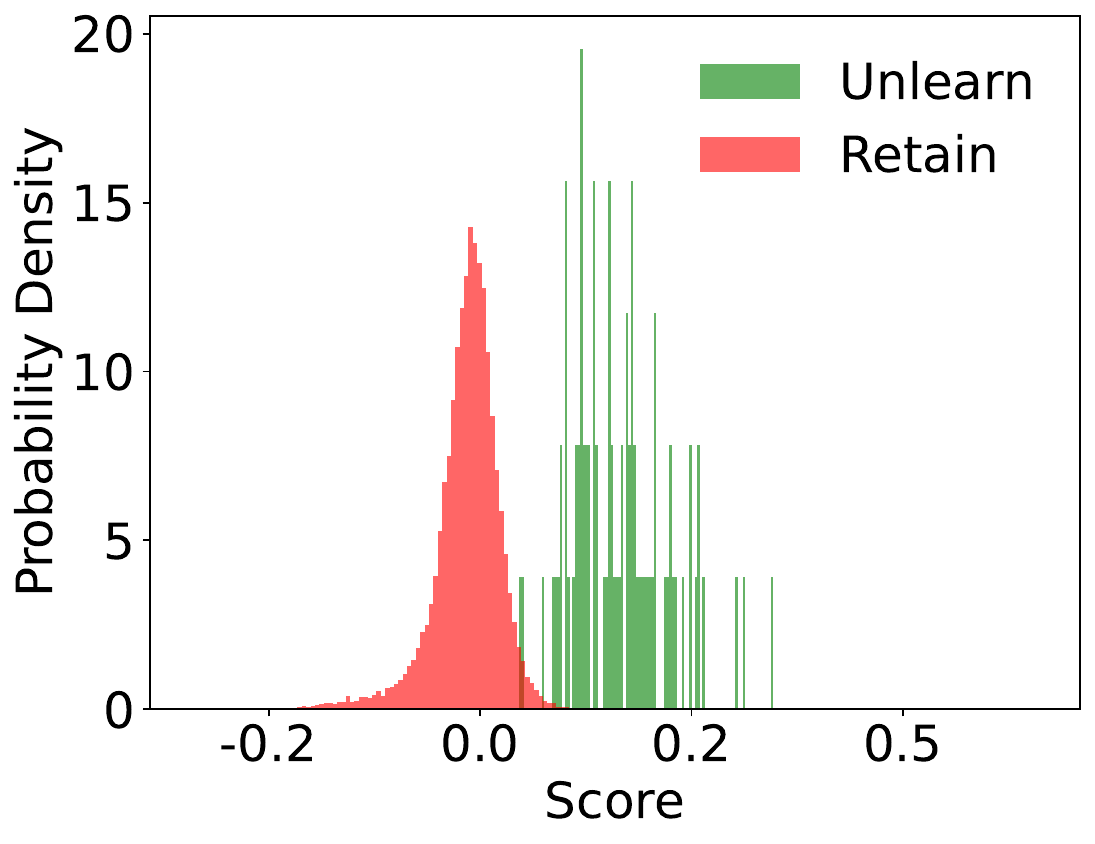}
         \caption{50\% Part-Class}
     \end{subfigure}     
\begin{subfigure}[b]{0.15\linewidth}
         \centering
         \includegraphics[height=0.08\textheight]{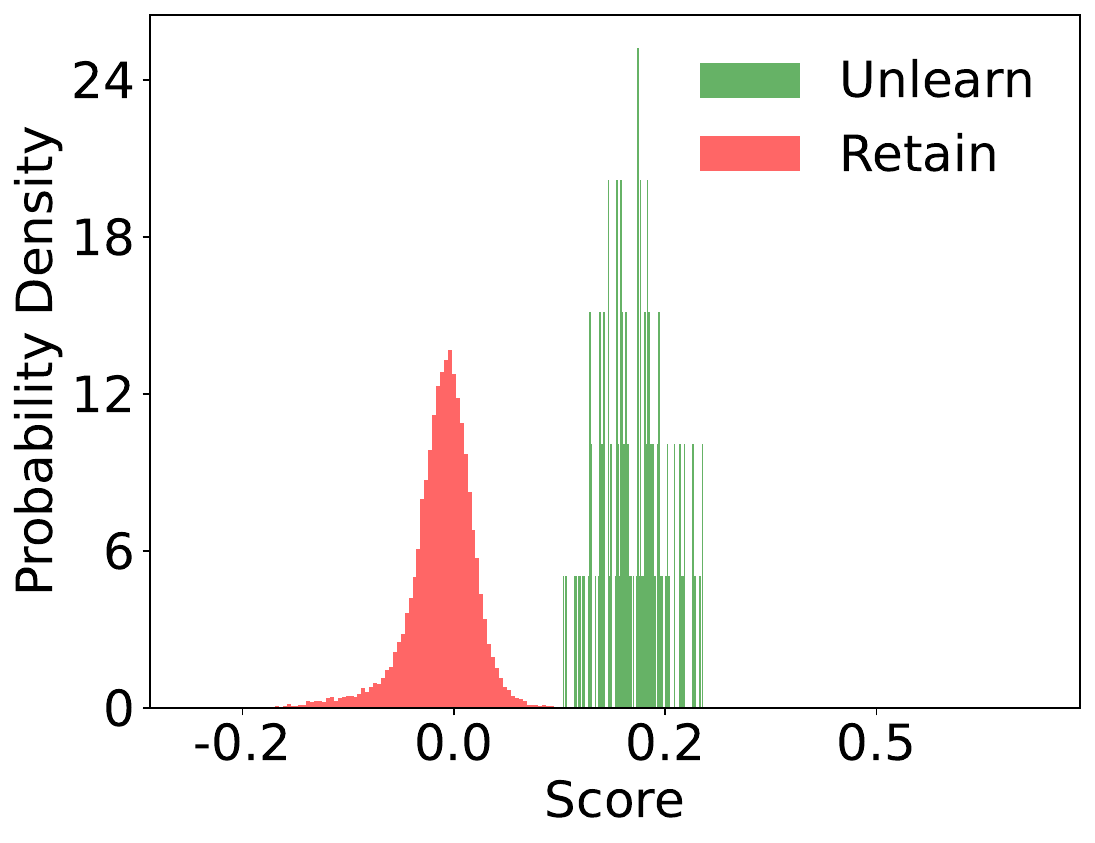}
         \caption{70\% Part-Class}
     \end{subfigure}     
\begin{subfigure}[b]{0.15\linewidth}
         \centering
         \includegraphics[height=0.08\textheight]{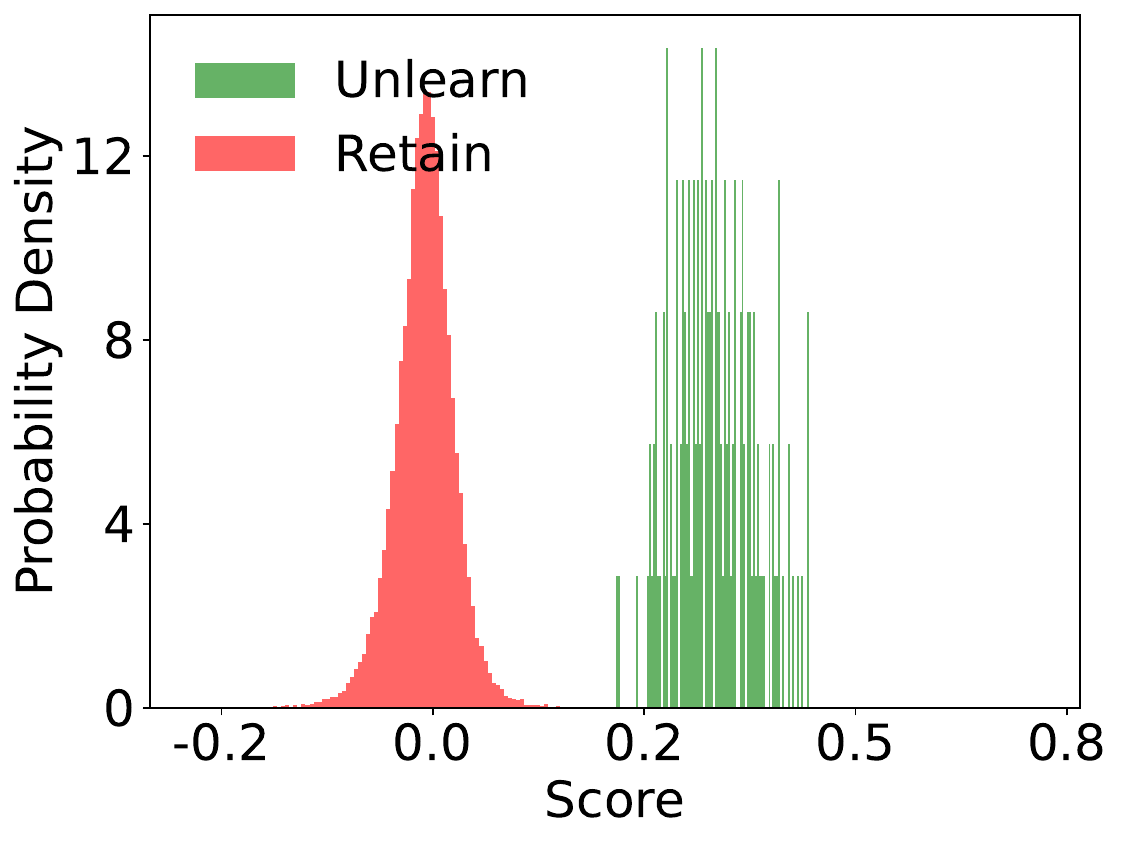}
         \caption{90\% Part-Class}
     \end{subfigure}     
\begin{subfigure}[b]{0.15\linewidth}
         \centering
         \includegraphics[height=0.08\textheight]{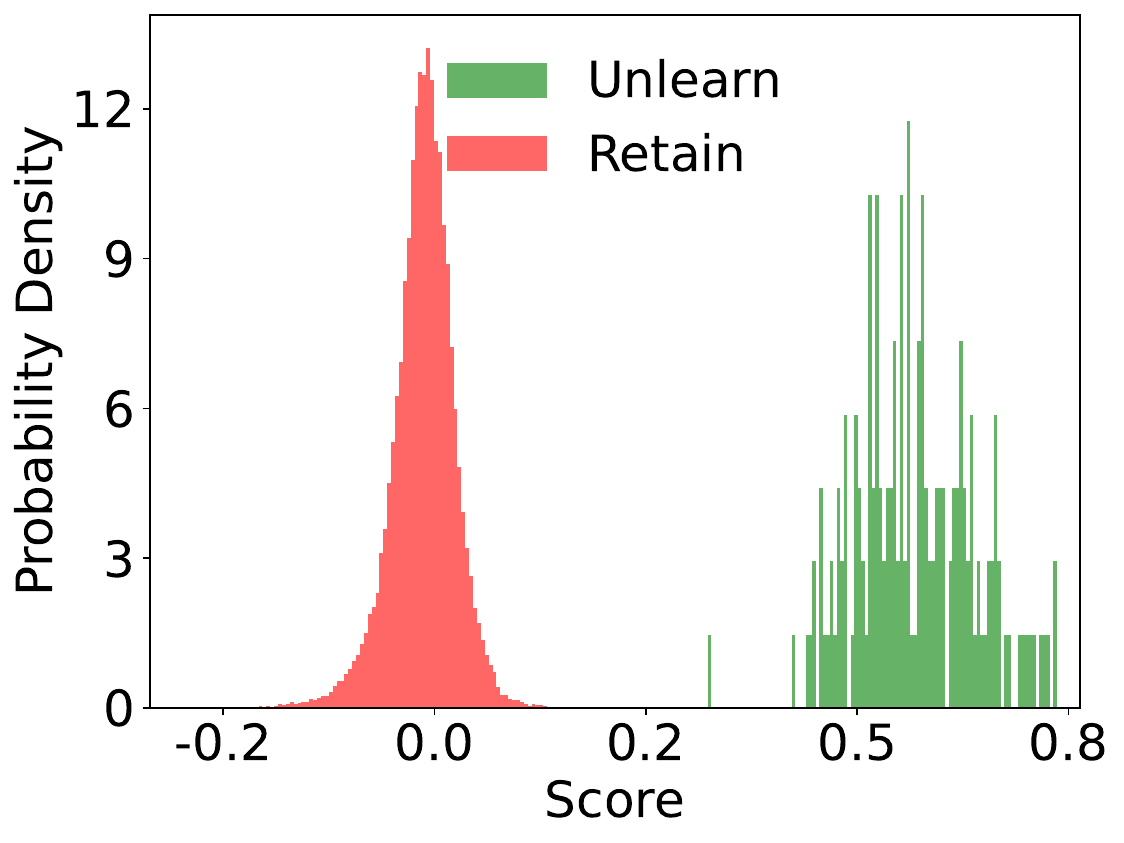}
         \caption{Total Class}
     \end{subfigure}     
     \hfill
\begin{subfigure}[b]{0.16\linewidth}
         \centering
         \includegraphics[height=0.08\textheight]{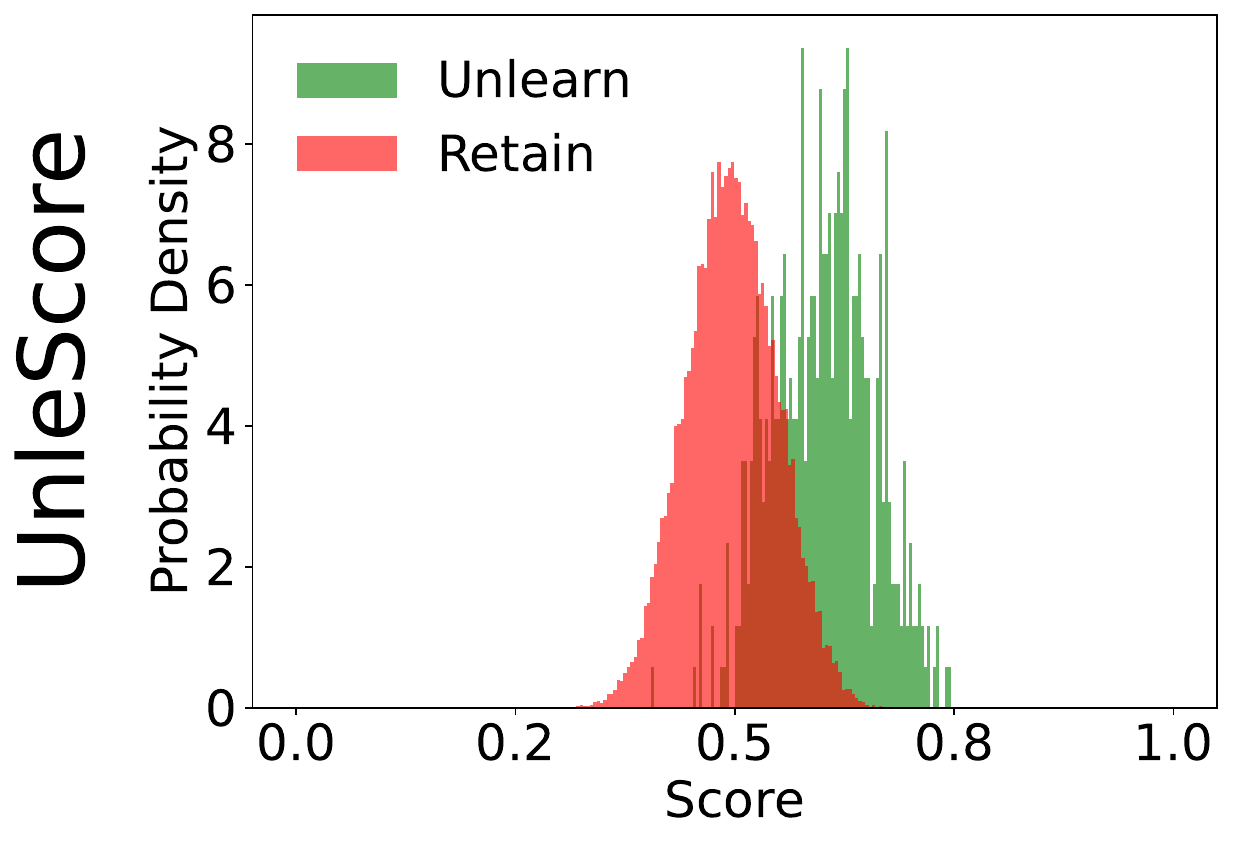}
         \caption{Random Sample}
     \end{subfigure}    
\begin{subfigure}[b]{0.15\linewidth}
         \centering
         \includegraphics[height=0.08\textheight]{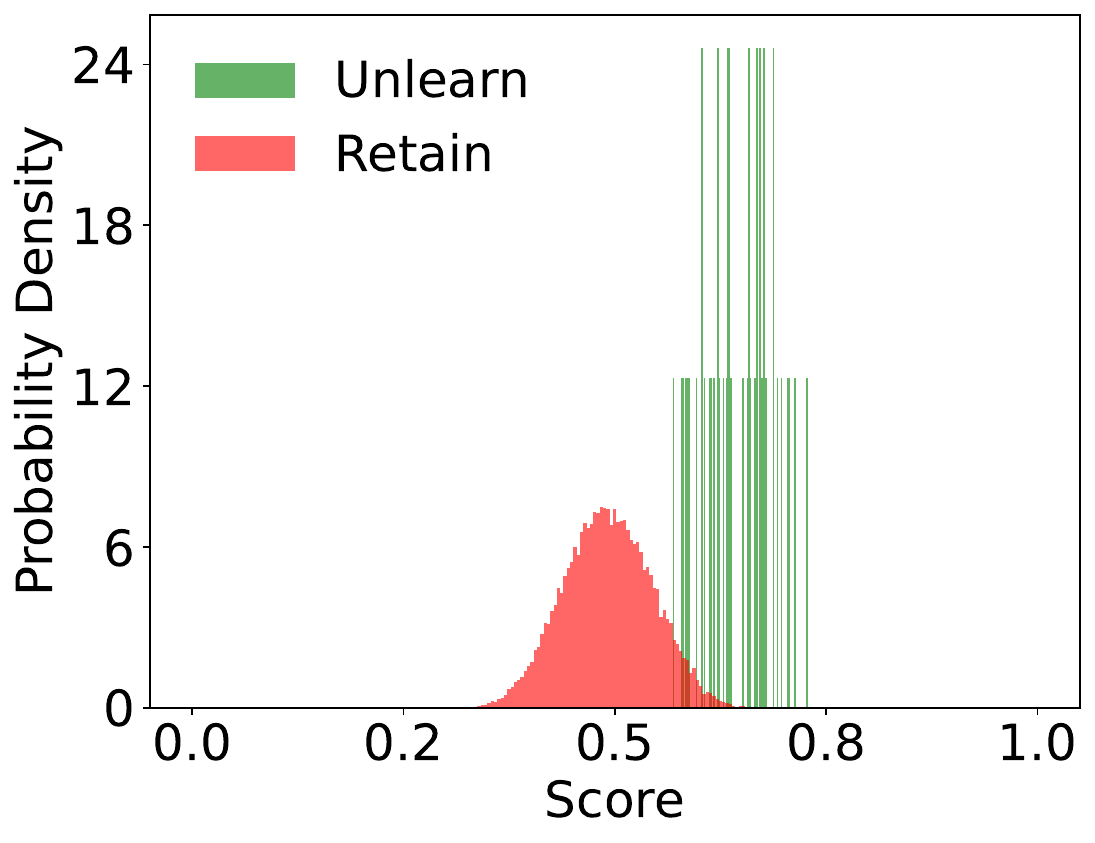}
         \caption{30\% Part-Class}
     \end{subfigure}    
\begin{subfigure}[b]{0.15\linewidth}
         \centering
         \includegraphics[height=0.08\textheight]{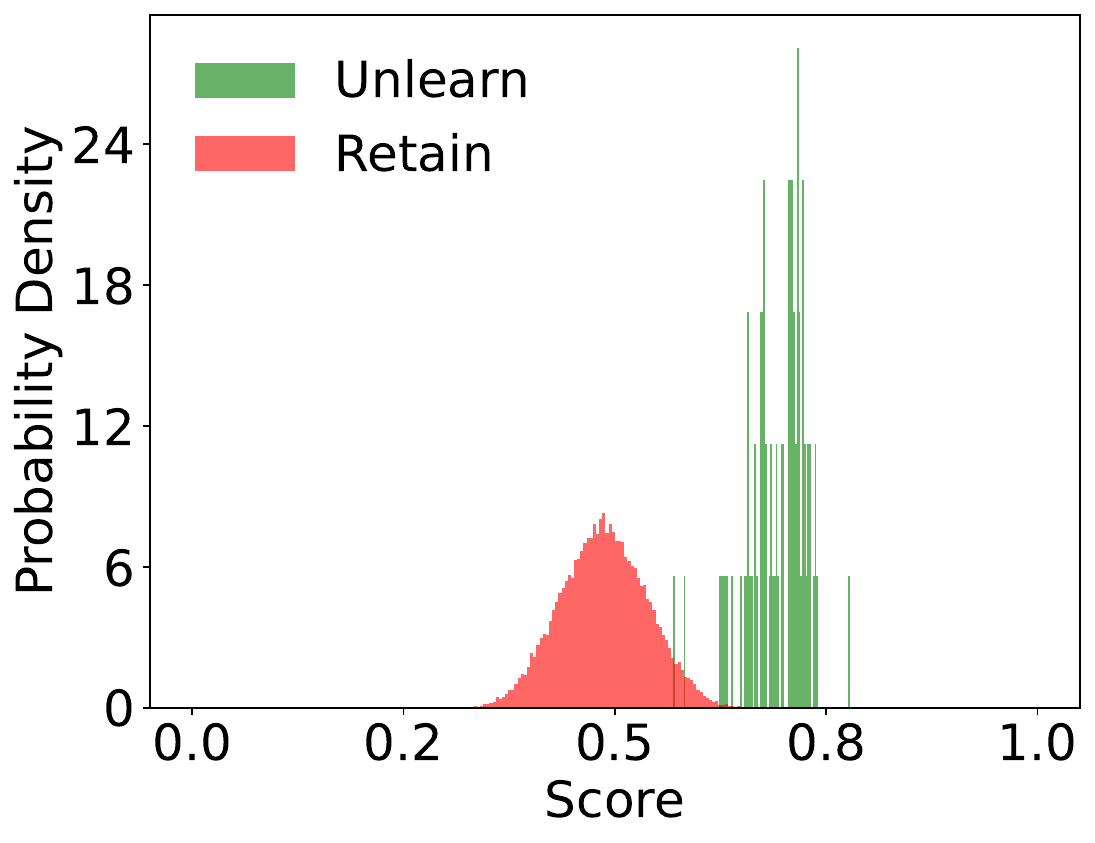}
         \caption{50\% Part-Class}
     \end{subfigure}    
\begin{subfigure}[b]{0.15\linewidth}
         \centering
         \includegraphics[height=0.08\textheight]{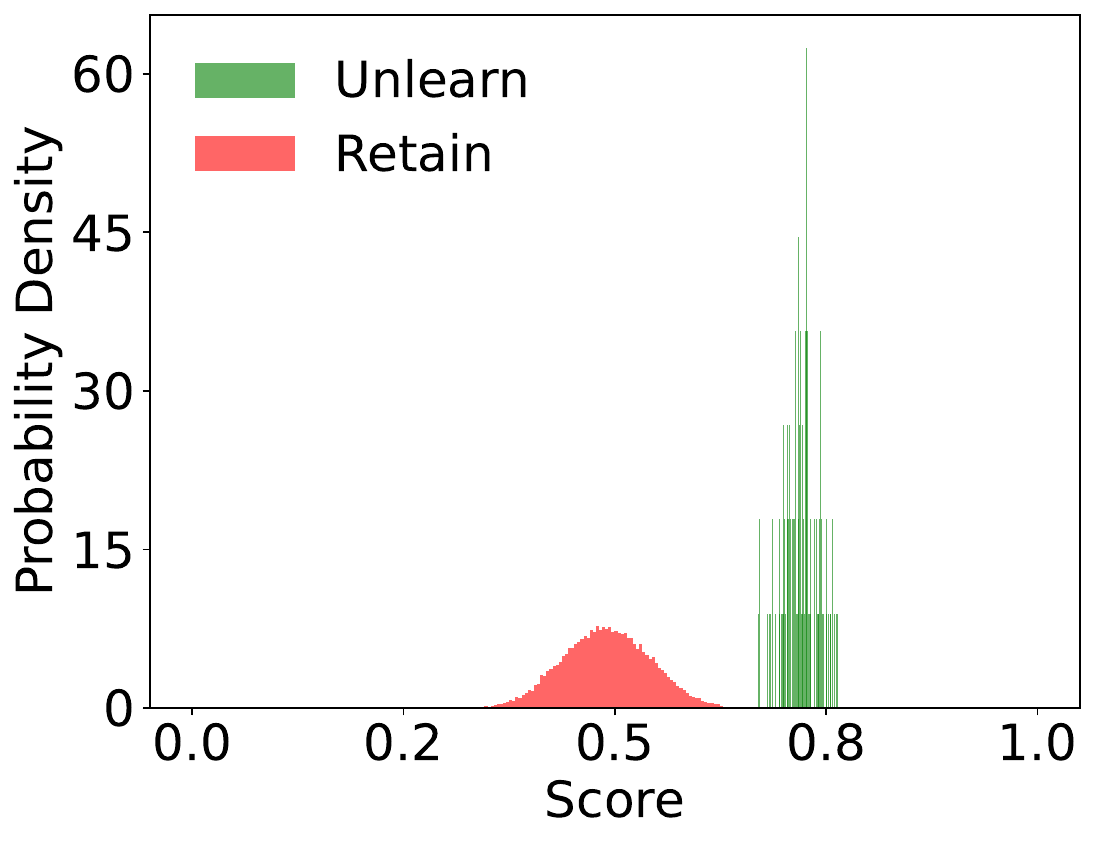}
         \caption{70\% Part-Class}
     \end{subfigure}    
\begin{subfigure}[b]{0.15\linewidth}
         \centering
         \includegraphics[height=0.08\textheight]{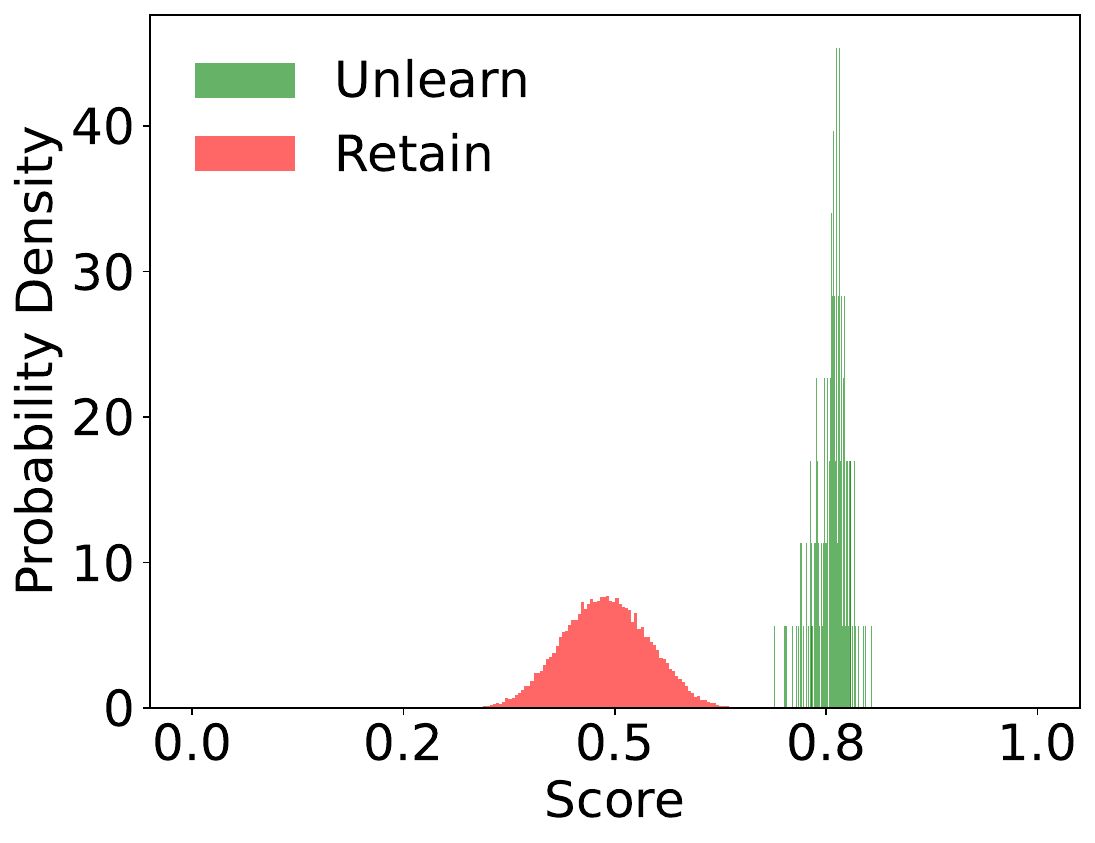}
         \caption{90\% Part-Class}
     \end{subfigure}    
\begin{subfigure}[b]{0.15\linewidth}
         \centering
         \includegraphics[height=0.08\textheight]{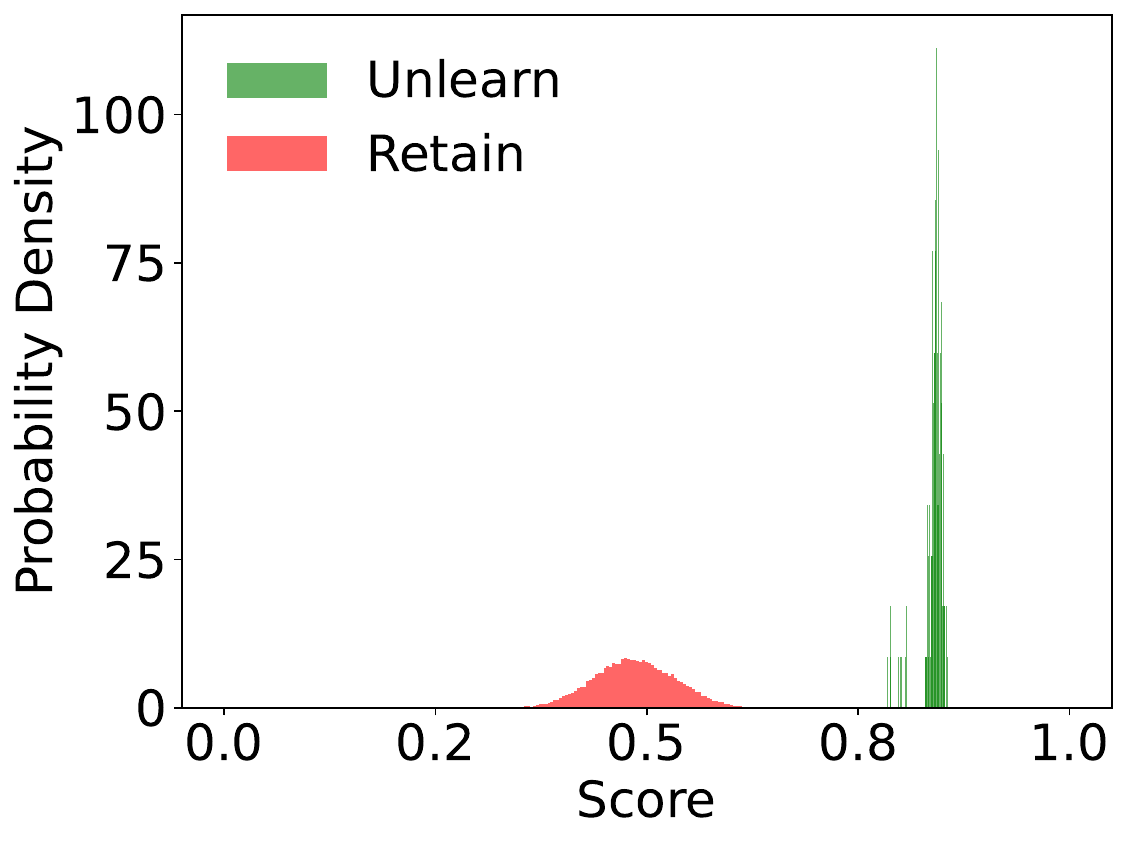}
         \caption{Total Class}
     \end{subfigure}       
     \caption{Score Distributions of Unlearning Metrics with Different Unlearning Tasks on Texas.}
    \label{fig:unlearningscores_Texas}
\end{figure*}

\begin{figure*}[h]
    \centering
    \begin{subfigure}[b]{0.16\linewidth}
         \centering
         \includegraphics[height=0.08\textheight]{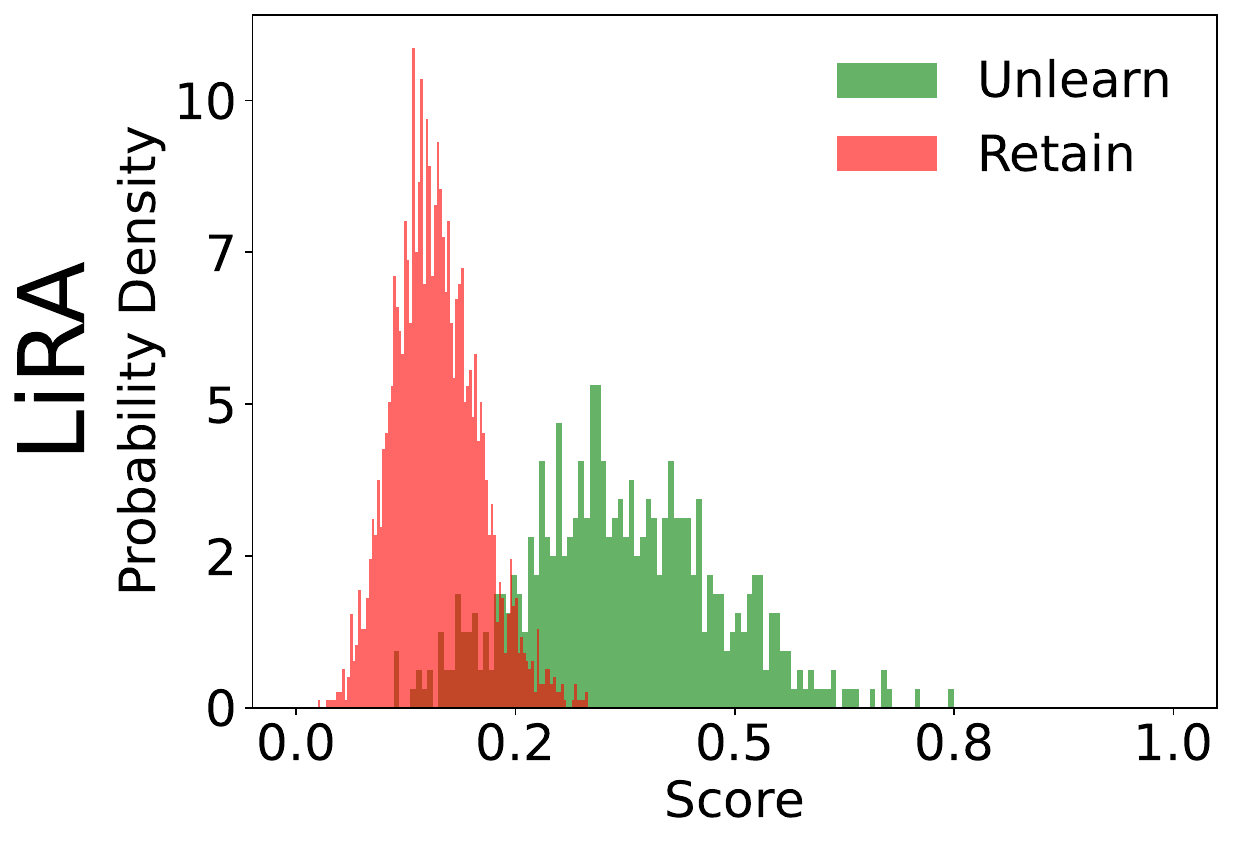}
         \caption{Random Sample}
     \end{subfigure}
     \begin{subfigure}[b]{0.15\linewidth}
         \centering
         \includegraphics[height=0.08\textheight]{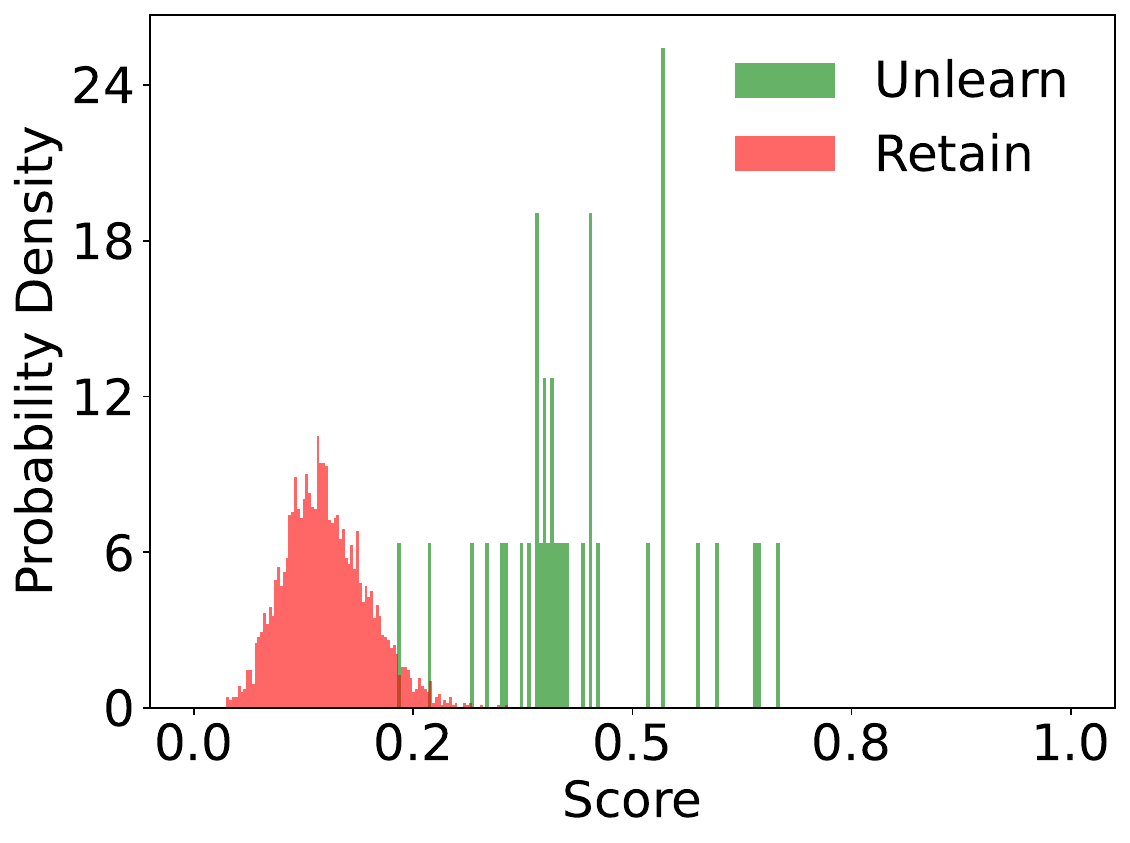}
         \caption{30\% Part-Class}
     \end{subfigure}
     \begin{subfigure}[b]{0.15\linewidth}
         \centering
         \includegraphics[height=0.08\textheight]{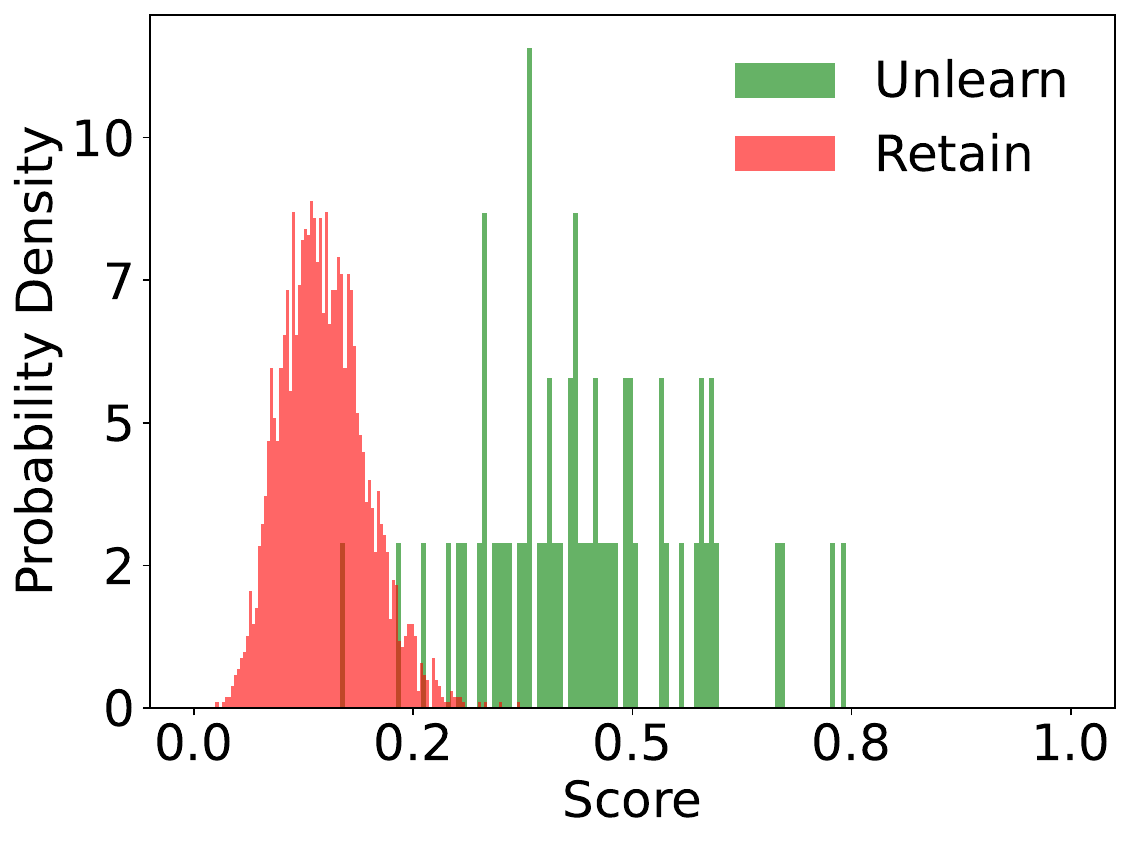}
         \caption{50\% Part-Class}
     \end{subfigure}
     \begin{subfigure}[b]{0.15\linewidth}
         \centering
         \includegraphics[height=0.08\textheight]{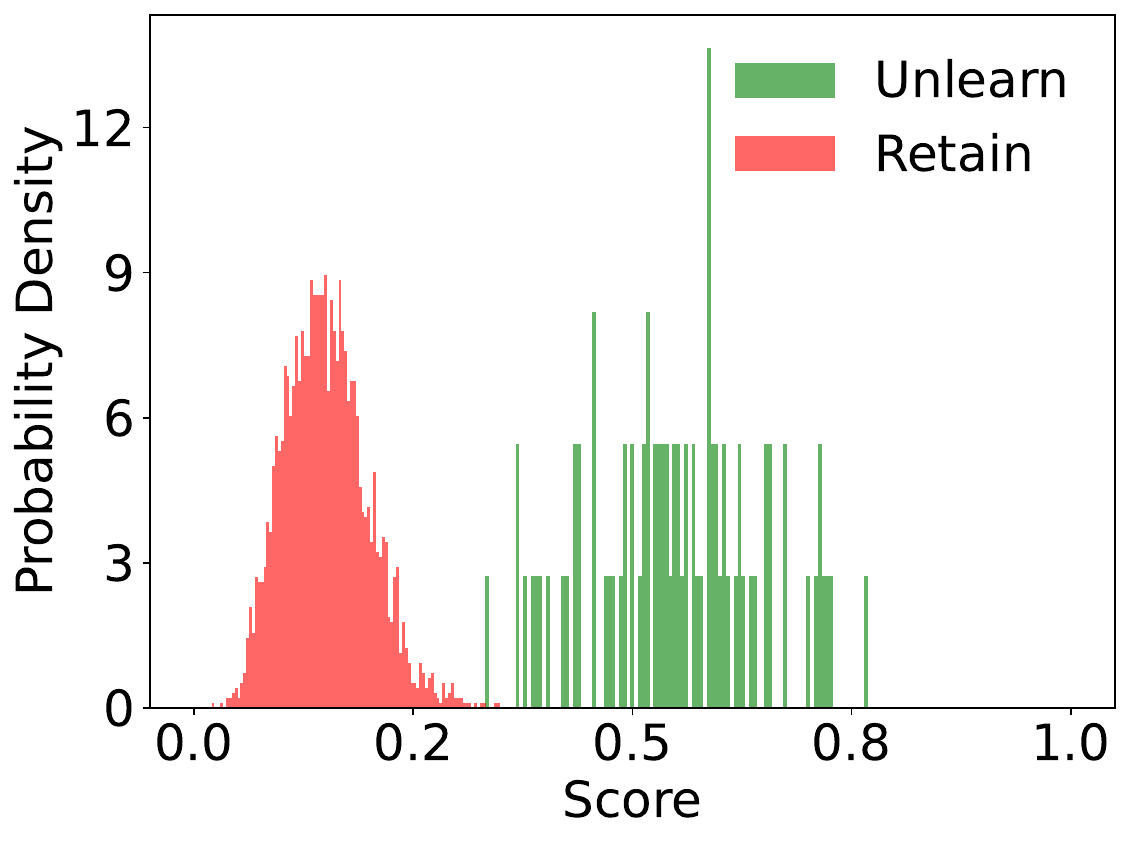}
         \caption{70\% Part-Class}
     \end{subfigure}
     \begin{subfigure}[b]{0.15\linewidth}
         \centering
         \includegraphics[height=0.08\textheight]{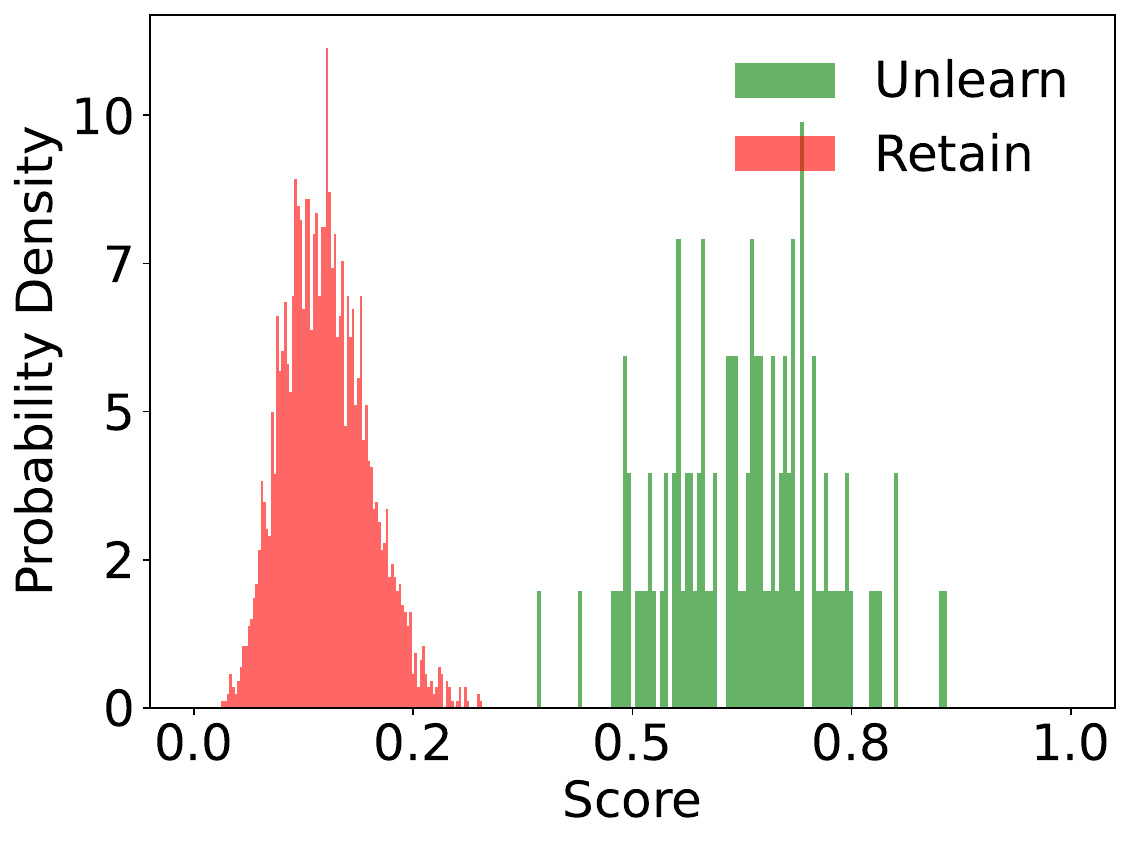}
         \caption{90\% Part-Class}
     \end{subfigure}
     \begin{subfigure}[b]{0.15\linewidth}
         \centering
         \includegraphics[height=0.08\textheight]{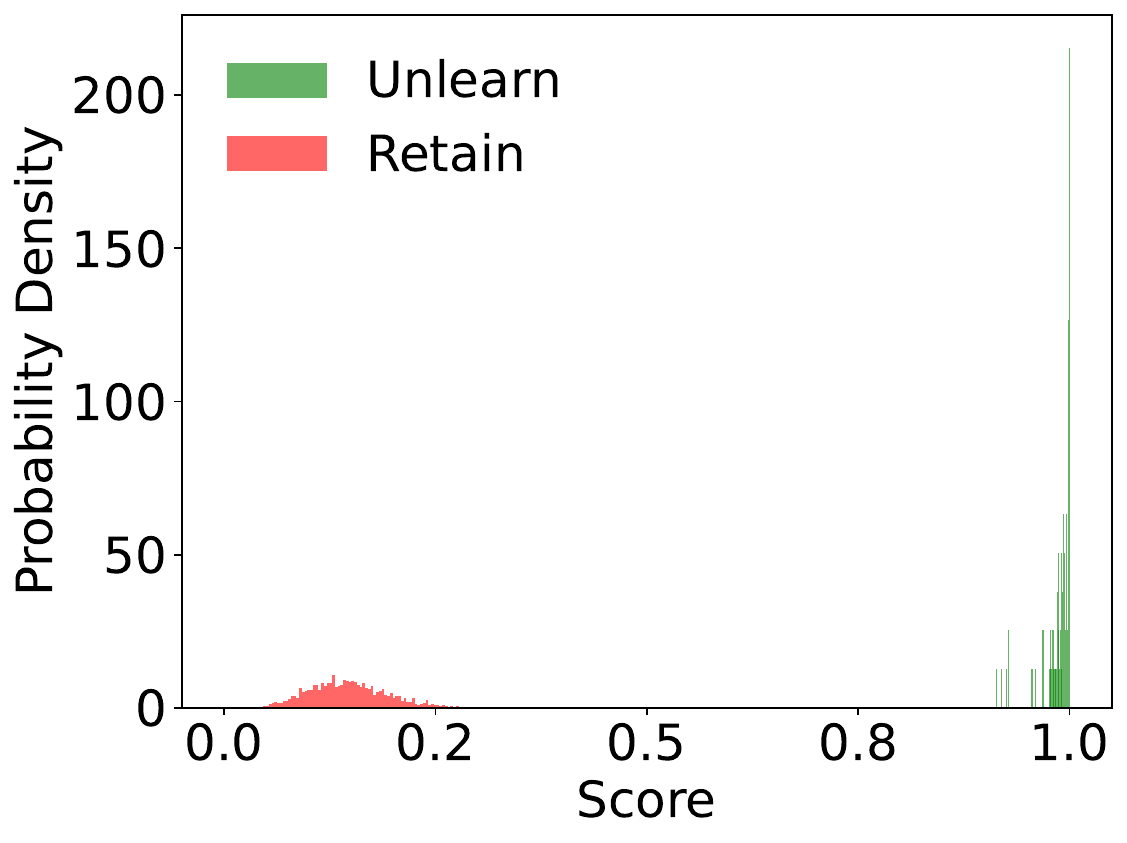}
         \caption{Total Class}
     \end{subfigure}
     \hfill
     \begin{subfigure}[b]{0.16\linewidth}
         \centering
         \includegraphics[height=0.08\textheight]{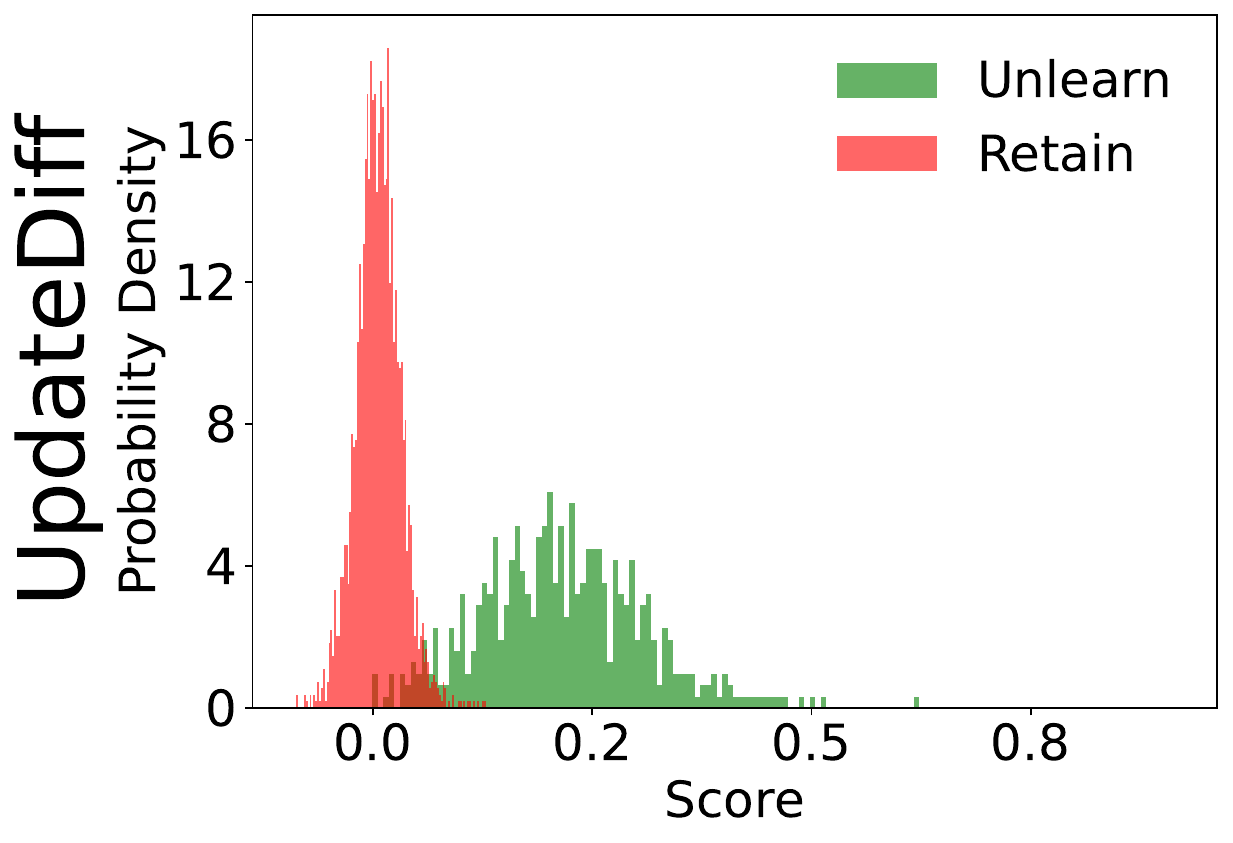}
         \caption{Random Sample}
     \end{subfigure}     
\begin{subfigure}[b]{0.15\linewidth}
         \centering
         \includegraphics[height=0.08\textheight]{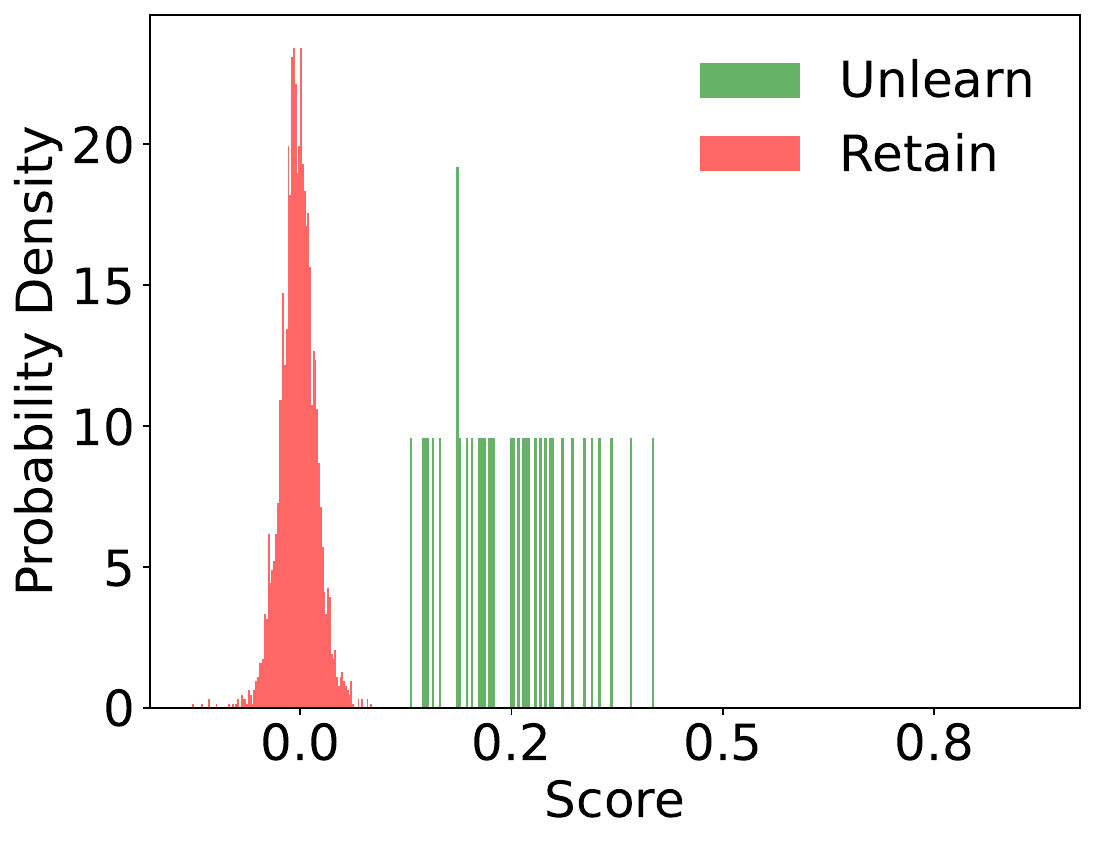}
         \caption{30\% Part-Class}
     \end{subfigure}     
\begin{subfigure}[b]{0.15\linewidth}
         \centering
         \includegraphics[height=0.08\textheight]{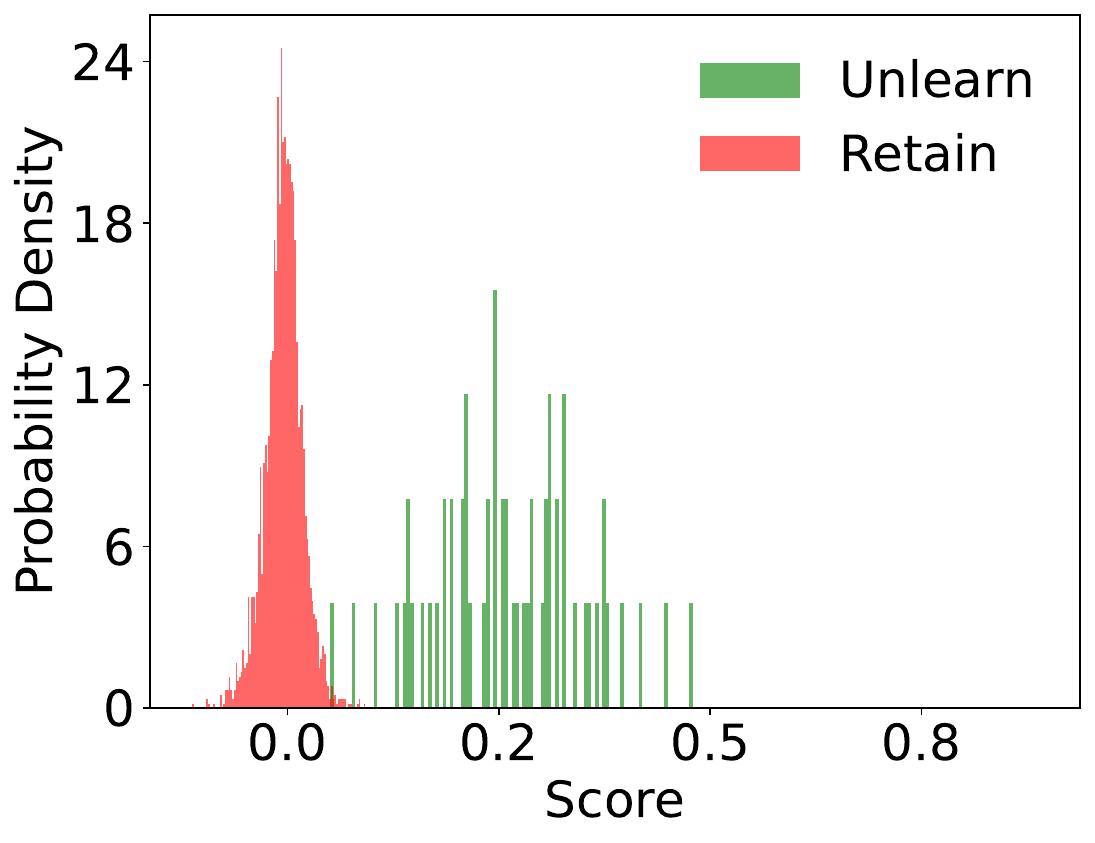}
         \caption{50\% Part-Class}
     \end{subfigure}     
\begin{subfigure}[b]{0.15\linewidth}
         \centering
         \includegraphics[height=0.08\textheight]{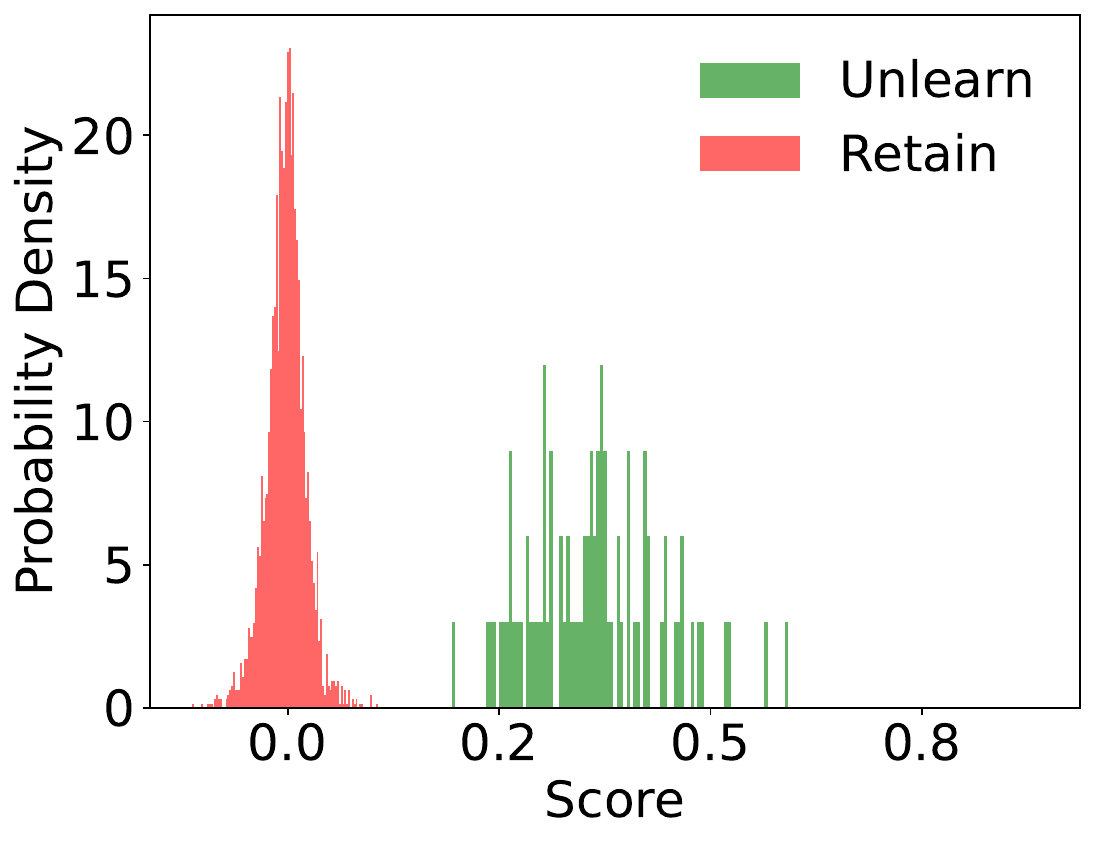}
         \caption{70\% Part-Class}
     \end{subfigure}     
\begin{subfigure}[b]{0.15\linewidth}
         \centering
         \includegraphics[height=0.08\textheight]{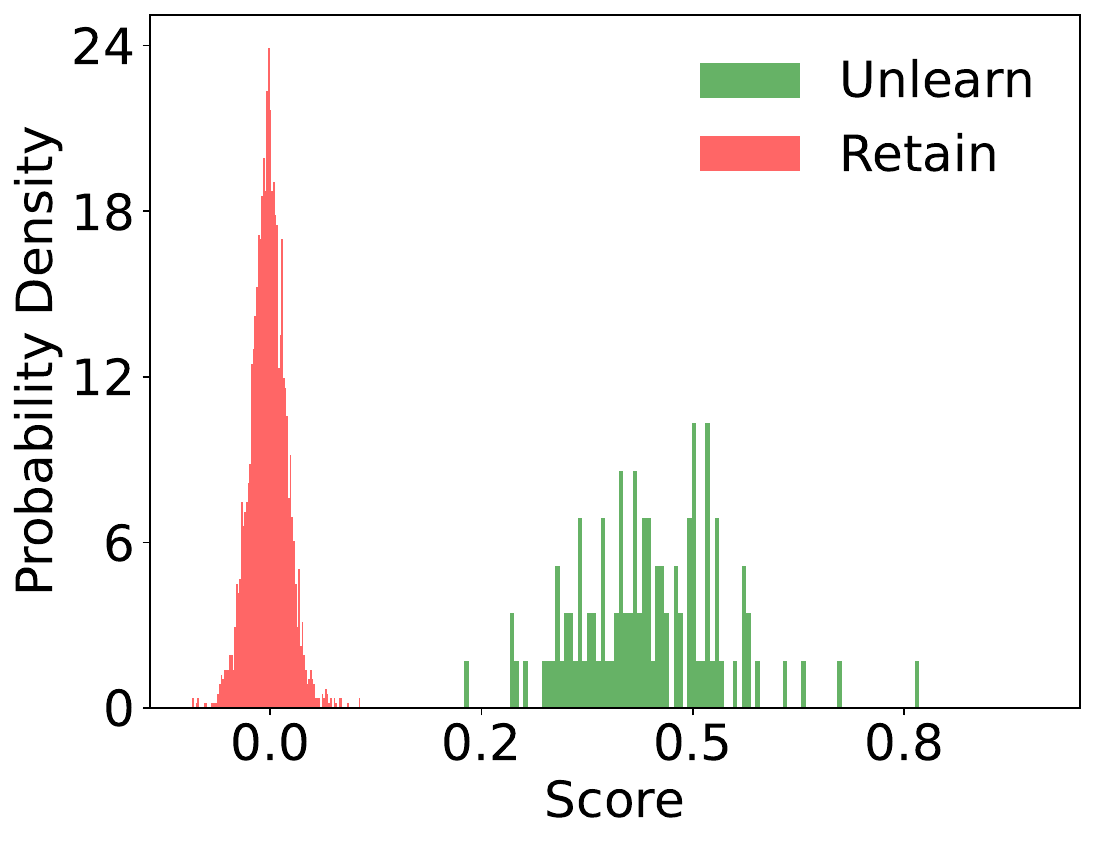}
         \caption{90\% Part-Class}
     \end{subfigure}     
\begin{subfigure}[b]{0.15\linewidth}
         \centering
         \includegraphics[height=0.08\textheight]{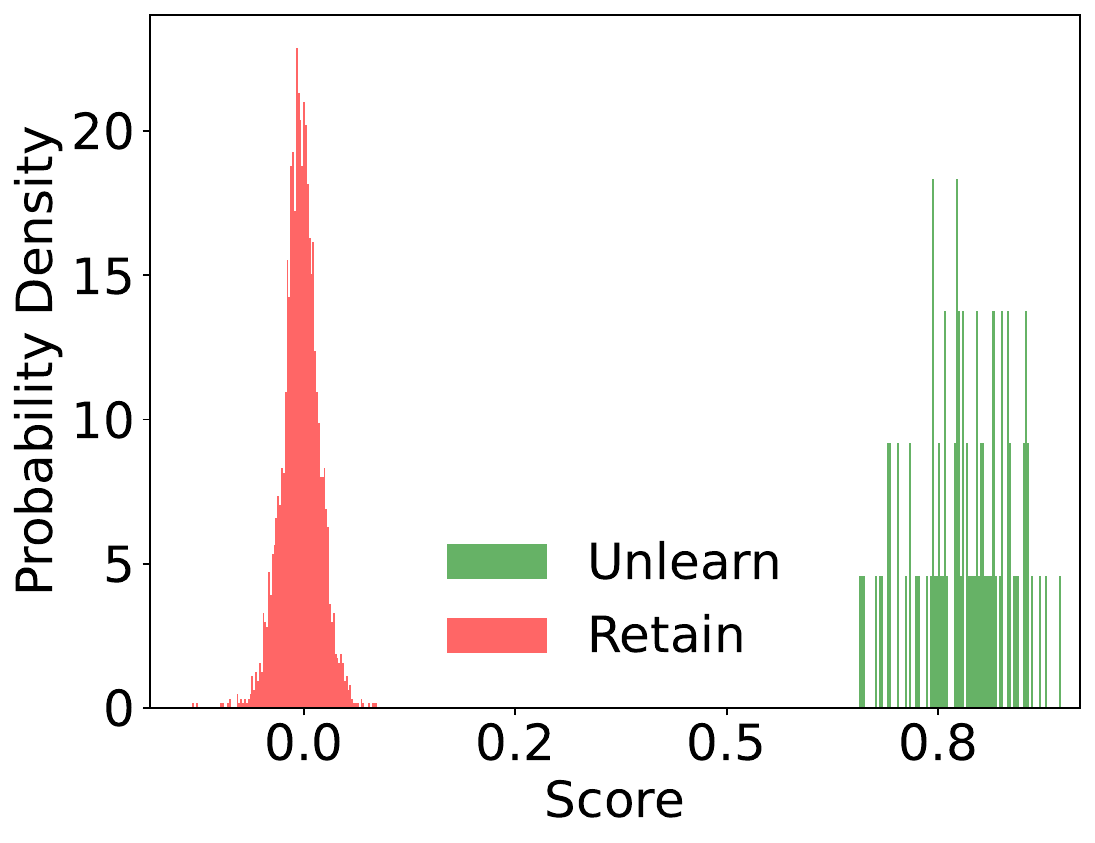}
         \caption{Total Class}
     \end{subfigure}     
     \hfill
\begin{subfigure}[b]{0.16\linewidth}
         \centering
         \includegraphics[height=0.08\textheight]{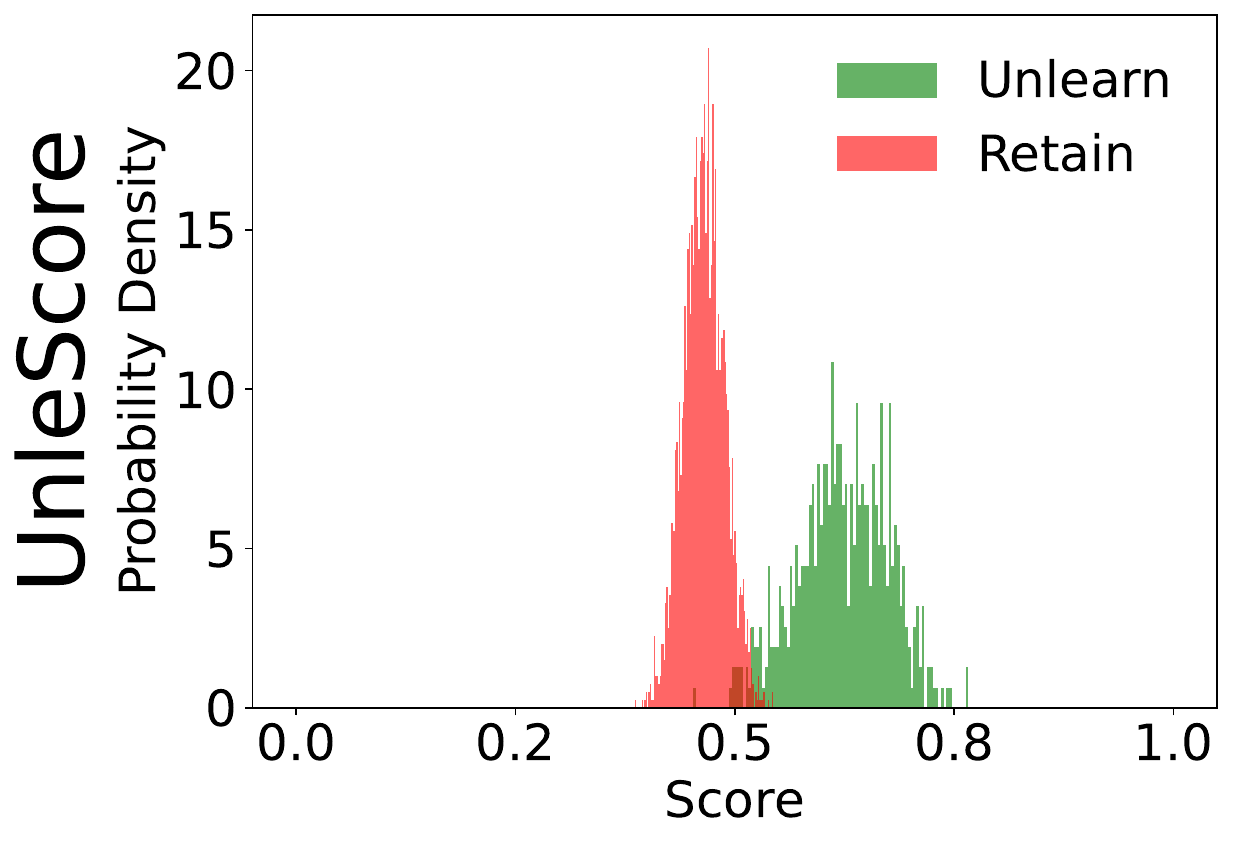}
         \caption{Random Sample}
     \end{subfigure}    
\begin{subfigure}[b]{0.15\linewidth}
         \centering
         \includegraphics[height=0.08\textheight]{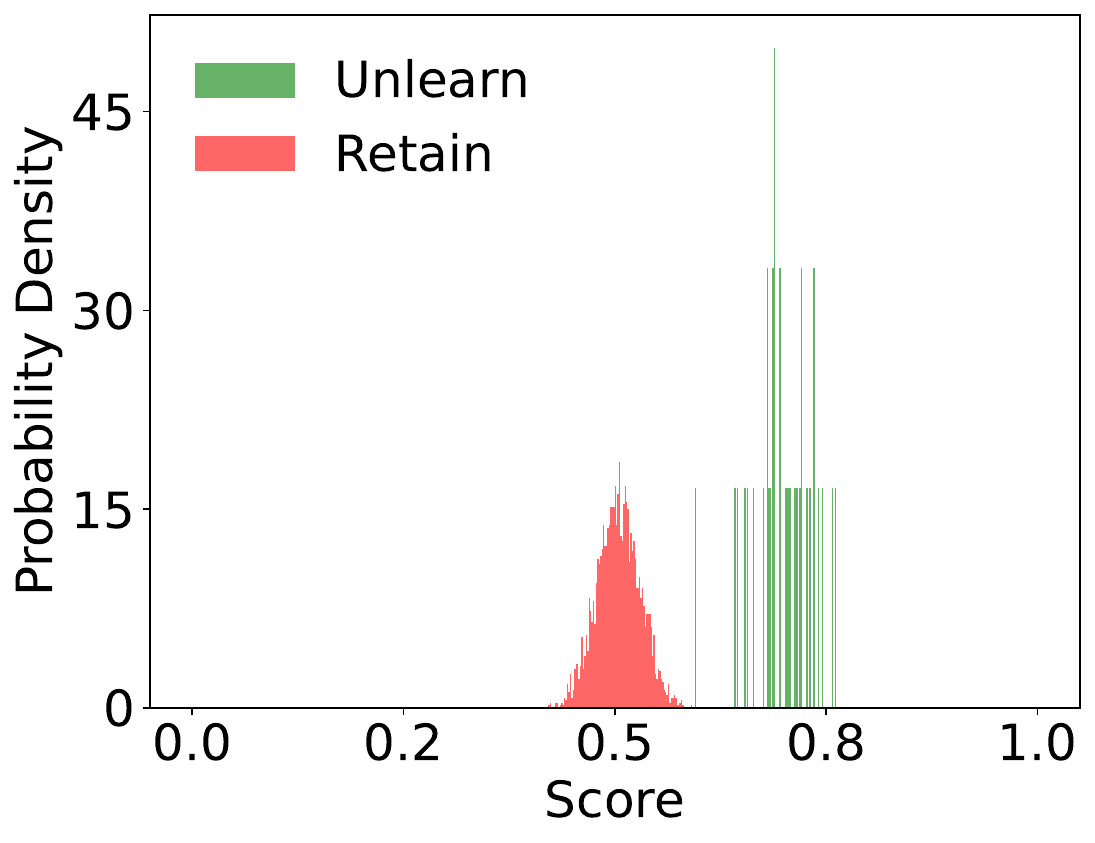}
         \caption{30\% Part-Class}
     \end{subfigure}    
\begin{subfigure}[b]{0.15\linewidth}
         \centering
         \includegraphics[height=0.08\textheight]{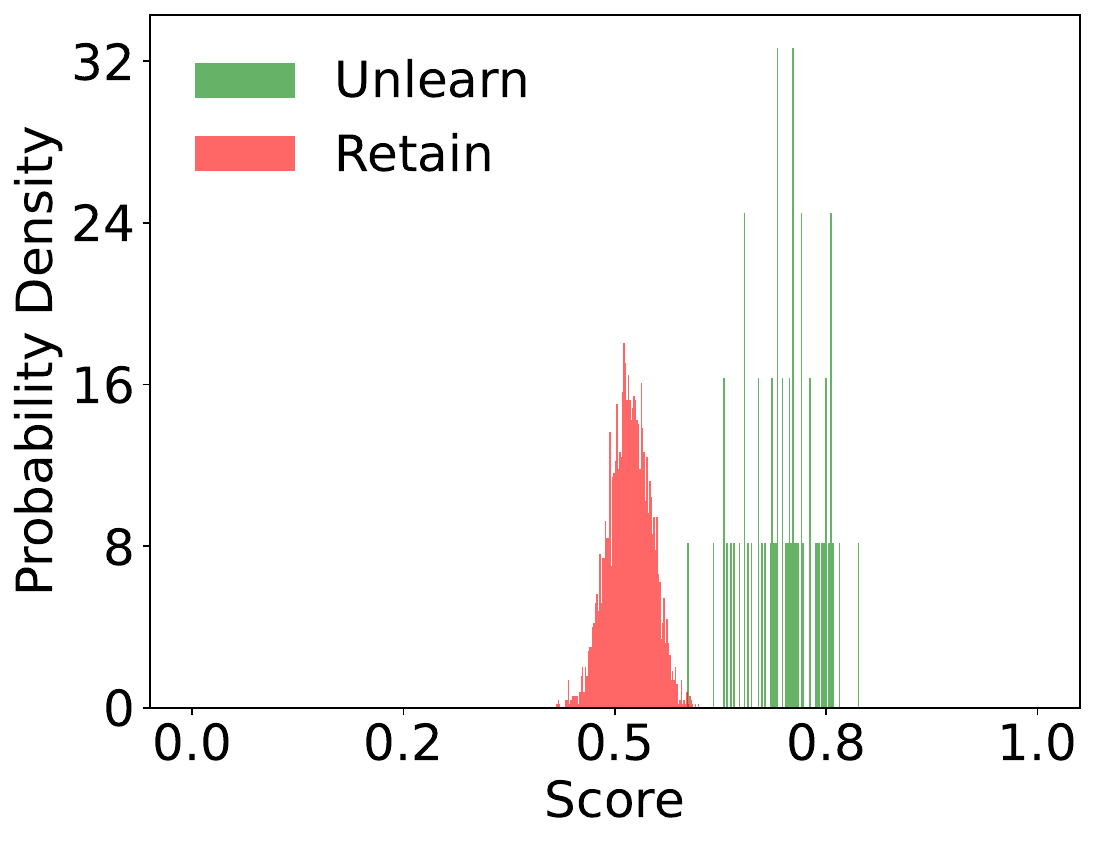}
         \caption{50\% Part-Class}
     \end{subfigure}    
\begin{subfigure}[b]{0.15\linewidth}
         \centering
         \includegraphics[height=0.08\textheight]{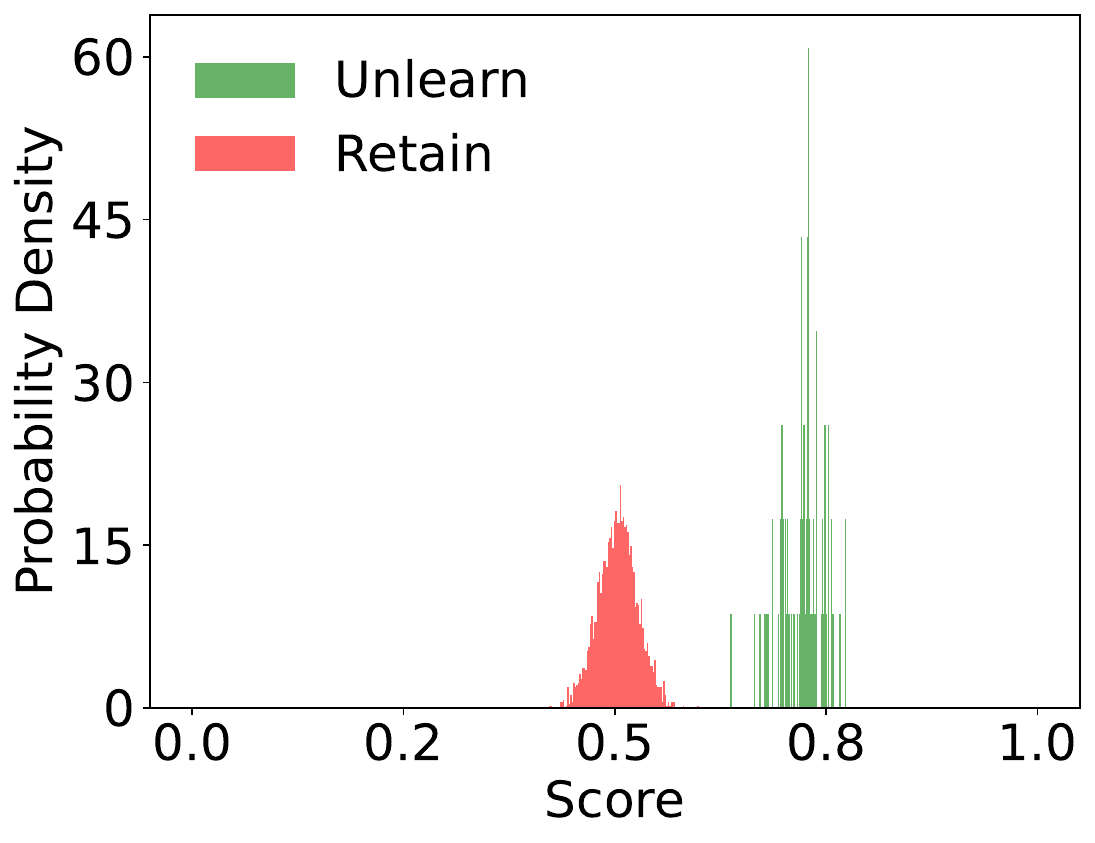}
         \caption{70\% Part-Class}
     \end{subfigure}    
\begin{subfigure}[b]{0.15\linewidth}
         \centering
         \includegraphics[height=0.08\textheight]{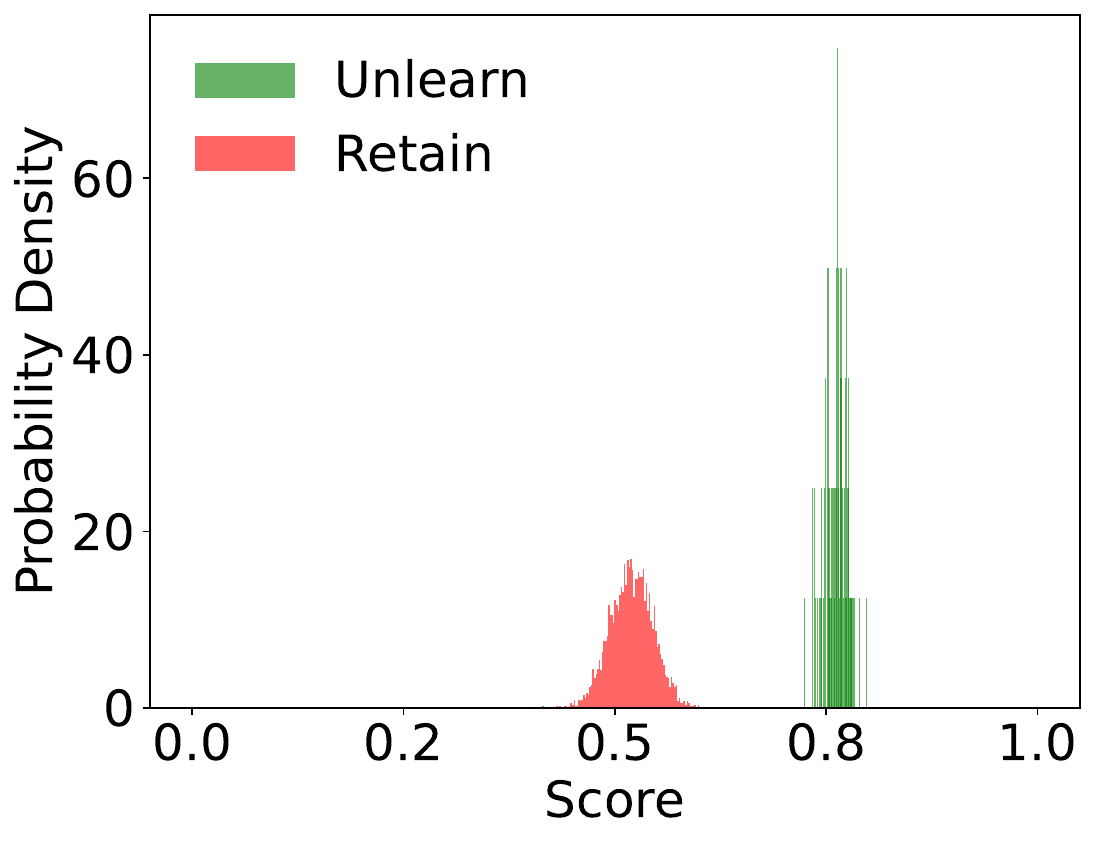}
         \caption{90\% Part-Class}
     \end{subfigure}    
\begin{subfigure}[b]{0.15\linewidth}
         \centering
         \includegraphics[height=0.08\textheight]{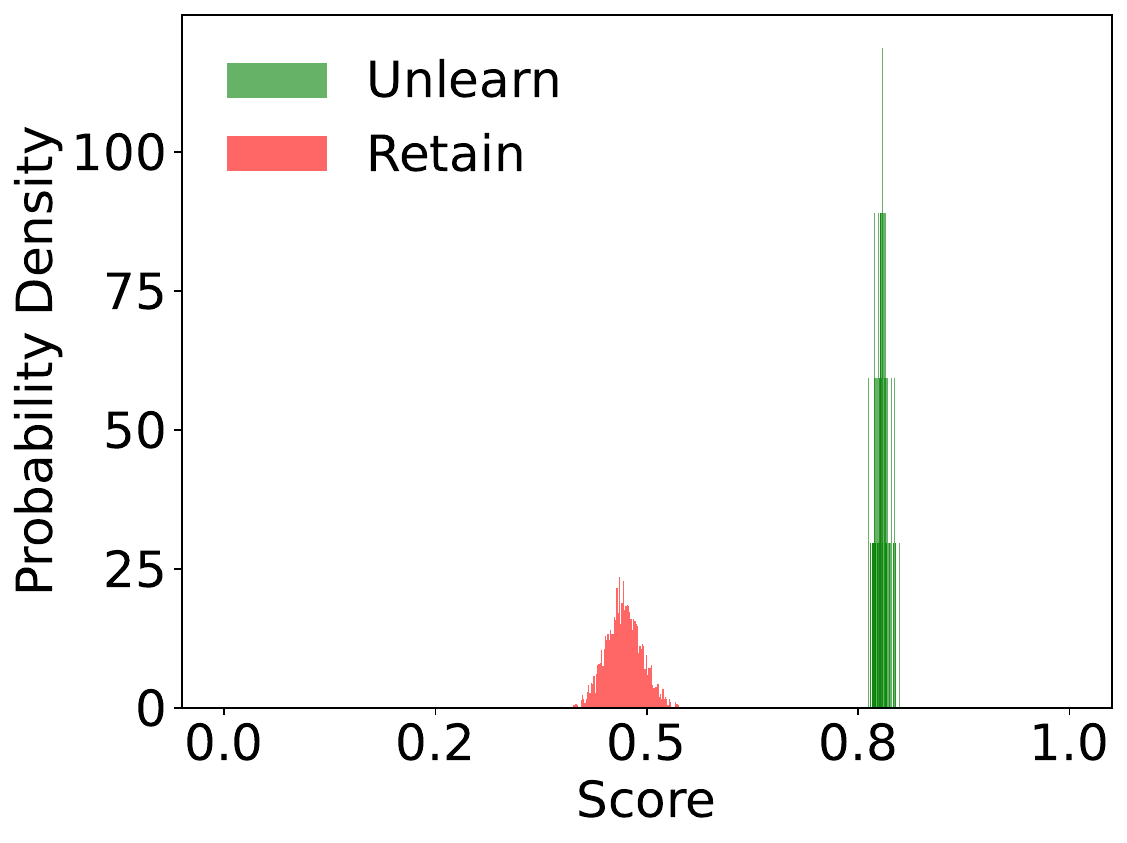}
         \caption{Total Class}
     \end{subfigure}       
     \caption{Score Distributions of Unlearning Metrics with Different Unlearning Tasks on Location.}
    \label{fig:unlearningscores_Location}
\end{figure*}

\section{Additional benchmark results of approximate unlearning algorithms}\label{addition_approximate_benchmarks}

\subsection{Additional results of unlearning utility}\label{addition_utility}

Table~\ref{unlearning_utility_auc} shows the AUC scores of 7 approximate unlearning baselines.

\subsection{Additional results of unlearning resilience}\label{addition_resil}

We provide the \textit{NMI\_TPR@FPR=0.01‰} results and AUC scores for unlearning baselines across five datasets. Figure~\ref{continue_auc_cifar10_cifar100} reports the AUC results of approximate baselines for the total class tasks of CIFAR10 and CIFAR100. Figure~\ref{continue_random_tpr} and \ref{continue_class_percent_tpr_5}  show the results on random class unlearning and partial total class unlearning tasks across 5 datasets.

\begin{figure}[ht]
    \centering
     \begin{subfigure}[b]{0.4\textwidth}
         \centering
         \includegraphics[width=\linewidth]{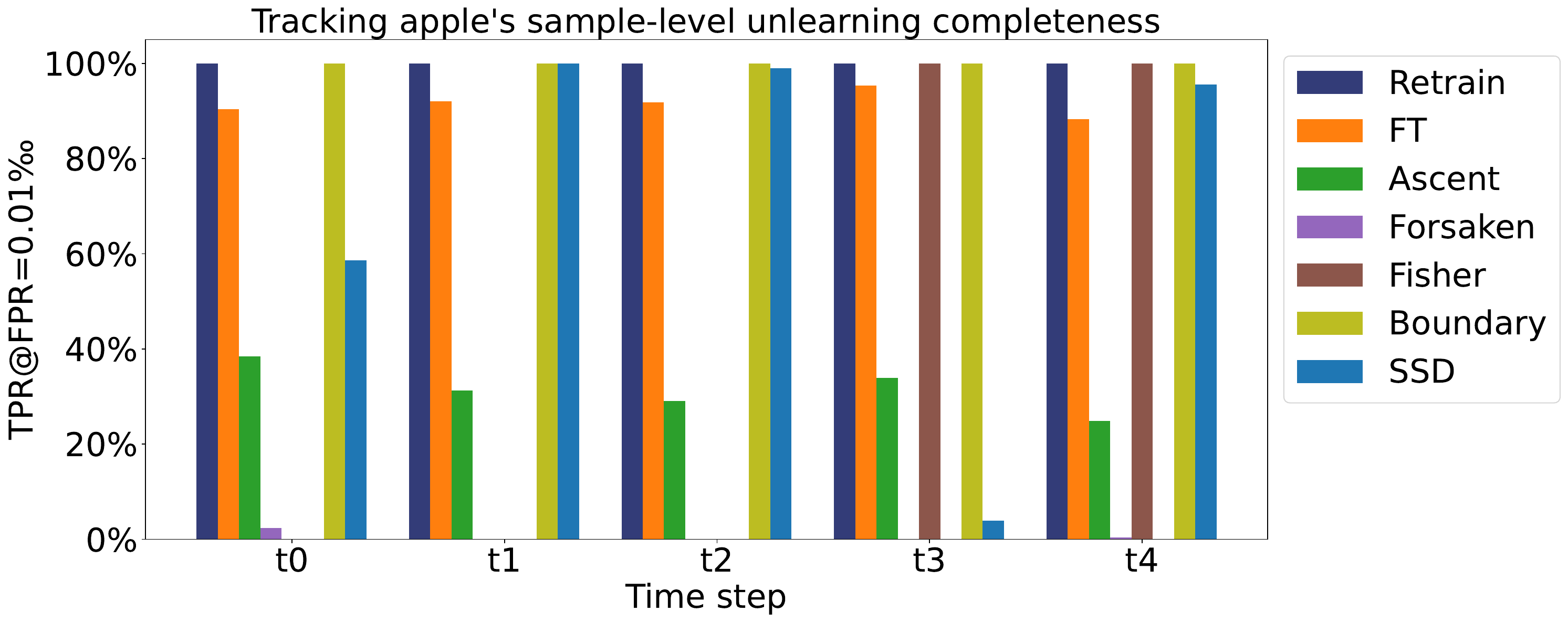}
         \caption{Total class unlearning resilience results (TPR@FPR=0.01‰) of baselines on Location}
         \label{fig:y equals x}
     \end{subfigure}
     \hfill
     \begin{subfigure}[b]{0.4\textwidth}
         \centering
         \includegraphics[width=\linewidth]{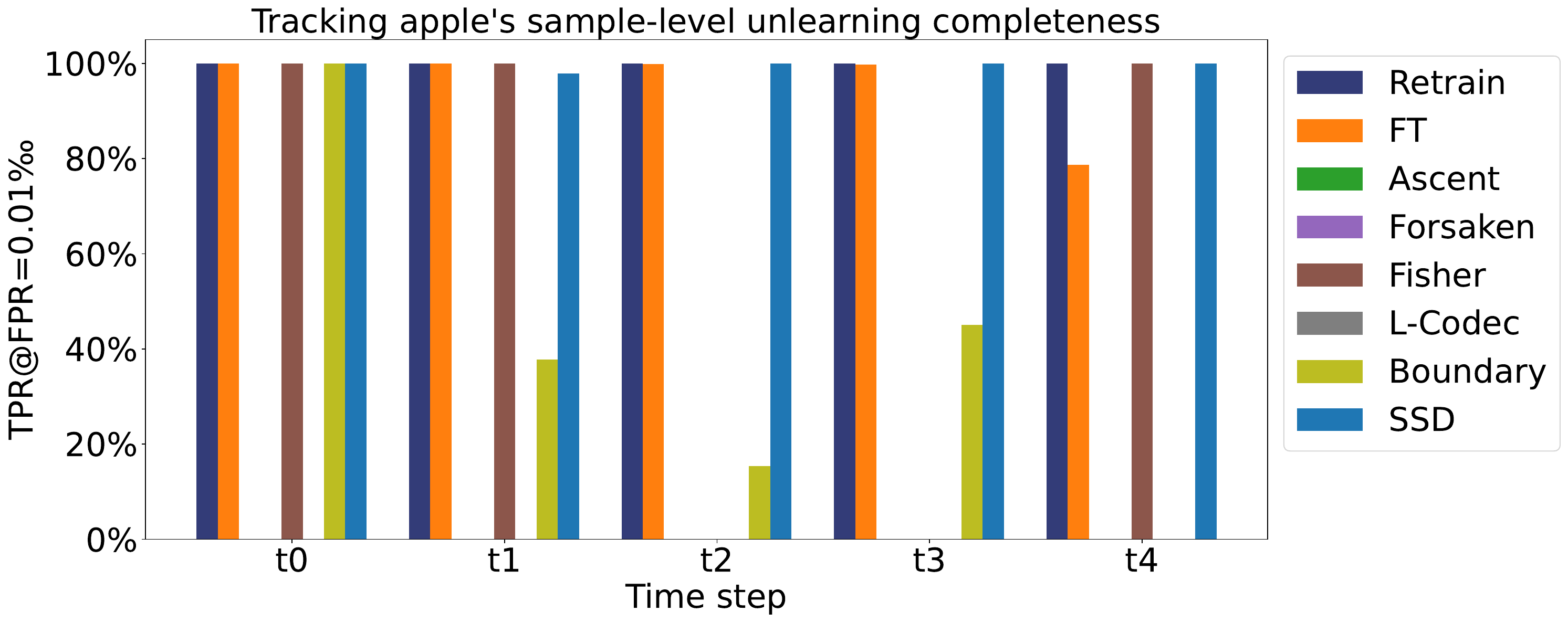}
         \caption{Purchase}
         \label{fig:three sin x}
     \end{subfigure}
     \begin{subfigure}[b]{0.4\textwidth}
         \centering
         \includegraphics[width=\linewidth]{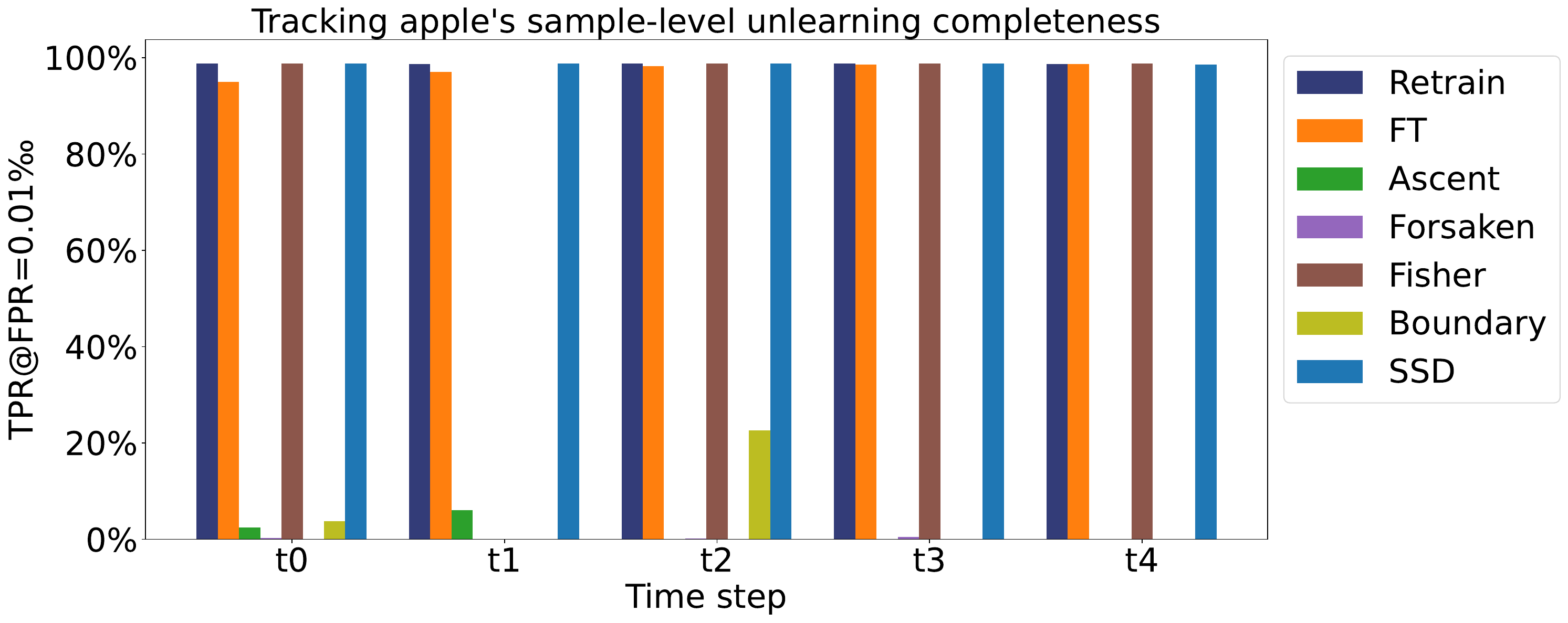}
         \caption{Texas}
         \label{fig:three sin x}
     \end{subfigure}
    \caption{Total class unlearning resilience results (TPR@FPR=0.01‰) of baselines on Location, Purchase, and Texas}
    \label{continue_class_one_tpr}
\end{figure}

\begin{figure}[h]
    \centering
    \centering
     \begin{subfigure}[b]{0.47\textwidth}
         \centering
         \includegraphics[width=\linewidth]{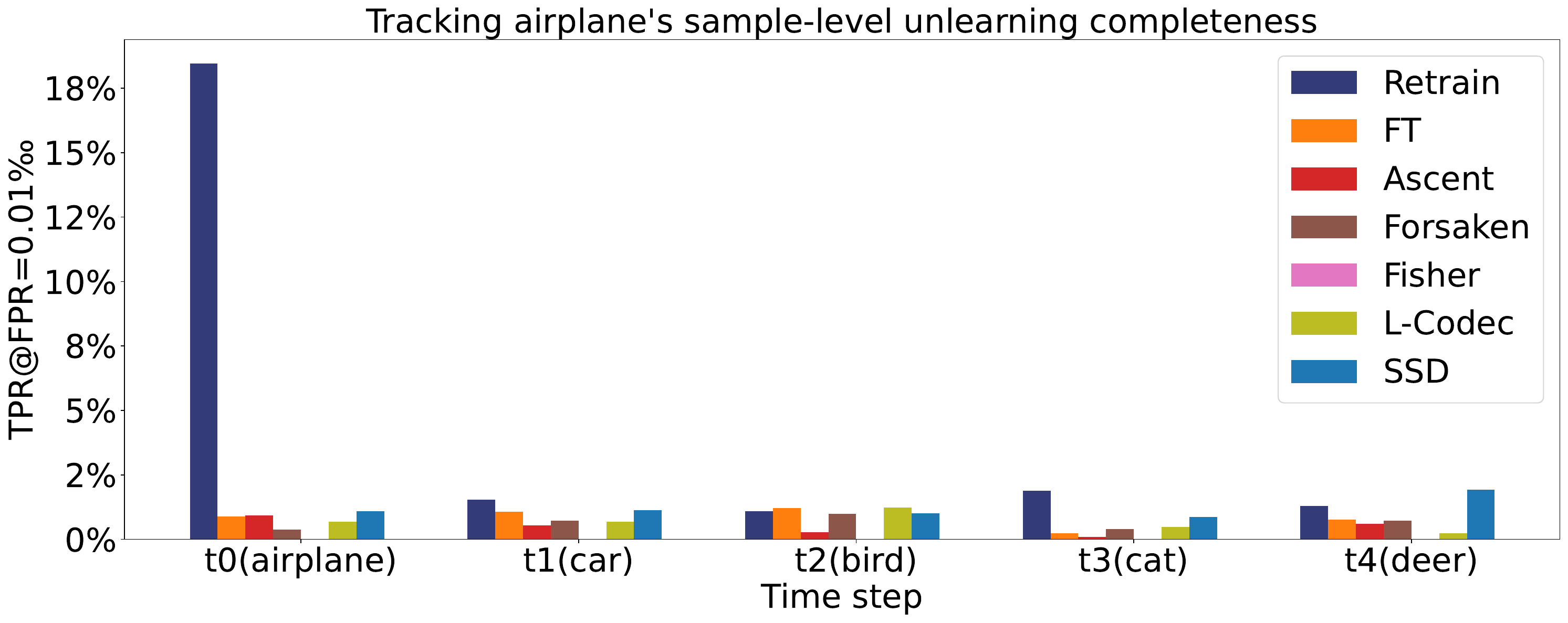}
         \caption{CIFAR10}
         \label{fig:y equals x}
     \end{subfigure}
     \hfill
     \begin{subfigure}[b]{0.47\textwidth}
         \centering
         \includegraphics[width=\linewidth]{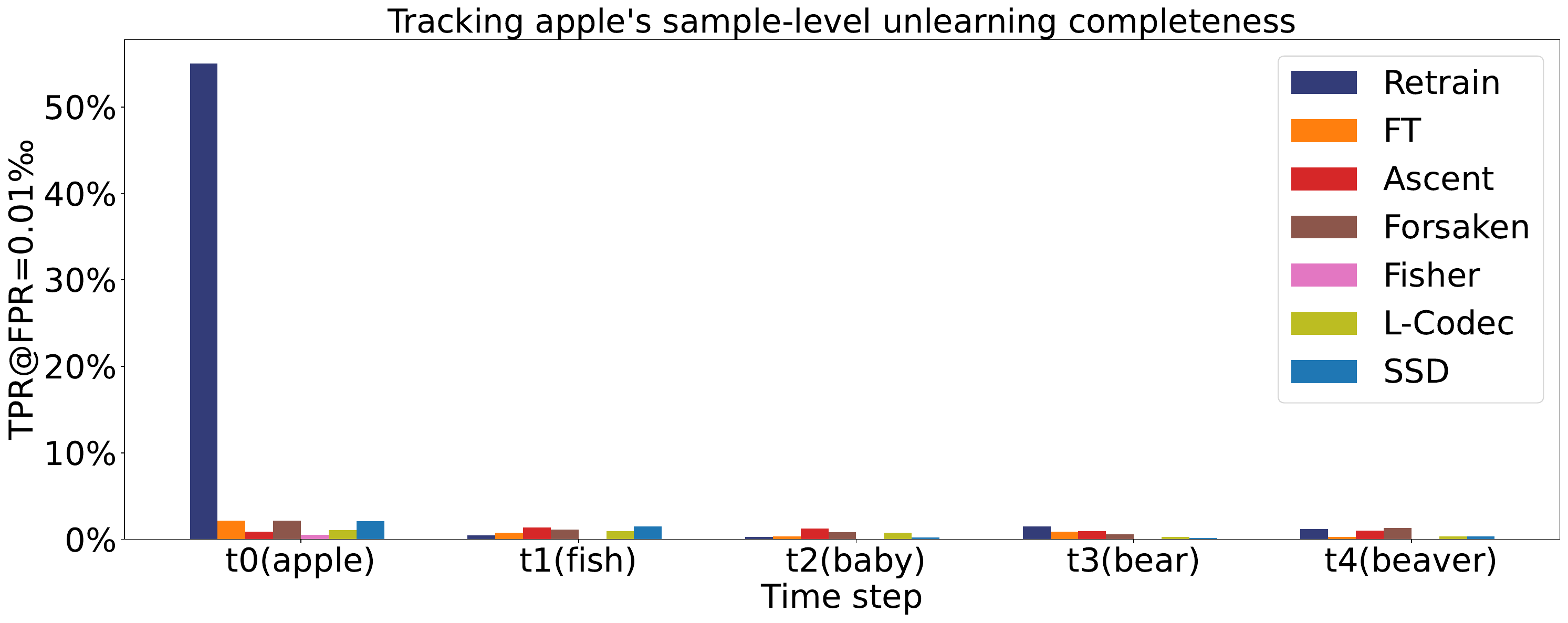}
         \caption{CIFAR100}
         \label{fig:three sin x}
     \end{subfigure}
     \begin{subfigure}[b]{0.47\textwidth}
         \centering
         \includegraphics[width=\linewidth]{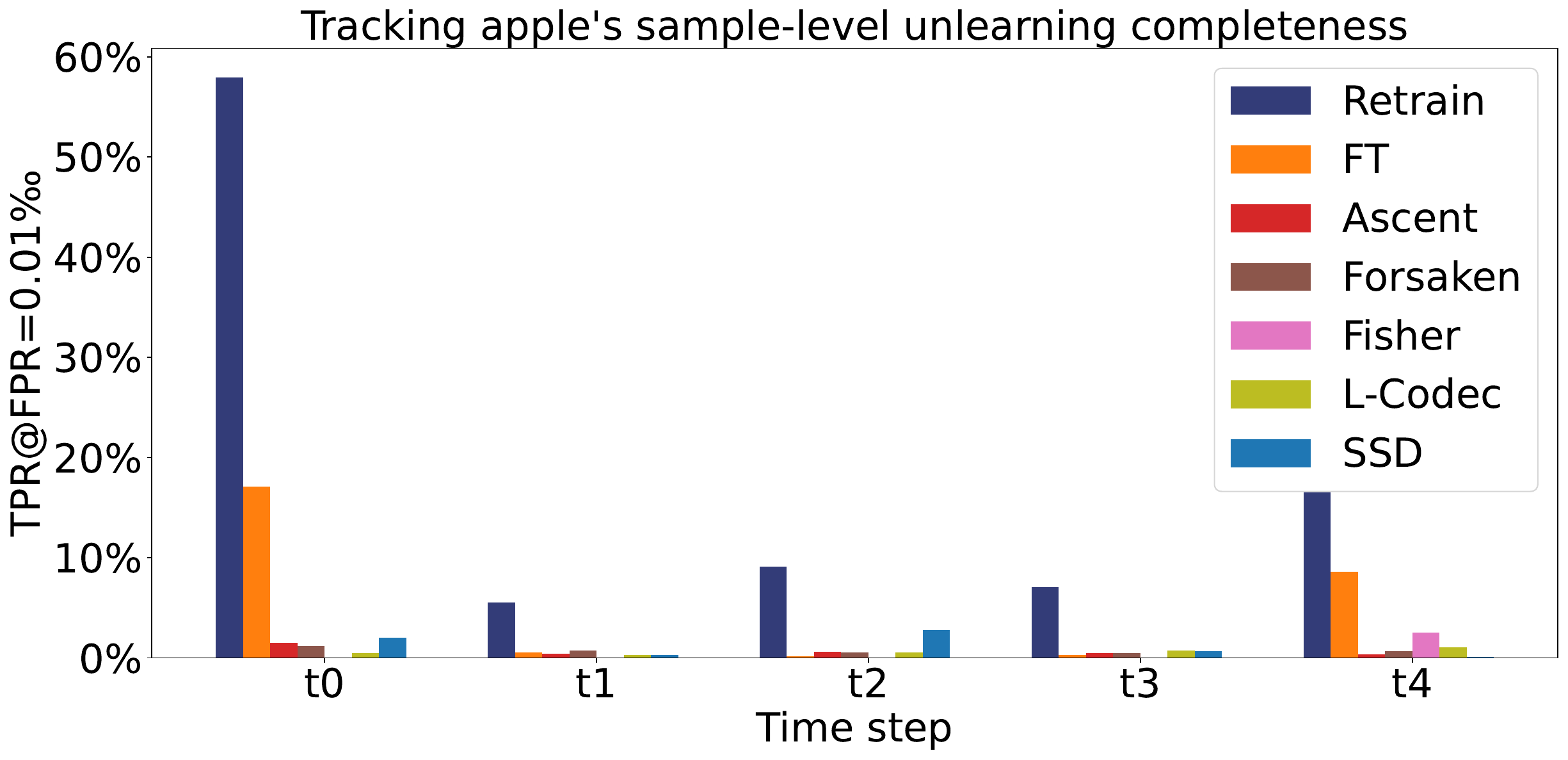}
         \caption{Location}
         \label{fig:three sin x}
     \end{subfigure}
    \begin{subfigure}[b]{0.47\textwidth}
         \centering
         \includegraphics[width=\linewidth]{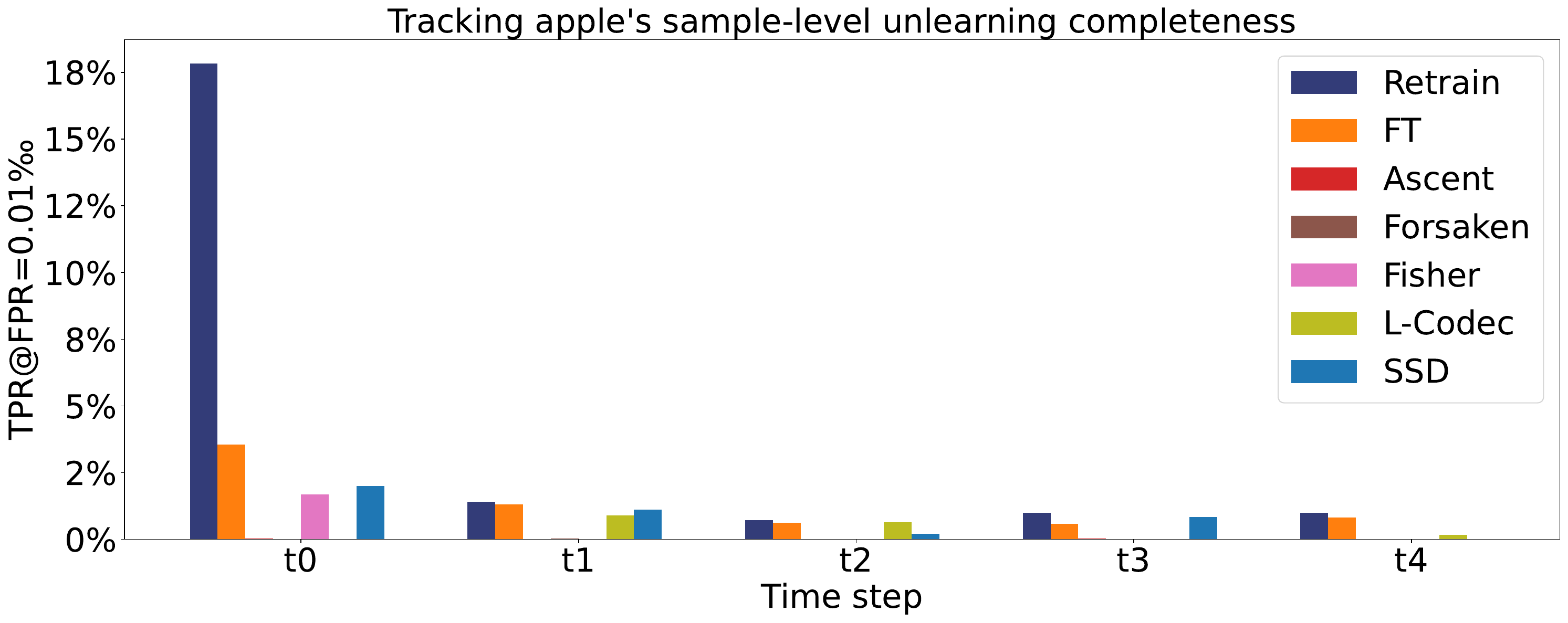}
         \caption{Purchase}
         \label{fig:three sin x}
     \end{subfigure}
    \begin{subfigure}[b]{0.47\textwidth}
         \centering
         \includegraphics[width=\linewidth]{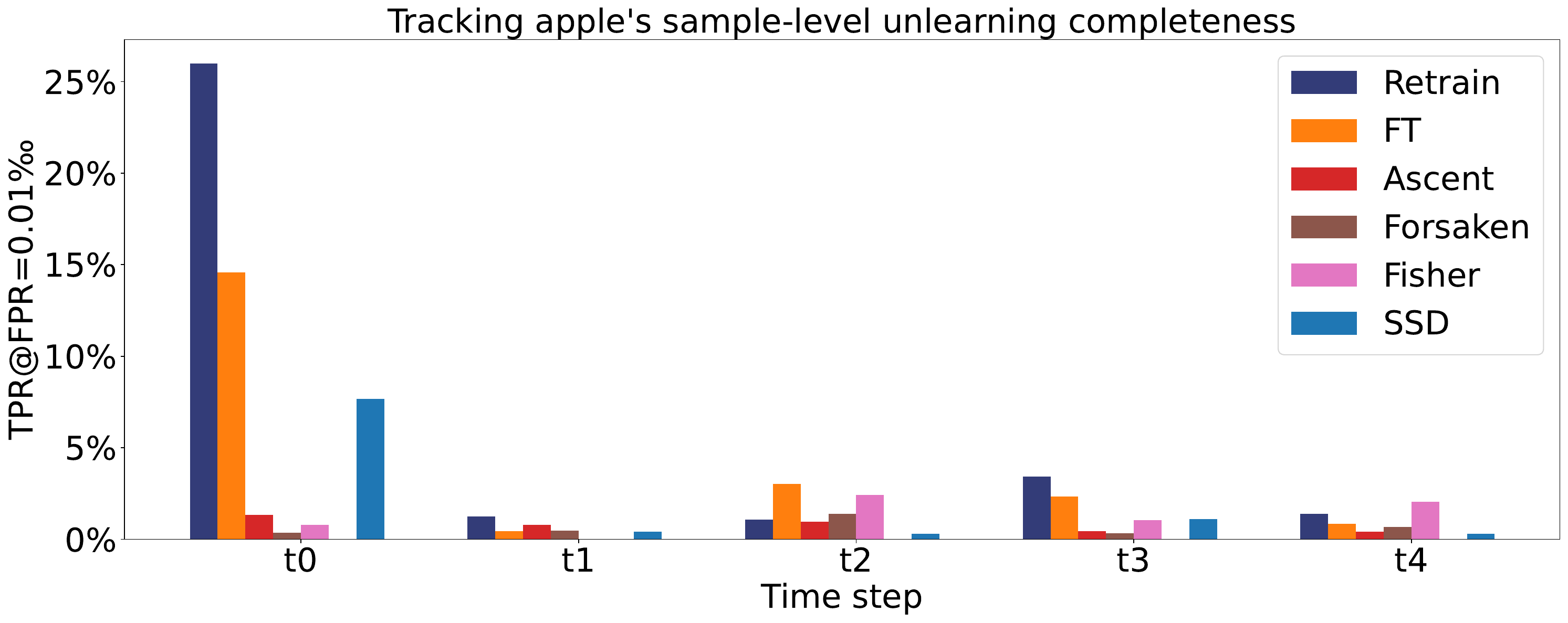}
         \caption{Texas}
         \label{fig:three sin x}
     \end{subfigure}
    \caption{Random sample unlearning resilience results (TPR@FPR=0.01‰) of baselines}
    \label{continue_random_tpr}
\end{figure}

\begin{figure}[h]
    \centering
    \centering
     \begin{subfigure}[b]{0.47\textwidth}
         \centering
         \includegraphics[width=\linewidth]{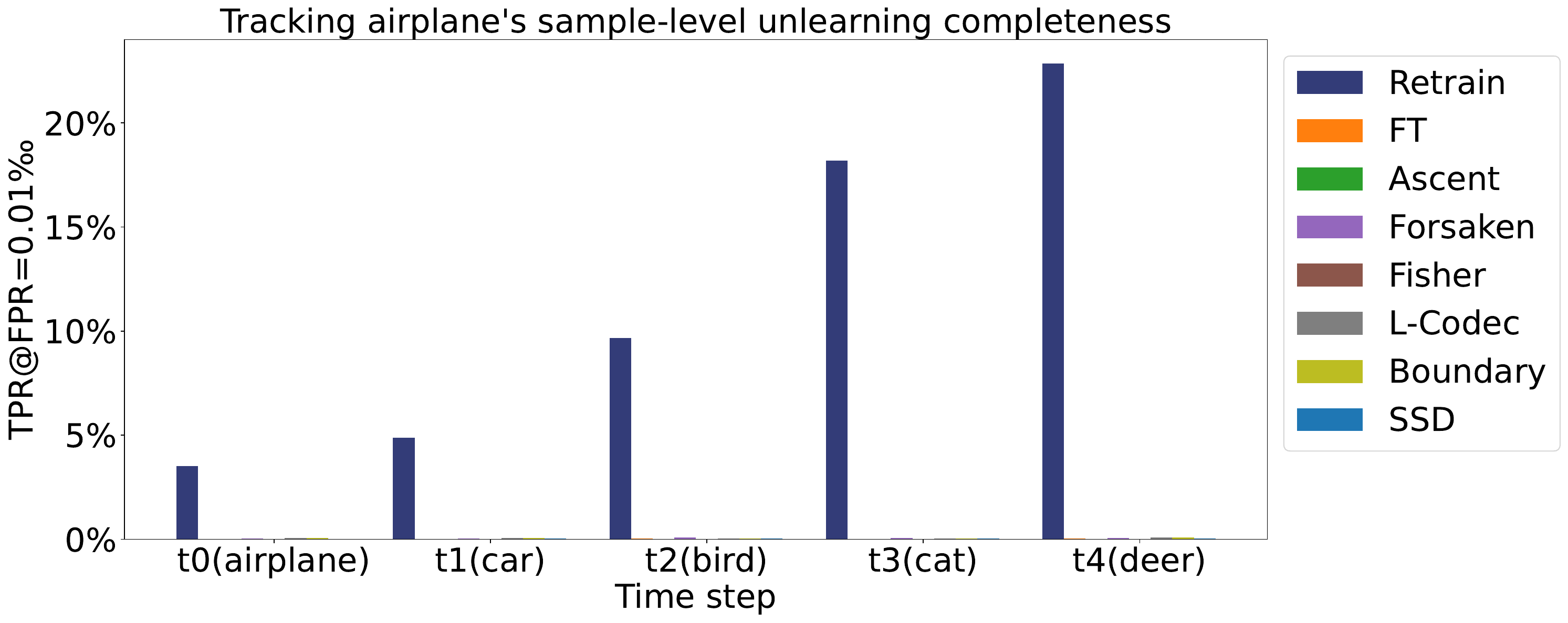}
         \caption{CIFAR10}
         \label{fig:y equals x}
     \end{subfigure}
     \hfill
     \begin{subfigure}[b]{0.47\textwidth}
         \centering
         \includegraphics[width=\linewidth]{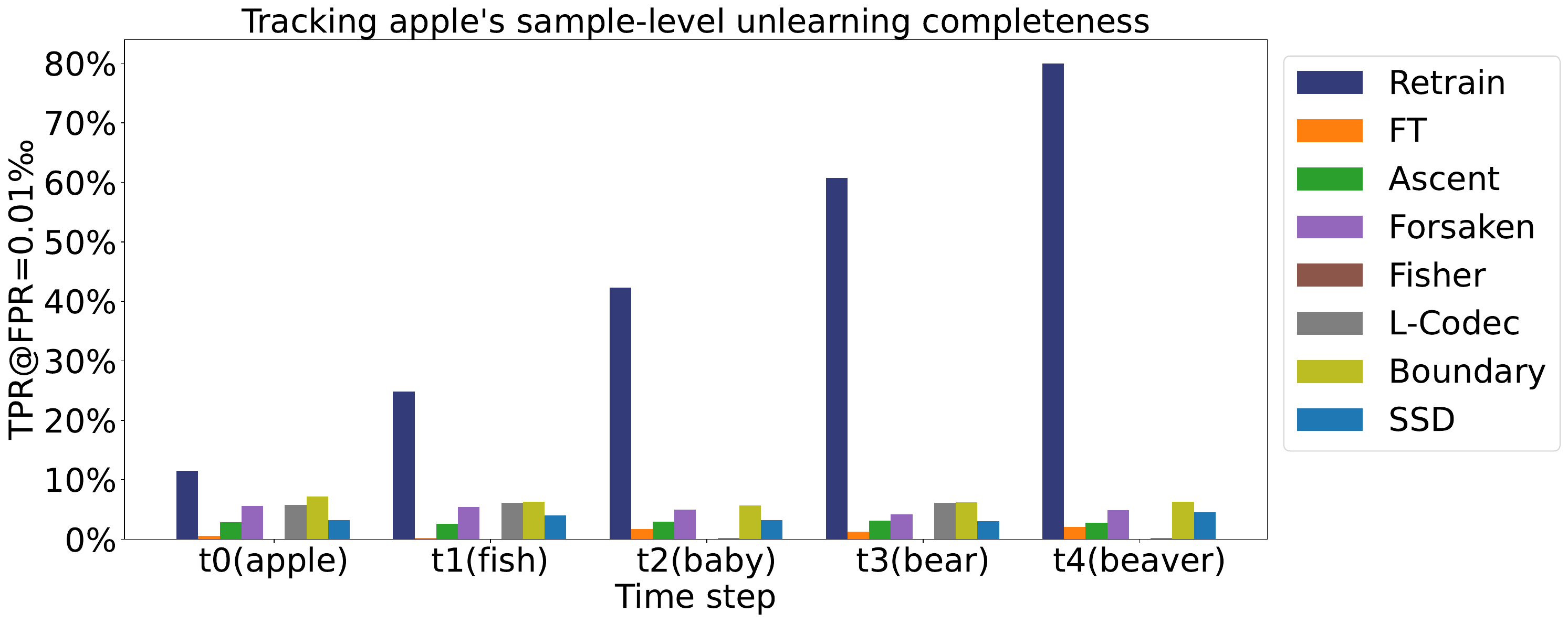}
         \caption{CIFAR100}
         \label{fig:three sin x}
     \end{subfigure}
     \begin{subfigure}[b]{0.47\textwidth}
         \centering
         \includegraphics[width=\linewidth]{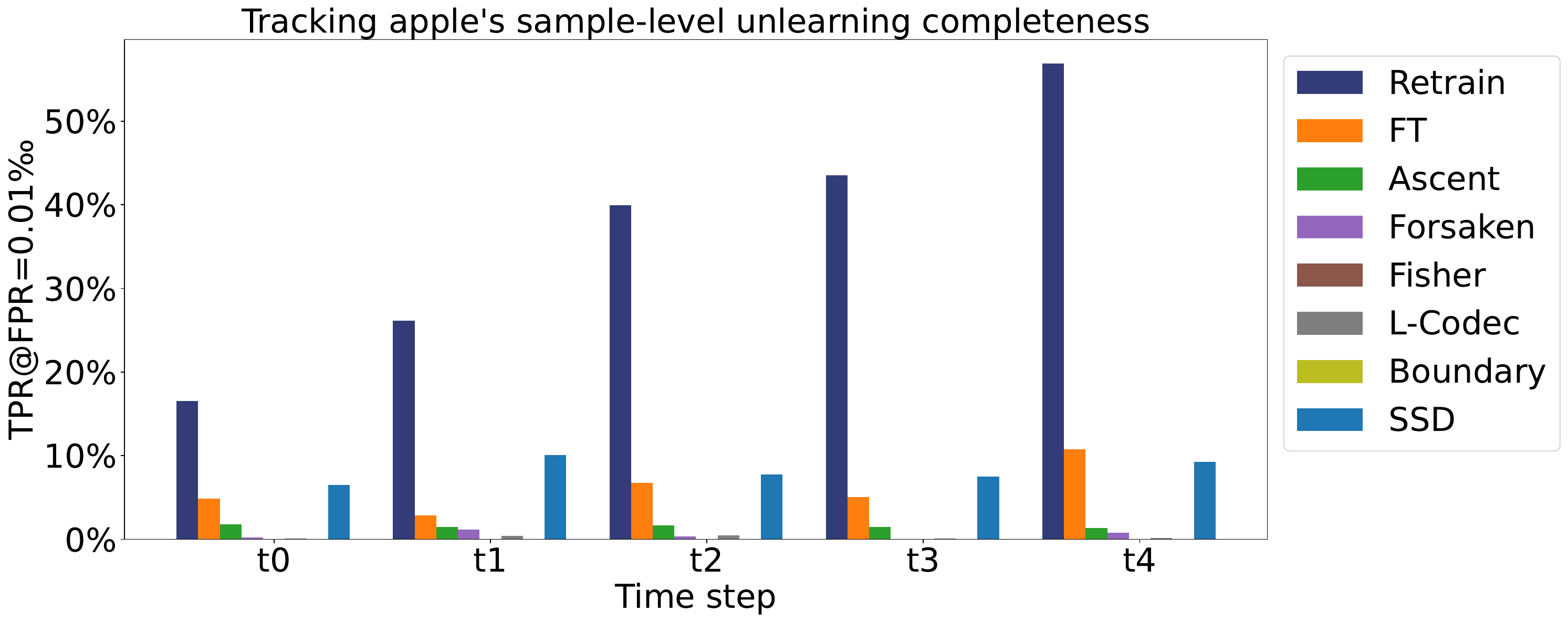}
         \caption{Location}
         \label{fig:three sin x}
     \end{subfigure}
     \begin{subfigure}[b]{0.47\textwidth}
         \centering
         \includegraphics[width=\linewidth]{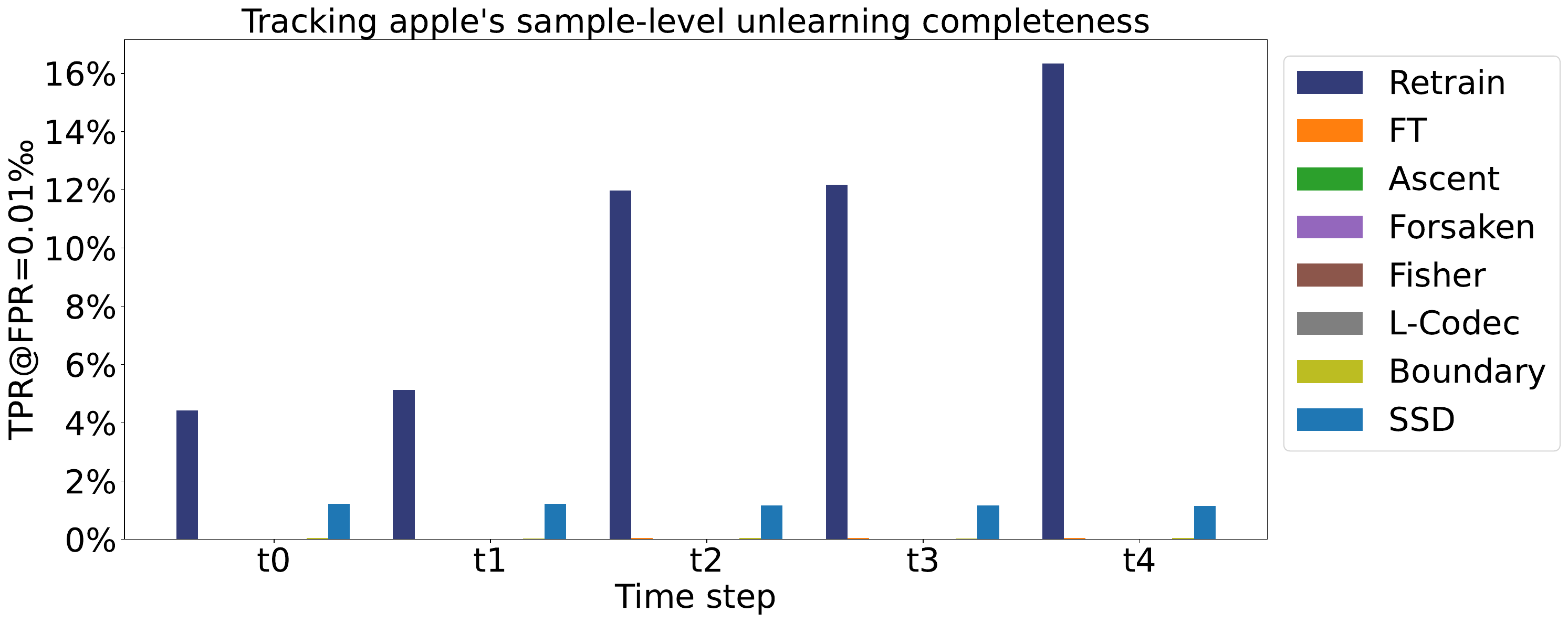}
         \caption{Purchase}
         \label{fig:three sin x}
     \end{subfigure}
     \begin{subfigure}[b]{0.47\textwidth}
         \centering
         \includegraphics[width=\linewidth]{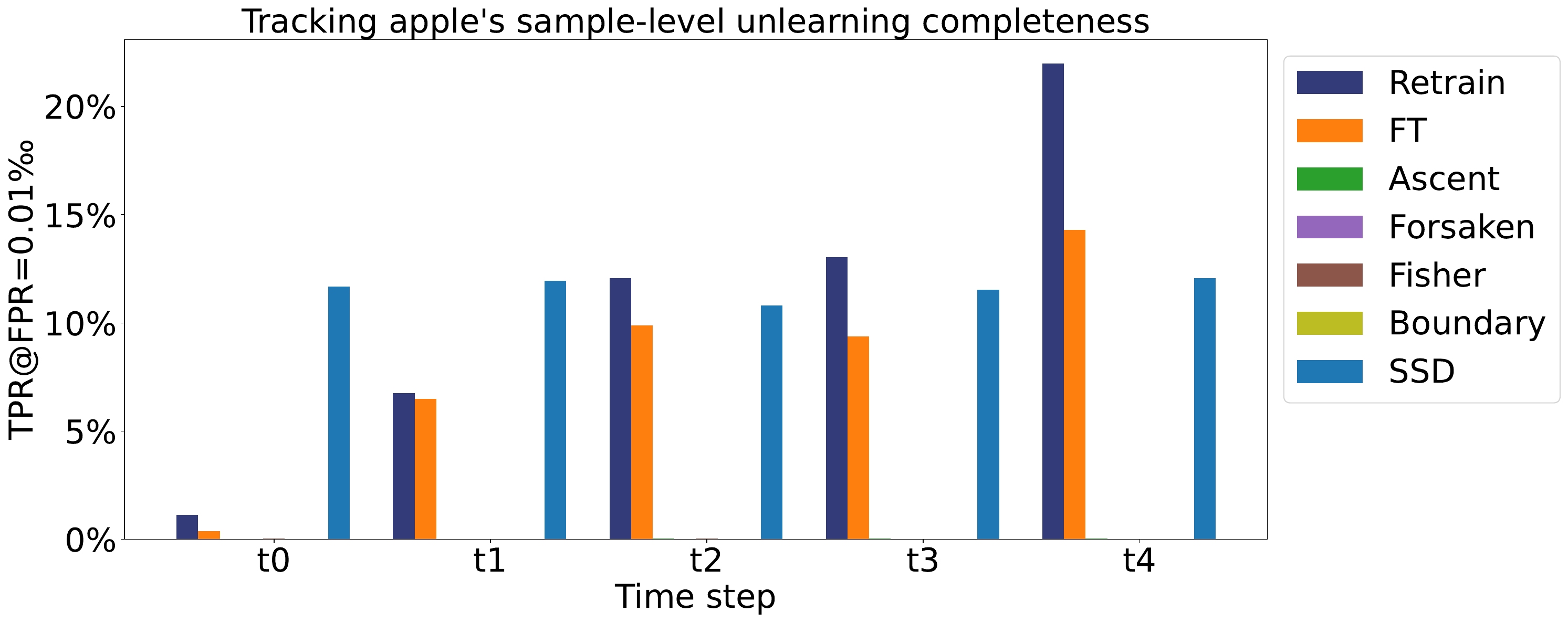}
         \caption{Texas}
         \label{fig:three sin x}
     \end{subfigure}
    \caption{Partial class unlearning resilience results (TPR@FPR=0.01‰) of baselines}
    \label{continue_class_percent_tpr_5}
\end{figure}

\subsection{Additional results of unlearning equity}\label{addition_equity}

Figure~\ref{fairtpr_classpercent},~\ref{fairauc_classpercent},~\ref{fairauc_totalclass} present the unlearning equity results of different approximate unlearning methods.
\begin{figure*}[h]
    \centering
    \includegraphics[width=0.9\linewidth]{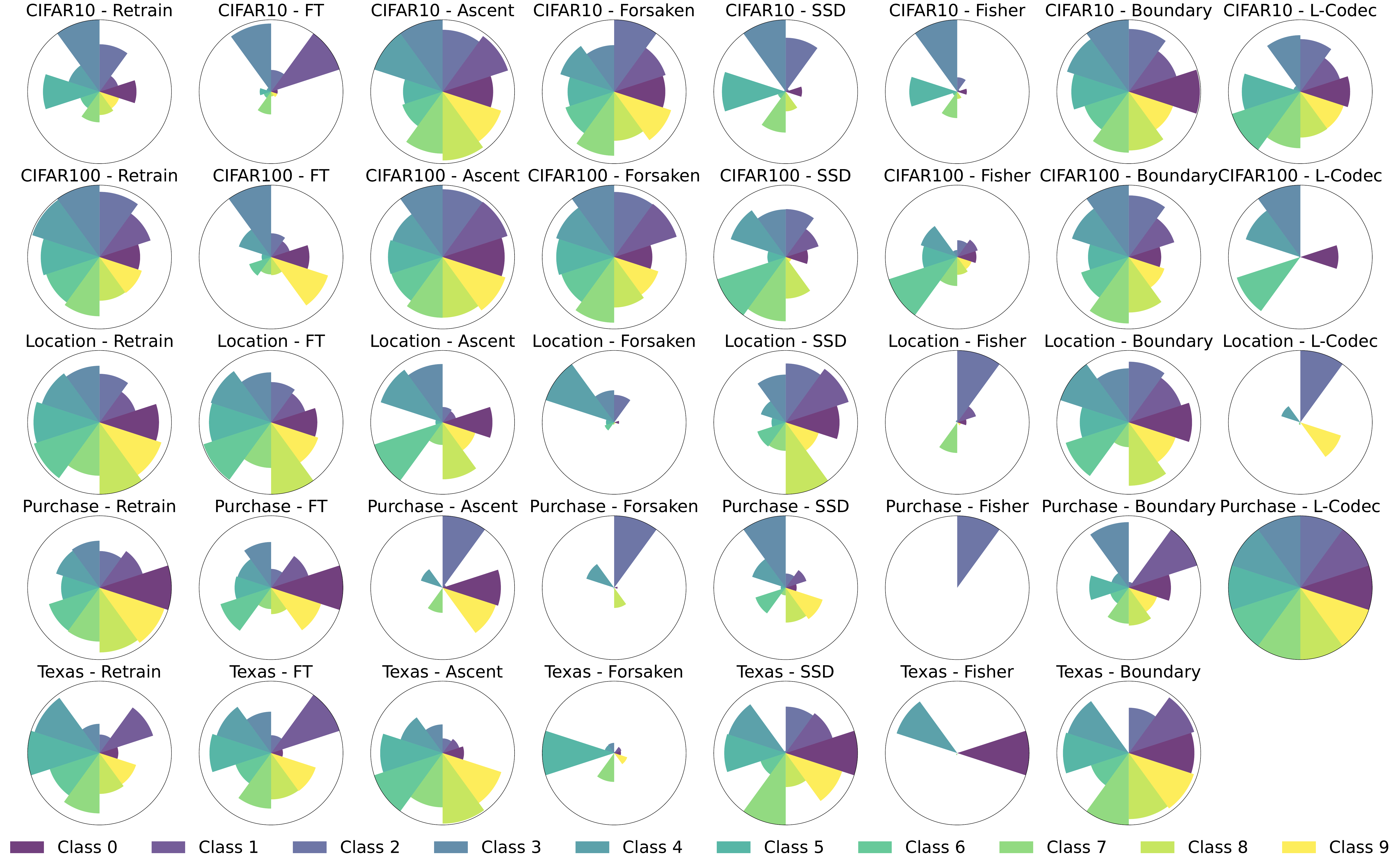}
    \caption{Partial class unlearning relative results (TPR@FPR=0.01‰)}
    \label{fairtpr_classpercent}
\end{figure*}
\begin{figure*}[h]
    \centering
    \includegraphics[width=0.9\linewidth]{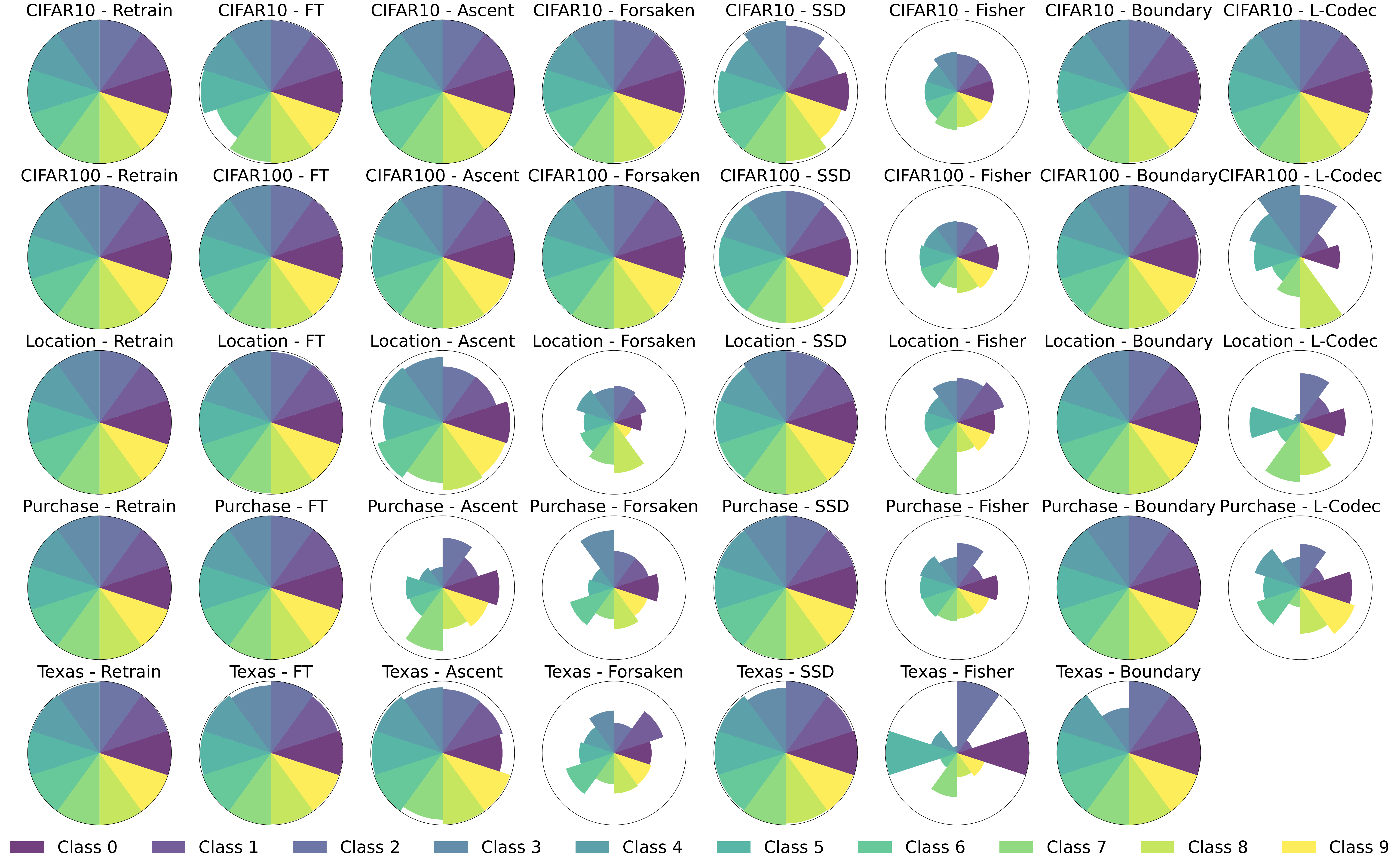}
    \caption{Total class unlearning results (AUC) of 10 classes}
    \label{fairauc_totalclass}
\end{figure*}
\begin{figure*}[h]
    \centering
    \includegraphics[width=0.9\linewidth]{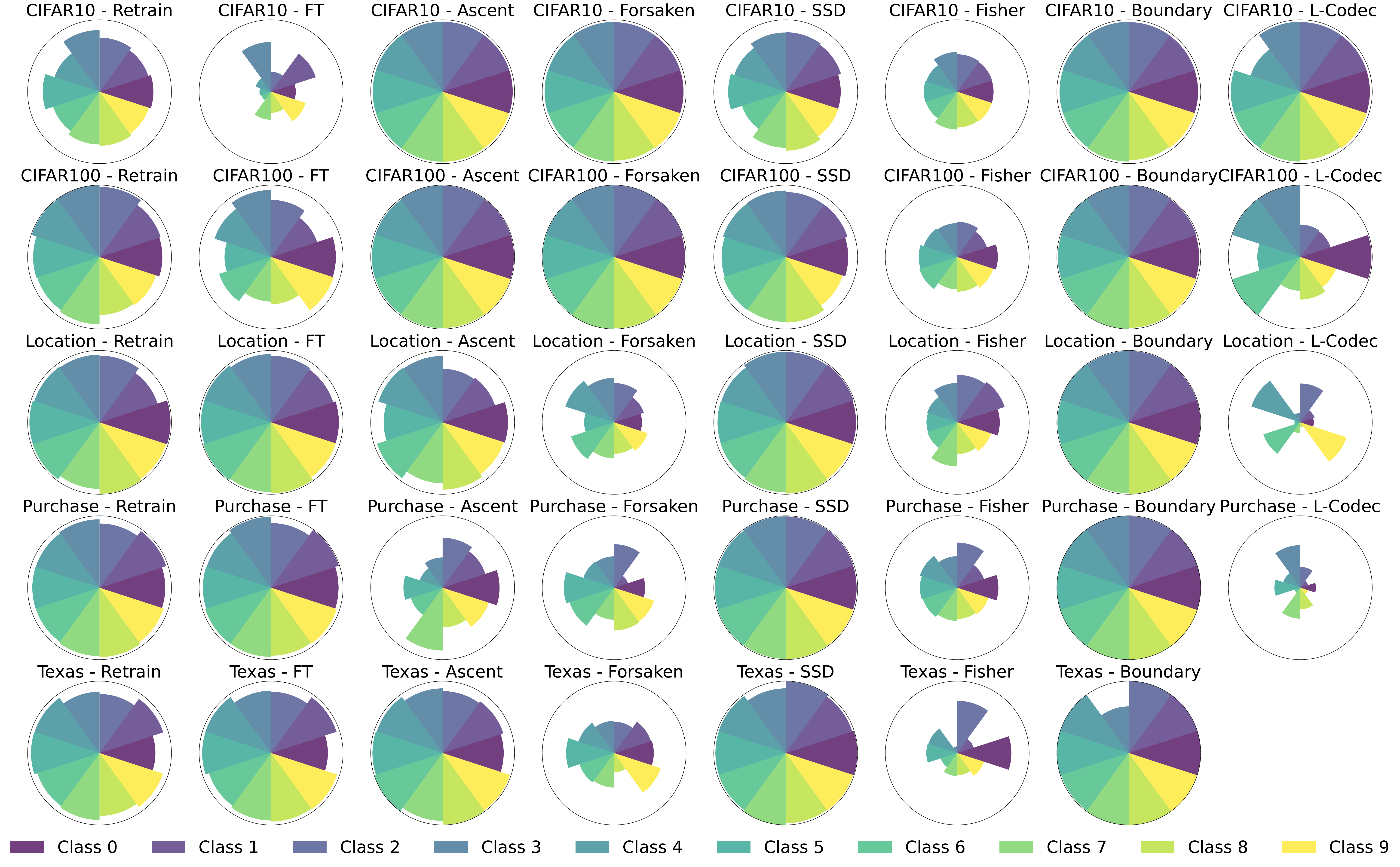}
    \caption{Partial class unlearning results (AUC) of 10 classes}
    \label{fairauc_classpercent}
\end{figure*}

\end{document}